\documentclass[11pt]{article}
\usepackage{geometry}
\usepackage[utf8]{inputenc}
\usepackage[T1]{fontenc}
\usepackage[numbers]{natbib}
\usepackage{hyperref}
\usepackage{url}
\usepackage{booktabs}
\usepackage{amsmath}
\usepackage{amssymb}
\usepackage{amsfonts}
\usepackage{bm}
\usepackage{mathtools}
\usepackage{amsthm}
\usepackage{microtype}
\usepackage[table]{xcolor}
\usepackage{xspace}
\usepackage{graphicx}
\usepackage{wrapfig}
\usepackage{multicol}
\usepackage{multirow}
\usepackage{csquotes}
\usepackage{longtable}
\usepackage{array}
\usepackage{enumitem}
\usepackage{float}
\usepackage{placeins}
\usepackage{algorithm}
\usepackage{algpseudocode}
\usepackage{caption}
\usepackage{subcaption}
\usepackage{tikz}
\usepackage{pgfplots}
\usepackage{pgfplotstable}

\usepackage[nameinlink,capitalize]{cleveref}

\usepackage{xparse}

\let\MATCHAorigfigure\figure
\let\MATCHAendorigfigure\endfigure
\let\MATCHAorigtable\table
\let\MATCHAendorigtable\endtable
\let\MATCHAorigsection\section
\let\MATCHAorigparagraph\paragraph
\let\MATCHAorigsubsection\subsection

\newcounter{matchaFigureCall}
\newcounter{matchaTableCall}

\newcommand{\MATCHApairedstored}{}
\newsavebox{\MATCHAtableTwoBox}

\newcommand{\MATCHAsectionstored}{}
\newcommand{\MATCHAparagraphstored}{}
\newcommand{\MATCHAparagraphside}{r}
\newcommand{\MATCHAparagraphwidth}{0.40}
\newif\ifMATCHAsectionpending
\newif\ifMATCHAparagraphpending
\MATCHAsectionpendingfalse
\MATCHAparagraphpendingfalse

\newcommand{\MATCHAsetlocalwidth}{%
  \setlength{\columnwidth}{\linewidth}%
  \setlength{\textwidth}{\linewidth}%
}

\newcommand{\MATCHApairedfigures}[2]{%
  \MATCHAorigfigure[t]
  \centering
  \begin{minipage}[t]{0.47\textwidth}
    \MATCHAsetlocalwidth
    \centering
    #1
  \end{minipage}\hfill
  \begin{minipage}[t]{0.47\textwidth}
    \MATCHAsetlocalwidth
    \centering
    #2
  \end{minipage}
  \MATCHAendorigfigure
}

\newcommand{\MATCHAACLFigureBox}[1]{%
  \begin{minipage}[t]{7.7cm}
    \setlength{\hsize}{7.7cm}%
    \setlength{\linewidth}{7.7cm}%
    \setlength{\columnwidth}{7.7cm}%
    \setlength{\textwidth}{16cm}%
    \fontfamily{ptm}\selectfont
    \captionsetup{font=small,skip=3pt,width=7.7cm}%
    \centering
    #1
  \end{minipage}%
}

\newcommand{\MATCHApairedfiguresACLGeometry}[2]{%
  \MATCHAorigfigure[t]
  \centering
  \MATCHAACLFigureBox{#1}\hfill
  \MATCHAACLFigureBox{#2}
  \MATCHAendorigfigure
}

\newcommand{\MATCHAstoreTableTwo}[1]{%
  \global\setbox\MATCHAtableTwoBox=\hbox{%
    \begin{minipage}[c]{0.485\textwidth}
      \MATCHAsetlocalwidth
      \centering
      \scriptsize
      \setlength{\tabcolsep}{2.7pt}%
      \renewcommand{\arraystretch}{0.98}%
      \captionsetup{type=table,font=small,skip=3pt,width=\linewidth}
      #1
    \end{minipage}%
  }%
}

\newcommand{\MATCHATableTwoTableThreeRow}[1]{%
  \MATCHAorigfigure[t]
  \centering
  \usebox{\MATCHAtableTwoBox}\hfill
  \begin{minipage}[c]{0.485\textwidth}
    \MATCHAsetlocalwidth
    \centering
    \scriptsize
    \setlength{\tabcolsep}{2.7pt}%
    \renewcommand{\arraystretch}{0.98}%
    \captionsetup{type=table,font=small,skip=3pt,width=\linewidth}
    #1
  \end{minipage}
  \MATCHAendorigfigure
}

\newcommand{\MATCHAdefaultfigure}[2]{%
  \MATCHAorigfigure[#1]
    \centering
    #2
  \MATCHAendorigfigure
}

\newcommand{\MATCHAdefaulttable}[2]{%
  \MATCHAorigtable[#1]
    \centering
    #2
  \MATCHAendorigtable
}

\newcommand{\MATCHAstoreparagraphfigure}[3]{%
  \global\long\def\MATCHAparagraphstored{#3}%
  \gdef\MATCHAparagraphside{#1}%
  \gdef\MATCHAparagraphwidth{#2}%
  \global\MATCHAparagraphpendingtrue
}

\newcommand{\MATCHAstoreparagraphtable}[3]{%
  \global\long\def\MATCHAparagraphstored{%
    \begin{wraptable}{#1}{#2\textwidth}
      \vspace{-0.45\baselineskip}
      \MATCHAsetlocalwidth
      \centering
      \scriptsize
      \renewcommand{\arraystretch}{0.94}%
      #3
      \vspace{-0.25\baselineskip}
    \end{wraptable}%
  }%
  \global\MATCHAparagraphpendingtrue
}

\newcommand{\MATCHAreleaseparagraphfigure}{%
  \begin{wrapfigure}{\MATCHAparagraphside}{\MATCHAparagraphwidth\textwidth}
    \vspace{-0.45\baselineskip}
    \MATCHAsetlocalwidth
    \centering
    \MATCHAparagraphstored
    \vspace{-0.25\baselineskip}
  \end{wrapfigure}%
}

\newcommand{\MATCHAFigureFiveInFiveOne}{%
  \begin{wrapfigure}[12]{r}{0.33\textwidth}
    \vspace{-0.35\baselineskip}
    \MATCHAsetlocalwidth
    \centering
    \captionsetup{font=small,skip=2pt}%
    \includegraphics[width=\linewidth]{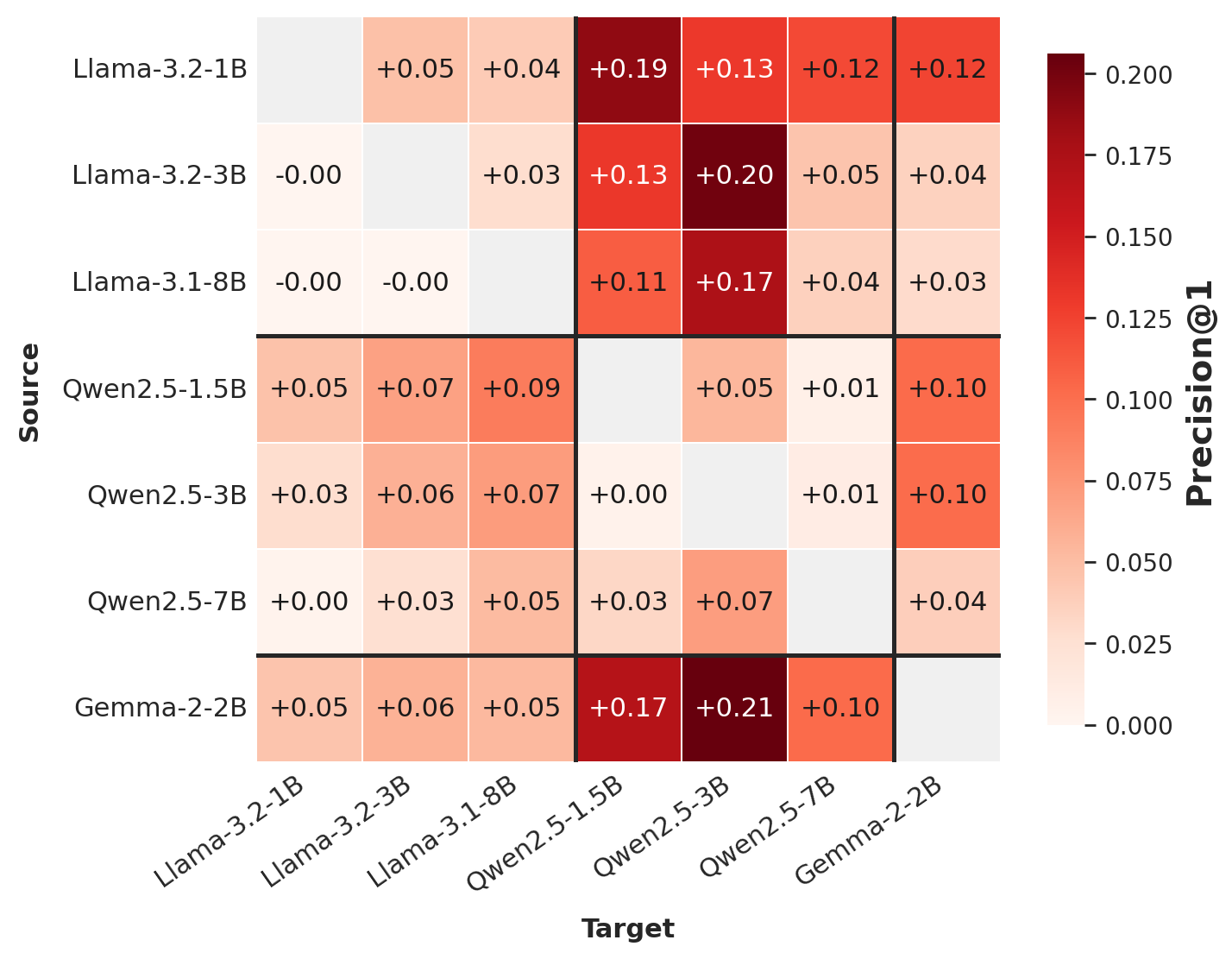}%
    \caption{Per-pair Precision@1 difference (MATCHA $-$ relative depth).}%
    \label{fig:wk-vs-linear-main}%
    \vspace{-0.25\baselineskip}
  \end{wrapfigure}%
}

\newcommand{\MATCHAFigureSixInFiveTwo}{%
  \begin{wrapfigure}[18]{r}{0.43\textwidth}
    \vspace{-0.45\baselineskip}
    \MATCHAsetlocalwidth
    \centering
    \captionsetup{font=small,skip=2pt}%
    \includegraphics[width=\linewidth]{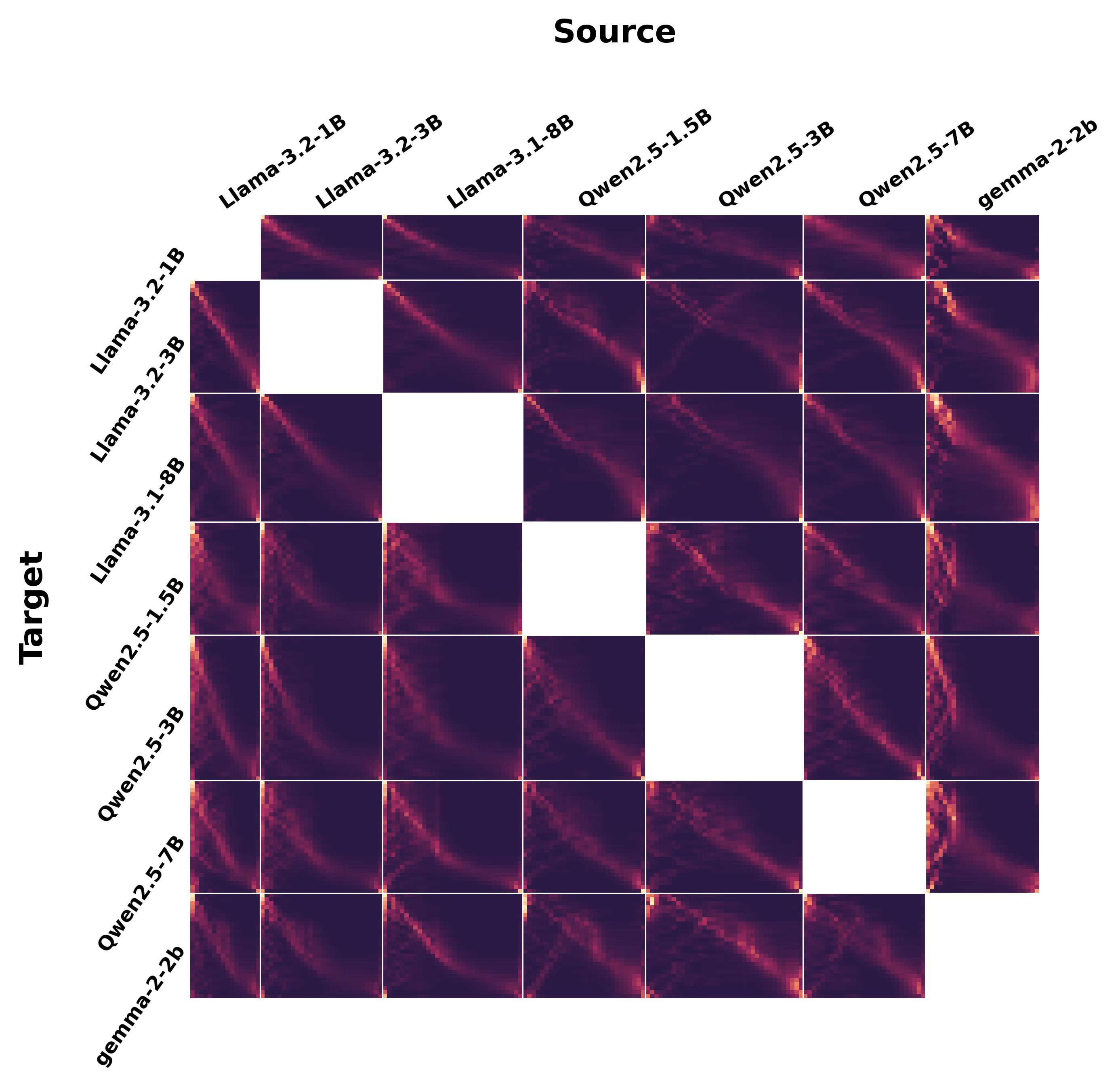}%
    \caption{The final, learned $\alpha$ obtained by our method.}%
    \vspace{-0.25\baselineskip}
  \end{wrapfigure}%
}

\newcommand{\MATCHAregisterFigureSixLabel}[1]{%
  \gdef\MATCHAFigureSixSourceLabel{}%
  \begingroup
    \renewcommand{\label}[1]{\gdef\MATCHAFigureSixSourceLabel{##1}}%
    \renewcommand{\caption}[2][]{}
    \renewcommand{\includegraphics}[2][]{}
    \setbox0=\vbox{\hsize=\linewidth #1}%
  \endgroup
  \ifx\MATCHAFigureSixSourceLabel\empty\else
    \label{\MATCHAFigureSixSourceLabel}%
  \fi
}

\newcommand{\MATCHAfigureDispatch}[2]{%
  \stepcounter{matchaFigureCall}%
  \ifcase\value{matchaFigureCall}\relax
  \or
    \global\long\def\MATCHApairedstored{#2}%
  \or
    \MATCHApairedfiguresACLGeometry{\MATCHApairedstored}{#2}%
  \or
    \global\long\def\MATCHAsectionstored{#2}%
    \global\MATCHAsectionpendingtrue
  \or
    \MATCHAstoreparagraphfigure{r}{0.38}{#2}%
  \or
    \relax
  \or
    \MATCHAregisterFigureSixLabel{#2}%
  \or
    \global\long\def\MATCHApairedstored{#2}%
  \or
    \MATCHApairedfigures{\MATCHApairedstored}{#2}%
  \or
    \MATCHAstoreparagraphfigure{r}{0.42}{#2}%
  \else
    \MATCHAdefaultfigure{#1}{#2}%
  \fi
}

\RenewDocumentEnvironment{figure}{O{!htbp}+b}
  {\MATCHAfigureDispatch{#1}{#2}}
  {}
\RenewDocumentEnvironment{figure*}{O{!htbp}+b}
  {\MATCHAfigureDispatch{#1}{#2}}
  {}

\newcommand{\MATCHAtableDispatch}[2]{%
  \stepcounter{matchaTableCall}%
  \ifcase\value{matchaTableCall}\relax
  \or
    \MATCHAdefaulttable{t}{#2}%
  \or
    \MATCHAstoreTableTwo{#2}%
  \or
    \MATCHATableTwoTableThreeRow{#2}%
  \else
    \MATCHAdefaulttable{#1}{#2}%
  \fi
}

\RenewDocumentEnvironment{table}{O{!htbp}+b}
  {\MATCHAtableDispatch{#1}{#2}}
  {}
\RenewDocumentEnvironment{table*}{O{!htbp}+b}
  {\MATCHAtableDispatch{#1}{#2}}
  {}

\RenewDocumentCommand{\section}{s o m}{%
  \IfBooleanTF{#1}
    {\MATCHAorigsection*{#3}}
    {\IfNoValueTF{#2}{\MATCHAorigsection{#3}}{\MATCHAorigsection[#2]{#3}}}%
  \ifMATCHAsectionpending
    \global\MATCHAsectionpendingfalse
    \par
    \begin{wrapfigure}{r}{0.37\textwidth}
      \vspace{-0.45\baselineskip}
      \MATCHAsetlocalwidth
      \centering
      \MATCHAsectionstored
      \vspace{-0.25\baselineskip}
    \end{wrapfigure}%
  \fi
}

\RenewDocumentCommand{\subsection}{s o m}{%
  \IfBooleanTF{#1}
    {\MATCHAorigsubsection*{#3}}
    {\IfNoValueTF{#2}{\MATCHAorigsubsection{#3}}{\MATCHAorigsubsection[#2]{#3}}}%
  \ifnum\pdfstrcmp{\detokenize{#3}}{Learned Layer Correspondences}=0\relax
    \par\MATCHAFigureSixInFiveTwo
  \fi
}

\RenewDocumentCommand{\paragraph}{s o m}{%
  \ifMATCHAparagraphpending
    \par
    \global\MATCHAparagraphpendingfalse
    \ifx\MATCHAparagraphkind\MATCHAtablekind
      \MATCHAparagraphstored
    \else
      \MATCHAreleaseparagraphfigure
    \fi
  \fi
  \ifnum\pdfstrcmp{\detokenize{#3}}{Results}=0\relax
    \par\MATCHAFigureFiveInFiveOne
  \else
    \ifnum\pdfstrcmp{\detokenize{#3}}{Results.}=0\relax
      \par\MATCHAFigureFiveInFiveOne
    \fi
  \fi
  \IfBooleanTF{#1}
    {\MATCHAorigparagraph*{#3}}
    {\IfNoValueTF{#2}{\MATCHAorigparagraph{#3}}{\MATCHAorigparagraph[#2]{#3}}}%
}

\newcommand{\MATCHAfigurekind}{F}
\newcommand{\MATCHAtablekind}{T}
\let\MATCHAparagraphkind\MATCHAfigurekind

\renewcommand{\MATCHAstoreparagraphfigure}[3]{%
  \global\long\def\MATCHAparagraphstored{#3}%
  \gdef\MATCHAparagraphside{#1}%
  \gdef\MATCHAparagraphwidth{#2}%
  \global\let\MATCHAparagraphkind\MATCHAfigurekind
  \global\MATCHAparagraphpendingtrue
}
\renewcommand{\MATCHAstoreparagraphtable}[3]{%
  \global\long\def\MATCHAparagraphstored{%
    \begin{wraptable}{#1}{#2\textwidth}
      \vspace{-0.45\baselineskip}
      \MATCHAsetlocalwidth
      \centering
      \scriptsize
      \renewcommand{\arraystretch}{0.94}%
      #3
      \vspace{-0.25\baselineskip}
    \end{wraptable}%
  }%
  \global\let\MATCHAparagraphkind\MATCHAtablekind
  \global\MATCHAparagraphpendingtrue
}

\usetikzlibrary{arrows,arrows.meta,positioning,calc,fit}
\usepgfplotslibrary{groupplots}
\pgfplotsset{compat=1.18}

\newcommand{\X}{\mathbf{X}}
\newcommand{\Xs}{\mathbf{X}^s}
\newcommand{\Xshat}{\hat{\mathbf{X}}^s}
\newcommand{\Xt}{\mathbf{X}^t}
\newcommand{\Ast}{\mathcal{A}^{s \to t}}
\newcommand{\Ats}{\mathcal{A}^{t \to s}}
\newcommand{\Ms}{\mathcal{M}_s}
\newcommand{\Mt}{\mathcal{M}_t}

\newcommand{\R}{\mathbb{R}}
\DeclareMathOperator{\softmax}{softmax}

\theoremstyle{plain}
\newtheorem{proposition}{Proposition}
\newtheorem{lemma}[proposition]{Lemma}
\newtheorem{corollary}[proposition]{Corollary}
\theoremstyle{definition}
\newtheorem{assumption}{Assumption}
\theoremstyle{remark}

\crefname{assumption}{Assumption}{Assumptions}
\crefname{remark}{Remark}{Remarks}

\colorlet{alncolor}{orange!90!red}
\tikzset{
  st/.style={draw,rounded corners=3pt,inner sep=4pt,font=\small},
  lbl/.style={font=\scriptsize}
}
\definecolor{srccolor}{RGB}{40,100,180}
\definecolor{tgtcolor}{RGB}{180,70,70}

\newcommand{\ourmethod}{\textsc{MATCHA}\xspace}

\title{Learning Cross-Model Activation Alignments with Explicit Many-to-Many Layer Maps}

\begin{document}

\date{}
\maketitle

\begin{center}
\begin{tabular}{cc}
\textbf{Alina Sudakov}\textsuperscript{*} & \textbf{Guy Bar-Shalom}\textsuperscript{*} \\
Technion & Technion \\
\texttt{alinasudakov@campus.technion.ac.il} & \texttt{guybs99@gmail.com} \\

& \\
\textbf{Fabrizio Frasca}\textsuperscript{$\dagger$} & \textbf{Haggai Maron}\textsuperscript{$\dagger$} \\
Technion & Technion, NVIDIA Research \\

\end{tabular}
\end{center}

\begingroup
\renewcommand{\thefootnote}{\fnsymbol{footnote}}
\footnotetext[1]{Equal contribution.}
\footnotetext[2]{These authors jointly supervised this work.}
\endgroup

\begin{abstract}

LLMs are released at a rapid pace, raising a natural question: \emph{how do two independently trained models relate, both in which layers correspond and in how features transform between them?} We study this by learning an \emph{activation alignment}, a map from a source model's layerwise activations to a target's.
Our method, \ourmethod, factors this map into a \emph{layer map}, whose output is an \emph{explicit} target-by-source matrix that can be extracted and inspected, and a layer-shared \emph{feature map} between hidden spaces. Most of prior work fixes the layer correspondence in advance, pairing layers at roughly the same relative depth; in contrast, we learn both factors jointly from prompts. Across 42 pairs of seven models spanning three different families, \ourmethod reconstructs the target's activations more faithfully and improves retrieval-based metrics substantially, w.r.t. previous approaches. The recovered maps are broadly monotone in depth but, in contrast with most previous approaches, are consistently \emph{many-to-many}: each target layer draws on a band of source layers. Our alignments also enable \emph{transfer of activation-space interventions}, allowing steering vectors and probes developed for one model to transfer to another.

\end{abstract}

\section{Introduction}

Interpretability~\cite{rai2025practicalreviewmechanisticinterpretability, zhao2023explainabilitylargelanguagemodels}, control~\cite{turner2024steeringlanguagemodelsactivation} and compression tools~\cite{zhang2023h2oheavyhitteroracleefficient, liu2023scissorhandsexploitingpersistenceimportance} for LLMs are built in the representation space of a given ``target'' model. Examples include a probe \cite{alain2017understanding} that is trained to detect whether a model is about to answer incorrectly, a steering vector that suppresses refusals \cite{rimsky-etal-2024-steering, arditi2024refusal}, and a cached key–value state\cite{DBLP:journals/corr/abs-1911-02150}. Each is a function of the target's model hidden activations, and none of them is directly applicable to a successor or a competitor model. As open-weight releases are rolled-out at an increasing rate, the cost of rebuilding these tools per model, and the difficulty of comparing findings obtained on different models, both grow.

\definecolor{aoSrc}{HTML}{3E6D9C}     
\definecolor{aoTgt}{HTML}{A65A78}     
\definecolor{aoMethod}{HTML}{2E4A6B}  
\definecolor{aoLayer}{HTML}{4E7CA8}   
\definecolor{aoFeat}{HTML}{3A8B9A}    
\definecolor{aoOut}{HTML}{C0912F}     
\definecolor{aoLoss}{HTML}{B4674A}    
\definecolor{aoApps}{HTML}{7E5A86}    
\definecolor{aoPrompt}{HTML}{6B6FB0}  
\definecolor{aoNeutral}{HTML}{566374} 
\definecolor{aoStageBg}{HTML}{FDFAF0} 

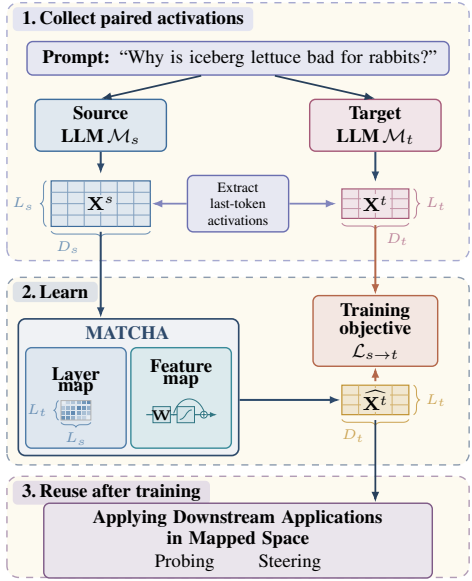
\begin{figure}[t]
\centering
\resizebox{0.8\linewidth}{!}{%
\begin{tikzpicture}[
    x=1cm,y=1cm,
    font=\scriptsize,
    >=Latex,
    arrow/.style={-{Latex[length=1.6mm]}, line width=0.75pt, draw=aoMethod!85!black},
    dataarrow/.style={-{Latex[length=1.6mm]}, line width=0.85pt, draw=aoMethod},
    trainarrow/.style={-{Latex[length=1.6mm]}, line width=0.85pt, draw=aoLoss},
    module/.style={draw=black!45, rounded corners=2.5pt, line width=0.6pt,
                   align=center, fill=white, inner sep=3pt},
    model/.style={module, minimum width=1.95cm, minimum height=0.58cm, text width=1.75cm},
    stage/.style={draw=black!25, rounded corners=5pt, dashed, line width=0.6pt},
    stagelabel/.style={rounded corners=2pt, inner xsep=4pt, inner ysep=1.5pt,
                       font=\scriptsize\bfseries, anchor=north west},
]

\draw[stage, draw=aoPrompt!60, fill=aoStageBg] (-3.45, 3.00) rectangle (3.45, 6.85);   
\draw[stage, draw=aoMethod!55, fill=aoStageBg] (-3.45,-0.05) rectangle (3.45, 2.75);   
\draw[stage, draw=aoApps!60,   fill=aoStageBg] (-3.45,-1.73) rectangle (3.45,-0.22);   

\node[stagelabel, fill=aoPrompt!18] at (-3.37, 6.79) {1.\,Collect paired activations};
\node[stagelabel, fill=aoMethod!16] at (-3.37, 2.69) {2.\,Learn};
\node[stagelabel, fill=aoApps!20]   at (-3.37,-0.28) {3.\,Reuse after training};

\node[module, fill=aoPrompt!16, draw=aoPrompt,
      minimum width=6.30cm, minimum height=0.50cm, text width=6.05cm] (prompt) at (0,6.02)
  {\textbf{Prompt:} ``Why is iceberg lettuce bad for rabbits?''};

\node[model, fill=aoSrc!16, draw=aoSrc]  (source) at (-2.05,5.02)
  {\textbf{Source LLM}\,$\mathcal{M}_s$};
\node[model, fill=aoTgt!16, draw=aoTgt]  (target) at (2.05,5.02)
  {\textbf{Target LLM}\,$\mathcal{M}_t$};

\draw[arrow] ($(prompt.south)+(-0.55,0)$) -- (source.north);
\draw[arrow] ($(prompt.south)+( 0.55,0)$) -- (target.north);

\coordinate (xs) at (-2.05,3.85);
\draw[fill=aoSrc!16, draw=aoSrc, line width=.5pt]
  ($(xs)+(-0.74,-0.34)$) rectangle ($(xs)+(0.74,0.34)$);
\foreach \u in {-0.4933,-0.2467,0,0.2467,0.4933}{
  \draw[aoSrc!65, line width=.28pt] ($(xs)+(\u,-0.34)$) -- ($(xs)+(\u,0.34)$);}
\foreach \w in {-0.17,0,0.17}{
  \draw[aoSrc!65, line width=.28pt] ($(xs)+(-0.74,\w)$) -- ($(xs)+(0.74,\w)$);}
\node[fill=aoSrc!16, inner sep=0.7pt] at (xs) {$\Xs$};
\draw[decorate,decoration={brace,amplitude=2.5pt},aoSrc!75,line width=.45pt]
  ($(xs)+(-0.81,-0.34)$) -- ($(xs)+(-0.81,0.34)$) node[midway,left=1.5pt,font=\tiny]{$L_s$};
\draw[decorate,decoration={brace,amplitude=2.5pt,mirror},aoSrc!75,line width=.45pt]
  ($(xs)+(-0.74,-0.41)$) -- ($(xs)+(0.74,-0.41)$) node[pos=0.17,below=0.5pt,font=\tiny]{$D_s$};

\coordinate (xt) at (2.05,3.85);
\draw[fill=aoTgt!16, draw=aoTgt, line width=.5pt]
  ($(xt)+(-0.50,-0.22)$) rectangle ($(xt)+(0.50,0.22)$);
\foreach \u in {-0.25,0,0.25}{
  \draw[aoTgt!65, line width=.28pt] ($(xt)+(\u,-0.22)$) -- ($(xt)+(\u,0.22)$);}
\foreach \w in {-0.0733,0.0733}{
  \draw[aoTgt!65, line width=.28pt] ($(xt)+(-0.50,\w)$) -- ($(xt)+(0.50,\w)$);}
\node[fill=aoTgt!16, inner sep=0.7pt] at (xt) {$\Xt$};
\draw[decorate,decoration={brace,amplitude=2.5pt},aoTgt!75,line width=.45pt]
  ($(xt)+(0.57,0.22)$) -- ($(xt)+(0.57,-0.22)$) node[midway,right=1.5pt,font=\tiny]{$L_t$};
\draw[decorate,decoration={brace,amplitude=2.5pt,mirror},aoTgt!75,line width=.45pt]
  ($(xt)+(-0.50,-0.29)$) -- ($(xt)+(0.50,-0.29)$) node[pos=0.83,below=0.5pt,font=\tiny]{$D_t$};

\draw[dataarrow] (source.south) -- ($(xs)+(0,0.44)$);
\draw[dataarrow] (target.south) -- ($(xt)+(0,0.32)$);

\node[module, fill=aoPrompt!16, draw=aoPrompt, minimum height=0.52cm, text width=1.30cm,
      font=\tiny, align=center] (lasttok) at (0,3.85) {Extract\\ last-token\\ activations};
\draw[arrow, draw=aoPrompt!85] (lasttok.west) -- ($(xs)+(0.76,0)$);
\draw[arrow, draw=aoPrompt!85] (lasttok.east) -- ($(xt)+(-0.58,0)$);

\node[draw=aoMethod, rounded corners=3pt, line width=0.7pt, fill=aoMethod!7,
      minimum width=3.29cm, minimum height=2.04cm] (method) at (-1.635,1.11) {};
\node[font=\scriptsize\bfseries, text=aoMethod] (maaca) at (-1.65,1.92) {MATCHA};

\node[module, fill=aoLayer!18, draw=aoLayer,
      minimum width=1.50cm, minimum height=1.50cm] (layerbox) at (-2.42,0.95) {};
\node[font=\scriptsize\bfseries, anchor=north] at (-2.42,1.54) {Layer};
\node[font=\scriptsize\bfseries, anchor=north] at (-2.42,1.325) {map};
\begin{scope}[shift={(-2.643,0.62)}, x=.078cm,y=.078cm]
  \foreach \r in {0,...,3}{\foreach \c in {0,...,5}{
    \pgfmathtruncatemacro{\dist}{abs(3*\c-5*\r)}
    \pgfmathtruncatemacro{\v}{max(14,84-5*\dist+5*mod(\r+2*\c,3))}
    \fill[aoLayer!\v] (\c,\r) rectangle ++(.72,.72);}}
  \draw[black!30, line width=.2pt] (0,0) rectangle (5.72,3.72);
\end{scope}
\draw[decorate,decoration={brace,amplitude=2pt},aoLayer!85,line width=.4pt]
  (-2.678,0.62) -- (-2.678,0.910) node[midway,left=1pt,font=\tiny]{$L_t$};
\draw[decorate,decoration={brace,amplitude=2pt,mirror},aoLayer!85,line width=.4pt]
  (-2.643,0.585) -- (-2.197,0.585) node[midway,below=0.5pt,font=\tiny]{$L_s$};

\node[module, fill=aoFeat!18, draw=aoFeat,
      minimum width=1.50cm, minimum height=1.50cm] (featbox) at (-0.85,0.95) {};
\node[font=\scriptsize\bfseries, anchor=north] at (-0.85,1.67) {Feature};
\node[font=\scriptsize\bfseries, anchor=north] at (-0.85,1.42) {map};
\begin{scope}[shift={(-0.85,0.70)}, font=\tiny]
  \draw[aoFeat!85!black, line width=0.5pt] (-0.52,0) -- (0.30,0);                       
  \draw[fill=aoFeat!22, draw=aoFeat!85!black, line width=0.4pt]                          
    (-0.42,-0.12) rectangle (-0.16,0.12);
  \node at (-0.29,0) {$\mathbf{W}$};
  \draw[fill=aoFeat!22, draw=aoFeat!85!black, line width=0.4pt]                          
    (-0.06,-0.12) rectangle (0.20,0.12);
  \draw[aoFeat!85!black, line width=0.4pt]
    (-0.02,-0.07) .. controls (0.07,-0.07) and (0.07,0.07) .. (0.16,0.07);              
  \draw[fill=white, draw=aoFeat!85!black, line width=0.45pt] (0.34,0) circle (0.05);     
  \draw[aoFeat!85!black, line width=0.35pt] (0.30,0)--(0.38,0) (0.34,-0.04)--(0.34,0.04);
  \draw[-{Latex[length=1.1mm]}, aoFeat!85!black, line width=0.5pt] (0.39,0) -- (0.54,0); 
  \fill[aoFeat!85!black] (-0.11,0) circle (0.012);                                        
  \draw[aoFeat!85!black, line width=0.4pt]
    (-0.11,0.02) .. controls (-0.02,0.30) and (0.24,0.30) .. (0.34,0.055);              
\end{scope}

\coordinate (xhat) at (2.05,0.92);
\draw[fill=aoOut!22, draw=aoOut, line width=.5pt]
  ($(xhat)+(-0.50,-0.22)$) rectangle ($(xhat)+(0.50,0.22)$);
\foreach \u in {-0.25,0,0.25}{
  \draw[aoOut!70, line width=.28pt] ($(xhat)+(\u,-0.22)$) -- ($(xhat)+(\u,0.22)$);}
\foreach \w in {-0.0733,0.0733}{
  \draw[aoOut!70, line width=.28pt] ($(xhat)+(-0.50,\w)$) -- ($(xhat)+(0.50,\w)$);}
\draw[decorate,decoration={brace,amplitude=2.5pt},aoOut!80,line width=.45pt]
  ($(xhat)+(0.57,0.22)$) -- ($(xhat)+(0.57,-0.22)$) node[midway,right=1.5pt,font=\tiny]{$L_t$};
\draw[decorate,decoration={brace,amplitude=2.5pt,mirror},aoOut!80,line width=.45pt]
  ($(xhat)+(-0.50,-0.29)$) -- ($(xhat)+(0.50,-0.29)$) node[pos=0.17,below=0.5pt,font=\tiny]{$D_t$};
\node[fill=aoOut!22, inner sep=0.7pt] at (xhat) {$\widehat \Xt$};

\node[module, fill=aoLoss!16, draw=aoLoss,
      minimum width=1.85cm, minimum height=0.72cm, text width=1.65cm] (loss) at (2.05,1.95)
  {\textbf{Training}\\[-1.5pt]\textbf{objective} $\mathcal{L}_{s\to t}$};

\draw[dataarrow] ($(xs)+(0,-0.42)$) -- (method.north -| xs);          
\draw[dataarrow] (method.east |- xhat) -- ($(xhat)+(-0.52,0)$);       
\draw[trainarrow] ($(xt)+(0,-0.32)$) -- (loss.north);                 
\draw[trainarrow] ($(xhat)+(0,0.28)$) -- (loss.south);                

\node[module, fill=aoApps!16, draw=aoApps,
      minimum width=5.80cm, minimum height=0.86cm, text width=5.55cm] (apps) at (0,-1.18)
  {\textbf{Applying Downstream Applications}\\[-0.5pt]\textbf{in Mapped Space}\\[0pt]
   Probing \qquad Steering};

\draw[arrow] ($(xhat)+(0,-0.30)$) -- (apps.north -| xhat);

\end{tikzpicture}%
}
\caption{\ourmethod aligns paired activations $\Xs,\Xs$ using an (interpretable) layer map, and a feature map. \ourmethod can be used for downstream applications.
}
\label{fig:teaser}
\end{figure}

Recent evidence suggests that rebuilding these tools may be avoidable. Independently trained LLMs appear to converge on similar representational content, differing in the
coordinates they express it in and the depths at which it emerges rather than in what they encode it. This similarity has been documented across model families and scales \citep{huh2024platonicrepresentationhypothesis,pmlr-v97-kornblith19a,moschella2023relative,morcos2018insightsrepresentationalsimilarityneural}.
If two models encode related information, it should be possible to map the layer-wise activations produced by one model 
to those produced by the other on the same prompts. We call such a mapping an \emph{activation alignment} and it can serve two relevant purposes. (i) It allows an application developed for one model to be reused with another without retraining the application or modifying either model; prior work has transferred probes, steering interventions, and key-value caches in this way \citep{wang2026commandv,fu2026cachetocache,oozeer2025activation}. (ii) It can reveal structural relationships between the two models, such as correspondences between their layers. 

\begin{figure}[t]
\centering
\definecolor{srcln}{HTML}{2F5597}\definecolor{srcfl}{HTML}{DCE6F6}
\definecolor{tgtln}{HTML}{A03030}\definecolor{tgtfl}{HTML}{FADEDC}
\definecolor{acc}{HTML}{3F3D7A}
\begin{tikzpicture}[
  sbox/.style={draw=srcln,fill=srcfl,rounded corners=1pt,
               minimum width=6.5mm,minimum height=2mm,inner sep=0pt,line width=.3pt},
  tbox/.style={draw=tgtln,fill=tgtfl,rounded corners=1pt,
               minimum width=6.5mm,minimum height=2mm,inner sep=0pt,line width=.3pt},
  map/.style={-{Stealth[length=1.1mm,width=.9mm]},draw=acc,line width=.8pt,shorten >=.5pt},
  mesh/.style={draw=black!60,line width=.35pt,shorten >=.5pt},
  row/.style={font=\scriptsize\bfseries,rotate=90},
  rowsub/.style={font=\tiny\itshape,rotate=90,black!55},
  ttl/.style={font=\scriptsize\bfseries,align=center},
  sub/.style={font=\tiny,align=center,black!55,inner sep=.5pt},
  lbl/.style={font=\tiny},
  x=1cm,y=1cm]
  \def\dy{0.25}\def\tx{0.95}\def\pw{2.10}\def\ph{2.33}
  \def\ymid{-0.60}

  \node[row] at (-0.98,\ymid) {Fixed mappings};
  \node[rowsub] at (-0.79,\ymid) {(prior work)};
  \draw[-{Stealth[length=.9mm]},black!45,line width=.35pt] (-0.42,0.05) -- (-0.42,{-\dy*5-0.04});
  \foreach \p/\px/\ttlname/\subname in {%
      0/0/{Relative depth}/{one-to-one},
      1/1/{Terminal}/{one-to-one},
      2/2/{Similarity\,/\,HOT}/{many-to-many\\\emph{per model pair}}}{
    \begin{scope}[shift={({\px*\pw},0)}]
      \foreach \i in {0,...,5}{\node[sbox] (P\p S\i) at (0,{-\dy*\i}) {};}
      \foreach \j in {0,...,3}{\node[tbox] (P\p T\j) at (\tx,{-\dy*\j}) {};}
      \node[ttl] at ({\tx/2},0.80) {\ttlname};
      \node[sub,anchor=north] at ({\tx/2},0.62) {\subname};
      \node[lbl,srcln,anchor=base] at (0,{-\dy*5-0.26}) {source};
      \node[lbl,tgtln,anchor=base] at (\tx,{-\dy*5-0.26}) {target};
    \end{scope}}
  \foreach \j/\i in {0/0,1/2,2/3,3/5}{\draw[map] (P0S\i.east) -- (P0T\j.west);}
  \foreach \j/\i in {0/2,1/3,2/4,3/5}{\draw[map] (P1S\i.east) -- (P1T\j.west);}
  \foreach \j in {0,...,3}{\foreach \i in {0,...,5}{\draw[mesh] (P2S\i.east) -- (P2T\j.west);}}
  \foreach \j/\i/\w in {0/0/.5,0/1/.5, 1/1/.45,1/2/.55, 2/3/.6,2/4/.4, 3/4/.45,3/5/.55}{
      \pgfmathsetmacro{\lw}{.55+.35*\w}\pgfmathsetmacro{\ah}{.95+.20*\w}%
      \pgfmathsetmacro{\aw}{.82*\ah}%
      \draw[map,line width=\lw pt,
            -{Stealth[length=\ah mm,width=\aw mm]}] (P2S\i.east) -- (P2T\j.west);}

  \draw[black!12,line width=.4pt] (-1.12,{-\ph+0.62}) -- (5.48,{-\ph+0.62});
  \begin{scope}[shift={(0,-\ph)}]
    \node[row,acc] at (-0.98,\ymid) {Ours};
    \node[rowsub,acc!70] at (-0.79,\ymid) {(many-to-many, per input)};
    \draw[-{Stealth[length=.9mm]},black!45,line width=.35pt] (-0.42,0.05) -- (-0.42,{-\dy*5-0.04});
    \foreach \p/\px/\inp in {3/0/{$x_1$},4/1/{$x_2$},5/2/{$x_3$}}{
      \begin{scope}[shift={({\px*\pw},0)}]
        \foreach \i in {0,...,5}{\node[sbox] (P\p S\i) at (0,{-\dy*\i}) {};}
        \foreach \j in {0,...,3}{\node[tbox] (P\p T\j) at (\tx,{-\dy*\j}) {};}
        \node[ttl] at ({\tx/2},0.40) {input \inp};
        \node[lbl,srcln,anchor=base] at (0,{-\dy*5-0.26}) {source};
        \node[lbl,tgtln,anchor=base] at (\tx,{-\dy*5-0.26}) {target};
      \end{scope}}
    \node[font=\normalsize,acc,inner sep=0pt,fill=white] at (1.525,\ymid) {$\neq$};
    \node[font=\normalsize,acc,inner sep=0pt,fill=white] at (3.625,\ymid) {$\neq$};
  \end{scope}
  \foreach \p in {3,4,5}{\foreach \j in {0,...,3}{\foreach \i in {0,...,5}{
      \draw[mesh] (P\p S\i.east) -- (P\p T\j.west);}}}
  \foreach \p/\j/\i/\w in {%
      3/0/0/.8,3/0/1/.2, 3/1/1/.7,3/1/2/.3, 3/2/2/.75,3/2/3/.25, 3/3/3/.6,3/3/4/.4,
      4/0/2/.35,4/0/5/.65, 4/1/0/.3,4/1/4/.7, 4/2/1/.25,4/2/5/.75, 4/3/2/.4,4/3/3/.6,
      5/0/0/.5,5/0/3/.3,5/0/5/.2, 5/1/1/.4,5/1/2/.35,5/1/4/.25,
      5/2/0/.3,5/2/4/.45,5/2/5/.25, 5/3/1/.2,5/3/3/.35,5/3/5/.45}{
      \pgfmathsetmacro{\lw}{.55+.35*\w}\pgfmathsetmacro{\ah}{.95+.20*\w}%
      \pgfmathsetmacro{\aw}{.82*\ah}%
      \draw[map,line width=\lw pt,
            -{Stealth[length=\ah mm,width=\aw mm]}] (P\p S\i.east) -- (P\p T\j.west);}
\end{tikzpicture}
\caption{Layer maps from source ($L_s{=}6$) to target ($L_t{=}4$), with shallow layers on top.
Prior maps are fixed; ours is learned and input-dependent.
}
\label{fig:differnt-layer-maps}
\end{figure}
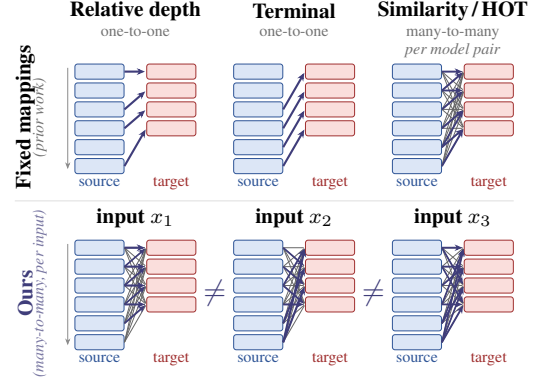
Most prior work focuses primarily on (i) and, in fact, assumes a fixed \emph{one-to-one} layer mapping --- assuming each target layer has a \emph{single counterpart} --- that is \emph{independent of both the input and the specific pair of models}. Representative examples include aligning layers at the same ``relative depth''~\citep{wolfram2025layerssimilardepthsgenerate,wang2026commandv}, or the ``terminal'' alignment strategy~\citep{fu2026cachetocache}; see \Cref{fig:differnt-layer-maps}. More recent work relaxes this assumption by allowing soft \emph{many-to-many} layer mappings that may vary across model pairs and are typically derived using optimal transport or feature similarity~\citep{shah2026representational,cui2026transport, kornblith2019similarityneuralnetworkrepresentations}.
However, it remains unclear how to extend these methods to the input-dependent setting, which is particularly relevant when transferring an input-dependent tool, such as a probe, from one model to another. In addition, even if such an extension were possible, it would require solving a separate optimization problem for each input.

\paragraph{Our approach.} 
Our key idea is to recast activation alignment as a \emph{learning} problem: rather than solving a separate optimization problem per input, we train a lightweight network to predict the mapping directly from the input, producing input-dependent maps with a single forward pass.
Concretely, we introduce \ourmethod, a cross-attention-based architecture that learns an \emph{interpretable}, \emph{input-dependent} mapping between layers; see \Cref{fig:teaser}.
Each target layer is associated with a learned query, while each source layer produces an input-dependent key.
Their attention scores form a target-layers $\times$ source-layers ($L_t \times L_s$) correspondence matrix that determines how source-layer representations are combined to reconstruct each target layer, on a per-input basis.
This matrix can be read off directly, yielding a legible layer alignment for every input.
To align features, we apply a single projection shared across all layers, followed by a feed-forward network (\texttt{FF}) with a residual connection.
The complete mapping is trained end to end, jointly optimizing the layer and feature transformations.

\paragraph{Main findings and contributions.} We analyze \ourmethod both quantitatively and qualitatively, yielding the following main contributions. \textbf{Reconstruction and retrieval:} \emph{learning} the correspondence pays off: across 42 pairs of seven models spanning three families and scales from 1B to 8B, \ourmethod reconstructs target activations more accurately than all baselines and achieves stronger performance across multiple retrieval metrics. \ourmethod is particularly effective when transferring between models from different families or with different depths, where fixed mappings are intuitively least plausible. \textbf{Layer correspondences:} \ourmethod reveals interesting cross-model layer correspondences. More specifically, the resulting layer maps are \emph{not one-to-one}: they 
follow an approximately monotone correspondence, while each target layer draws on a \emph{band} of neighboring source layers.
In fact, restricting the map to a single source layer per target, or even a top-k combination
of source layers, degrades every retrieval metric, whereas the full multi-layer map outperforms all those restricted variants. \textbf{Applications:} \ourmethod is useful beyond reconstruction, for \emph{behavioral teleportation}: we show that a hallucination-detection probe and a steering intervention developed for a target model can be transferred to a source model, with \ourmethod yielding better transfer performance than all baselines on both tasks, across four different datasets. Finally, we note that \ourmethod's latency remains extremely low --- transferring between activation spaces takes only approximately $0.06$ milliseconds. 
\section{Related Work}


\paragraph{Cross-model activation maps.}
Most prior work estimates activation maps between prespecified pairs of LLM layers. Layers are commonly paired by relative depth, as in Command-V \citep{wolfram2025layerssimilardepthsgenerate,wang2026commandv}, or backwards from the output, as in the terminal alignment of Cache-to-Cache \citep{fu2026cachetocache}. Other work fits linear maps for individual layer pairs and treats the pairing itself as a hyperparameter \citep{huang-etal-2025-cross}. These approaches 
impose a fixed one-to-one correspondence between layers. Several recent works \emph{point} toward more flexible alternatives. The Linear Representation Transferability Hypothesis sketches a many-to-many extension \citep{journals/corr/abs-2506-00653}, cross-coders identify feature spaces shared across layers \citep{lindsey2024crosscoders}, and Hierarchical Optimal Transport estimates a soft correspondence jointly across layers and features \citep{shah2026representational}. 

\paragraph{Applications of activation maps.}
Activation maps have been used to transfer steering interventions
\citep{wang2026commandv,huang-etal-2025-cross,oozeer2025activation,balasubramanian2026masterkeyhypothesisunlocking}, key--value caches \citep{fu2026cachetocache}, and mechanistic circuits \citep{cristofano2026universalrefusalcircuitsllms}. They have also enabled direct communication between language-model agents through intermediate activations \citep{ramesh2025communicating}, and the reuse of representations for classification, out-of-distribution detection, and more
\citep{gorbett2026characterizinglinearalignmentlanguage}. Beyond inference-time transfer, activation-derived correspondences have been used to guide cross-architecture weight-space merging \citep{cui2026transportmergecrossarchitecturemerging}.

\paragraph{Evaluating cross-model representations.}
Representation similarity across neural networks has been studied using metrics such as CKA~\citep{kornblith2019similarityneuralnetworkrepresentations}, which compares the pairwise relationships among examples across representation spaces that may have different dimensionalities, SVCCA \citep{raghu2017svccasingularvectorcanonical}, and mutual nearest-neighbor similarity \citep{huh2024platonicrepresentationhypothesis,Klabunde_2025}; see \citet{sucholutsky2025getting} for a broader survey. Differently to us, these methods compare the geometry of representation spaces without constructing an explicit map between them. Closer to our setting, model stitching evaluates functional compatibility by transforming an intermediate representation from one network and injecting it into another \citep{bansal2021revisiting}. Compatibility is then assessed through the stitched model's performance on some task.
 \citet{oozeer2025activation}, constructs a feature map between a \emph{single} pair of layers,
and evaluates its quality 
by measuring the effect on model behavior, for example, by replacing a target-layer activation with its mapped reconstruction and computing the KL divergence between the resulting and original next-token distributions. In this paper, we evaluate reconstruction quality from two complementary perspectives: representation retrieval quality and performance on downstream applications, including behavior-based evaluations of the kind described above.


\section{Problem Formulation}
\label{sec:formalization}
\begin{figure}[t]
  \centering
  \resizebox{0.75\columnwidth}{!}{%
  \begin{tikzpicture}[
    >=Stealth,
    promptbox/.style={draw=black!45, semithick, rounded corners=3pt, fill=black!4,
      minimum width=3.2cm, minimum height=0.6cm, align=center, font=\small\itshape},
    llmbox/.style={draw, semithick, rounded corners=3pt,
      minimum width=2.7cm, minimum height=0.72cm, align=center, font=\small},
    flow/.style={->, semithick, shorten >=1pt},
    ax/.style={font=\scriptsize, text=black!45}
  ]
    \def\cw{0.20}\def\chh{0.17}          
    \def\dx{0.26}\def\dy{0.18}           
    \def\gt{0.95}                        

    \node[promptbox] (p) at (0,2.95) {``Is the earth flat?''};
    \draw[semithick] (p.south) -- (0,2.58);
    \draw[flow,srccolor] (0,2.58) -| (-1.9,2.14);
    \draw[flow,tgtcolor] (0,2.58) -| ( 1.9,2.14);

    \node[llmbox, draw=srccolor, fill=srccolor!5, text=srccolor] (srcbox) at (-1.9,1.78)
      {source LLM~$(\Ms)$};
    \node[llmbox, draw=tgtcolor, fill=tgtcolor!5, text=tgtcolor] (tgtbox) at ( 1.9,1.78)
      {target LLM~$(\Mt)$};
    \draw[flow,srccolor] (-1.9,1.42) -- (-1.9,{\gt+\dy+0.06});
    \draw[flow,tgtcolor] ( 1.9,1.42) -- ( 1.9,{\gt+\dy+0.06});

    \def\sR{-1.45}                                  
    \pgfmathsetmacro{\sL}{\sR-5*\cw}\pgfmathsetmacro{\sB}{\gt-6*\chh}
    \fill[srccolor!10] (\sL,\gt) -- (\sR,\gt) -- ++(\dx,\dy) -- ++(-5*\cw,0) -- cycle; 
    \fill[srccolor!35] (\sR,\gt) -- ++(\dx,\dy) -- ++(0,-6*\chh) -- (\sR,\sB) -- cycle; 
    \fill[srccolor!28] ({\sR-\cw},\gt) -- (\sR,\gt) -- ++(\dx,\dy)
                       -- ++(-\cw,0) -- cycle;                                   
    \foreach \r in {0,...,5}{\foreach \c in {0,...,4}{
        \pgfmathsetmacro{\xx}{\sL+\cw*\c}\pgfmathsetmacro{\yy}{\gt-\chh*\r}
        \ifnum\c=4 \fill[srccolor!45] (\xx,\yy) rectangle ++(\cw,-\chh);
        \else       \fill[srccolor!16] (\xx,\yy) rectangle ++(\cw,-\chh);\fi
        \draw[white,line width=.5pt] (\xx,\yy) rectangle ++(\cw,-\chh);}}
    \draw[srccolor!60,thin] (\sL,\gt) -- ++(\dx,\dy) -- ++(5*\cw,0) -- ++(0,-6*\chh);
    \draw[srccolor,semithick]                                   
      ({\sR-\cw},\gt) -- (\sR,\gt) -- ++(\dx,\dy) -- ++(0,-6*\chh) -- (\sR,\sB) -- ({\sR-\cw},\sB) -- cycle;
    \draw[srccolor,semithick] (\sR,\gt) -- (\sR,\sB);
    \draw[srccolor,semithick] ({\sR-\cw},\gt) -- ++(\dx,\dy);
    \node[font=\small,text=srccolor,anchor=north] at ({\sR-\cw/2},{\sB-0.52}) {$\Xs$};

    \def\tL{1.45}
    \pgfmathsetmacro{\tR}{\tL+3*\cw}\pgfmathsetmacro{\tB}{\gt-4*\chh}
    \fill[tgtcolor!10] (\tL,\gt) -- (\tR,\gt) -- ++(\dx,\dy) -- ++(-3*\cw,0) -- cycle;
    \fill[tgtcolor!35] (\tR,\gt) -- ++(\dx,\dy) -- ++(0,-4*\chh) -- (\tR,\tB) -- cycle;
    \fill[tgtcolor!28] ({\tR-\cw},\gt) -- (\tR,\gt) -- ++(\dx,\dy) -- ++(-\cw,0) -- cycle;
    \foreach \r in {0,...,3}{\foreach \c in {0,...,2}{
        \pgfmathsetmacro{\xx}{\tL+\cw*\c}\pgfmathsetmacro{\yy}{\gt-\chh*\r}
        \ifnum\c=2 \fill[tgtcolor!45] (\xx,\yy) rectangle ++(\cw,-\chh);
        \else       \fill[tgtcolor!16] (\xx,\yy) rectangle ++(\cw,-\chh);\fi
        \draw[white,line width=.5pt] (\xx,\yy) rectangle ++(\cw,-\chh);}}
    \draw[tgtcolor!60,thin] (\tL,\gt) -- ++(\dx,\dy) -- ++(3*\cw,0) -- ++(0,-4*\chh);
    \draw[tgtcolor,semithick]
      ({\tR-\cw},\gt) -- (\tR,\gt) -- ++(\dx,\dy) -- ++(0,-4*\chh) -- (\tR,\tB) -- ({\tR-\cw},\tB) -- cycle;
    \draw[tgtcolor,semithick] (\tR,\gt) -- (\tR,\tB);
    \draw[tgtcolor,semithick] ({\tR-\cw},\gt) -- ++(\dx,\dy);
    \node[font=\small,text=tgtcolor,anchor=north] at ({\tR-\cw/2},{\sB-0.52}) {$\Xt$};

    \node[ax,anchor=east] at ({\sL-0.10},{\gt-3*\chh}) {layers};
    \draw[->,black!45,thin] ({\sL-0.07},{\gt-0.02}) -- ({\sL-0.07},{\sB+0.02});
    \draw[->,black!45,thin] ({\sL+0.02},{\sB-0.14}) -- ({\sL+5*\cw-0.02},{\sB-0.14});
    \node[ax,anchor=north] at ({\sL+2.5*\cw},{\sB-0.17}) {tokens};
    \draw[->,black!45,thin] ({\sL+0.03},{\gt+0.02}) -- ({\sL+\dx-0.02},{\gt+\dy-0.01});
    \node[ax,anchor=south east,inner sep=1pt] at ({\sL+\dx},{\gt+\dy+0.02}) {features $d$};

    \draw[->, alncolor, line width=1.1pt]
      ({\sR+\dx+0.06},{\gt-1.2*\chh}) to[bend left=22]
      node[font=\small,text=alncolor,above=2pt,pos=.38] {$\Ast$}
      ({\tR-\cw+\dx*0.5},{\gt+\dy*0.5+0.04});
  \end{tikzpicture}%
  }
  \caption{Both models encode the same prompt. Highlighted are the last-token representations, $\Xs$ and $\Xt$, with $\Ast$ mapping $\Xs$ to $\Xt$.}
  \label{fig:objective}
\end{figure}
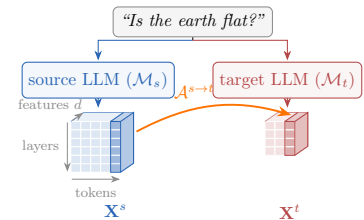
Let $\Ms$ and $\Mt$ denote a source and a target LLM, with $L_s$ and $L_t$
layers and hidden dimensionalities $D_s$ and $D_t$, respectively. For a prompt
$p$, we collect from every transformer block the hidden state of the final
token,\footnote{We use the final-token state: under causal masking, it is the only position with access to the full prompt and is commonly used for representation readout \citep{zou2023representation,kossen2024semanticentropyprobesrobust}  
} $(\Xs \in \R^{L_s \times D_s}, \Xt \in \R^{L_t \times D_t})$.

Our goal is to learn an \emph{activation alignment} map:
\begin{equation}
\label{eq:map}
    \Ast :
    \R^{L_s \times D_s}
    \longrightarrow
    \R^{L_t \times D_t},
\end{equation}
such that $\Ast(\Xs) \approx \Xt$ for activations produced from the same prompt; see \Cref{fig:objective}.


\paragraph{Training objective.} Given a dataset of prompts $\mathcal{D}=\{p_i\}_{i=1}^{N}$ and the corresponding
activation pairs $\{(\Xs_i,\Xt_i)\}_{i=1}^{N}$, we learn $\Ast$ by minimizing
the average reconstruction error,
\begin{equation}
    \mathcal{L}_{s\rightarrow t}
    =
    \frac{1}{N L_t D_t}
    \sum_{i=1}^{N}
    \left\|
        \Ast(\Xs_i)-\Xt_i
    \right\|_F^2.
    \label{eq:recon}
\end{equation}

\section{\ourmethod: \underline{M}any-to-many \underline{A}lignment \underline{T}hrough \underline{C}ross-attention over \underline{H}idden \underline{A}ctivations}

\label{sec:method}


The activation alignment map in ~\Cref{eq:map} constrains $\Ast$ only through its input and output shapes:
any map from $\R^{L_s\times D_s}$ to $\R^{L_t\times D_t}$ is admissible. 
Although many such maps can be constructed, designing a useful and legible one requires care, with the most expressive choices typically being the least interpretable. For example, an MLP applied to the flattened source stack could well minimize the objective in \Cref{eq:recon}, but any learned relationship between the layers of the two models would remain implicit in its weights. 

\paragraph{Guiding principles in constructing $\Ast$.}
Our construction is guided by two principles. (i) The alignment should make it clearly explicit both which source layers contribute to each target layer and how features are transformed between the corresponding representation spaces. 
(ii) The alignment should remain flexible enough to subsume prior approaches, in particular, one-to-one mappings~\citep{huang-etal-2025-cross,wolfram2025layerssimilardepthsgenerate,wang2026commandv,fu2026cachetocache}; where each target layer is constructed from \emph{exactly one} source layer. 
To realize these principles, we factorize $\Ast$ along the two axes of the problem: the \emph{layer map} and the \emph{feature map}. \Cref{fig:alignment map} illustrates this factorization; we now detail its construction.



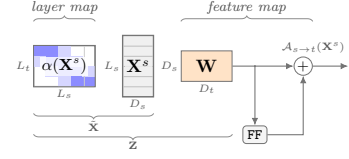
\begin{figure}[t]
\centering
\resizebox{0.75\linewidth}{!}{%
\begin{tikzpicture}[
  font=\footnotesize,
  >={Latex[length=1.5mm,width=1.3mm]},
  mat/.style={draw=black!55, inner sep=0pt, font=\small},
  dim/.style={font=\tiny, text=black!65, inner sep=1.2pt},
  lab/.style={fill=white, fill opacity=0.85, text opacity=1, inner sep=0.8pt},
  blk/.style={draw=black!55, rounded corners=2pt, fill=black!4,
              inner sep=2.5pt, font=\scriptsize},
  op/.style={circle, draw=black!55, minimum size=3.6mm, inner sep=0pt,
             font=\scriptsize},
  fl/.style={->, semithick, black!55},
  note/.style={font=\tiny, text=black!65, align=center, inner sep=1pt},
  stage/.style={font=\scriptsize\itshape, text=black!70, inner sep=1pt},
  brc/.style={decorate, decoration={brace, amplitude=3pt}, black!45, thin}
]
\path[use as bounding box] (-0.45,-1.85) rectangle (6.55,1.35);
\node[mat, minimum width=12mm, minimum height=8mm] (A) at (0.60,0) {};
\begin{scope}[shift={(A.south west)}, x=2.4mm, y=2mm]
  \foreach \c/\r in {0/3, 1/3, 1/2, 2/2, 2/1, 3/1, 3/0, 4/0}
    \fill[blue!45] (\c,\r) rectangle ++(1,1);
  \foreach \c/\r in {0/2, 2/3, 1/1, 3/2, 4/1, 2/0}
    \fill[blue!16] (\c,\r) rectangle ++(1,1);
  \draw[black!20, very thin] (0,0) grid (5,4);
\end{scope}
\draw[black!55] (A.south west) rectangle (A.north east);
\node[lab, font=\footnotesize] at (A.center) {$\alpha(\mathbf{X}^s)$};
\node[dim, anchor=east]  at (A.west)  {$L_t$};
\node[dim, anchor=north] at (A.south) {$L_s$};
\node[mat, minimum width=6mm, minimum height=12mm, fill=black!6] (X) at (2.05,0) {};
\begin{scope}[shift={(X.south west)}, x=6mm, y=2mm]
  \draw[black!25, very thin] (0,1)--(1,1) (0,2)--(1,2) (0,3)--(1,3)
                             (0,4)--(1,4) (0,5)--(1,5);
\end{scope}
\draw[black!55] (X.south west) rectangle (X.north east);
\node[lab, fill=black!6, font=\small] at (X.center) {$\mathbf{X}^s$};
\node[dim, anchor=east]  at (X.west)  {$L_s$};
\node[dim, anchor=north] at (X.south) {$D_s$};
\node[mat, minimum width=10mm, minimum height=6mm, fill=orange!22] (W) at (3.40,0)
      {$\mathbf{W}$};
\node[dim, anchor=east]  at (W.west)  {$D_s$};
\node[dim, anchor=north] at (W.south) {$D_t$};
\coordinate (j) at (4.35,0);
\node[op]  (plus) at (5.30,0)     {$+$};
\node[blk] (ff)   at (4.35,-1.35) {\texttt{FF}};
\draw[fl] (W.east)   -- (plus.west);
\draw[fl] (j)        -- (ff.north);
\draw[fl] (ff.east)  -| (plus.south);
\fill[black!55] (j) circle (0.55pt);
\draw[fl] (plus.east) -- ++(0.70,0);
\node[note, anchor=south] at (5.50,0.22) {$\mathcal{A}_{s\to t}(\mathbf{X}^s)$};
\draw[brc] (X.south east |- 0,-0.93) -- (A.south west |- 0,-0.93)
      node[midway, below=3pt, note] {$\tilde{\mathbf{X}}$};
\draw[brc] (W.south east |- 0,-1.34) -- (A.south west |- 0,-1.34)
      node[midway, below=3pt, note] {$\mathbf{Z}$};
\draw[brc] (A.west |- 0,0.88) -- (A.east |- 0,0.88)
      node[midway, above=3pt, stage] {layer map};
\draw[brc] (W.west |- 0,0.88) -- (plus.east |- 0,0.88)
      node[midway, above=3pt, stage] {feature map};
\end{tikzpicture}}
\caption{Visualization of our alignment map $\Ast$.
  }
\label{fig:alignment map}
\end{figure}

\paragraph{Layer map overview.}
We instantiate the layer map as a function $\alpha:\mathbb{R}^{L_s\times D_s}\rightarrow\mathbb{R}^{L_t\times L_s}$
that reads the source activations $\Xs$ and returns the layer correspondence between the
two models. Its output, $\alpha(\Xs) \in \mathbb{R}^{L_t \times L_s}$, is a matrix whose entry $(\ell,j)$ is the  coefficient with
which source layer $j$ additively 
contributes to target layer $\ell$. Applying it to $\Xs$ therefore
mixes the source stack along the layer axis alone, leaving the feature axis untouched:
\begin{equation}
\label{eq:layer-map}
    \tilde{\X}^s = \alpha(\Xs) \Xs \in \mathbb{R}^{L_t\times D_s}.
\end{equation}
Importantly, because the layer correspondence is materialized as an explicit $L_t\times L_s$ matrix, it can be
read off directly. 
We now describe the parameterization of the layer-mapping matrix $\alpha$. The main idea is to use cross-attention, with queries representing target-model layers and keys representing source-model layers.



\paragraph{Constructing the layer map.}

Target activations $\Xt$ are precisely what the alignment must predict at inference time, and are available only during training. Due to this, the layer map and, in particular, the target-layer queries cannot depend on $\Xt$. Thus, we associate each target layer $\ell$ with a learned, \emph{input-independent} query $q_\ell$; these are gathered across layers into (learnable) matrix $ Q \in \R^{L_t \times d}$.
In contrast to the above, we, instead, do have access to source activations $\Xs$, so we compute an \emph{input-dependent} Key matrix as follows,
\begin{equation}
    K(\Xs) = \Xs \mathbf{W}^K \in \mathbb{R}^{L_s \times d},
\end{equation}
where $\mathbf{W}^K \in \mathbb{R}^{D_s \times d}$ is a learnable matrix. 
The layer-mapping matrix is then defined by row-wise cross-attention between the target-layer queries and source-layer keys:
\begin{equation}
    \!\!\alpha(X)
    =
    \operatorname{softmax}\footnote{We treat the softmax temperature as a hyperparameter, omitted from \Cref{eq:attn} for brevity.}\!\left(
        \frac{Q K(\Xs)^\top}{\sqrt{d}}
    \!\right)
    \in \R^{L_t \times L_s},
    \label{eq:attn}
\end{equation}
so that $\alpha_{\ell,j}$ is the weight assigned to source layer $j$, when
reconstructing target layer $\ell$. 

\paragraph{Generalizing constant maps.} Under mild conditions, our layer map in \Cref{eq:attn} can approximate all constant one-to-one maps used in prior work, including `terminal' and `relative-depth' mappings. In-fact, it can also approximate any constant soft many-to-many map. For further details and proofs, see \Cref{app:expressivity}.

\paragraph{Feature map.}
The feature map acts (independently) on the remaining axis, carrying the mixed representations from the
source model's hidden space into the target's. It applies a learnable projection
$\mathbf{W}$ shared across layers, followed by a residual
feed-forward block \texttt{FF}\footnote{\texttt{FF} is a two-layer network,
\texttt{Linear}$\to$\texttt{GELU}$\to$\texttt{Linear}, applied to the feature dimension.},
\begin{align}
    \mathbf{Z} &= \tilde{\X}^s \mathbf{W} \in \mathbb{R}^{L_t\times D_t}, \label{eq:feature-map1}
    \\
    \Ast(\Xs) &= \mathbf{Z} + \texttt{FF}(\mathbf{Z}), \label{eq:feature-map2}
\end{align}
to insure a fully expressive feature transformation.
\section{Experiments}

We evaluate our alignment method from three complementary perspectives. (i) We provide a \emph{quantitative} evaluation of the learned activation alignment map $\Ast$ directly through its reconstruction performance. (ii) We provide a \emph{qualitative} analysis -- analyzing the layer correspondences we learned between different LLMs, by inspecting our layer-map ($\alpha$). (iii) We examine the utility of the learned alignment in two downstream applications: probing and steering. In both these settings, we assume that a probe or steering vector is available only for the target model and we use it by mapping $\Xs$ to the activation space of the target model. We present only the main results here; further details and extended results are provided in \Cref{sec:appendix-results}.

\paragraph{LLMs.}
We evaluate our method on seven instruction-tuned LLMs: Llama-3.2-1B, Llama-3.2-3B, Llama-3.1-8B~\cite{dubey2024llama3herdmodels}, Qwen2.5-1.5B, Qwen2.5-1.3B, Qwen2.5-1.7B~\cite{qwen2, qwen2.5}, gemma-2-2b~\cite{team2024gemma}. We learn alignments between every non-degenerate pair,
yielding 42 such pairs. The selected models span multiple model families, parameter scales, architectural variants, and training regimes, allowing us to test whether the proposed alignment generalizes across (substantial) differences in model design.

\paragraph{Methods in comparison.}
We compare against two categories of methods. The first consists of methods adopted directly from prior work --- we dub those \emph{external baselines}. The second consists of \emph{additional baselines} designed either to isolate the contributions of our architectural choices or adapted from prior methods developed for closely related tasks. In the latter case, although these methods were not originally proposed for directly aligning source and target models, they can be incorporated into our setting with only minor modifications and are therefore included as additional baselines, using our same training pipeline. \textbf{External baselines:} 
(1) \texttt{Command-V}: \cite{wang2026commandv} constructs layer correspondences according to the relative depth strategy, mapping target layer $l_t$ to source layer $l_s$ as $l_s=\left\lfloor \frac{|L_s|}{|L_t|}l_t \right\rfloor$; 
then it fits a linear regression between the representations of each corresponding layer pair\footnote{\texttt{Command-V} trains alignment adapters only for even-numbered layers of the target model. In our setting, we extend its relative-depth mapping to all target layers.}.
(2) \texttt{HOT}: \cite{cui2026transport, shah2026representational} uses activation-based hierarchical optimal transport to infer soft feature-level correspondences between every source-target layer pair, aggregates the resulting transport costs to infer a global layer-level correspondence, and uses these transport plans to map between the activation spaces. \textbf{Additional baselines:} (1) \texttt{Terminal}: motivated by \cite{fu2026cachetocache}, it applies only when aligning a deeper source model with a shallower target one. It constructs the map using the final $L_t$ layers of the source model; this is only applicable if $L_s \geq L_t$.
(2) \texttt{Relative Depth}~\cite{wolfram2025layerssimilardepthsgenerate, fu2026cachetocache}: follows our method but replaces $\alpha$ (see \Cref{eq:attn}) with the relative-depth layer mapping used in \texttt{Command-V}. (3) \texttt{CKA}~\cite{kornblith2019similarityneuralnetworkrepresentations}: replaces $\alpha$ with a fixed map obtained by computing CKA similarity for every source--target layer pair and applying a softmax over the source layers for each target layer. While \texttt{Terminal}, \texttt{Relative Depth} and \texttt{Command-V} are using one-to-one layer-maps, the rest use soft, many-to-many maps. 


\paragraph{Setup.}
Following Command-V~\cite{wang2026commandv}, we train an alignment for each ordered pair of models on the LIMA dataset~\cite{zhou2023lima} using the formulation introduced in \Cref{sec:formalization}; \Cref{fig:objective} provides a high-level illustration of the pipeline. LIMA is a compact instruction-following dataset containing 1,030 training conversations and 300 held-out test prompts spanning diverse domains and tasks. Its combination of small size and broad prompt diversity makes it suitable for learning generic alignments across many model pairs. We divide the training split into training and validation subsets, select hyperparameters for each model pair on the validation set, and evaluate alignment quality on the official test set. 
We train each configuration with three random seeds and report the mean and standard deviation; further details are provided in \Cref{sec:appendix-results}.

\subsection{Evaluation of Cross-Model Alignment}
\label{subsec:Quantitative}
We consider the following evaluation metrics.

\paragraph{Evaluation metrics.}
Motivated by prior work~\cite{moschella2023relative, bansal2021revisiting, lample2018word}, we evaluate the learned mappings using nearest-neighbor retrieval and reconstruction quality. We provide brief intuitive descriptions below and defer their formal definitions to \Cref{sec:appendix-metrics}. Let $\hat{\mathbf{X}}^{t} \coloneqq \Ast(\Xs)$ denote the mapped representation of source sample $\Xs$. \textbf{Recall@1} (R@1) measures whether the ground-truth target representation $\Xt$ is the nearest neighbor of $\hat{\mathbf{X}}^{t}$ in the target space. Conversely, \textbf{Precision@1} (P@1) measures whether the mapped representation $\hat{\mathbf{X}}^{t}$ is the nearest neighbor of $\Xt$. \textbf{Jaccard@5} (J@5) compares the sets of five nearest neighbors of the source
and target representations in their spaces, respectively, and computes the size of their intersection divided by the size of their union.
\textbf{Tok-Jac@5} (TJ@5) applies the same Jaccard measure to the sets of top-5 predicted next-tokens, conveying a similar, but behavior-related metric. Finally, \textbf{MSE loss} is the mean squared error between the mapped and target representations and directly measures reconstruction accuracy. Higher values are better for the retrieval-based metrics, whereas lower MSE indicates better performance.

\paragraph{Results.} \begin{table*}[t]\centering\small
\resizebox{\textwidth}{!}{%
\begin{tabular}{lcccccccccccc}
\toprule
 & \multicolumn{4}{c}{comparable-size} & \multicolumn{4}{c}{Bigger$\to$Smaller} & \multicolumn{4}{c}{Smaller$\to$Bigger} \\
\cmidrule(lr){2-5} \cmidrule(lr){6-9} \cmidrule(lr){10-13}
Method & TJ@5 & R@1 & P@1 & J@5 & TJ@5 & R@1 & P@1 & J@5 & TJ@5 & R@1 & P@1 & J@5 \\
\midrule
\texttt{Terminal} & \underline{0.449}\,{$\scriptscriptstyle\pm0.001$} & 0.716\,{$\scriptscriptstyle\pm0.001$} & 0.915\,{$\scriptscriptstyle\pm0.001$} & 0.376\,{$\scriptscriptstyle\pm0.000$} & 0.407\,{$\scriptscriptstyle\pm0.001$} & \underline{0.759}\,{$\scriptscriptstyle\pm0.000$} & 0.874\,{$\scriptscriptstyle\pm0.000$} & 0.381\,{$\scriptscriptstyle\pm0.000$} & -- & -- & -- & -- \\
\texttt{Relative Depth} & 0.447\,{$\scriptscriptstyle\pm0.002$} & \underline{0.727}\,{$\scriptscriptstyle\pm0.001$} & \underline{0.917}\,{$\scriptscriptstyle\pm0.001$} & \underline{0.377}\,{$\scriptscriptstyle\pm0.000$} & \underline{0.411}\,{$\scriptscriptstyle\pm0.001$} & 0.757\,{$\scriptscriptstyle\pm0.000$} & \underline{0.918}\,{$\scriptscriptstyle\pm0.000$} & \underline{0.395}\,{$\scriptscriptstyle\pm0.000$} & \underline{0.435}\,{$\scriptscriptstyle\pm0.001$} & \underline{0.686}\,{$\scriptscriptstyle\pm0.000$} & \underline{0.856}\,{$\scriptscriptstyle\pm0.000$} & \underline{0.381}\,{$\scriptscriptstyle\pm0.000$} \\
\texttt{CKA} & 0.361\,{$\scriptscriptstyle\pm0.004$} & 0.544\,{$\scriptscriptstyle\pm0.002$} & 0.457\,{$\scriptscriptstyle\pm0.002$} & 0.204\,{$\scriptscriptstyle\pm0.001$} & 0.283\,{$\scriptscriptstyle\pm0.002$} & 0.634\,{$\scriptscriptstyle\pm0.001$} & 0.561\,{$\scriptscriptstyle\pm0.001$} & 0.256\,{$\scriptscriptstyle\pm0.001$} & 0.311\,{$\scriptscriptstyle\pm0.002$} & 0.567\,{$\scriptscriptstyle\pm0.001$} & 0.464\,{$\scriptscriptstyle\pm0.002$} & 0.217\,{$\scriptscriptstyle\pm0.001$} \\
\midrule
\texttt{Command-V} & 0.316 & 0.541 & 0.395 & 0.321 & 0.294 & 0.617 & 0.517 & 0.347 & 0.281 & 0.627 & 0.465 & 0.325 \\
HOT & 0.091 & 0.030 & 0.343 & 0.253 & 0.093 & 0.041 & 0.434 & 0.259 & 0.101 & 0.037 & 0.401 & 0.263 \\
\midrule
\ourmethod (Ours) & \textbf{0.455}\,{$\scriptscriptstyle\pm0.002$} & \textbf{0.833}\,{$\scriptscriptstyle\pm0.002$} & \textbf{0.968}\,{$\scriptscriptstyle\pm0.001$} & \textbf{0.415}\,{$\scriptscriptstyle\pm0.001$} & \textbf{0.422}\,{$\scriptscriptstyle\pm0.001$} & \textbf{0.847}\,{$\scriptscriptstyle\pm0.002$} & \textbf{0.958}\,{$\scriptscriptstyle\pm0.002$} & \textbf{0.433}\,{$\scriptscriptstyle\pm0.001$} & \textbf{0.449}\,{$\scriptscriptstyle\pm0.001$} & \textbf{0.808}\,{$\scriptscriptstyle\pm0.001$} & \textbf{0.962}\,{$\scriptscriptstyle\pm0.001$} & \textbf{0.413}\,{$\scriptscriptstyle\pm0.000$} \\
\midrule
\rowcolor{gray!15}
\textbf{Improvement} & $\mathbf{+}$\textbf{0.005} & $\mathbf{+}$\textbf{0.106} & $\mathbf{+}$\textbf{0.052} & $\mathbf{+}$\textbf{0.038} & $\mathbf{+}$\textbf{0.011} & $\mathbf{+}$\textbf{0.088} & $\mathbf{+}$\textbf{0.040} & $\mathbf{+}$\textbf{0.038} & $\mathbf{+}$\textbf{0.014} & $\mathbf{+}$\textbf{0.123} & $\mathbf{+}$\textbf{0.106} & $\mathbf{+}$\textbf{0.032} \\
\bottomrule\end{tabular}}
\caption{\textbf{Alignment quality} across comparable-size, larger-to-smaller, and smaller-to-larger model pairs. Best results are in \textbf{bold}, and second-best results are underlined. per-pair breakdowns are provided in \Cref{sec:appendix-metrics-raw}.}
\label{tab:main-metrics}
    \vspace{-0.4cm}
\end{table*}

In \Cref{tab:main-metrics}, we compare our method with the baselines across all proposed evaluation metrics; we report the corresponding MSE losses separately in \Cref{sec:appendix-mse}. We group the results into three categories according to the relative sizes of the source and target models: models of comparable size, larger-to-smaller mappings, and smaller-to-larger mappings. \ourmethod consistently achieves the best performance across all metrics and size configurations; for example, we observe a large gain of around 0.12 in the R@1 metric, in the smaller $\rightarrow$ larger setup . In general, we observe the largest gains when transferring from smaller to larger models, suggesting that expressive, learnable mappings, such as ours, are especially important when aligning activation spaces across models of different sizes. Among the baselines, \texttt{Relative Depth} achieves the second-best overall performance. MSE losses exhibit the same overall trend as discussed above; see \Cref{sec:appendix-mse} 

\paragraph{Mapping across model families.}
\begin{figure}[h]
\vspace{-0.2cm}
\centering
\includegraphics[width=0.8\linewidth]{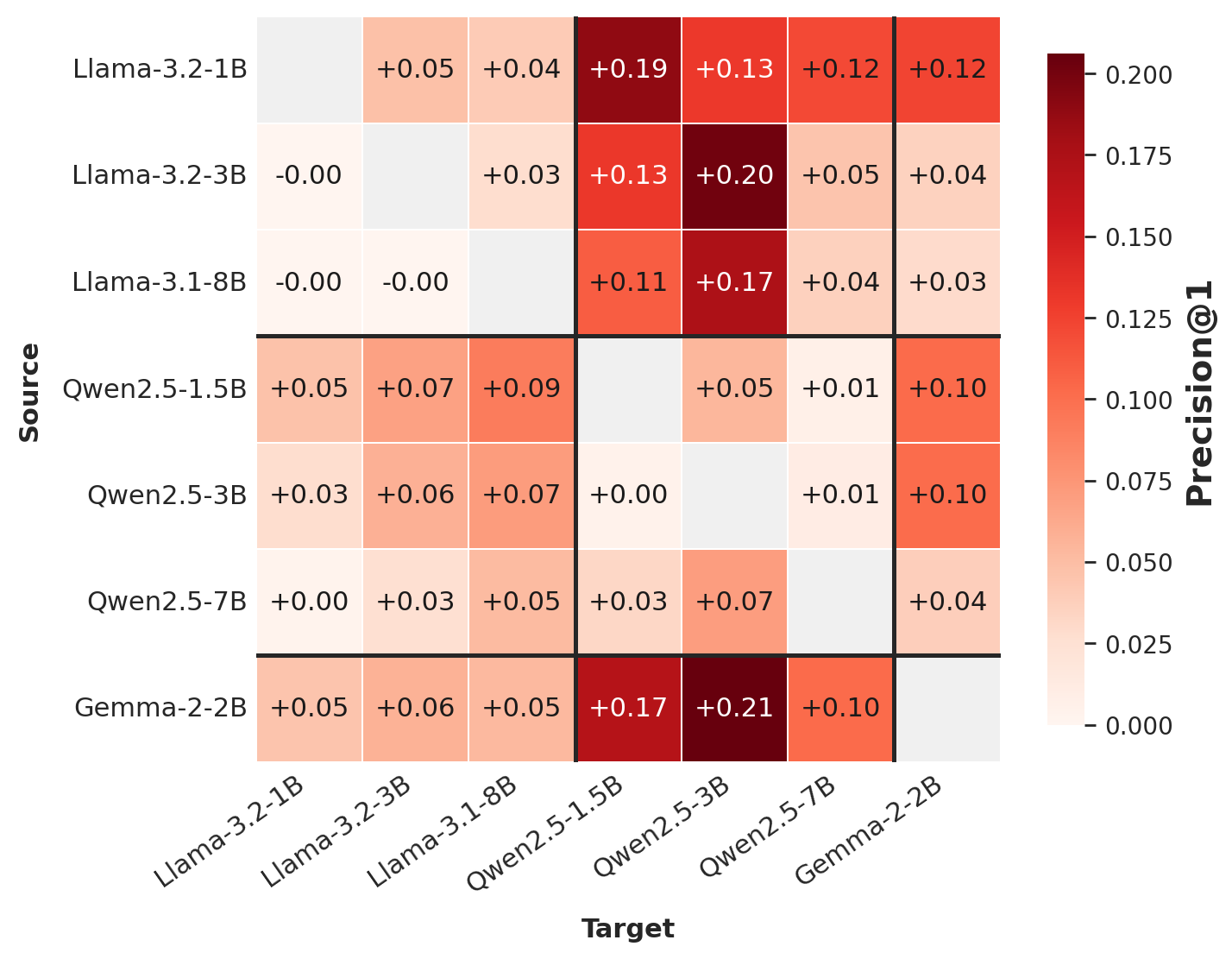}
\caption{Per-pair Precision@1 difference (\ourmethod $-$ \texttt{relative depth}).}
\label{fig:wk-vs-linear-main}
\vspace{-0.2cm}
\end{figure}

For a more fine-grained understanding of our alignment results, in \Cref{fig:wk-vs-linear-main}, we show them grouped by model family. In particular, we compare \ourmethod with \texttt{Relative-Depth} (the second-best overall method in terms of quantitative results; see \Cref{tab:main-metrics}), by reporting the difference in Recall@1 across model-family pairings. Interestingly, we observe our method achieves its largest gains when aligning models from different families. This suggests that simpler, fixed, mapping hypotheses may be more suitable when aligning models within the same architectural families, while highlighting the importance of more flexible, multi-layer mappings (such as ours) when the source and target architectures differ more substantially.


\subsection{Learned Layer Correspondences}

\begin{figure}[h]
\vspace{-0.2cm}
    \centering
    \includegraphics[width=\columnwidth]{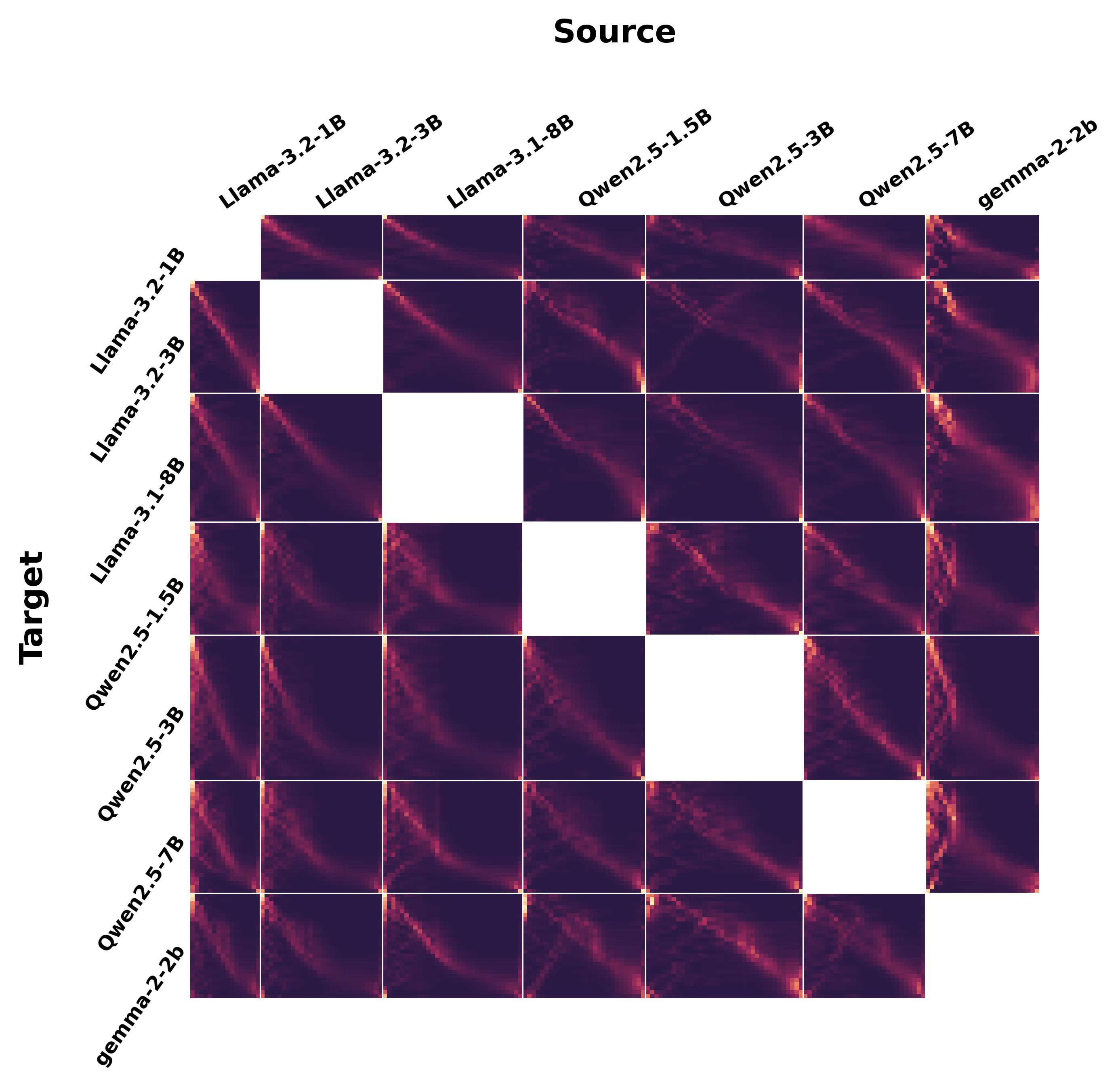}
    \caption{The final, learned $\alpha$ obtained by our method. 
    }
    \label{fig:wk-plus-ffn}
    \vspace{-0.2cm}
\end{figure}

\Cref{fig:wk-plus-ffn} visualizes the layer-alignment maps, $\alpha$, learned by \ourmethod across all model pairs, averaged over samples from a held-out validation subset of the LIMA training set.
We observe a consistent emerging pattern: the learned correspondences are generally \emph{not} one-to-one, with each target layer reconstructed as a \emph{weighted combination} of multiple sources. Additionally, while a significant portion of the alignment mass is concentrated in the \emph{proximity} of the relative-depth correspondence line, 
the salient correspondence motifs often depart from the typical \emph{relative} or \emph{terminal} matching hypotheses. Especially when the source and target models differ in depth, the near-diagonal structure becomes more curved, producing a characteristic ``belly'' around the middle of the network. Our per-pair quantitative breakdown, presented in \Cref{sec:appendix-metrics-raw},
shows that alignments involving Gemma consistently yield lower performance. This trend is echoed in \Cref{fig:wk-plus-ffn}, where maps involving Gemma exhibit a slightly different structure from those of all other model pairs, with a less clear layer correspondence. 
From these observations, we conclude that both alignment quality and correspondence motifs may depend on the specific model pair, with some representation spaces exhibiting clearer cross-layer correspondences than others. 

\subsection{Downstream Applications}
\label{subsec:Downstream applications}
We evaluate the alignment methods on the task of \emph{transfer of activation-space interventions}, 
in which $\Ast$ is used to transfer a target-model activation-space operator
to the source model. We consider two tasks: \emph{steering} \cite{turner2024steeringlanguagemodelsactivation} and \emph{probing} \cite{belinkov-2022-probing,alain2017understanding}. In both cases, the goal is to transfer, to the source model, an activation-spaces intervention
available only for the target model. These are obtained, in particular, by applying a steering vector or an activation probe. The idea is depicted in \Cref{fig:steering,fig:probing}: it consists in approximating the (unknown or unavailable) source steering model $S^s$ (resp., probe model $P^s$), via the target's $S^t$ (resp., $P^t$), assuming activation alignment maps $\Ast, \Ats$ are available:
\begin{align}
  &\text{Steering: }S^s(\Xs) \!\approx\! \Ats\!\left(S^t\!\left(\Ast(\Xs)\right)\right), \\
  &\text{Probing: }P^s(\Xs) \!\approx\! P^t\!\left(\Ast(\Xs)\right).
\end{align}

\begin{figure}[t]
  \centering
  \begin{subfigure}[b]{0.49\columnwidth}
    \centering
    \resizebox{0.8\linewidth}{!}{%
    \begin{tikzpicture}[>=Stealth]
      \path[use as bounding box] (-0.9,-1.45) rectangle (3.9,1.75);
      \node[st, draw=srccolor, fill=srccolor!10, text=srccolor] (xs)  at (0,   1.0) {$\Xs$};
      \node[st, draw=tgtcolor, fill=tgtcolor!10, text=tgtcolor] (xt)  at (2.6, 1.0) {$\Xt$};
      \node[st, draw=srccolor, fill=srccolor!10, text=srccolor] (fxs) at (0,  -0.8) {$S^s(\Xs)$};
      \node[st, draw=tgtcolor, fill=tgtcolor!10, text=tgtcolor] (fxt) at (2.6,-0.8) {$S^t(\Xt)$};
      \draw[dashed, gray!35] (1.3, 1.45) -- (1.3, -1.25);
      \node[lbl, text=srccolor] at (0,   1.45) {source};
      \node[lbl, text=tgtcolor] at (2.6, 1.45) {target};
      \draw[->, orange, thick] (xs)  -- node[above, lbl] {$\Ast$} (xt);
      \draw[->, tgtcolor]      (xt)  -- node[right, lbl] {$S^t$}  (fxt);
      \draw[->, orange, thick] (fxt) -- node[below, lbl] {$\Ats$} (fxs);
      \draw[->, srccolor, dashed, thick] (xs) -- node[left, lbl] {$S^s$} (fxs);
    \end{tikzpicture}}
    \caption{Steering.}
    \label{fig:steering}
  \end{subfigure}
  \hfill
  \begin{subfigure}[b]{0.49\columnwidth}
    \centering
    \resizebox{0.8\linewidth}{!}{%
    \begin{tikzpicture}[>=Stealth]
      \path[use as bounding box] (-0.9,-1.45) rectangle (3.9,1.75);
      \node[st, draw=srccolor, fill=srccolor!10, text=srccolor] (xs)    at (0,   1.0) {$\Xs$};
      \node[st, draw=tgtcolor, fill=tgtcolor!10, text=tgtcolor] (xt)    at (2.6, 1.0) {$\Xt$};
      \node[st, draw=gray!70,  fill=gray!10]                    (score) at (2.6,-0.8) {score $\in\mathbb{R}$};
      \draw[dashed, gray!35] (1.3, 1.45) -- (1.3, -1.25);
      \node[lbl, text=srccolor] at (0,   1.45) {source};
      \node[lbl, text=tgtcolor] at (2.6, 1.45) {target};
      \draw[->, orange, thick] (xs) -- node[above, lbl] {$\Ast$} (xt);
      \draw[->, tgtcolor]      (xt) -- node[right, lbl] {$P^t$}  (score);
      \draw[->, srccolor, dashed, thick] (xs) -- node[below left, lbl] {$P^s$} (score);
    \end{tikzpicture}}
    \caption{Probing.}
    \label{fig:probing}
  \end{subfigure}
  \caption{\textbf{Transfer of activation-space interventions.} Dashed arrows denote source-side operators, which are approximated by the solid path through the target model.}
  \label{fig:teleportation}
  \vspace{-0.3cm}
\end{figure}
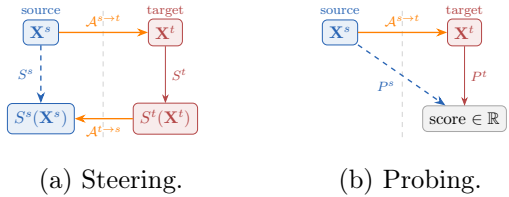
    \vspace{-0.2cm}
\paragraph{Steering setup.}
Following \citet{arditi2024refusal}, we study refusal suppression using steering vectors computed as Difference in Mean (DiM) activations of harmful and harmless prompts in the target model. 
For each prompt, we map only the final input token's activation at each layer into the target space, apply the target steering vector, and map the resulting steering-induced change back to the source space. We then recompute the final input token's activations while injecting these transferred changes at the corresponding source-model layers. We evaluate the generated responses on three datasets: AdvBench \citep{zou2023universal}, JailbreakBench \citep{chao2024jailbreakbench}, and HarmBench \citep{mazeika2024harmbench} using LlamaGuard3~8B \citep{dubey2024llama3herdmodels} as the safety classifier. Full details of the datasets, probe training, steering-vector construction, transfer algorithm, and evaluation procedure are provided in~\Cref{sec:appendix-probing,sec:appendix-steering}. 
\paragraph{Probing setup.}
We consider hallucination detection on TriviaQA \citep{joshi-etal-2017-triviaqa}: we train a linear probe at each layer of the target model, select the best-performing target probe by a validation set, and apply it to source activations mapped into the target representation space.

\paragraph{Steering results.}
\begin{figure}[t]
    \centering
    \includegraphics[width=0.8\linewidth]{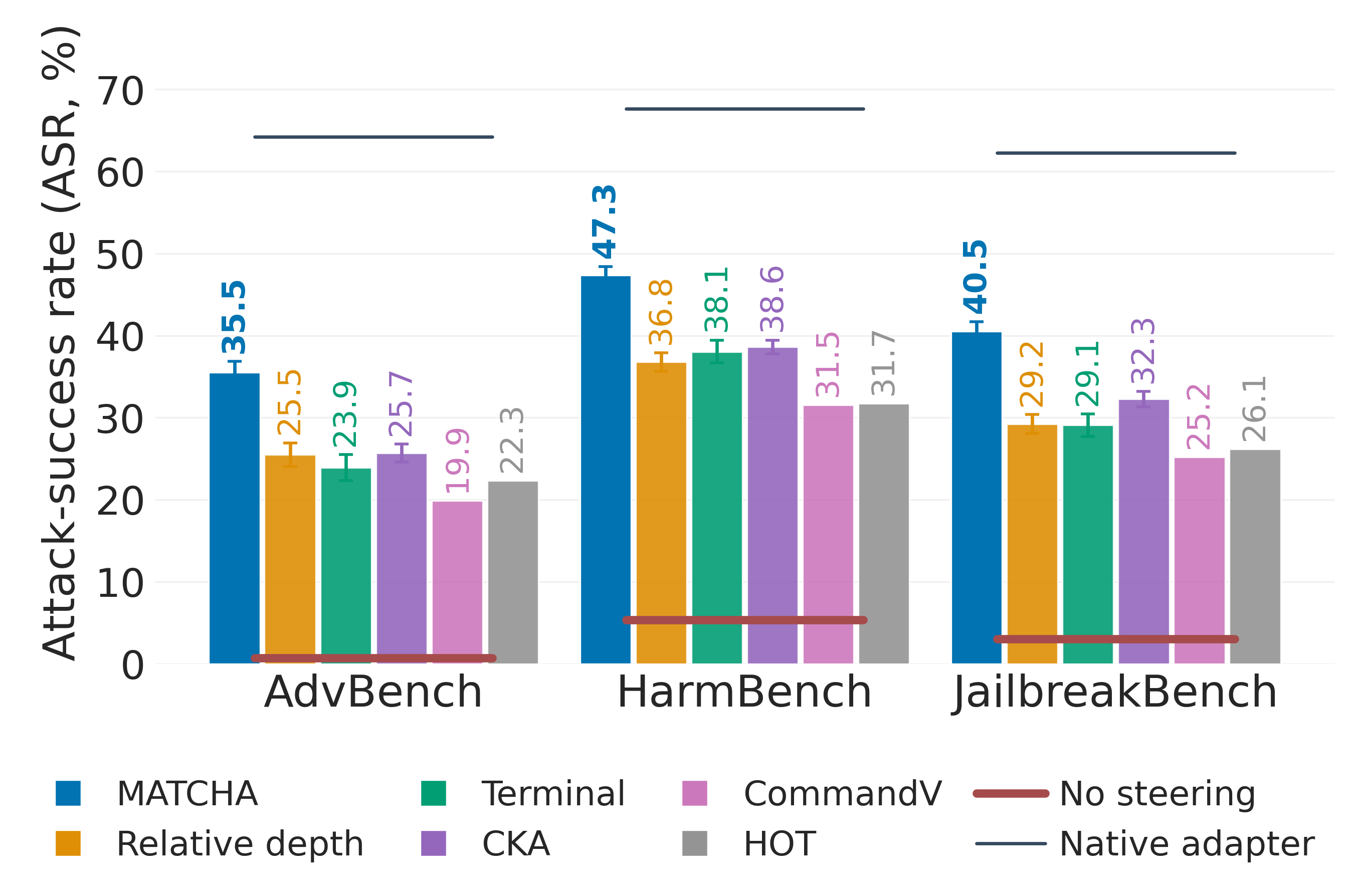}
    \caption{\textbf{Steering results.} We report per-dataset mean test ASR, averaged over all pairs.
}
    \label{fig:steering_results}
    \vspace{-0.3cm}
\end{figure}

Results are shown in \Cref{fig:steering_results}\footnote{The \texttt{Terminal} baseline is averaged only on applicable pairs -- consistent with \Cref{tab:main-metrics}.}, where the metric of interest is \emph{Attack-success rate} (ASR) \%; higher is better. For each dataset, we report the performance of all alignment methods, along with a lower bound corresponding to performance without steering and an upper bound given by the ``native adapter,'' which is trained directly on the target model. Across all three datasets, \ourmethod outperforms the other alignment techniques by a wide margin, where the second best method is the \texttt{Relative Depth} baseline.

\paragraph{Probing results.}
\begin{table}[t]
\centering
\small
\resizebox{\columnwidth}{!}{%
\begin{tabular}{lccc}
\toprule
Method & Same-size & Bigger$\to$Smaller & Smaller$\to$Bigger \\
\midrule
\texttt{Terminal} & \underline{0.647}\,{$\scriptscriptstyle\pm0.006$} & 0.634\,{$\scriptscriptstyle\pm0.004$} & -- \\
\texttt{Relative Depth} & 0.646\,{$\scriptscriptstyle\pm0.007$} & \underline{0.654}\,{$\scriptscriptstyle\pm0.004$} & \underline{0.664}\,{$\scriptscriptstyle\pm0.002$} \\
\texttt{CKA} & 0.549\,{$\scriptscriptstyle\pm0.009$} & 0.593\,{$\scriptscriptstyle\pm0.005$} & 0.597\,{$\scriptscriptstyle\pm0.003$} \\
\texttt{Command-V} & 0.589 & 0.605 & 0.601 \\
\texttt{HOT} & 0.642 & 0.639 & 0.652 \\
\midrule
\textbf{MATCHA} (Ours) & \textbf{0.683}\,{$\scriptscriptstyle\pm0.007$} & \textbf{0.658}\,{$\scriptscriptstyle\pm0.007$} & \textbf{0.693}\,{$\scriptscriptstyle\pm0.004$} \\
\bottomrule
\end{tabular}
}
\caption{Probing transfer (TriviaQA test AUC $\uparrow$) by size group --- the receiver's \textbf{own probe} applied to the transferred representation. 
}
\label{tab:probing-main-tgt}
    \vspace{-0.3cm}
\end{table}

Results are reported in \Cref{tab:probing-main-tgt}. \ourmethod outperforms all baselines, in some cases by a large margin. For example, in the smaller-to-larger setting, it improves AUC by approximately three points over the second-best method. Finally, we again note, that \texttt{Relative Depth} is the strongest baseline overall.


\subsection{Ablation Studies}
In this section, we (i) ablate our multi layer maps; and (ii) ablate components of our architecture.

\paragraph{Ablating Multi-layer mapping.} 
\begin{figure}[h!]
\centering
\includegraphics[width=0.8\linewidth]{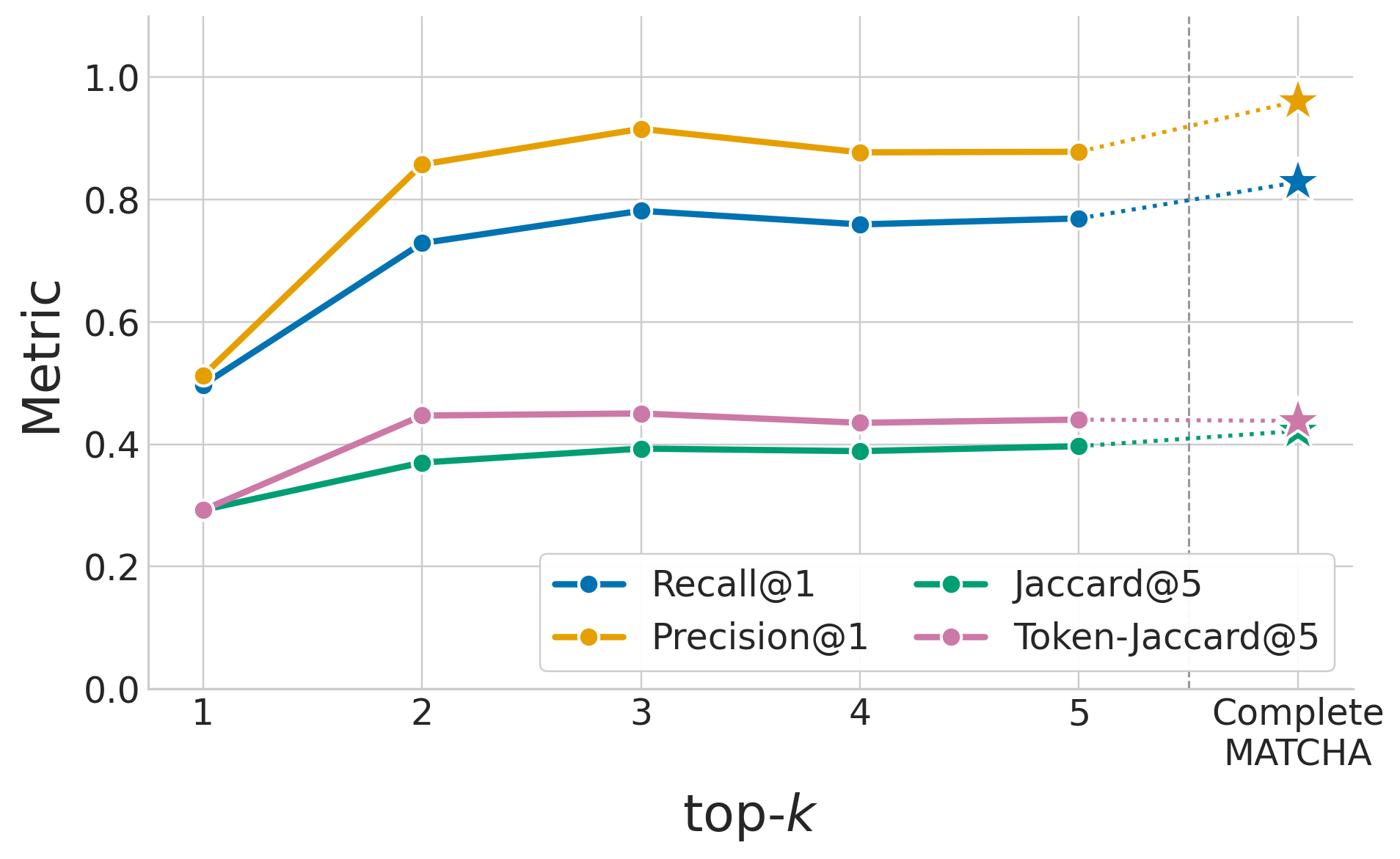}
\caption{Retrieval performance versus top-$k$ source layers, averaged across all LLM pairs.}

\label{fig:ksweep}
\end{figure}
We investigate the relevance of the multi-layer mapping estimated by our approach. In \Cref{fig:ksweep} we show the impact of restricting the layer-map obtained by \ourmethod for each target layer to its top-$k$ source layers, ranked by alignment mass\footnote{We renormalize scores across the selected top-$k$ source layers with softmax.}. The results show that, when restricted to one-to-one alignments ($k=1$), our method performs substantially worse than alignments combining multiple source layers ($k>1$). Moreover, our full alignment map consistently outperforms all of its top-$k$ approximations. 
These findings provide additional evidence that our learned cross-model layer correspondences are calculating not only a better performing, but also an inherently multi-layer map.

\paragraph{Impact of design choices.} 
\begin{table}[t]
\centering
\small
\resizebox{\columnwidth}{!}{%
\begin{tabular}{lcccc}
\toprule
Variant & R@1 & P@1 & J@5 & TJ@5 \\
\midrule
MATCHA-K & \underline{0.715}\,{$\scriptscriptstyle\pm0.002$} & \underline{0.834}\,{$\scriptscriptstyle\pm0.001$} & \underline{0.368}\,{$\scriptscriptstyle\pm0.001$} & \underline{0.416}\,{$\scriptscriptstyle\pm0.001$} \\
MATCHA-FL & 0.707\,{$\scriptscriptstyle\pm0.002$} & 0.794\,{$\scriptscriptstyle\pm0.002$} & 0.340\,{$\scriptscriptstyle\pm0.001$} & 0.412\,{$\scriptscriptstyle\pm0.001$} \\
\texttt{MLP} & 0.460\,{$\scriptscriptstyle\pm0.017$} & 0.817\,{$\scriptscriptstyle\pm0.009$} & 0.318\,{$\scriptscriptstyle\pm0.003$} & 0.364\,{$\scriptscriptstyle\pm0.003$} \\
\midrule
\textbf{MATCHA} (Ours) & \textbf{0.829}\,{$\scriptscriptstyle\pm0.001$} & \textbf{0.961}\,{$\scriptscriptstyle\pm0.001$} & \textbf{0.422}\,{$\scriptscriptstyle\pm0.001$} & \textbf{0.438}\,{$\scriptscriptstyle\pm0.001$} \\
\bottomrule
\end{tabular}}
\caption{\textbf{Architecture ablation}. The results are averaged over all 42 LLM pairs.
}\label{tab:ablation-arch}
    \vspace{-0.3cm}
\end{table}

To evaluate our architectural design choices, we study the following variants of our method: (1) \ourmethod-$K$, in which the key matrix is an \emph{input-independent} learnable parameter rather than a projection of $\Xs$; (2) \ourmethod with a Fully Learned layer map (\ourmethod-FL), in which the layer-mapping coefficients $\alpha$ are parameters learned without constraints; and (3) \texttt{MLP}, a simple \texttt{MLP} applied to flattened representations from the two activation spaces. Results are reported in \Cref{tab:ablation-arch}. Each modification to \ourmethod degrades performance across all metrics. Importantly, both \ourmethod-$K$ and \ourmethod-FL produce \emph{input-independent} layer maps and perform significantly worse than \ourmethod, which learns an \emph{input-dependent} map. This suggests that layer correspondences between LLMs may not be fixed, and depend on the specific input. Finally, the unstructured \textsc{MLP} baseline performs substantially worse then all ablation methods, highlighting the value of explicitly decoupling the layer-map and feature-map.

\section{Conclusion}
We formulated cross-LLM activation alignment as a learning problem, in which layer correspondences are learned rather than assumed. We introduced \ourmethod, a highly efficient method that applies the learned alignment in approximately $0.06$ milliseconds and decomposes it into an interpretable cross-layer map and a feature transformation. 

Across 42 ordered model pairs, \ourmethod consistently outperforms prior baselines, with particularly strong gains for cross-family and smaller-to-larger transfers. \ourmethod learned layer-maps further reveal that cross-model layer correspondences are broadly depth-monotone, yet, \emph{inherently multi-layer}. Beyond improving alignment quality, \ourmethod enables probes and steering interventions to transfer across models, providing a practical framework for reusing activation-based tools with newly introduced LLMs without retraining the tools for each such new model.



\clearpage

\section{Ethical Considerations}
Although our method can facilitate the reuse of interpretability and control tools, its ability to transfer steering interventions is dual-use and could be misused to weaken model safeguards; we therefore recommend careful safety evaluation and appropriate access controls before deployment.

\section*{Limitations}
Despite its effectiveness, our approach has several limitations. It requires access to internal activations, or hidden states, and is therefore applicable only to open-source models or when implemented directly by closed-source model providers. In addition, we train a separate activation map for each model pair because activation spaces may differ across models. We leave the development of a joint activation map trained across multiple model pairs to future work.

\section*{Acknowledgements}
G.B. is supported by the Jacobs Qualcomm PhD Fellowship. F.F. conducted this work supported by an Aly Kaufman Post-Doctoral Fellowship. H.M.\ is supported by the Israel Science Foundation through a personal grant (ISF 264/23)
and an equipment grant (ISF 532/23), and by the Career Advancement Chairs in Artificial
Intelligence--Schmidt Futures.

\bibliographystyle{plainnat}
\bibliography{custom}


\clearpage
\onecolumn

\RenewDocumentCommand{\twocolumn}{o}{\onecolumn}

\RenewDocumentEnvironment{multicols}{m o o}
  {\par}
  {\par}

\RenewDocumentEnvironment{multicols*}{m o o}
  {\par}
  {\par}

\appendix

\global\MATCHAsectionpendingfalse
\global\MATCHAparagraphpendingfalse
\FloatBarrier

\FloatBarrier

\RenewDocumentEnvironment{figure}{O{!htbp}+b}
  {%
    \ifnum\value{figure}<12
      \par
      \medskip
      \noindent\hfill
      \begin{minipage}{0.60\textwidth}
        \MATCHAsetlocalwidth
        \centering
        \captionsetup{type=figure}
        #2
      \end{minipage}
      \hfill\mbox{}\par
      \medskip
    \else
      \MATCHAorigfigure[#1]
      \centering
      \begin{minipage}{0.60\textwidth}
        \MATCHAsetlocalwidth
        \centering
        #2
      \end{minipage}
      \MATCHAendorigfigure
    \fi
  }
  {}

\RenewDocumentEnvironment{figure*}{O{!htbp}+b}
  {%
    \MATCHAorigfigure[#1]
    \centering
    \begin{minipage}{0.80\textwidth}
      \MATCHAsetlocalwidth
      \centering
      #2
    \end{minipage}
    \MATCHAendorigfigure
  }
  {}

\begingroup
\RenewDocumentCommand{\paragraph}{s o m}{}

\appendix

\section*{Appendix}

\section{Additional Implementation Details}
\label{sec:appendix-impl-details}

\subsection{Activation Normalization}
\label{sec:appendix-normalization}

Before training, we perform a calibration pass over the alignment-training data and
estimate, for every layer of every model, a mean $\mu_\ell$ and standard deviation
$\sigma_\ell$, pooled over all activation entries collected for that layer (across
examples, token positions, and hidden dimensions). Each source and target layer $\ell$ is
normalized using its own statistics,
\[
    x_{\mathrm{norm}} = \frac{x - \mu_\ell}{\sigma_\ell},
\]
before $\Ast$ is applied and the reconstruction loss of
Eq.~\Cref{eq:recon} is computed; this prevents layers with larger
activation scales from dominating the objective. $\Ast$'s output is denormalized back with
the target layer's own statistics,
\[
    x = x_{\mathrm{norm}} \cdot \sigma_\ell + \mu_\ell,
\]
whenever it needs to be read out on the target layer's original activation scale---for
example when used inside the target model during activation-editing transfer, or when
computing the MSE/retrieval metrics (see \Cref{tab:main-metrics}), which we report in this original, denormalized scale rather than the
normalized training coordinates. The normalization statistics are fixed once, at
calibration time, and reused unchanged throughout training, validation, and testing.

\subsection{Extended Metric Definitions}
\label{sec:appendix-metrics}

This appendix restates the four retrieval-style metrics of \S\ref{subsec:Quantitative} precisely,
in the same notation. Let $\hat{X}^s_i := \Ast(\Xs_i)$ denote the mapped representation of
held-out example $i$, and let $\mathrm{KNN}_K(\cdot)$ denote the index set of the $K$
nearest neighbors of a representation under pairwise mean-squared distance.

\begin{itemize}
    \item \textbf{Recall@K}: $\frac{1}{N}\sum_i \mathbb{1}\!\left[i \in
    \mathrm{KNN}^{t}_K\big(\hat{X}^s_i\big)\right]$, querying from the mapped
    representation into the target space---does the true target land among the
    representations nearest to the prediction?

\paragraph{Recall@K.}

\begin{figure}[H]
  \centering
  \scalebox{0.75}{
  \begin{tikzpicture}[
    pnode/.style={circle, fill=orange!85!black, text=white, inner sep=4pt, font=\small\bfseries},
    tnode/.style={circle, fill=tgtcolor,        text=white, inner sep=4pt, font=\small\bfseries},
    >=Stealth]
    \node[orange!85!black, font=\small\bfseries] at (0.0, 3.3)
      {predicted $\Xshat_i{=}\Ast(\Xs_i)$};
    \node[tgtcolor, font=\small\bfseries] at (4.2, 3.3) {target $\Xt_i$};
    \foreach \i/\y in {1/2.6, 2/1.3, 3/0.0}
      \node[font=\small, gray!65!black, anchor=east] at (-1.6, \y) {$i{=}\i$};
    \node[pnode] (p1) at (0.0, 2.6) {$\Xshat_1$};
    \node[pnode] (p2) at (0.0, 1.3) {$\Xshat_2$};
    \node[pnode] (p3) at (0.0, 0.0) {$\Xshat_3$};
    \node[tnode] (t1) at (4.2, 2.6) {$\Xt_1$};
    \node[tnode] (t2) at (4.2, 1.3) {$\Xt_2$};
    \node[tnode] (t3) at (4.2, 0.0) {$\Xt_3$};
    \draw[->, red!70!black, thick, bend right=18] (p1) to (t2);
    \draw[->, tgtcolor,     thick]                (p2) -- (t2);
    \draw[->, tgtcolor,     thick]                (p3) -- (t3);
    \node[red!70!black, font=\large\bfseries, left=6pt of p1] {\texttimes};
    \node[tgtcolor,     font=\large\bfseries, left=6pt of p2] {\checkmark};
    \node[tgtcolor,     font=\large\bfseries, left=6pt of p3] {\checkmark};
    \node[font=\small] at (2.1, -0.85) {Recall@1 $= \tfrac{2}{3} \approx 66.7\%$};
  \end{tikzpicture}
  }
  \caption{Recall@1 on three samples. Each predicted representation $\Xshat_i$
  (orange) queries its nearest neighbour in the target space (teal).
  A sample is a \textcolor{tgtcolor}{\textbf{hit}} (\checkmark) when $\Xt_i$ is
  that neighbour, and a \textcolor{red!70!black}{\textbf{miss}} ($\times$) otherwise.}
  \label{fig:recall-at-k}
\end{figure}
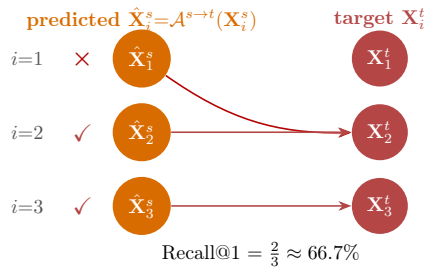

    \item \textbf{Precision@K}: $\frac{1}{N}\sum_i \mathbb{1}\!\left[i \in
    \mathrm{KNN}^{\hat{s}}_K\big(\Xt_i\big)\right]$, querying in the opposite direction,
    from the true target into the mapped space---does the prediction land among the
    representations nearest to the true target?

\paragraph{Precision@K.}

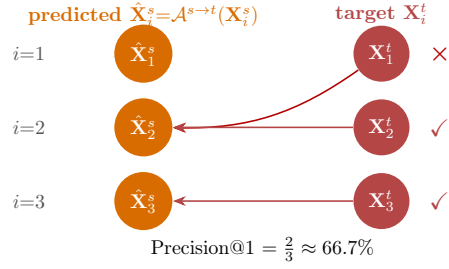
\begin{figure}[H]
  \centering
  \scalebox{0.75}{
  \begin{tikzpicture}[
    pnode/.style={circle, fill=orange!85!black, text=white, inner sep=4pt, font=\small\bfseries},
    tnode/.style={circle, fill=tgtcolor,        text=white, inner sep=4pt, font=\small\bfseries},
    >=Stealth]
    \node[orange!85!black, font=\small\bfseries] at (0.0, 3.3)
      {predicted $\Xshat_i{=}\Ast(\Xs_i)$};
    \node[tgtcolor, font=\small\bfseries] at (4.2, 3.3) {target $\Xt_i$};
    \foreach \i/\y in {1/2.6, 2/1.3, 3/0.0}
      \node[font=\small, gray!65!black, anchor=east] at (-1.6, \y) {$i{=}\i$};
    \node[pnode] (p1) at (0.0, 2.6) {$\Xshat_1$};
    \node[pnode] (p2) at (0.0, 1.3) {$\Xshat_2$};
    \node[pnode] (p3) at (0.0, 0.0) {$\Xshat_3$};
    \node[tnode] (t1) at (4.2, 2.6) {$\Xt_1$};
    \node[tnode] (t2) at (4.2, 1.3) {$\Xt_2$};
    \node[tnode] (t3) at (4.2, 0.0) {$\Xt_3$};
    \draw[->, red!70!black, thick, bend left=18] (t1) to (p2);
    \draw[->, tgtcolor,     thick]               (t2) -- (p2);
    \draw[->, tgtcolor,     thick]               (t3) -- (p3);
    \node[red!70!black, font=\large\bfseries, right=6pt of t1] {\texttimes};
    \node[tgtcolor,     font=\large\bfseries, right=6pt of t2] {\checkmark};
    \node[tgtcolor,     font=\large\bfseries, right=6pt of t3] {\checkmark};
    \node[font=\small] at (2.1, -0.85) {Precision@1 $= \tfrac{2}{3} \approx 66.7\%$};
  \end{tikzpicture}
  }
  \caption{Precision@1 on three samples. Each target representation $\Xt_i$
  (teal) queries its nearest neighbour in the predicted space (orange).
  A sample is a \textcolor{tgtcolor}{\textbf{hit}} (\checkmark) when $\Xshat_i$ is
  that neighbour, and a \textcolor{red!70!black}{\textbf{miss}} ($\times$) otherwise.}
  \label{fig:precision-at-k}
\end{figure}

    \item \textbf{Jaccard@K}: $\frac{1}{N}\sum_i \big|\mathcal{N}^s_K(i) \cap
    \mathcal{N}^{\hat s}_K(i)\big| \,/\, \big|\mathcal{N}^s_K(i) \cup \mathcal{N}^{\hat
    s}_K(i)\big|$, where $\mathcal{N}^s_K(i)$ and $\mathcal{N}^{\hat s}_K(i)$ are the index
    sets of the $K$ nearest neighbors of example $i$ in the source space and of $\hat
    X^s_i$ in the mapped space, respectively---does the local neighborhood structure of the
    source space survive the mapping, penalizing both neighbors that are lost and
    neighbors that are spuriously introduced.

\begin{figure}[H]
  \centering
  \scalebox{0.75}{
  \begin{tikzpicture}[
    snode/.style={circle, fill=srccolor,        text=white, inner sep=4pt, font=\small\bfseries},
    pnode/.style={circle, fill=orange!85!black, text=white, inner sep=4pt, font=\small\bfseries},
    >=Stealth]
    \node[srccolor,        font=\small\bfseries] at (0.0, 3.3) {source $\Xs_i$};
    \node[orange!85!black, font=\small\bfseries] at (4.2, 3.3)
      {predicted $\Xshat_i{=}\Ast(\Xs_i)$};
    \foreach \i/\y in {1/2.6, 2/1.3, 3/0.0}
      \node[font=\small, gray!65!black, anchor=east] at (-1.6, \y) {$i{=}\i$};
    \node[snode] (s1) at (0.0, 2.6) {$\Xs_1$};
    \node[snode] (s2) at (0.0, 1.3) {$\Xs_2$};
    \node[snode] (s3) at (0.0, 0.0) {$\Xs_3$};
    \node[pnode] (p1) at (4.2, 2.6) {$\Xshat_1$};
    \node[pnode] (p2) at (4.2, 1.3) {$\Xshat_2$};
    \node[pnode] (p3) at (4.2, 0.0) {$\Xshat_3$};
    \foreach \i in {1,2,3}
      \draw[->, orange!60!black, dashed, thin] (s\i) -- (p\i);
    \draw[->, srccolor, thick, bend right=50] (s1) to (s3);
    \draw[->, srccolor, thick, bend left=22]  (s2) to (s1);
    \draw[->, srccolor, thick, bend left=22]  (s3) to (s2);
    \draw[->, tgtcolor,     thick, bend right=50] (p1) to (p3);
    \draw[->, red!70!black, thick, bend left=28]  (p2) to (p3);
    \draw[->, tgtcolor,     thick, bend left=22]  (p3) to (p2);
    \node[tgtcolor,     font=\large\bfseries, right=6pt of p1] {\checkmark};
    \node[red!70!black, font=\large\bfseries, right=6pt of p2] {\texttimes};
    \node[tgtcolor,     font=\large\bfseries, right=6pt of p3] {\checkmark};
    \node[font=\scriptsize, gray!60!black, align=left, anchor=west] at (5.3, 1.3)
      {$\cap{=}\emptyset,\;\cup{=}2$};
    \node[font=\small] at (2.1, -0.85) {Jaccard@1 $= \tfrac{2}{3} \approx 66.7\%$};
  \end{tikzpicture}
  }
  \caption{Jaccard@1 on three samples. For each hit (\checkmark), $|\cap|{=}|\cup|{=}1$, giving
  score~1. For the miss (\texttimes), $|\cap|{=}0$ and $|\cup|{=}2$---yielding score~0.
  The union penalty distinguishes Jaccard from Overlap when $K{>}1$.}
  \label{fig:jaccard-at-k}
\end{figure}
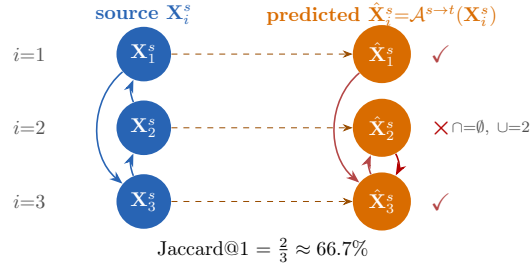
    
    \item \textbf{Token-Jaccard@K}: the same intersection-over-union definition as
    Jaccard@K, but with $\mathcal{N}^s_K(i)$ and $\mathcal{N}^{\hat s}_K(i)$ computed after
    projecting every example through the target model's unembedding matrix into
    next-token-distribution space, rather than compared as raw hidden activations. This is
    the only one of the four metrics that depends on the target model's own read-out
    rather than purely on the geometry of its activation space.
\end{itemize}

\endgroup
%
\section{Extended Experimental Section}
\label{sec:appendix-results}

\paragraph{Hyperparameters}

Our hyperparameters are reported in \Cref{tab:hp-grid-space}, and selected ones based on validation are in \Cref{tab:hp-best-wkffn}, \Cref{tab:hp-best-maacak}, \Cref{tab:hp-best-maacafl},
\Cref{tab:hp-best-maacanoffn}, \Cref{tab:hp-best-fixedffn}, \Cref{tab:hp-best-cka}, \Cref{tab:hp-best-mlp}, \Cref{tab:hp-best-hotnorm}.
\begin{table}[h!]
\centering
\begin{tabular}{@{}lcl@{}}
\toprule
\textbf{Hyperparameter} & \textbf{\#} & \textbf{Values swept} \\
\midrule
\multicolumn{3}{@{}l}{\textbf{MATCHA}} \\
\texttt{learning\_rate}         & 2 & 0.0003, 0.001 \\
\texttt{attention\_temperature} & 4 & 0.5, 1, 4, 8 \\
\texttt{bottleneck\_dim}        & 3 & 256, 512, 1024 \\
\texttt{num\_epochs}            & 3 & 10, 30, 60 \\
\midrule
\multicolumn{3}{@{}l}{\textbf{MATCHA-K}} \\
\texttt{learning\_rate}         & 2 & 0.0003, 0.001 \\
\bottomrule
\end{tabular}
\caption{Hyperparameter grid per method.}\label{tab:hp-grid-space}
\end{table}


\begin{table*}[t]
  \centering
  \small
  \setlength{\tabcolsep}{4pt}
  \renewcommand{\arraystretch}{0.92}
  \begin{tabular}{lrrrrr}
    \toprule
    \textbf{Target} & \textbf{LR} & \textbf{Temp.} & \textbf{Bottleneck} & \textbf{Epochs} & \textbf{Val loss} \\
    \midrule
    \multicolumn{6}{l}{\textbf{Source: Llama-3.2-1B}} \\
    Llama-3.2-3B & 0.0003 & 1 & 512 & 30 & 0.1736 \\
    Llama-3.1-8B & 0.0003 & 1 & 256 & 30 & 0.1450 \\
    Qwen2.5-1.5B & 0.001 & 1 & 1024 & 30 & 0.0391 \\
    Qwen2.5-3B & 0.001 & 0.5 & 256 & 30 & 0.0426 \\
    Qwen2.5-7B & 0.001 & 0.5 & 512 & 60 & 0.0703 \\
    Gemma-2-2B & 0.001 & 1 & 512 & 30 & 0.0930 \\
\bottomrule
  \end{tabular}
  \caption{Best \textbf{MATCHA} hyperparameters per pair (min.\ validation loss).}
  \label{tab:hp-best-wkffn}
\end{table*}


\begin{table*}[t]
  \centering
  \small
  \setlength{\tabcolsep}{4pt}
  \renewcommand{\arraystretch}{0.92}
  \begin{tabular}{lrrrrr}
    \toprule
    \textbf{Target} & \textbf{LR} & \textbf{Temp.} & \textbf{Bottleneck} & \textbf{Epochs} & \textbf{Val loss} \\
    \midrule
    \multicolumn{6}{l}{\textbf{Source: Llama-3.2-1B}} \\
    Llama-3.2-3B & 0.001 & 0.5 & 1024 & 60 & 0.2092 \\
    Llama-3.1-8B & 0.001 & 0.5 & 1024 & 60 & 0.1731 \\
    Qwen2.5-1.5B & 0.001 & 0.5 & 1024 & 60 & 0.0445 \\
    Qwen2.5-3B & 0.001 & 0.5 & 1024 & 60 & 0.0514 \\
    Qwen2.5-7B & 0.001 & 0.5 & 1024 & 60 & 0.0879 \\
    Gemma-2-2B & 0.001 & 0.5 & 1024 & 60 & 0.1058 \\
\bottomrule
  \end{tabular}
  \caption{Best \textbf{MATCHA-K} hyperparameters per pair (min.\ validation loss). MATCHA-K replaces the dynamic key with an input-independent one.}
  \label{tab:hp-best-maacak}
\end{table*}


\begin{table*}[t]
  \centering
  \small
  \setlength{\tabcolsep}{4pt}
  \renewcommand{\arraystretch}{0.92}
  \begin{tabular}{lrrrrr}
    \toprule
    \textbf{Target} & \textbf{LR} & \textbf{Temp.} & \textbf{Bottleneck} & \textbf{Epochs} & \textbf{Val loss} \\
    \midrule
    \multicolumn{6}{l}{\textbf{Source: Llama-3.2-1B}} \\
    Llama-3.2-3B & 0.001 & 0.5 & 512 & 30 & 0.1878 \\
    Llama-3.1-8B & 0.0003 & 0.5 & 512 & 30 & 0.1584 \\
    Qwen2.5-1.5B & 0.001 & 0.5 & 512 & 30 & 0.0431 \\
    Qwen2.5-3B & 0.001 & 0.5 & 512 & 30 & 0.0480 \\
    Qwen2.5-7B & 0.001 & 0.5 & 512 & 30 & 0.0778 \\
    Gemma-2-2B & 0.001 & 0.5 & 512 & 30 & 0.1004 \\
\bottomrule
  \end{tabular}
  \caption{Best \textbf{MATCHA-FL} hyperparameters per pair (min.\ validation loss). MATCHA-FL uses a single fully-learned logit map.}
  \label{tab:hp-best-maacafl}
\end{table*}


\begin{table*}[t]
  \centering
  \small
  \setlength{\tabcolsep}{4pt}
  \renewcommand{\arraystretch}{0.92}
  \begin{tabular}{lrrrrr}
    \toprule
    \textbf{Target} & \textbf{LR} & \textbf{Temp.} & \textbf{Bottleneck} & \textbf{Epochs} & \textbf{Val loss} \\
    \midrule
    \multicolumn{6}{l}{\textbf{Source: Llama-3.2-1B}} \\
    Llama-3.2-3B & 0.001 & 1 & 512 & 10 & 0.2186 \\
    Llama-3.1-8B & 0.001 & 0.5 & 256 & 10 & 0.1975 \\
    Qwen2.5-1.5B & 0.001 & 8 & 256 & 30 & 0.0696 \\
    Qwen2.5-3B & 0.001 & 0.5 & 256 & 30 & 0.0807 \\
    Qwen2.5-7B & 0.001 & 0.5 & 512 & 30 & 0.1073 \\
    Gemma-2-2B & 0.001 & 1 & 256 & 10 & 0.1219 \\
\bottomrule
  \end{tabular}
  \caption{Best \textbf{MATCHA (no FFN)} hyperparameters per pair (min.\ validation loss). MATCHA (no FFN) omits the FFN value path; trained from scratch, not derived from MATCHA's checkpoint.}
  \label{tab:hp-best-maacanoffn}
\end{table*}


\begin{table*}[t]
  \centering
  \small
  \setlength{\tabcolsep}{4pt}
  \renewcommand{\arraystretch}{0.92}
  \begin{tabular}{lrrrrrr}
    \toprule
     & \multicolumn{3}{c}{\textbf{Relative depth}} & \multicolumn{3}{c}{\textbf{Terminal}} \\
    \cmidrule(lr){2-4}\cmidrule(lr){5-7}
    \textbf{Target} & \textbf{LR} & \textbf{Ep.} & \textbf{Val} & \textbf{LR} & \textbf{Ep.} & \textbf{Val} \\
    \midrule
    \multicolumn{7}{l}{\textbf{Source: Llama-3.2-1B}} \\
    Llama-3.2-3B & 0.0003 & 30 & 0.2169 & 0.0003 & 30 & 0.3659 \\
    Llama-3.1-8B & 0.0003 & 30 & 0.2023 & 0.0003 & 30 & 0.3629 \\
    Qwen2.5-1.5B & 0.001 & 30 & 0.0698 & 0.001 & 30 & 0.1403 \\
    Qwen2.5-3B & 0.001 & 30 & 0.0952 & 0.001 & 30 & 0.2019 \\
    Qwen2.5-7B & 0.001 & 30 & 0.1047 & 0.001 & 10 & 0.1832 \\
\bottomrule
  \end{tabular}
  \caption{Best \textbf{Relative depth} and \textbf{Terminal} hyperparameters per pair (min.\ validation loss); both share the LR $\times$ epochs grid.}
  \label{tab:hp-best-fixedffn}
\end{table*}


\begin{table*}[t]
  \centering
  \small
  \setlength{\tabcolsep}{4pt}
  \renewcommand{\arraystretch}{0.92}
  \begin{tabular}{lrrr}
    \toprule
    \textbf{Target} & \textbf{LR} & \textbf{Epochs} & \textbf{Val loss} \\
    \midrule
    \multicolumn{4}{l}{\textbf{Source: Llama-3.2-1B}} \\
    Llama-3.2-3B & 0.0003 & 60 & 0.2725 \\
    Llama-3.1-8B & 0.0003 & 30 & 0.2567 \\
    Qwen2.5-1.5B & 0.001 & 30 & 0.1350 \\
    Qwen2.5-3B & 0.001 & 30 & 0.1276 \\
    Qwen2.5-7B & 0.0003 & 60 & 0.1587 \\
    Gemma-2-2B & 0.0003 & 30 & 0.2173 \\
\bottomrule
  \end{tabular}
  \caption{Best \textbf{CKA} hyperparameters per pair (min.\ validation loss). CKA shares the fixed-map methods' grid (learning rate $\times$ epochs).}
  \label{tab:hp-best-cka}
\end{table*}


\begin{table*}[t]
  \centering
  \scriptsize
  \setlength{\tabcolsep}{4pt}
  \renewcommand{\arraystretch}{0.92}
  \begin{tabular}{lrrr}
    \toprule
    \textbf{Target} & \textbf{LR} & \textbf{Epochs} & \textbf{Val loss} \\
    \midrule
    \multicolumn{4}{l}{\textbf{Source: Llama-3.2-1B}} \\
    Llama-3.2-3B & 0.0003 & 60 & 0.1740 \\
    Llama-3.1-8B & 0.0003 & 60 & 0.1448 \\
    Qwen2.5-1.5B & 0.0003 & 60 & 0.0394 \\
    Qwen2.5-3B & 0.0003 & 60 & 0.0428 \\
    Qwen2.5-7B & 0.0003 & 60 & 0.0720 \\
    Gemma-2-2B & 0.0003 & 60 & 0.0913 \\
    \midrule
    \multicolumn{4}{l}{\textbf{Source: Llama-3.2-3B}} \\
    Llama-3.2-1B & 0.0003 & 60 & 0.1357 \\
    Llama-3.1-8B & 0.0003 & 60 & 0.1571 \\
    Qwen2.5-1.5B & 0.0003 & 60 & 0.0423 \\
    Qwen2.5-3B & 0.0003 & 60 & 0.0465 \\
    Qwen2.5-7B & 0.0003 & 60 & 0.0776 \\
    Gemma-2-2B & 0.0003 & 60 & 0.1002 \\
    \midrule
    \multicolumn{4}{l}{\textbf{Source: Llama-3.1-8B}} \\
    Llama-3.2-1B & 0.0003 & 60 & 0.1415 \\
    Llama-3.2-3B & 0.0003 & 60 & 0.2003 \\
    Qwen2.5-1.5B & 0.0003 & 60 & 0.0440 \\
    Qwen2.5-3B & 0.0003 & 60 & 0.0486 \\
    Qwen2.5-7B & 0.0003 & 60 & 0.0786 \\
    Gemma-2-2B & 0.0003 & 60 & 0.1024 \\
    \midrule
    \multicolumn{4}{l}{\textbf{Source: Qwen2.5-1.5B}} \\
    Llama-3.2-1B & 0.0003 & 60 & 0.1293 \\
    Llama-3.2-3B & 0.0003 & 60 & 0.1892 \\
    Llama-3.1-8B & 0.0003 & 60 & 0.1548 \\
    Qwen2.5-3B & 0.0003 & 60 & 0.0400 \\
    Qwen2.5-7B & 0.0003 & 60 & 0.0696 \\
    Gemma-2-2B & 0.0003 & 60 & 0.0941 \\
    \midrule
    \multicolumn{4}{l}{\textbf{Source: Qwen2.5-3B}} \\
    Llama-3.2-1B & 0.0003 & 60 & 0.1363 \\
    Llama-3.2-3B & 0.0003 & 60 & 0.1952 \\
    Llama-3.1-8B & 0.0003 & 60 & 0.1589 \\
    Qwen2.5-1.5B & 0.0003 & 60 & 0.0386 \\
    Qwen2.5-7B & 0.0003 & 60 & 0.0687 \\
    Gemma-2-2B & 0.0003 & 60 & 0.0974 \\
    \midrule
    \multicolumn{4}{l}{\textbf{Source: Qwen2.5-7B}} \\
    Llama-3.2-1B & 0.0003 & 60 & 0.1416 \\
    Llama-3.2-3B & 0.0003 & 60 & 0.1996 \\
    Llama-3.1-8B & 0.0003 & 60 & 0.1619 \\
    Qwen2.5-1.5B & 0.0003 & 60 & 0.0404 \\
    Qwen2.5-3B & 0.0003 & 60 & 0.0443 \\
    Gemma-2-2B & 0.0003 & 60 & 0.0991 \\
    \midrule
    \multicolumn{4}{l}{\textbf{Source: Gemma-2-2B}} \\
    Llama-3.2-1B & 0.0003 & 60 & 0.1305 \\
    Llama-3.2-3B & 0.0003 & 60 & 0.1866 \\
    Llama-3.1-8B & 0.0003 & 60 & 0.1519 \\
    Qwen2.5-1.5B & 0.0003 & 60 & 0.0406 \\
    Qwen2.5-3B & 0.0003 & 60 & 0.0436 \\
    Qwen2.5-7B & 0.0003 & 60 & 0.0734 \\
    \bottomrule
  \end{tabular}
  \caption{Best \textbf{MLP} hyperparameters per pair (min.\ validation loss). Learning rate and number of epochs are searched (LR $\in\{0.0003, 0.001\}$, epochs $\in\{10,30,60\}$); layers and hidden width are fixed (2, 10{,}000).}
  \label{tab:hp-best-mlp}
\end{table*}


\begin{table*}[t]
  \centering
  \small
  \setlength{\tabcolsep}{4pt}
  \renewcommand{\arraystretch}{0.92}
  \begin{tabular}{lrr}
    \toprule
    \textbf{Target} & \textbf{$k$ (\texttt{routing\_top\_k})} & \textbf{Val loss} \\
    \midrule
    \multicolumn{3}{l}{\textbf{Source: Llama-3.2-1B}} \\
    Llama-3.2-3B & 3 & 0.1384 \\
    Llama-3.1-8B & 3 & 0.1721 \\
    Qwen2.5-1.5B & 10 & 5.4011 \\
    Qwen2.5-3B & 10 & 6.6036 \\
    Qwen2.5-7B & 5 & 5.3680 \\
    Gemma-2-2B & 5 & 33.8884 \\
    \midrule
    \multicolumn{3}{l}{\textbf{Source: Llama-3.2-3B}} \\
    Llama-3.2-1B & 3 & 0.2571 \\
    Llama-3.1-8B & 3 & 0.1617 \\
    Qwen2.5-1.5B & 3 & 5.4139 \\
    Qwen2.5-3B & 5 & 6.5828 \\
    Qwen2.5-7B & 10 & 5.3624 \\
    Gemma-2-2B & 10 & 33.8854 \\
    \midrule
    \multicolumn{3}{l}{\textbf{Source: Llama-3.1-8B}} \\
    Llama-3.2-1B & 3 & 0.2566 \\
    Llama-3.2-3B & 3 & 0.1274 \\
    Qwen2.5-1.5B & 3 & 5.3795 \\
    Qwen2.5-3B & 3 & 6.4682 \\
    Qwen2.5-7B & 10 & 5.2887 \\
    Gemma-2-2B & 3 & 33.8018 \\
    \midrule
    \multicolumn{3}{l}{\textbf{Source: Qwen2.5-1.5B}} \\
    Llama-3.2-1B & 10 & 0.2957 \\
    Llama-3.2-3B & 10 & 0.1575 \\
    Llama-3.1-8B & 10 & 0.1959 \\
    Qwen2.5-3B & 3 & 5.8200 \\
    Qwen2.5-7B & 10 & 5.2962 \\
    Gemma-2-2B & 10 & 35.1998 \\
    \midrule
    \multicolumn{3}{l}{\textbf{Source: Qwen2.5-3B}} \\
    Llama-3.2-1B & 10 & 0.2935 \\
    Llama-3.2-3B & 10 & 0.1564 \\
    Llama-3.1-8B & 10 & 0.1940 \\
    Qwen2.5-1.5B & 3 & 4.6408 \\
    Qwen2.5-7B & 10 & 5.1453 \\
    Gemma-2-2B & 10 & 35.1838 \\
    \midrule
    \multicolumn{3}{l}{\textbf{Source: Qwen2.5-7B}} \\
    Llama-3.2-1B & 3 & 0.2792 \\
    Llama-3.2-3B & 3 & 0.1498 \\
    Llama-3.1-8B & 3 & 0.1851 \\
    Qwen2.5-1.5B & 3 & 5.0661 \\
    Qwen2.5-3B & 3 & 6.0459 \\
    Gemma-2-2B & 10 & 33.8265 \\
    \midrule
    \multicolumn{3}{l}{\textbf{Source: Gemma-2-2B}} \\
    Llama-3.2-1B & 5 & 0.2811 \\
    Llama-3.2-3B & 5 & 0.1502 \\
    Llama-3.1-8B & 5 & 0.1860 \\
    Qwen2.5-1.5B & 5 & 5.3062 \\
    Qwen2.5-3B & 5 & 6.5408 \\
    Qwen2.5-7B & 5 & 5.3828 \\
    \bottomrule
  \end{tabular}
  \caption{Best \textbf{HOT} routing-$k$ per pair (min.\ validation loss); $k$ caps how many donor layers each receiver layer routes from, swept over $\{3,5,10\}$.}
  \label{tab:hp-best-hotnorm}
\end{table*}

\paragraph{Additional training details.}
\label{sec:appendix-training}
We use the AdamW optimizer \cite{kingma2014adam}
(weight decay $0$) under a cosine schedule with $10\%$
linear warm-up, with early stopping (patience
$5$). 
selected per pair by the sweeps above.

\paragraph{The hidden state we use.} Every alignment map is fit on each model's last-token residual-stream state at every layer, extracted identically for donor and receiver
(chat-templated prompt, one forward pass, frozen model).

\subsection{Models Description}
\label{sec:appendix-models}

\paragraph{Models.} \Cref{tab:roster} lists the seven instruction-tuned LLMs; the groups in \Cref{tab:main-metrics} are as follows.



\begin{table}[t]\centering
\setlength{\tabcolsep}{4pt}
\begin{tabular}{@{}llrr@{}}
\toprule
\textbf{Model} & \textbf{Family} & \textbf{Layers} & \textbf{Hidden} \\
\midrule
Llama-3.2-1B & Llama & 16 & 2048 \\
Gemma-2-2B   & Gemma & 26 & 2304 \\
Qwen2.5-1.5B & Qwen  & 28 & 1536 \\
Llama-3.2-3B & Llama & 28 & 3072 \\
Qwen2.5-7B   & Qwen  & 28 & 3584 \\
Llama-3.1-8B & Llama & 32 & 4096 \\
Qwen2.5-3B   & Qwen  & 36 & 2048 \\
\bottomrule
\end{tabular}
\caption{The seven instruction-tuned LLMs used, with their features.}
\label{tab:roster}
\end{table}

\begingroup\small
\begin{itemize}\setlength{\itemsep}{4pt}
\item \textbf{Bigger$\to$Smaller (deeper donor)} --- 18 directional pairs: Qwen2.5-1.5B (28) $\to$\allowbreak\ Llama-3.2-1B (16), Qwen2.5-1.5B (28) $\to$\allowbreak\ Gemma-2-2B (26), Gemma-2-2B (26) $\to$\allowbreak\ Llama-3.2-1B (16), Qwen2.5-3B (36) $\to$\allowbreak\ Llama-3.2-1B (16), Qwen2.5-3B (36) $\to$\allowbreak\ Qwen2.5-1.5B (28), Qwen2.5-3B (36) $\to$\allowbreak\ Gemma-2-2B (26), Qwen2.5-3B (36) $\to$\allowbreak\ Llama-3.2-3B (28), Qwen2.5-3B (36) $\to$\allowbreak\ Qwen2.5-7B (28), Qwen2.5-3B (36) $\to$\allowbreak\ Llama-3.1-8B (32), Llama-3.2-3B (28) $\to$\allowbreak\ Llama-3.2-1B (16), Llama-3.2-3B (28) $\to$\allowbreak\ Gemma-2-2B (26), Qwen2.5-7B (28) $\to$\allowbreak\ Llama-3.2-1B (16), Qwen2.5-7B (28) $\to$\allowbreak\ Gemma-2-2B (26), Llama-3.1-8B (32) $\to$\allowbreak\ Llama-3.2-1B (16), Llama-3.1-8B (32) $\to$\allowbreak\ Qwen2.5-1.5B (28), Llama-3.1-8B (32) $\to$\allowbreak\ Gemma-2-2B (26), Llama-3.1-8B (32) $\to$\allowbreak\ Llama-3.2-3B (28), Llama-3.1-8B (32) $\to$\allowbreak\ Qwen2.5-7B (28).
\item \textbf{Smaller$\to$Bigger (shallower donor)} --- 18 directional pairs: Llama-3.2-1B (16) $\to$\allowbreak\ Qwen2.5-1.5B (28), Llama-3.2-1B (16) $\to$\allowbreak\ Gemma-2-2B (26), Llama-3.2-1B (16) $\to$\allowbreak\ Qwen2.5-3B (36), Llama-3.2-1B (16) $\to$\allowbreak\ Llama-3.2-3B (28), Llama-3.2-1B (16) $\to$\allowbreak\ Qwen2.5-7B (28), Llama-3.2-1B (16) $\to$\allowbreak\ Llama-3.1-8B (32), Qwen2.5-1.5B (28) $\to$\allowbreak\ Qwen2.5-3B (36), Qwen2.5-1.5B (28) $\to$\allowbreak\ Llama-3.1-8B (32), Gemma-2-2B (26) $\to$\allowbreak\ Qwen2.5-1.5B (28), Gemma-2-2B (26) $\to$\allowbreak\ Qwen2.5-3B (36), Gemma-2-2B (26) $\to$\allowbreak\ Llama-3.2-3B (28), Gemma-2-2B (26) $\to$\allowbreak\ Qwen2.5-7B (28), Gemma-2-2B (26) $\to$\allowbreak\ Llama-3.1-8B (32), Llama-3.2-3B (28) $\to$\allowbreak\ Qwen2.5-3B (36), Llama-3.2-3B (28) $\to$\allowbreak\ Llama-3.1-8B (32), Qwen2.5-7B (28) $\to$\allowbreak\ Qwen2.5-3B (36), Qwen2.5-7B (28) $\to$\allowbreak\ Llama-3.1-8B (32), Llama-3.1-8B (32) $\to$\allowbreak\ Qwen2.5-3B (36).
\item \textbf{Same-size (equal depth)} --- 6 directional pairs: Qwen2.5-1.5B (28) $\to$\allowbreak\ Llama-3.2-3B (28), Qwen2.5-1.5B (28) $\to$\allowbreak\ Qwen2.5-7B (28), Llama-3.2-3B (28) $\to$\allowbreak\ Qwen2.5-1.5B (28), Llama-3.2-3B (28) $\to$\allowbreak\ Qwen2.5-7B (28), Qwen2.5-7B (28) $\to$\allowbreak\ Qwen2.5-1.5B (28), Qwen2.5-7B (28) $\to$\allowbreak\ Llama-3.2-3B (28).
\end{itemize}
\endgroup

\subsection{Dataset Description}
\label{sec:appendix-datasets}

\paragraph{Datasets.}
We use six data sources, with separate roles in alignment training, probing,
steering-vector construction, and final steering evaluation:
\begin{itemize}\setlength{\itemsep}{3pt}

\item \textbf{LIMA}~\citep{zhou2023lima} is a curated instruction-following
corpus of $1{,}000$ diverse prompt--response demonstrations, covering
factual questions, advice, planning, creative writing, and other
instruction types. We use only the prompts: the alignment
train, validation, and test activations are extracted from disjoint subsets
of LIMA prompts. No task labels or LIMA responses are used to train the
alignment maps.

\item \textbf{TriviaQA}~\citep{joshi-etal-2017-triviaqa} is a large-scale factual
question-answering dataset containing trivia questions, reference answers,
and supporting evidence documents. We use the questions and reference
answers to construct the probing task: each model-generated response is
assigned a binary correctness label, which is used to train and evaluate
the layerwise correctness probes.

\item \textbf{AdvBench}~\citep{zou2023universal} is a collection of harmful
behaviors formulated as instructions, designed to test whether a
safety-aligned model attempts to comply with a prohibited request. We use
AdvBench only as a held-out benchmark for evaluating refusal-suppression
steering.

\item \textbf{HarmBench}~\citep{mazeika2024harmbench} is a standardized
red-teaming benchmark covering a broad range of harmful behaviors and
misuse categories. We use its behavior prompts as a second held-out
evaluation set for measuring the effect of transferred steering
interventions.

\item \textbf{JailbreakBench}~\citep{chao2024jailbreakbench} contains $100$
unique and realizable text-based misuse behaviors, annotated by misuse
category and source. We use these behaviors as the third held-out steering
benchmark.

\item \textbf{Refusal-direction prompts}~\citep{arditi2024refusal} are the
harmful and harmless prompt splits released with the
\texttt{refusal direction} implementation. For each model, we use the
training prompts to construct a contrastive refusal direction at every
layer and the validation prompts to select the intervention layer and
coefficient. These prompts are used only for steering-vector construction
and validation and are kept disjoint from the AdvBench, HarmBench, and
JailbreakBench prompts used for final evaluation.

\end{itemize}

\subsection{Additional Results}
\label{sec:appendix-additional}
In the following, we report the full breakdown of the results and include two additional variants of \ourmethod as further ablations.

\paragraph{\ourmethod without \texttt{FF}.} This variant removes the \texttt{FF} component from \ourmethod.

\paragraph{\ourmethod with a fixed average map.} We compute an average map using a held-out validation set and keep it fixed during evaluation.

\subsubsection{Retrival-Based Metrics per Model Pair}
\label{sec:appendix-metrics-raw}
In \Cref{tab:pp-recall1}, \Cref{tab:pp-jaccard5}, \Cref{tab:pp-precision1}, \Cref{tab:pp-tj5}, we report the pairwise breakdown of the results.


\onecolumn
\begingroup\fontsize{4.6pt}{5.6pt}\selectfont\setlength{\tabcolsep}{0.6pt}\renewcommand{\arraystretch}{1.1}
\setlength\LTleft{0pt}\setlength\LTright{0pt}\begin{longtable}{@{\extracolsep{\fill}}lccccccccccc}
\caption{Token-Jaccard@5 ($\uparrow$): full per-pair results, seed mean\,$\pm$\,std. Best \textbf{bold}, second \underline{underlined}; grouped by source (donor) LLM.}\label{tab:pp-tj5}\\
\toprule
Target & \textbf{Term} & \textbf{Rel-dep} & \textbf{CKA} & \textbf{MATCHA} & \textbf{MATCHA-K} & \textbf{MATCHA-FL} & \textbf{MATCHA-avg} & \textbf{MATCHA-no FF} & \textbf{CmdV} & \textbf{HOT} & \textbf{MLP} \\
\midrule\endfirsthead
\multicolumn{12}{c}{\tablename\ \thetable{} -- continued}\\ \toprule
Target & \textbf{Term} & \textbf{Rel-dep} & \textbf{CKA} & \textbf{MATCHA} & \textbf{MATCHA-K} & \textbf{MATCHA-FL} & \textbf{MATCHA-avg} & \textbf{MATCHA-no FF} & \textbf{CmdV} & \textbf{HOT} & \textbf{MLP} \\
\midrule\endhead
\bottomrule\endfoot
\multicolumn{12}{l}{\textbf{Source: Llama-3.2-1B}} \\
Llama-3.2-3B & 0.453\,{$\scriptscriptstyle\pm0.008$} & 0.455\,{$\scriptscriptstyle\pm0.002$} & 0.327\,{$\scriptscriptstyle\pm0.002$} & \underline{0.455}\,{$\scriptscriptstyle\pm0.002$} & 0.427\,{$\scriptscriptstyle\pm0.005$} & 0.411\,{$\scriptscriptstyle\pm0.006$} & \textbf{0.457}\,{$\scriptscriptstyle\pm0.003$} & 0.433\,{$\scriptscriptstyle\pm0.001$} & 0.431 & 0.188 & 0.418\,{$\scriptscriptstyle\pm0.003$} \\
Llama-3.1-8B & \underline{0.461}\,{$\scriptscriptstyle\pm0.003$} & 0.442\,{$\scriptscriptstyle\pm0.008$} & 0.259\,{$\scriptscriptstyle\pm0.004$} & 0.461\,{$\scriptscriptstyle\pm0.002$} & 0.413\,{$\scriptscriptstyle\pm0.007$} & 0.410\,{$\scriptscriptstyle\pm0.012$} & \textbf{0.464}\,{$\scriptscriptstyle\pm0.002$} & 0.413\,{$\scriptscriptstyle\pm0.005$} & 0.393 & 0.182 & 0.400\,{$\scriptscriptstyle\pm0.001$} \\
Qwen2.5-1.5B & 0.474\,{$\scriptscriptstyle\pm0.008$} & 0.468\,{$\scriptscriptstyle\pm0.004$} & 0.378\,{$\scriptscriptstyle\pm0.003$} & \underline{0.493}\,{$\scriptscriptstyle\pm0.009$} & 0.490\,{$\scriptscriptstyle\pm0.004$} & 0.484\,{$\scriptscriptstyle\pm0.003$} & \textbf{0.495}\,{$\scriptscriptstyle\pm0.002$} & 0.443\,{$\scriptscriptstyle\pm0.005$} & 0.315 & 0.020 & 0.466\,{$\scriptscriptstyle\pm0.009$} \\
Qwen2.5-3B & 0.431\,{$\scriptscriptstyle\pm0.007$} & 0.435\,{$\scriptscriptstyle\pm0.001$} & 0.339\,{$\scriptscriptstyle\pm0.006$} & \textbf{0.455}\,{$\scriptscriptstyle\pm0.002$} & 0.428\,{$\scriptscriptstyle\pm0.002$} & 0.431\,{$\scriptscriptstyle\pm0.007$} & \underline{0.455}\,{$\scriptscriptstyle\pm0.002$} & 0.421\,{$\scriptscriptstyle\pm0.009$} & 0.317 & 0.052 & 0.402\,{$\scriptscriptstyle\pm0.001$} \\
Qwen2.5-7B & \textbf{0.436}\,{$\scriptscriptstyle\pm0.003$} & \underline{0.435}\,{$\scriptscriptstyle\pm0.004$} & 0.322\,{$\scriptscriptstyle\pm0.005$} & 0.428\,{$\scriptscriptstyle\pm0.004$} & 0.422\,{$\scriptscriptstyle\pm0.004$} & 0.425\,{$\scriptscriptstyle\pm0.007$} & 0.434\,{$\scriptscriptstyle\pm0.003$} & 0.412\,{$\scriptscriptstyle\pm0.008$} & 0.297 & 0.104 & 0.416\,{$\scriptscriptstyle\pm0.008$} \\
gemma-2-2b & 0.280\,{$\scriptscriptstyle\pm0.006$} & 0.286\,{$\scriptscriptstyle\pm0.002$} & 0.101\,{$\scriptscriptstyle\pm0.004$} & \textbf{0.311}\,{$\scriptscriptstyle\pm0.013$} & 0.271\,{$\scriptscriptstyle\pm0.005$} & \underline{0.294}\,{$\scriptscriptstyle\pm0.024$} & 0.147\,{$\scriptscriptstyle\pm0.148$} & 0.227\,{$\scriptscriptstyle\pm0.005$} & 0.016 & 0.028 & 0.292\,{$\scriptscriptstyle\pm0.005$} \\
\midrule
\multicolumn{12}{l}{\textbf{Source: Llama-3.2-3B}} \\
Llama-3.2-1B & 0.485\,{$\scriptscriptstyle\pm0.003$} & 0.491\,{$\scriptscriptstyle\pm0.004$} & 0.332\,{$\scriptscriptstyle\pm0.003$} & \underline{0.493}\,{$\scriptscriptstyle\pm0.002$} & 0.468\,{$\scriptscriptstyle\pm0.002$} & 0.450\,{$\scriptscriptstyle\pm0.006$} & \textbf{0.504}\,{$\scriptscriptstyle\pm0.006$} & 0.493\,{$\scriptscriptstyle\pm0.005$} & 0.489 & 0.152 & 0.392\,{$\scriptscriptstyle\pm0.009$} \\
Llama-3.1-8B & 0.525\,{$\scriptscriptstyle\pm0.001$} & \textbf{0.529}\,{$\scriptscriptstyle\pm0.005$} & 0.205\,{$\scriptscriptstyle\pm0.009$} & 0.523\,{$\scriptscriptstyle\pm0.004$} & 0.471\,{$\scriptscriptstyle\pm0.011$} & 0.451\,{$\scriptscriptstyle\pm0.005$} & 0.528\,{$\scriptscriptstyle\pm0.002$} & 0.501\,{$\scriptscriptstyle\pm0.003$} & \underline{0.528} & 0.229 & 0.383\,{$\scriptscriptstyle\pm0.002$} \\
Qwen2.5-1.5B & 0.459\,{$\scriptscriptstyle\pm0.001$} & 0.456\,{$\scriptscriptstyle\pm0.006$} & 0.355\,{$\scriptscriptstyle\pm0.003$} & \underline{0.482}\,{$\scriptscriptstyle\pm0.008$} & 0.482\,{$\scriptscriptstyle\pm0.007$} & 0.463\,{$\scriptscriptstyle\pm0.005$} & \textbf{0.485}\,{$\scriptscriptstyle\pm0.003$} & 0.436\,{$\scriptscriptstyle\pm0.004$} & 0.333 & 0.027 & 0.450\,{$\scriptscriptstyle\pm0.007$} \\
Qwen2.5-3B & 0.442\,{$\scriptscriptstyle\pm0.006$} & 0.441\,{$\scriptscriptstyle\pm0.003$} & 0.324\,{$\scriptscriptstyle\pm0.006$} & \textbf{0.471}\,{$\scriptscriptstyle\pm0.003$} & 0.435\,{$\scriptscriptstyle\pm0.005$} & 0.433\,{$\scriptscriptstyle\pm0.003$} & \underline{0.467}\,{$\scriptscriptstyle\pm0.001$} & 0.433\,{$\scriptscriptstyle\pm0.008$} & 0.374 & 0.072 & 0.380\,{$\scriptscriptstyle\pm0.004$} \\
Qwen2.5-7B & \underline{0.446}\,{$\scriptscriptstyle\pm0.004$} & 0.444\,{$\scriptscriptstyle\pm0.002$} & 0.298\,{$\scriptscriptstyle\pm0.016$} & \textbf{0.449}\,{$\scriptscriptstyle\pm0.004$} & 0.441\,{$\scriptscriptstyle\pm0.002$} & 0.431\,{$\scriptscriptstyle\pm0.010$} & 0.393\,{$\scriptscriptstyle\pm0.092$} & 0.422\,{$\scriptscriptstyle\pm0.011$} & 0.409 & 0.103 & 0.389\,{$\scriptscriptstyle\pm0.005$} \\
gemma-2-2b & 0.276\,{$\scriptscriptstyle\pm0.004$} & 0.280\,{$\scriptscriptstyle\pm0.012$} & 0.099\,{$\scriptscriptstyle\pm0.003$} & \textbf{0.315}\,{$\scriptscriptstyle\pm0.007$} & 0.263\,{$\scriptscriptstyle\pm0.004$} & 0.279\,{$\scriptscriptstyle\pm0.023$} & 0.232\,{$\scriptscriptstyle\pm0.142$} & 0.224\,{$\scriptscriptstyle\pm0.015$} & 0.021 & 0.035 & \underline{0.282}\,{$\scriptscriptstyle\pm0.009$} \\
\midrule
\multicolumn{12}{l}{\textbf{Source: Llama-3.1-8B}} \\
Llama-3.2-1B & 0.481\,{$\scriptscriptstyle\pm0.003$} & 0.481\,{$\scriptscriptstyle\pm0.001$} & 0.361\,{$\scriptscriptstyle\pm0.002$} & \underline{0.494}\,{$\scriptscriptstyle\pm0.003$} & 0.461\,{$\scriptscriptstyle\pm0.006$} & 0.442\,{$\scriptscriptstyle\pm0.006$} & \textbf{0.494}\,{$\scriptscriptstyle\pm0.002$} & 0.491\,{$\scriptscriptstyle\pm0.004$} & 0.476 & 0.145 & 0.376\,{$\scriptscriptstyle\pm0.007$} \\
Llama-3.2-3B & 0.520\,{$\scriptscriptstyle\pm0.004$} & \underline{0.523}\,{$\scriptscriptstyle\pm0.001$} & 0.355\,{$\scriptscriptstyle\pm0.002$} & 0.518\,{$\scriptscriptstyle\pm0.003$} & 0.488\,{$\scriptscriptstyle\pm0.002$} & 0.455\,{$\scriptscriptstyle\pm0.004$} & 0.522\,{$\scriptscriptstyle\pm0.004$} & 0.501\,{$\scriptscriptstyle\pm0.003$} & \textbf{0.531} & 0.216 & 0.371\,{$\scriptscriptstyle\pm0.006$} \\
Qwen2.5-1.5B & 0.458\,{$\scriptscriptstyle\pm0.002$} & 0.471\,{$\scriptscriptstyle\pm0.004$} & 0.374\,{$\scriptscriptstyle\pm0.005$} & \underline{0.488}\,{$\scriptscriptstyle\pm0.001$} & 0.473\,{$\scriptscriptstyle\pm0.014$} & \textbf{0.490}\,{$\scriptscriptstyle\pm0.008$} & 0.487\,{$\scriptscriptstyle\pm0.003$} & 0.459\,{$\scriptscriptstyle\pm0.008$} & 0.365 & 0.025 & 0.435\,{$\scriptscriptstyle\pm0.005$} \\
Qwen2.5-3B & 0.445\,{$\scriptscriptstyle\pm0.003$} & 0.446\,{$\scriptscriptstyle\pm0.005$} & 0.344\,{$\scriptscriptstyle\pm0.007$} & \textbf{0.477}\,{$\scriptscriptstyle\pm0.001$} & 0.441\,{$\scriptscriptstyle\pm0.004$} & 0.439\,{$\scriptscriptstyle\pm0.009$} & \underline{0.472}\,{$\scriptscriptstyle\pm0.007$} & 0.450\,{$\scriptscriptstyle\pm0.002$} & 0.389 & 0.078 & 0.360\,{$\scriptscriptstyle\pm0.006$} \\
Qwen2.5-7B & 0.465\,{$\scriptscriptstyle\pm0.001$} & 0.466\,{$\scriptscriptstyle\pm0.004$} & 0.283\,{$\scriptscriptstyle\pm0.015$} & \textbf{0.478}\,{$\scriptscriptstyle\pm0.004$} & 0.449\,{$\scriptscriptstyle\pm0.002$} & 0.449\,{$\scriptscriptstyle\pm0.007$} & 0.380\,{$\scriptscriptstyle\pm0.177$} & \underline{0.467}\,{$\scriptscriptstyle\pm0.002$} & 0.423 & 0.102 & 0.389\,{$\scriptscriptstyle\pm0.003$} \\
gemma-2-2b & 0.285\,{$\scriptscriptstyle\pm0.009$} & 0.286\,{$\scriptscriptstyle\pm0.007$} & 0.146\,{$\scriptscriptstyle\pm0.002$} & \textbf{0.323}\,{$\scriptscriptstyle\pm0.009$} & 0.264\,{$\scriptscriptstyle\pm0.006$} & \underline{0.287}\,{$\scriptscriptstyle\pm0.009$} & 0.160\,{$\scriptscriptstyle\pm0.143$} & 0.246\,{$\scriptscriptstyle\pm0.010$} & 0.021 & 0.025 & 0.267\,{$\scriptscriptstyle\pm0.007$} \\
\midrule
\multicolumn{12}{l}{\textbf{Source: Qwen2.5-1.5B}} \\
Llama-3.2-1B & \underline{0.418}\,{$\scriptscriptstyle\pm0.004$} & 0.414\,{$\scriptscriptstyle\pm0.002$} & 0.374\,{$\scriptscriptstyle\pm0.006$} & 0.414\,{$\scriptscriptstyle\pm0.009$} & \textbf{0.420}\,{$\scriptscriptstyle\pm0.006$} & 0.407\,{$\scriptscriptstyle\pm0.011$} & 0.404\,{$\scriptscriptstyle\pm0.024$} & 0.402\,{$\scriptscriptstyle\pm0.006$} & 0.316 & 0.104 & 0.397\,{$\scriptscriptstyle\pm0.003$} \\
Llama-3.2-3B & \textbf{0.393}\,{$\scriptscriptstyle\pm0.007$} & \underline{0.392}\,{$\scriptscriptstyle\pm0.007$} & 0.333\,{$\scriptscriptstyle\pm0.006$} & 0.387\,{$\scriptscriptstyle\pm0.004$} & 0.379\,{$\scriptscriptstyle\pm0.004$} & 0.376\,{$\scriptscriptstyle\pm0.001$} & 0.331\,{$\scriptscriptstyle\pm0.054$} & 0.370\,{$\scriptscriptstyle\pm0.005$} & 0.199 & 0.122 & 0.375\,{$\scriptscriptstyle\pm0.001$} \\
Llama-3.1-8B & 0.388\,{$\scriptscriptstyle\pm0.001$} & 0.387\,{$\scriptscriptstyle\pm0.003$} & 0.255\,{$\scriptscriptstyle\pm0.021$} & \underline{0.388}\,{$\scriptscriptstyle\pm0.003$} & 0.360\,{$\scriptscriptstyle\pm0.004$} & 0.364\,{$\scriptscriptstyle\pm0.007$} & \textbf{0.391}\,{$\scriptscriptstyle\pm0.003$} & 0.362\,{$\scriptscriptstyle\pm0.003$} & 0.217 & 0.102 & 0.372\,{$\scriptscriptstyle\pm0.005$} \\
Qwen2.5-3B & 0.470\,{$\scriptscriptstyle\pm0.000$} & 0.463\,{$\scriptscriptstyle\pm0.002$} & 0.410\,{$\scriptscriptstyle\pm0.004$} & \underline{0.472}\,{$\scriptscriptstyle\pm0.004$} & 0.458\,{$\scriptscriptstyle\pm0.008$} & 0.446\,{$\scriptscriptstyle\pm0.003$} & \textbf{0.480}\,{$\scriptscriptstyle\pm0.004$} & 0.413\,{$\scriptscriptstyle\pm0.006$} & 0.229 & 0.070 & 0.417\,{$\scriptscriptstyle\pm0.003$} \\
Qwen2.5-7B & \textbf{0.441}\,{$\scriptscriptstyle\pm0.001$} & 0.438\,{$\scriptscriptstyle\pm0.004$} & 0.381\,{$\scriptscriptstyle\pm0.011$} & \underline{0.439}\,{$\scriptscriptstyle\pm0.004$} & 0.431\,{$\scriptscriptstyle\pm0.003$} & 0.425\,{$\scriptscriptstyle\pm0.006$} & 0.437\,{$\scriptscriptstyle\pm0.004$} & 0.427\,{$\scriptscriptstyle\pm0.003$} & 0.227 & 0.093 & 0.413\,{$\scriptscriptstyle\pm0.003$} \\
gemma-2-2b & 0.273\,{$\scriptscriptstyle\pm0.008$} & 0.272\,{$\scriptscriptstyle\pm0.006$} & 0.135\,{$\scriptscriptstyle\pm0.003$} & \textbf{0.296}\,{$\scriptscriptstyle\pm0.008$} & 0.278\,{$\scriptscriptstyle\pm0.003$} & \underline{0.291}\,{$\scriptscriptstyle\pm0.008$} & 0.210\,{$\scriptscriptstyle\pm0.151$} & 0.207\,{$\scriptscriptstyle\pm0.018$} & 0.020 & 0.037 & 0.282\,{$\scriptscriptstyle\pm0.003$} \\
\midrule
\multicolumn{12}{l}{\textbf{Source: Qwen2.5-3B}} \\
Llama-3.2-1B & 0.424\,{$\scriptscriptstyle\pm0.001$} & \textbf{0.434}\,{$\scriptscriptstyle\pm0.001$} & 0.316\,{$\scriptscriptstyle\pm0.004$} & \underline{0.430}\,{$\scriptscriptstyle\pm0.001$} & 0.425\,{$\scriptscriptstyle\pm0.002$} & 0.412\,{$\scriptscriptstyle\pm0.002$} & 0.420\,{$\scriptscriptstyle\pm0.023$} & 0.417\,{$\scriptscriptstyle\pm0.006$} & 0.363 & 0.108 & 0.386\,{$\scriptscriptstyle\pm0.005$} \\
Llama-3.2-3B & \underline{0.414}\,{$\scriptscriptstyle\pm0.003$} & 0.408\,{$\scriptscriptstyle\pm0.004$} & 0.339\,{$\scriptscriptstyle\pm0.004$} & \textbf{0.415}\,{$\scriptscriptstyle\pm0.008$} & 0.391\,{$\scriptscriptstyle\pm0.007$} & 0.391\,{$\scriptscriptstyle\pm0.010$} & 0.413\,{$\scriptscriptstyle\pm0.005$} & 0.393\,{$\scriptscriptstyle\pm0.002$} & 0.350 & 0.146 & 0.376\,{$\scriptscriptstyle\pm0.001$} \\
Llama-3.1-8B & 0.415\,{$\scriptscriptstyle\pm0.005$} & 0.404\,{$\scriptscriptstyle\pm0.002$} & 0.229\,{$\scriptscriptstyle\pm0.023$} & \underline{0.416}\,{$\scriptscriptstyle\pm0.003$} & 0.381\,{$\scriptscriptstyle\pm0.004$} & 0.374\,{$\scriptscriptstyle\pm0.002$} & \textbf{0.417}\,{$\scriptscriptstyle\pm0.001$} & 0.381\,{$\scriptscriptstyle\pm0.003$} & 0.351 & 0.137 & 0.375\,{$\scriptscriptstyle\pm0.007$} \\
Qwen2.5-1.5B & 0.511\,{$\scriptscriptstyle\pm0.003$} & 0.526\,{$\scriptscriptstyle\pm0.001$} & 0.491\,{$\scriptscriptstyle\pm0.005$} & \textbf{0.529}\,{$\scriptscriptstyle\pm0.004$} & \underline{0.528}\,{$\scriptscriptstyle\pm0.009$} & 0.518\,{$\scriptscriptstyle\pm0.011$} & 0.516\,{$\scriptscriptstyle\pm0.025$} & 0.435\,{$\scriptscriptstyle\pm0.029$} & 0.363 & 0.050 & 0.470\,{$\scriptscriptstyle\pm0.005$} \\
Qwen2.5-7B & \underline{0.481}\,{$\scriptscriptstyle\pm0.004$} & 0.479\,{$\scriptscriptstyle\pm0.002$} & 0.322\,{$\scriptscriptstyle\pm0.004$} & \textbf{0.483}\,{$\scriptscriptstyle\pm0.004$} & 0.464\,{$\scriptscriptstyle\pm0.001$} & 0.459\,{$\scriptscriptstyle\pm0.008$} & 0.402\,{$\scriptscriptstyle\pm0.136$} & 0.456\,{$\scriptscriptstyle\pm0.013$} & 0.385 & 0.130 & 0.412\,{$\scriptscriptstyle\pm0.009$} \\
gemma-2-2b & 0.288\,{$\scriptscriptstyle\pm0.006$} & \underline{0.302}\,{$\scriptscriptstyle\pm0.002$} & 0.145\,{$\scriptscriptstyle\pm0.003$} & \textbf{0.314}\,{$\scriptscriptstyle\pm0.006$} & 0.282\,{$\scriptscriptstyle\pm0.006$} & 0.294\,{$\scriptscriptstyle\pm0.010$} & 0.058\,{$\scriptscriptstyle\pm0.015$} & 0.193\,{$\scriptscriptstyle\pm0.009$} & 0.018 & 0.044 & 0.282\,{$\scriptscriptstyle\pm0.006$} \\
\midrule
\multicolumn{12}{l}{\textbf{Source: Qwen2.5-7B}} \\
Llama-3.2-1B & 0.446\,{$\scriptscriptstyle\pm0.005$} & 0.441\,{$\scriptscriptstyle\pm0.006$} & 0.319\,{$\scriptscriptstyle\pm0.020$} & \underline{0.446}\,{$\scriptscriptstyle\pm0.010$} & 0.428\,{$\scriptscriptstyle\pm0.005$} & 0.421\,{$\scriptscriptstyle\pm0.002$} & \textbf{0.449}\,{$\scriptscriptstyle\pm0.003$} & 0.434\,{$\scriptscriptstyle\pm0.001$} & 0.408 & 0.102 & 0.381\,{$\scriptscriptstyle\pm0.008$} \\
Llama-3.2-3B & 0.441\,{$\scriptscriptstyle\pm0.001$} & 0.439\,{$\scriptscriptstyle\pm0.001$} & 0.355\,{$\scriptscriptstyle\pm0.004$} & \underline{0.442}\,{$\scriptscriptstyle\pm0.003$} & 0.423\,{$\scriptscriptstyle\pm0.003$} & 0.407\,{$\scriptscriptstyle\pm0.009$} & \textbf{0.443}\,{$\scriptscriptstyle\pm0.003$} & 0.410\,{$\scriptscriptstyle\pm0.004$} & 0.359 & 0.140 & 0.371\,{$\scriptscriptstyle\pm0.007$} \\
Llama-3.1-8B & 0.435\,{$\scriptscriptstyle\pm0.008$} & 0.433\,{$\scriptscriptstyle\pm0.006$} & 0.256\,{$\scriptscriptstyle\pm0.008$} & \textbf{0.443}\,{$\scriptscriptstyle\pm0.005$} & 0.400\,{$\scriptscriptstyle\pm0.003$} & 0.393\,{$\scriptscriptstyle\pm0.010$} & \underline{0.441}\,{$\scriptscriptstyle\pm0.003$} & 0.402\,{$\scriptscriptstyle\pm0.012$} & 0.374 & 0.122 & 0.367\,{$\scriptscriptstyle\pm0.003$} \\
Qwen2.5-1.5B & 0.516\,{$\scriptscriptstyle\pm0.003$} & 0.513\,{$\scriptscriptstyle\pm0.006$} & 0.443\,{$\scriptscriptstyle\pm0.008$} & \underline{0.528}\,{$\scriptscriptstyle\pm0.003$} & \textbf{0.530}\,{$\scriptscriptstyle\pm0.005$} & 0.527\,{$\scriptscriptstyle\pm0.001$} & 0.438\,{$\scriptscriptstyle\pm0.143$} & 0.498\,{$\scriptscriptstyle\pm0.015$} & 0.368 & 0.060 & 0.450\,{$\scriptscriptstyle\pm0.013$} \\
Qwen2.5-3B & 0.509\,{$\scriptscriptstyle\pm0.003$} & 0.504\,{$\scriptscriptstyle\pm0.003$} & 0.407\,{$\scriptscriptstyle\pm0.002$} & \underline{0.518}\,{$\scriptscriptstyle\pm0.006$} & 0.492\,{$\scriptscriptstyle\pm0.008$} & 0.469\,{$\scriptscriptstyle\pm0.013$} & \textbf{0.522}\,{$\scriptscriptstyle\pm0.008$} & 0.487\,{$\scriptscriptstyle\pm0.019$} & 0.388 & 0.123 & 0.393\,{$\scriptscriptstyle\pm0.007$} \\
gemma-2-2b & 0.288\,{$\scriptscriptstyle\pm0.001$} & \underline{0.298}\,{$\scriptscriptstyle\pm0.010$} & 0.169\,{$\scriptscriptstyle\pm0.009$} & \textbf{0.321}\,{$\scriptscriptstyle\pm0.002$} & 0.287\,{$\scriptscriptstyle\pm0.001$} & 0.296\,{$\scriptscriptstyle\pm0.012$} & 0.058\,{$\scriptscriptstyle\pm0.025$} & 0.248\,{$\scriptscriptstyle\pm0.028$} & 0.021 & 0.033 & 0.280\,{$\scriptscriptstyle\pm0.003$} \\
\midrule
\multicolumn{12}{l}{\textbf{Source: gemma-2-2b}} \\
Llama-3.2-1B & 0.404\,{$\scriptscriptstyle\pm0.000$} & 0.415\,{$\scriptscriptstyle\pm0.003$} & 0.304\,{$\scriptscriptstyle\pm0.008$} & \textbf{0.422}\,{$\scriptscriptstyle\pm0.007$} & 0.411\,{$\scriptscriptstyle\pm0.001$} & 0.417\,{$\scriptscriptstyle\pm0.006$} & \underline{0.419}\,{$\scriptscriptstyle\pm0.007$} & 0.392\,{$\scriptscriptstyle\pm0.005$} & 0.362 & 0.080 & 0.395\,{$\scriptscriptstyle\pm0.008$} \\
Llama-3.2-3B & 0.415\,{$\scriptscriptstyle\pm0.003$} & 0.413\,{$\scriptscriptstyle\pm0.001$} & 0.286\,{$\scriptscriptstyle\pm0.007$} & \textbf{0.426}\,{$\scriptscriptstyle\pm0.002$} & 0.406\,{$\scriptscriptstyle\pm0.004$} & 0.399\,{$\scriptscriptstyle\pm0.008$} & \underline{0.425}\,{$\scriptscriptstyle\pm0.003$} & 0.356\,{$\scriptscriptstyle\pm0.007$} & 0.258 & 0.121 & 0.387\,{$\scriptscriptstyle\pm0.002$} \\
Llama-3.1-8B & 0.403\,{$\scriptscriptstyle\pm0.003$} & 0.400\,{$\scriptscriptstyle\pm0.001$} & 0.274\,{$\scriptscriptstyle\pm0.007$} & \textbf{0.419}\,{$\scriptscriptstyle\pm0.005$} & 0.405\,{$\scriptscriptstyle\pm0.003$} & 0.385\,{$\scriptscriptstyle\pm0.007$} & \underline{0.417}\,{$\scriptscriptstyle\pm0.002$} & 0.319\,{$\scriptscriptstyle\pm0.002$} & 0.216 & 0.115 & 0.379\,{$\scriptscriptstyle\pm0.007$} \\
Qwen2.5-1.5B & 0.457\,{$\scriptscriptstyle\pm0.002$} & 0.462\,{$\scriptscriptstyle\pm0.004$} & 0.403\,{$\scriptscriptstyle\pm0.007$} & \textbf{0.474}\,{$\scriptscriptstyle\pm0.008$} & 0.467\,{$\scriptscriptstyle\pm0.002$} & 0.469\,{$\scriptscriptstyle\pm0.012$} & \underline{0.474}\,{$\scriptscriptstyle\pm0.008$} & 0.414\,{$\scriptscriptstyle\pm0.030$} & 0.040 & 0.029 & 0.443\,{$\scriptscriptstyle\pm0.006$} \\
Qwen2.5-3B & 0.424\,{$\scriptscriptstyle\pm0.008$} & 0.414\,{$\scriptscriptstyle\pm0.005$} & 0.324\,{$\scriptscriptstyle\pm0.002$} & \textbf{0.444}\,{$\scriptscriptstyle\pm0.002$} & 0.432\,{$\scriptscriptstyle\pm0.011$} & 0.420\,{$\scriptscriptstyle\pm0.003$} & \underline{0.444}\,{$\scriptscriptstyle\pm0.003$} & 0.364\,{$\scriptscriptstyle\pm0.013$} & 0.076 & 0.064 & 0.397\,{$\scriptscriptstyle\pm0.008$} \\
Qwen2.5-7B & 0.415\,{$\scriptscriptstyle\pm0.003$} & 0.417\,{$\scriptscriptstyle\pm0.004$} & 0.390\,{$\scriptscriptstyle\pm0.007$} & \textbf{0.423}\,{$\scriptscriptstyle\pm0.001$} & 0.405\,{$\scriptscriptstyle\pm0.003$} & \underline{0.419}\,{$\scriptscriptstyle\pm0.005$} & 0.367\,{$\scriptscriptstyle\pm0.096$} & 0.348\,{$\scriptscriptstyle\pm0.031$} & 0.201 & 0.114 & 0.395\,{$\scriptscriptstyle\pm0.003$} \\
\midrule\midrule
\textbf{\#\,1st-place wins} & \textbf{3} & \textbf{2} & \textbf{0} & \textbf{21} & \textbf{2} & \textbf{1} & \textbf{12} & \textbf{0} & \textbf{1} & \textbf{0} & \textbf{0} \\
\#\,2nd-place wins & 5 & 5 & 0 & 15 & 1 & 4 & 9 & 1 & 1 & 0 & 1 \\
\end{longtable}\endgroup
\twocolumn

\onecolumn
\begingroup\fontsize{4.6pt}{5.6pt}\selectfont\setlength{\tabcolsep}{0.6pt}\renewcommand{\arraystretch}{1.1}
\setlength\LTleft{0pt}\setlength\LTright{0pt}\begin{longtable}{@{\extracolsep{\fill}}lccccccccccc}
\caption{Recall@1 ($\uparrow$): full per-pair results, seed mean\,$\pm$\,std. Best \textbf{bold}, second \underline{underlined}; grouped by source (donor) LLM.}\label{tab:pp-recall1}\\
\toprule
Target & \textbf{Term} & \textbf{Rel-dep} & \textbf{CKA} & \textbf{MATCHA} & \textbf{MATCHA-K} & \textbf{MATCHA-FL} & \textbf{MATCHA-avg} & \textbf{MATCHA-no FF} & \textbf{CmdV} & \textbf{HOT} & \textbf{MLP} \\
\midrule\endfirsthead
\multicolumn{12}{c}{\tablename\ \thetable{} -- continued}\\ \toprule
Target & \textbf{Term} & \textbf{Rel-dep} & \textbf{CKA} & \textbf{MATCHA} & \textbf{MATCHA-K} & \textbf{MATCHA-FL} & \textbf{MATCHA-avg} & \textbf{MATCHA-no FF} & \textbf{CmdV} & \textbf{HOT} & \textbf{MLP} \\
\midrule\endhead
\bottomrule\endfoot
\multicolumn{12}{l}{\textbf{Source: Llama-3.2-1B}} \\
Llama-3.2-3B & 0.397\,{$\scriptscriptstyle\pm0.001$} & 0.783\,{$\scriptscriptstyle\pm0.001$} & 0.737\,{$\scriptscriptstyle\pm0.002$} & \underline{0.857}\,{$\scriptscriptstyle\pm0.003$} & 0.745\,{$\scriptscriptstyle\pm0.019$} & 0.758\,{$\scriptscriptstyle\pm0.005$} & 0.853\,{$\scriptscriptstyle\pm0.002$} & 0.744\,{$\scriptscriptstyle\pm0.009$} & \textbf{0.904} & 0.063 & 0.728\,{$\scriptscriptstyle\pm0.006$} \\
Llama-3.1-8B & 0.338\,{$\scriptscriptstyle\pm0.001$} & 0.800\,{$\scriptscriptstyle\pm0.001$} & 0.735\,{$\scriptscriptstyle\pm0.002$} & 0.862\,{$\scriptscriptstyle\pm0.002$} & 0.777\,{$\scriptscriptstyle\pm0.002$} & 0.779\,{$\scriptscriptstyle\pm0.002$} & \underline{0.863}\,{$\scriptscriptstyle\pm0.004$} & 0.730\,{$\scriptscriptstyle\pm0.028$} & \textbf{0.898} & 0.053 & 0.745\,{$\scriptscriptstyle\pm0.008$} \\
Qwen2.5-1.5B & 0.445\,{$\scriptscriptstyle\pm0.001$} & 0.612\,{$\scriptscriptstyle\pm0.001$} & 0.498\,{$\scriptscriptstyle\pm0.004$} & \underline{0.762}\,{$\scriptscriptstyle\pm0.006$} & 0.735\,{$\scriptscriptstyle\pm0.005$} & 0.679\,{$\scriptscriptstyle\pm0.014$} & \textbf{0.770}\,{$\scriptscriptstyle\pm0.002$} & 0.593\,{$\scriptscriptstyle\pm0.026$} & 0.500 & 0.037 & 0.610\,{$\scriptscriptstyle\pm0.007$} \\
Qwen2.5-3B & 0.330\,{$\scriptscriptstyle\pm0.001$} & 0.591\,{$\scriptscriptstyle\pm0.000$} & 0.435\,{$\scriptscriptstyle\pm0.001$} & \textbf{0.765}\,{$\scriptscriptstyle\pm0.007$} & 0.640\,{$\scriptscriptstyle\pm0.017$} & 0.630\,{$\scriptscriptstyle\pm0.023$} & \underline{0.758}\,{$\scriptscriptstyle\pm0.002$} & 0.488\,{$\scriptscriptstyle\pm0.039$} & 0.451 & 0.040 & 0.554\,{$\scriptscriptstyle\pm0.013$} \\
Qwen2.5-7B & 0.412\,{$\scriptscriptstyle\pm0.001$} & 0.668\,{$\scriptscriptstyle\pm0.001$} & 0.573\,{$\scriptscriptstyle\pm0.003$} & \underline{0.856}\,{$\scriptscriptstyle\pm0.003$} & 0.619\,{$\scriptscriptstyle\pm0.025$} & 0.693\,{$\scriptscriptstyle\pm0.009$} & \textbf{0.862}\,{$\scriptscriptstyle\pm0.002$} & 0.620\,{$\scriptscriptstyle\pm0.012$} & 0.557 & 0.046 & 0.632\,{$\scriptscriptstyle\pm0.008$} \\
gemma-2-2b & 0.393\,{$\scriptscriptstyle\pm0.002$} & 0.672\,{$\scriptscriptstyle\pm0.003$} & 0.532\,{$\scriptscriptstyle\pm0.005$} & \textbf{0.805}\,{$\scriptscriptstyle\pm0.005$} & \underline{0.746}\,{$\scriptscriptstyle\pm0.005$} & 0.739\,{$\scriptscriptstyle\pm0.005$} & 0.511\,{$\scriptscriptstyle\pm0.260$} & 0.710\,{$\scriptscriptstyle\pm0.021$} & 0.733 & 0.021 & 0.672\,{$\scriptscriptstyle\pm0.002$} \\
\midrule
\multicolumn{12}{l}{\textbf{Source: Llama-3.2-3B}} \\
Llama-3.2-1B & 0.773\,{$\scriptscriptstyle\pm0.002$} & 0.903\,{$\scriptscriptstyle\pm0.002$} & 0.740\,{$\scriptscriptstyle\pm0.001$} & 0.915\,{$\scriptscriptstyle\pm0.002$} & 0.740\,{$\scriptscriptstyle\pm0.021$} & 0.778\,{$\scriptscriptstyle\pm0.011$} & \underline{0.917}\,{$\scriptscriptstyle\pm0.002$} & 0.838\,{$\scriptscriptstyle\pm0.001$} & \textbf{0.955} & 0.115 & 0.557\,{$\scriptscriptstyle\pm0.013$} \\
Llama-3.1-8B & 0.777\,{$\scriptscriptstyle\pm0.001$} & 0.914\,{$\scriptscriptstyle\pm0.001$} & 0.783\,{$\scriptscriptstyle\pm0.004$} & 0.924\,{$\scriptscriptstyle\pm0.001$} & 0.838\,{$\scriptscriptstyle\pm0.007$} & 0.782\,{$\scriptscriptstyle\pm0.011$} & \underline{0.925}\,{$\scriptscriptstyle\pm0.004$} & 0.847\,{$\scriptscriptstyle\pm0.002$} & \textbf{0.969} & 0.089 & 0.556\,{$\scriptscriptstyle\pm0.008$} \\
Qwen2.5-1.5B & 0.644\,{$\scriptscriptstyle\pm0.001$} & 0.647\,{$\scriptscriptstyle\pm0.001$} & 0.345\,{$\scriptscriptstyle\pm0.003$} & \underline{0.797}\,{$\scriptscriptstyle\pm0.008$} & 0.675\,{$\scriptscriptstyle\pm0.004$} & 0.615\,{$\scriptscriptstyle\pm0.022$} & \textbf{0.798}\,{$\scriptscriptstyle\pm0.001$} & 0.642\,{$\scriptscriptstyle\pm0.005$} & 0.541 & 0.033 & 0.445\,{$\scriptscriptstyle\pm0.026$} \\
Qwen2.5-3B & 0.554\,{$\scriptscriptstyle\pm0.001$} & 0.588\,{$\scriptscriptstyle\pm0.000$} & 0.376\,{$\scriptscriptstyle\pm0.001$} & \underline{0.799}\,{$\scriptscriptstyle\pm0.004$} & 0.544\,{$\scriptscriptstyle\pm0.017$} & 0.556\,{$\scriptscriptstyle\pm0.010$} & \textbf{0.802}\,{$\scriptscriptstyle\pm0.001$} & 0.615\,{$\scriptscriptstyle\pm0.011$} & 0.491 & 0.044 & 0.399\,{$\scriptscriptstyle\pm0.020$} \\
Qwen2.5-7B & 0.766\,{$\scriptscriptstyle\pm0.003$} & \underline{0.767}\,{$\scriptscriptstyle\pm0.001$} & 0.479\,{$\scriptscriptstyle\pm0.005$} & \textbf{0.892}\,{$\scriptscriptstyle\pm0.008$} & 0.587\,{$\scriptscriptstyle\pm0.010$} & 0.676\,{$\scriptscriptstyle\pm0.013$} & 0.697\,{$\scriptscriptstyle\pm0.352$} & 0.644\,{$\scriptscriptstyle\pm0.004$} & 0.600 & 0.044 & 0.461\,{$\scriptscriptstyle\pm0.020$} \\
gemma-2-2b & \underline{0.814}\,{$\scriptscriptstyle\pm0.001$} & 0.804\,{$\scriptscriptstyle\pm0.001$} & 0.536\,{$\scriptscriptstyle\pm0.003$} & \textbf{0.853}\,{$\scriptscriptstyle\pm0.008$} & 0.785\,{$\scriptscriptstyle\pm0.004$} & 0.720\,{$\scriptscriptstyle\pm0.007$} & 0.718\,{$\scriptscriptstyle\pm0.246$} & 0.737\,{$\scriptscriptstyle\pm0.007$} & 0.809 & 0.021 & 0.485\,{$\scriptscriptstyle\pm0.020$} \\
\midrule
\multicolumn{12}{l}{\textbf{Source: Llama-3.1-8B}} \\
Llama-3.2-1B & 0.738\,{$\scriptscriptstyle\pm0.002$} & \underline{0.922}\,{$\scriptscriptstyle\pm0.000$} & 0.745\,{$\scriptscriptstyle\pm0.006$} & 0.907\,{$\scriptscriptstyle\pm0.005$} & 0.727\,{$\scriptscriptstyle\pm0.013$} & 0.777\,{$\scriptscriptstyle\pm0.017$} & 0.909\,{$\scriptscriptstyle\pm0.001$} & 0.871\,{$\scriptscriptstyle\pm0.003$} & \textbf{0.943} & 0.101 & 0.480\,{$\scriptscriptstyle\pm0.019$} \\
Llama-3.2-3B & 0.929\,{$\scriptscriptstyle\pm0.000$} & \underline{0.956}\,{$\scriptscriptstyle\pm0.001$} & 0.783\,{$\scriptscriptstyle\pm0.005$} & 0.944\,{$\scriptscriptstyle\pm0.003$} & 0.772\,{$\scriptscriptstyle\pm0.004$} & 0.776\,{$\scriptscriptstyle\pm0.017$} & 0.948\,{$\scriptscriptstyle\pm0.001$} & 0.902\,{$\scriptscriptstyle\pm0.002$} & \textbf{0.980} & 0.107 & 0.442\,{$\scriptscriptstyle\pm0.009$} \\
Qwen2.5-1.5B & 0.629\,{$\scriptscriptstyle\pm0.001$} & 0.694\,{$\scriptscriptstyle\pm0.001$} & 0.473\,{$\scriptscriptstyle\pm0.006$} & \underline{0.827}\,{$\scriptscriptstyle\pm0.007$} & 0.711\,{$\scriptscriptstyle\pm0.019$} & 0.654\,{$\scriptscriptstyle\pm0.011$} & \textbf{0.830}\,{$\scriptscriptstyle\pm0.009$} & 0.690\,{$\scriptscriptstyle\pm0.020$} & 0.537 & 0.037 & 0.397\,{$\scriptscriptstyle\pm0.023$} \\
Qwen2.5-3B & 0.627\,{$\scriptscriptstyle\pm0.002$} & 0.642\,{$\scriptscriptstyle\pm0.001$} & 0.439\,{$\scriptscriptstyle\pm0.002$} & \underline{0.829}\,{$\scriptscriptstyle\pm0.004$} & 0.590\,{$\scriptscriptstyle\pm0.016$} & 0.572\,{$\scriptscriptstyle\pm0.024$} & \textbf{0.836}\,{$\scriptscriptstyle\pm0.002$} & 0.621\,{$\scriptscriptstyle\pm0.036$} & 0.489 & 0.043 & 0.334\,{$\scriptscriptstyle\pm0.006$} \\
Qwen2.5-7B & 0.763\,{$\scriptscriptstyle\pm0.003$} & \underline{0.814}\,{$\scriptscriptstyle\pm0.001$} & 0.600\,{$\scriptscriptstyle\pm0.005$} & \textbf{0.919}\,{$\scriptscriptstyle\pm0.003$} & 0.559\,{$\scriptscriptstyle\pm0.014$} & 0.709\,{$\scriptscriptstyle\pm0.014$} & 0.747\,{$\scriptscriptstyle\pm0.309$} & 0.694\,{$\scriptscriptstyle\pm0.005$} & 0.633 & 0.042 & 0.432\,{$\scriptscriptstyle\pm0.015$} \\
gemma-2-2b & 0.791\,{$\scriptscriptstyle\pm0.002$} & \underline{0.834}\,{$\scriptscriptstyle\pm0.000$} & 0.570\,{$\scriptscriptstyle\pm0.001$} & \textbf{0.885}\,{$\scriptscriptstyle\pm0.002$} & 0.805\,{$\scriptscriptstyle\pm0.006$} & 0.726\,{$\scriptscriptstyle\pm0.012$} & 0.637\,{$\scriptscriptstyle\pm0.214$} & 0.780\,{$\scriptscriptstyle\pm0.001$} & 0.809 & 0.019 & 0.428\,{$\scriptscriptstyle\pm0.023$} \\
\midrule
\multicolumn{12}{l}{\textbf{Source: Qwen2.5-1.5B}} \\
Llama-3.2-1B & \underline{0.821}\,{$\scriptscriptstyle\pm0.000$} & 0.678\,{$\scriptscriptstyle\pm0.002$} & 0.655\,{$\scriptscriptstyle\pm0.003$} & \textbf{0.833}\,{$\scriptscriptstyle\pm0.014$} & 0.776\,{$\scriptscriptstyle\pm0.004$} & 0.762\,{$\scriptscriptstyle\pm0.019$} & 0.696\,{$\scriptscriptstyle\pm0.254$} & 0.543\,{$\scriptscriptstyle\pm0.032$} & 0.430 & 0.035 & 0.736\,{$\scriptscriptstyle\pm0.003$} \\
Llama-3.2-3B & 0.603\,{$\scriptscriptstyle\pm0.001$} & 0.602\,{$\scriptscriptstyle\pm0.002$} & 0.526\,{$\scriptscriptstyle\pm0.007$} & \textbf{0.784}\,{$\scriptscriptstyle\pm0.002$} & 0.646\,{$\scriptscriptstyle\pm0.005$} & 0.641\,{$\scriptscriptstyle\pm0.002$} & 0.415\,{$\scriptscriptstyle\pm0.312$} & 0.479\,{$\scriptscriptstyle\pm0.030$} & 0.398 & 0.024 & \underline{0.657}\,{$\scriptscriptstyle\pm0.009$} \\
Llama-3.1-8B & 0.518\,{$\scriptscriptstyle\pm0.001$} & 0.619\,{$\scriptscriptstyle\pm0.001$} & 0.694\,{$\scriptscriptstyle\pm0.009$} & \textbf{0.771}\,{$\scriptscriptstyle\pm0.006$} & 0.666\,{$\scriptscriptstyle\pm0.004$} & 0.680\,{$\scriptscriptstyle\pm0.013$} & \underline{0.767}\,{$\scriptscriptstyle\pm0.003$} & 0.542\,{$\scriptscriptstyle\pm0.004$} & 0.416 & 0.016 & 0.702\,{$\scriptscriptstyle\pm0.016$} \\
Qwen2.5-3B & 0.596\,{$\scriptscriptstyle\pm0.002$} & 0.726\,{$\scriptscriptstyle\pm0.001$} & 0.562\,{$\scriptscriptstyle\pm0.003$} & \textbf{0.790}\,{$\scriptscriptstyle\pm0.008$} & 0.715\,{$\scriptscriptstyle\pm0.002$} & 0.686\,{$\scriptscriptstyle\pm0.018$} & \underline{0.790}\,{$\scriptscriptstyle\pm0.003$} & 0.378\,{$\scriptscriptstyle\pm0.005$} & 0.448 & 0.019 & 0.624\,{$\scriptscriptstyle\pm0.010$} \\
Qwen2.5-7B & 0.766\,{$\scriptscriptstyle\pm0.002$} & 0.831\,{$\scriptscriptstyle\pm0.001$} & 0.598\,{$\scriptscriptstyle\pm0.004$} & \underline{0.846}\,{$\scriptscriptstyle\pm0.001$} & 0.729\,{$\scriptscriptstyle\pm0.003$} & 0.741\,{$\scriptscriptstyle\pm0.011$} & \textbf{0.850}\,{$\scriptscriptstyle\pm0.004$} & 0.560\,{$\scriptscriptstyle\pm0.091$} & 0.481 & 0.016 & 0.664\,{$\scriptscriptstyle\pm0.018$} \\
gemma-2-2b & 0.617\,{$\scriptscriptstyle\pm0.002$} & 0.560\,{$\scriptscriptstyle\pm0.002$} & 0.558\,{$\scriptscriptstyle\pm0.001$} & \textbf{0.745}\,{$\scriptscriptstyle\pm0.014$} & \underline{0.717}\,{$\scriptscriptstyle\pm0.005$} & 0.708\,{$\scriptscriptstyle\pm0.015$} & 0.555\,{$\scriptscriptstyle\pm0.346$} & 0.474\,{$\scriptscriptstyle\pm0.045$} & 0.379 & 0.008 & 0.688\,{$\scriptscriptstyle\pm0.008$} \\
\midrule
\multicolumn{12}{l}{\textbf{Source: Qwen2.5-3B}} \\
Llama-3.2-1B & 0.758\,{$\scriptscriptstyle\pm0.002$} & \underline{0.793}\,{$\scriptscriptstyle\pm0.003$} & 0.667\,{$\scriptscriptstyle\pm0.004$} & \textbf{0.846}\,{$\scriptscriptstyle\pm0.004$} & 0.787\,{$\scriptscriptstyle\pm0.003$} & 0.758\,{$\scriptscriptstyle\pm0.020$} & 0.709\,{$\scriptscriptstyle\pm0.271$} & 0.601\,{$\scriptscriptstyle\pm0.009$} & 0.414 & 0.035 & 0.637\,{$\scriptscriptstyle\pm0.013$} \\
Llama-3.2-3B & 0.700\,{$\scriptscriptstyle\pm0.001$} & 0.577\,{$\scriptscriptstyle\pm0.001$} & 0.629\,{$\scriptscriptstyle\pm0.002$} & \textbf{0.737}\,{$\scriptscriptstyle\pm0.024$} & 0.673\,{$\scriptscriptstyle\pm0.004$} & 0.656\,{$\scriptscriptstyle\pm0.012$} & \underline{0.732}\,{$\scriptscriptstyle\pm0.018$} & 0.621\,{$\scriptscriptstyle\pm0.016$} & 0.384 & 0.022 & 0.574\,{$\scriptscriptstyle\pm0.019$} \\
Llama-3.1-8B & 0.698\,{$\scriptscriptstyle\pm0.000$} & 0.627\,{$\scriptscriptstyle\pm0.002$} & 0.727\,{$\scriptscriptstyle\pm0.010$} & \textbf{0.796}\,{$\scriptscriptstyle\pm0.005$} & 0.702\,{$\scriptscriptstyle\pm0.008$} & 0.688\,{$\scriptscriptstyle\pm0.024$} & \underline{0.793}\,{$\scriptscriptstyle\pm0.006$} & 0.616\,{$\scriptscriptstyle\pm0.024$} & 0.416 & 0.017 & 0.617\,{$\scriptscriptstyle\pm0.021$} \\
Qwen2.5-1.5B & \underline{0.800}\,{$\scriptscriptstyle\pm0.003$} & 0.791\,{$\scriptscriptstyle\pm0.002$} & 0.648\,{$\scriptscriptstyle\pm0.008$} & \textbf{0.820}\,{$\scriptscriptstyle\pm0.019$} & 0.787\,{$\scriptscriptstyle\pm0.002$} & 0.768\,{$\scriptscriptstyle\pm0.007$} & 0.713\,{$\scriptscriptstyle\pm0.213$} & 0.570\,{$\scriptscriptstyle\pm0.031$} & 0.475 & 0.018 & 0.597\,{$\scriptscriptstyle\pm0.006$} \\
Qwen2.5-7B & 0.809\,{$\scriptscriptstyle\pm0.001$} & 0.797\,{$\scriptscriptstyle\pm0.003$} & 0.632\,{$\scriptscriptstyle\pm0.001$} & \textbf{0.886}\,{$\scriptscriptstyle\pm0.006$} & 0.768\,{$\scriptscriptstyle\pm0.001$} & \underline{0.809}\,{$\scriptscriptstyle\pm0.010$} & 0.680\,{$\scriptscriptstyle\pm0.377$} & 0.648\,{$\scriptscriptstyle\pm0.069$} & 0.484 & 0.020 & 0.623\,{$\scriptscriptstyle\pm0.044$} \\
gemma-2-2b & 0.701\,{$\scriptscriptstyle\pm0.003$} & 0.556\,{$\scriptscriptstyle\pm0.001$} & 0.549\,{$\scriptscriptstyle\pm0.004$} & \textbf{0.782}\,{$\scriptscriptstyle\pm0.011$} & \underline{0.736}\,{$\scriptscriptstyle\pm0.003$} & 0.710\,{$\scriptscriptstyle\pm0.009$} & 0.278\,{$\scriptscriptstyle\pm0.010$} & 0.459\,{$\scriptscriptstyle\pm0.062$} & 0.399 & 0.008 & 0.606\,{$\scriptscriptstyle\pm0.015$} \\
\midrule
\multicolumn{12}{l}{\textbf{Source: Qwen2.5-7B}} \\
Llama-3.2-1B & 0.819\,{$\scriptscriptstyle\pm0.001$} & 0.824\,{$\scriptscriptstyle\pm0.001$} & 0.719\,{$\scriptscriptstyle\pm0.008$} & \underline{0.852}\,{$\scriptscriptstyle\pm0.002$} & 0.790\,{$\scriptscriptstyle\pm0.014$} & 0.782\,{$\scriptscriptstyle\pm0.008$} & \textbf{0.856}\,{$\scriptscriptstyle\pm0.002$} & 0.687\,{$\scriptscriptstyle\pm0.022$} & 0.641 & 0.058 & 0.561\,{$\scriptscriptstyle\pm0.011$} \\
Llama-3.2-3B & 0.751\,{$\scriptscriptstyle\pm0.001$} & 0.749\,{$\scriptscriptstyle\pm0.001$} & 0.699\,{$\scriptscriptstyle\pm0.002$} & \textbf{0.824}\,{$\scriptscriptstyle\pm0.006$} & 0.706\,{$\scriptscriptstyle\pm0.009$} & 0.713\,{$\scriptscriptstyle\pm0.014$} & \underline{0.815}\,{$\scriptscriptstyle\pm0.005$} & 0.715\,{$\scriptscriptstyle\pm0.018$} & 0.699 & 0.030 & 0.504\,{$\scriptscriptstyle\pm0.018$} \\
Llama-3.1-8B & 0.636\,{$\scriptscriptstyle\pm0.000$} & 0.753\,{$\scriptscriptstyle\pm0.002$} & 0.737\,{$\scriptscriptstyle\pm0.003$} & \textbf{0.855}\,{$\scriptscriptstyle\pm0.002$} & 0.755\,{$\scriptscriptstyle\pm0.011$} & 0.746\,{$\scriptscriptstyle\pm0.007$} & \underline{0.852}\,{$\scriptscriptstyle\pm0.003$} & 0.703\,{$\scriptscriptstyle\pm0.027$} & 0.674 & 0.022 & 0.563\,{$\scriptscriptstyle\pm0.022$} \\
Qwen2.5-1.5B & 0.763\,{$\scriptscriptstyle\pm0.003$} & 0.764\,{$\scriptscriptstyle\pm0.002$} & 0.615\,{$\scriptscriptstyle\pm0.003$} & \textbf{0.855}\,{$\scriptscriptstyle\pm0.000$} & \underline{0.794}\,{$\scriptscriptstyle\pm0.002$} & 0.779\,{$\scriptscriptstyle\pm0.004$} & 0.720\,{$\scriptscriptstyle\pm0.237$} & 0.776\,{$\scriptscriptstyle\pm0.019$} & 0.526 & 0.033 & 0.489\,{$\scriptscriptstyle\pm0.054$} \\
Qwen2.5-3B & 0.604\,{$\scriptscriptstyle\pm0.003$} & 0.733\,{$\scriptscriptstyle\pm0.001$} & 0.543\,{$\scriptscriptstyle\pm0.001$} & \textbf{0.885}\,{$\scriptscriptstyle\pm0.005$} & 0.764\,{$\scriptscriptstyle\pm0.004$} & 0.738\,{$\scriptscriptstyle\pm0.017$} & \underline{0.884}\,{$\scriptscriptstyle\pm0.002$} & 0.710\,{$\scriptscriptstyle\pm0.053$} & 0.545 & 0.039 & 0.464\,{$\scriptscriptstyle\pm0.020$} \\
gemma-2-2b & 0.775\,{$\scriptscriptstyle\pm0.001$} & 0.735\,{$\scriptscriptstyle\pm0.001$} & 0.549\,{$\scriptscriptstyle\pm0.001$} & \textbf{0.841}\,{$\scriptscriptstyle\pm0.014$} & \underline{0.778}\,{$\scriptscriptstyle\pm0.003$} & 0.745\,{$\scriptscriptstyle\pm0.012$} & 0.402\,{$\scriptscriptstyle\pm0.018$} & 0.654\,{$\scriptscriptstyle\pm0.012$} & 0.648 & 0.016 & 0.519\,{$\scriptscriptstyle\pm0.018$} \\
\midrule
\multicolumn{12}{l}{\textbf{Source: gemma-2-2b}} \\
Llama-3.2-1B & 0.721\,{$\scriptscriptstyle\pm0.001$} & 0.763\,{$\scriptscriptstyle\pm0.001$} & 0.627\,{$\scriptscriptstyle\pm0.007$} & \underline{0.859}\,{$\scriptscriptstyle\pm0.003$} & 0.775\,{$\scriptscriptstyle\pm0.001$} & 0.734\,{$\scriptscriptstyle\pm0.012$} & \textbf{0.866}\,{$\scriptscriptstyle\pm0.002$} & 0.716\,{$\scriptscriptstyle\pm0.001$} & 0.777 & 0.053 & 0.621\,{$\scriptscriptstyle\pm0.012$} \\
Llama-3.2-3B & 0.679\,{$\scriptscriptstyle\pm0.002$} & 0.714\,{$\scriptscriptstyle\pm0.001$} & 0.626\,{$\scriptscriptstyle\pm0.002$} & \underline{0.788}\,{$\scriptscriptstyle\pm0.002$} & 0.736\,{$\scriptscriptstyle\pm0.006$} & 0.716\,{$\scriptscriptstyle\pm0.001$} & 0.785\,{$\scriptscriptstyle\pm0.004$} & 0.599\,{$\scriptscriptstyle\pm0.024$} & \textbf{0.819} & 0.032 & 0.606\,{$\scriptscriptstyle\pm0.008$} \\
Llama-3.1-8B & 0.590\,{$\scriptscriptstyle\pm0.002$} & 0.746\,{$\scriptscriptstyle\pm0.003$} & 0.648\,{$\scriptscriptstyle\pm0.009$} & \textbf{0.814}\,{$\scriptscriptstyle\pm0.008$} & 0.726\,{$\scriptscriptstyle\pm0.006$} & 0.698\,{$\scriptscriptstyle\pm0.005$} & \underline{0.810}\,{$\scriptscriptstyle\pm0.009$} & 0.698\,{$\scriptscriptstyle\pm0.009$} & 0.789 & 0.023 & 0.626\,{$\scriptscriptstyle\pm0.007$} \\
Qwen2.5-1.5B & 0.577\,{$\scriptscriptstyle\pm0.001$} & 0.580\,{$\scriptscriptstyle\pm0.001$} & 0.437\,{$\scriptscriptstyle\pm0.002$} & \underline{0.714}\,{$\scriptscriptstyle\pm0.010$} & 0.672\,{$\scriptscriptstyle\pm0.001$} & 0.602\,{$\scriptscriptstyle\pm0.010$} & \textbf{0.719}\,{$\scriptscriptstyle\pm0.005$} & 0.504\,{$\scriptscriptstyle\pm0.025$} & 0.498 & 0.022 & 0.513\,{$\scriptscriptstyle\pm0.008$} \\
Qwen2.5-3B & 0.471\,{$\scriptscriptstyle\pm0.002$} & 0.526\,{$\scriptscriptstyle\pm0.002$} & 0.362\,{$\scriptscriptstyle\pm0.008$} & \textbf{0.694}\,{$\scriptscriptstyle\pm0.004$} & 0.586\,{$\scriptscriptstyle\pm0.010$} & 0.557\,{$\scriptscriptstyle\pm0.014$} & \underline{0.694}\,{$\scriptscriptstyle\pm0.004$} & 0.420\,{$\scriptscriptstyle\pm0.075$} & 0.491 & 0.023 & 0.446\,{$\scriptscriptstyle\pm0.022$} \\
Qwen2.5-7B & 0.662\,{$\scriptscriptstyle\pm0.002$} & \underline{0.676}\,{$\scriptscriptstyle\pm0.001$} & 0.494\,{$\scriptscriptstyle\pm0.011$} & \textbf{0.776}\,{$\scriptscriptstyle\pm0.014$} & 0.653\,{$\scriptscriptstyle\pm0.006$} & 0.669\,{$\scriptscriptstyle\pm0.005$} & 0.582\,{$\scriptscriptstyle\pm0.381$} & 0.562\,{$\scriptscriptstyle\pm0.021$} & 0.619 & 0.031 & 0.532\,{$\scriptscriptstyle\pm0.015$} \\
\midrule\midrule
\textbf{\#\,1st-place wins} & \textbf{0} & \textbf{0} & \textbf{0} & \textbf{25} & \textbf{0} & \textbf{0} & \textbf{10} & \textbf{0} & \textbf{7} & \textbf{0} & \textbf{0} \\
\#\,2nd-place wins & 3 & 7 & 0 & 12 & 5 & 1 & 13 & 0 & 0 & 0 & 1 \\
\end{longtable}\endgroup
\twocolumn

\onecolumn
\begingroup\fontsize{4.6pt}{5.6pt}\selectfont\setlength{\tabcolsep}{0.6pt}\renewcommand{\arraystretch}{1.1}
\setlength\LTleft{0pt}\setlength\LTright{0pt}\begin{longtable}{@{\extracolsep{\fill}}lccccccccccc}
\caption{Precision@1 ($\uparrow$): full per-pair results, seed mean\,$\pm$\,std. Best \textbf{bold}, second \underline{underlined}; grouped by source (donor) LLM.}\label{tab:pp-precision1}\\
\toprule
Target & \textbf{Term} & \textbf{Rel-dep} & \textbf{CKA} & \textbf{MATCHA} & \textbf{MATCHA-K} & \textbf{MATCHA-FL} & \textbf{MATCHA-avg} & \textbf{MATCHA-no FF} & \textbf{CmdV} & \textbf{HOT} & \textbf{MLP} \\
\midrule\endfirsthead
\multicolumn{12}{c}{\tablename\ \thetable{} -- continued}\\ \toprule
Target & \textbf{Term} & \textbf{Rel-dep} & \textbf{CKA} & \textbf{MATCHA} & \textbf{MATCHA-K} & \textbf{MATCHA-FL} & \textbf{MATCHA-avg} & \textbf{MATCHA-no FF} & \textbf{CmdV} & \textbf{HOT} & \textbf{MLP} \\
\midrule\endhead
\bottomrule\endfoot
\multicolumn{12}{l}{\textbf{Source: Llama-3.2-1B}} \\
Llama-3.2-3B & 0.606\,{$\scriptscriptstyle\pm0.002$} & 0.928\,{$\scriptscriptstyle\pm0.002$} & 0.702\,{$\scriptscriptstyle\pm0.006$} & \underline{0.976}\,{$\scriptscriptstyle\pm0.001$} & 0.868\,{$\scriptscriptstyle\pm0.010$} & 0.867\,{$\scriptscriptstyle\pm0.009$} & \textbf{0.977}\,{$\scriptscriptstyle\pm0.001$} & 0.948\,{$\scriptscriptstyle\pm0.005$} & 0.669 & 0.702 & 0.922\,{$\scriptscriptstyle\pm0.006$} \\
Llama-3.1-8B & 0.594\,{$\scriptscriptstyle\pm0.002$} & 0.948\,{$\scriptscriptstyle\pm0.001$} & 0.651\,{$\scriptscriptstyle\pm0.003$} & \underline{0.989}\,{$\scriptscriptstyle\pm0.000$} & 0.896\,{$\scriptscriptstyle\pm0.001$} & 0.853\,{$\scriptscriptstyle\pm0.012$} & \textbf{0.989}\,{$\scriptscriptstyle\pm0.001$} & 0.960\,{$\scriptscriptstyle\pm0.005$} & 0.608 & 0.739 & 0.947\,{$\scriptscriptstyle\pm0.005$} \\
Qwen2.5-1.5B & 0.611\,{$\scriptscriptstyle\pm0.003$} & 0.748\,{$\scriptscriptstyle\pm0.003$} & 0.358\,{$\scriptscriptstyle\pm0.007$} & \underline{0.937}\,{$\scriptscriptstyle\pm0.002$} & 0.825\,{$\scriptscriptstyle\pm0.002$} & 0.761\,{$\scriptscriptstyle\pm0.017$} & \textbf{0.943}\,{$\scriptscriptstyle\pm0.001$} & 0.746\,{$\scriptscriptstyle\pm0.024$} & 0.336 & 0.210 & 0.864\,{$\scriptscriptstyle\pm0.008$} \\
Qwen2.5-3B & 0.529\,{$\scriptscriptstyle\pm0.001$} & 0.821\,{$\scriptscriptstyle\pm0.001$} & 0.347\,{$\scriptscriptstyle\pm0.004$} & \underline{0.951}\,{$\scriptscriptstyle\pm0.002$} & 0.718\,{$\scriptscriptstyle\pm0.006$} & 0.712\,{$\scriptscriptstyle\pm0.021$} & \textbf{0.956}\,{$\scriptscriptstyle\pm0.001$} & 0.699\,{$\scriptscriptstyle\pm0.010$} & 0.273 & 0.199 & 0.846\,{$\scriptscriptstyle\pm0.005$} \\
Qwen2.5-7B & 0.614\,{$\scriptscriptstyle\pm0.002$} & 0.854\,{$\scriptscriptstyle\pm0.003$} & 0.462\,{$\scriptscriptstyle\pm0.005$} & \underline{0.976}\,{$\scriptscriptstyle\pm0.002$} & 0.765\,{$\scriptscriptstyle\pm0.018$} & 0.797\,{$\scriptscriptstyle\pm0.013$} & \textbf{0.977}\,{$\scriptscriptstyle\pm0.001$} & 0.857\,{$\scriptscriptstyle\pm0.023$} & 0.297 & 0.372 & 0.897\,{$\scriptscriptstyle\pm0.003$} \\
gemma-2-2b & 0.522\,{$\scriptscriptstyle\pm0.000$} & 0.801\,{$\scriptscriptstyle\pm0.003$} & 0.482\,{$\scriptscriptstyle\pm0.004$} & \textbf{0.924}\,{$\scriptscriptstyle\pm0.004$} & 0.808\,{$\scriptscriptstyle\pm0.009$} & 0.790\,{$\scriptscriptstyle\pm0.008$} & 0.561\,{$\scriptscriptstyle\pm0.344$} & 0.826\,{$\scriptscriptstyle\pm0.004$} & 0.544 & 0.445 & \underline{0.882}\,{$\scriptscriptstyle\pm0.004$} \\
\midrule
\multicolumn{12}{l}{\textbf{Source: Llama-3.2-3B}} \\
Llama-3.2-1B & 0.810\,{$\scriptscriptstyle\pm0.001$} & \textbf{0.993}\,{$\scriptscriptstyle\pm0.000$} & 0.572\,{$\scriptscriptstyle\pm0.002$} & 0.988\,{$\scriptscriptstyle\pm0.002$} & 0.813\,{$\scriptscriptstyle\pm0.012$} & 0.819\,{$\scriptscriptstyle\pm0.013$} & \underline{0.992}\,{$\scriptscriptstyle\pm0.003$} & 0.971\,{$\scriptscriptstyle\pm0.015$} & 0.945 & 0.597 & 0.868\,{$\scriptscriptstyle\pm0.008$} \\
Llama-3.1-8B & 0.913\,{$\scriptscriptstyle\pm0.000$} & 0.970\,{$\scriptscriptstyle\pm0.000$} & 0.582\,{$\scriptscriptstyle\pm0.015$} & \underline{0.998}\,{$\scriptscriptstyle\pm0.000$} & 0.872\,{$\scriptscriptstyle\pm0.004$} & 0.771\,{$\scriptscriptstyle\pm0.007$} & \textbf{0.999}\,{$\scriptscriptstyle\pm0.000$} & 0.987\,{$\scriptscriptstyle\pm0.001$} & 0.950 & 0.758 & 0.905\,{$\scriptscriptstyle\pm0.005$} \\
Qwen2.5-1.5B & 0.818\,{$\scriptscriptstyle\pm0.003$} & 0.819\,{$\scriptscriptstyle\pm0.003$} & 0.222\,{$\scriptscriptstyle\pm0.008$} & \underline{0.951}\,{$\scriptscriptstyle\pm0.003$} & 0.770\,{$\scriptscriptstyle\pm0.003$} & 0.645\,{$\scriptscriptstyle\pm0.018$} & \textbf{0.955}\,{$\scriptscriptstyle\pm0.001$} & 0.752\,{$\scriptscriptstyle\pm0.012$} & 0.475 & 0.198 & 0.774\,{$\scriptscriptstyle\pm0.013$} \\
Qwen2.5-3B & 0.720\,{$\scriptscriptstyle\pm0.002$} & 0.764\,{$\scriptscriptstyle\pm0.003$} & 0.295\,{$\scriptscriptstyle\pm0.004$} & \underline{0.965}\,{$\scriptscriptstyle\pm0.001$} & 0.674\,{$\scriptscriptstyle\pm0.014$} & 0.562\,{$\scriptscriptstyle\pm0.010$} & \textbf{0.971}\,{$\scriptscriptstyle\pm0.002$} & 0.802\,{$\scriptscriptstyle\pm0.018$} & 0.417 & 0.209 & 0.760\,{$\scriptscriptstyle\pm0.010$} \\
Qwen2.5-7B & \underline{0.935}\,{$\scriptscriptstyle\pm0.002$} & 0.935\,{$\scriptscriptstyle\pm0.001$} & 0.277\,{$\scriptscriptstyle\pm0.005$} & \textbf{0.980}\,{$\scriptscriptstyle\pm0.003$} & 0.658\,{$\scriptscriptstyle\pm0.015$} & 0.645\,{$\scriptscriptstyle\pm0.018$} & 0.694\,{$\scriptscriptstyle\pm0.497$} & 0.848\,{$\scriptscriptstyle\pm0.042$} & 0.470 & 0.372 & 0.828\,{$\scriptscriptstyle\pm0.013$} \\
gemma-2-2b & 0.880\,{$\scriptscriptstyle\pm0.002$} & \underline{0.906}\,{$\scriptscriptstyle\pm0.001$} & 0.394\,{$\scriptscriptstyle\pm0.003$} & \textbf{0.942}\,{$\scriptscriptstyle\pm0.004$} & 0.787\,{$\scriptscriptstyle\pm0.004$} & 0.726\,{$\scriptscriptstyle\pm0.004$} & 0.709\,{$\scriptscriptstyle\pm0.413$} & 0.849\,{$\scriptscriptstyle\pm0.012$} & 0.789 & 0.448 & 0.813\,{$\scriptscriptstyle\pm0.013$} \\
\midrule
\multicolumn{12}{l}{\textbf{Source: Llama-3.1-8B}} \\
Llama-3.2-1B & 0.784\,{$\scriptscriptstyle\pm0.001$} & \underline{0.987}\,{$\scriptscriptstyle\pm0.000$} & 0.620\,{$\scriptscriptstyle\pm0.009$} & 0.986\,{$\scriptscriptstyle\pm0.001$} & 0.839\,{$\scriptscriptstyle\pm0.009$} & 0.839\,{$\scriptscriptstyle\pm0.013$} & \textbf{0.988}\,{$\scriptscriptstyle\pm0.001$} & 0.976\,{$\scriptscriptstyle\pm0.000$} & 0.968 & 0.626 & 0.846\,{$\scriptscriptstyle\pm0.011$} \\
Llama-3.2-3B & 0.966\,{$\scriptscriptstyle\pm0.000$} & \textbf{0.994}\,{$\scriptscriptstyle\pm0.000$} & 0.655\,{$\scriptscriptstyle\pm0.004$} & 0.990\,{$\scriptscriptstyle\pm0.001$} & 0.876\,{$\scriptscriptstyle\pm0.006$} & 0.828\,{$\scriptscriptstyle\pm0.013$} & \underline{0.993}\,{$\scriptscriptstyle\pm0.000$} & 0.978\,{$\scriptscriptstyle\pm0.002$} & 0.990 & 0.754 & 0.854\,{$\scriptscriptstyle\pm0.013$} \\
Qwen2.5-1.5B & 0.734\,{$\scriptscriptstyle\pm0.003$} & 0.842\,{$\scriptscriptstyle\pm0.003$} & 0.298\,{$\scriptscriptstyle\pm0.010$} & \underline{0.952}\,{$\scriptscriptstyle\pm0.006$} & 0.781\,{$\scriptscriptstyle\pm0.016$} & 0.701\,{$\scriptscriptstyle\pm0.006$} & \textbf{0.955}\,{$\scriptscriptstyle\pm0.004$} & 0.827\,{$\scriptscriptstyle\pm0.018$} & 0.532 & 0.195 & 0.764\,{$\scriptscriptstyle\pm0.015$} \\
Qwen2.5-3B & 0.776\,{$\scriptscriptstyle\pm0.002$} & 0.789\,{$\scriptscriptstyle\pm0.002$} & 0.298\,{$\scriptscriptstyle\pm0.016$} & \underline{0.963}\,{$\scriptscriptstyle\pm0.004$} & 0.731\,{$\scriptscriptstyle\pm0.004$} & 0.627\,{$\scriptscriptstyle\pm0.018$} & \textbf{0.969}\,{$\scriptscriptstyle\pm0.004$} & 0.778\,{$\scriptscriptstyle\pm0.028$} & 0.451 & 0.191 & 0.728\,{$\scriptscriptstyle\pm0.004$} \\
Qwen2.5-7B & 0.882\,{$\scriptscriptstyle\pm0.001$} & \underline{0.949}\,{$\scriptscriptstyle\pm0.000$} & 0.452\,{$\scriptscriptstyle\pm0.009$} & \textbf{0.985}\,{$\scriptscriptstyle\pm0.000$} & 0.754\,{$\scriptscriptstyle\pm0.005$} & 0.711\,{$\scriptscriptstyle\pm0.006$} & 0.667\,{$\scriptscriptstyle\pm0.553$} & 0.938\,{$\scriptscriptstyle\pm0.003$} & 0.542 & 0.383 & 0.818\,{$\scriptscriptstyle\pm0.003$} \\
gemma-2-2b & 0.847\,{$\scriptscriptstyle\pm0.001$} & \underline{0.916}\,{$\scriptscriptstyle\pm0.000$} & 0.466\,{$\scriptscriptstyle\pm0.003$} & \textbf{0.945}\,{$\scriptscriptstyle\pm0.002$} & 0.814\,{$\scriptscriptstyle\pm0.005$} & 0.753\,{$\scriptscriptstyle\pm0.008$} & 0.434\,{$\scriptscriptstyle\pm0.447$} & 0.886\,{$\scriptscriptstyle\pm0.002$} & 0.803 & 0.461 & 0.789\,{$\scriptscriptstyle\pm0.007$} \\
\midrule
\multicolumn{12}{l}{\textbf{Source: Qwen2.5-1.5B}} \\
Llama-3.2-1B & 0.919\,{$\scriptscriptstyle\pm0.001$} & 0.921\,{$\scriptscriptstyle\pm0.002$} & 0.684\,{$\scriptscriptstyle\pm0.005$} & \textbf{0.968}\,{$\scriptscriptstyle\pm0.006$} & 0.911\,{$\scriptscriptstyle\pm0.008$} & 0.888\,{$\scriptscriptstyle\pm0.013$} & \underline{0.933}\,{$\scriptscriptstyle\pm0.062$} & 0.822\,{$\scriptscriptstyle\pm0.011$} & 0.226 & 0.366 & 0.920\,{$\scriptscriptstyle\pm0.003$} \\
Llama-3.2-3B & 0.881\,{$\scriptscriptstyle\pm0.002$} & 0.880\,{$\scriptscriptstyle\pm0.002$} & 0.561\,{$\scriptscriptstyle\pm0.005$} & \textbf{0.948}\,{$\scriptscriptstyle\pm0.002$} & 0.875\,{$\scriptscriptstyle\pm0.002$} & 0.853\,{$\scriptscriptstyle\pm0.004$} & 0.736\,{$\scriptscriptstyle\pm0.220$} & 0.801\,{$\scriptscriptstyle\pm0.017$} & 0.157 & 0.397 & \underline{0.892}\,{$\scriptscriptstyle\pm0.007$} \\
Llama-3.1-8B & 0.776\,{$\scriptscriptstyle\pm0.002$} & 0.870\,{$\scriptscriptstyle\pm0.001$} & 0.636\,{$\scriptscriptstyle\pm0.009$} & \underline{0.961}\,{$\scriptscriptstyle\pm0.002$} & 0.887\,{$\scriptscriptstyle\pm0.003$} & 0.873\,{$\scriptscriptstyle\pm0.012$} & \textbf{0.961}\,{$\scriptscriptstyle\pm0.003$} & 0.836\,{$\scriptscriptstyle\pm0.007$} & 0.166 & 0.433 & 0.929\,{$\scriptscriptstyle\pm0.002$} \\
Qwen2.5-3B & 0.783\,{$\scriptscriptstyle\pm0.001$} & 0.922\,{$\scriptscriptstyle\pm0.002$} & 0.435\,{$\scriptscriptstyle\pm0.006$} & \underline{0.976}\,{$\scriptscriptstyle\pm0.001$} & 0.866\,{$\scriptscriptstyle\pm0.010$} & 0.825\,{$\scriptscriptstyle\pm0.015$} & \textbf{0.978}\,{$\scriptscriptstyle\pm0.000$} & 0.710\,{$\scriptscriptstyle\pm0.012$} & 0.241 & 0.284 & 0.901\,{$\scriptscriptstyle\pm0.002$} \\
Qwen2.5-7B & 0.972\,{$\scriptscriptstyle\pm0.001$} & 0.980\,{$\scriptscriptstyle\pm0.001$} & 0.519\,{$\scriptscriptstyle\pm0.002$} & \textbf{0.987}\,{$\scriptscriptstyle\pm0.001$} & 0.899\,{$\scriptscriptstyle\pm0.004$} & 0.895\,{$\scriptscriptstyle\pm0.012$} & \underline{0.987}\,{$\scriptscriptstyle\pm0.001$} & 0.853\,{$\scriptscriptstyle\pm0.024$} & 0.206 & 0.328 & 0.932\,{$\scriptscriptstyle\pm0.009$} \\
gemma-2-2b & 0.822\,{$\scriptscriptstyle\pm0.003$} & 0.798\,{$\scriptscriptstyle\pm0.002$} & 0.534\,{$\scriptscriptstyle\pm0.001$} & \textbf{0.901}\,{$\scriptscriptstyle\pm0.008$} & 0.839\,{$\scriptscriptstyle\pm0.000$} & 0.808\,{$\scriptscriptstyle\pm0.017$} & 0.710\,{$\scriptscriptstyle\pm0.337$} & 0.699\,{$\scriptscriptstyle\pm0.017$} & 0.149 & 0.336 & \underline{0.880}\,{$\scriptscriptstyle\pm0.003$} \\
\midrule
\multicolumn{12}{l}{\textbf{Source: Qwen2.5-3B}} \\
Llama-3.2-1B & 0.884\,{$\scriptscriptstyle\pm0.002$} & \underline{0.941}\,{$\scriptscriptstyle\pm0.001$} & 0.682\,{$\scriptscriptstyle\pm0.007$} & \textbf{0.969}\,{$\scriptscriptstyle\pm0.001$} & 0.926\,{$\scriptscriptstyle\pm0.004$} & 0.890\,{$\scriptscriptstyle\pm0.011$} & 0.910\,{$\scriptscriptstyle\pm0.106$} & 0.863\,{$\scriptscriptstyle\pm0.010$} & 0.289 & 0.392 & 0.902\,{$\scriptscriptstyle\pm0.003$} \\
Llama-3.2-3B & 0.916\,{$\scriptscriptstyle\pm0.001$} & 0.868\,{$\scriptscriptstyle\pm0.001$} & 0.712\,{$\scriptscriptstyle\pm0.006$} & \textbf{0.927}\,{$\scriptscriptstyle\pm0.028$} & 0.890\,{$\scriptscriptstyle\pm0.004$} & 0.871\,{$\scriptscriptstyle\pm0.009$} & \underline{0.927}\,{$\scriptscriptstyle\pm0.021$} & 0.858\,{$\scriptscriptstyle\pm0.013$} & 0.221 & 0.454 & 0.878\,{$\scriptscriptstyle\pm0.008$} \\
Llama-3.1-8B & 0.924\,{$\scriptscriptstyle\pm0.001$} & 0.888\,{$\scriptscriptstyle\pm0.002$} & 0.675\,{$\scriptscriptstyle\pm0.008$} & \underline{0.959}\,{$\scriptscriptstyle\pm0.002$} & 0.908\,{$\scriptscriptstyle\pm0.002$} & 0.892\,{$\scriptscriptstyle\pm0.013$} & \textbf{0.959}\,{$\scriptscriptstyle\pm0.003$} & 0.855\,{$\scriptscriptstyle\pm0.006$} & 0.241 & 0.498 & 0.915\,{$\scriptscriptstyle\pm0.003$} \\
Qwen2.5-1.5B & 0.937\,{$\scriptscriptstyle\pm0.001$} & \underline{0.962}\,{$\scriptscriptstyle\pm0.001$} & 0.521\,{$\scriptscriptstyle\pm0.004$} & \textbf{0.966}\,{$\scriptscriptstyle\pm0.008$} & 0.911\,{$\scriptscriptstyle\pm0.002$} & 0.859\,{$\scriptscriptstyle\pm0.009$} & 0.893\,{$\scriptscriptstyle\pm0.141$} & 0.825\,{$\scriptscriptstyle\pm0.021$} & 0.334 & 0.305 & 0.886\,{$\scriptscriptstyle\pm0.006$} \\
Qwen2.5-7B & 0.956\,{$\scriptscriptstyle\pm0.001$} & \underline{0.980}\,{$\scriptscriptstyle\pm0.001$} & 0.637\,{$\scriptscriptstyle\pm0.004$} & \textbf{0.991}\,{$\scriptscriptstyle\pm0.001$} & 0.920\,{$\scriptscriptstyle\pm0.003$} & 0.896\,{$\scriptscriptstyle\pm0.003$} & 0.701\,{$\scriptscriptstyle\pm0.503$} & 0.892\,{$\scriptscriptstyle\pm0.029$} & 0.331 & 0.425 & 0.921\,{$\scriptscriptstyle\pm0.012$} \\
gemma-2-2b & 0.847\,{$\scriptscriptstyle\pm0.001$} & 0.813\,{$\scriptscriptstyle\pm0.001$} & 0.534\,{$\scriptscriptstyle\pm0.001$} & \textbf{0.915}\,{$\scriptscriptstyle\pm0.003$} & 0.855\,{$\scriptscriptstyle\pm0.004$} & 0.824\,{$\scriptscriptstyle\pm0.009$} & 0.403\,{$\scriptscriptstyle\pm0.119$} & 0.754\,{$\scriptscriptstyle\pm0.038$} & 0.244 & 0.323 & \underline{0.859}\,{$\scriptscriptstyle\pm0.009$} \\
\midrule
\multicolumn{12}{l}{\textbf{Source: Qwen2.5-7B}} \\
Llama-3.2-1B & 0.914\,{$\scriptscriptstyle\pm0.001$} & 0.958\,{$\scriptscriptstyle\pm0.001$} & 0.667\,{$\scriptscriptstyle\pm0.005$} & \underline{0.962}\,{$\scriptscriptstyle\pm0.002$} & 0.892\,{$\scriptscriptstyle\pm0.012$} & 0.886\,{$\scriptscriptstyle\pm0.008$} & \textbf{0.967}\,{$\scriptscriptstyle\pm0.004$} & 0.917\,{$\scriptscriptstyle\pm0.010$} & 0.542 & 0.433 & 0.870\,{$\scriptscriptstyle\pm0.009$} \\
Llama-3.2-3B & 0.941\,{$\scriptscriptstyle\pm0.002$} & 0.940\,{$\scriptscriptstyle\pm0.001$} & 0.670\,{$\scriptscriptstyle\pm0.006$} & \underline{0.966}\,{$\scriptscriptstyle\pm0.001$} & 0.868\,{$\scriptscriptstyle\pm0.005$} & 0.865\,{$\scriptscriptstyle\pm0.008$} & \textbf{0.966}\,{$\scriptscriptstyle\pm0.001$} & 0.919\,{$\scriptscriptstyle\pm0.001$} & 0.557 & 0.467 & 0.862\,{$\scriptscriptstyle\pm0.006$} \\
Llama-3.1-8B & 0.862\,{$\scriptscriptstyle\pm0.001$} & 0.932\,{$\scriptscriptstyle\pm0.001$} & 0.620\,{$\scriptscriptstyle\pm0.006$} & \underline{0.984}\,{$\scriptscriptstyle\pm0.001$} & 0.906\,{$\scriptscriptstyle\pm0.006$} & 0.887\,{$\scriptscriptstyle\pm0.005$} & \textbf{0.986}\,{$\scriptscriptstyle\pm0.001$} & 0.934\,{$\scriptscriptstyle\pm0.004$} & 0.514 & 0.522 & 0.904\,{$\scriptscriptstyle\pm0.008$} \\
Qwen2.5-1.5B & 0.946\,{$\scriptscriptstyle\pm0.001$} & \underline{0.946}\,{$\scriptscriptstyle\pm0.001$} & 0.494\,{$\scriptscriptstyle\pm0.003$} & \textbf{0.978}\,{$\scriptscriptstyle\pm0.000$} & 0.893\,{$\scriptscriptstyle\pm0.002$} & 0.829\,{$\scriptscriptstyle\pm0.008$} & 0.700\,{$\scriptscriptstyle\pm0.483$} & 0.919\,{$\scriptscriptstyle\pm0.016$} & 0.506 & 0.294 & 0.842\,{$\scriptscriptstyle\pm0.028$} \\
Qwen2.5-3B & 0.784\,{$\scriptscriptstyle\pm0.002$} & 0.918\,{$\scriptscriptstyle\pm0.001$} & 0.426\,{$\scriptscriptstyle\pm0.002$} & \underline{0.988}\,{$\scriptscriptstyle\pm0.000$} & 0.842\,{$\scriptscriptstyle\pm0.009$} & 0.827\,{$\scriptscriptstyle\pm0.012$} & \textbf{0.990}\,{$\scriptscriptstyle\pm0.001$} & 0.872\,{$\scriptscriptstyle\pm0.037$} & 0.511 & 0.353 & 0.847\,{$\scriptscriptstyle\pm0.013$} \\
gemma-2-2b & \underline{0.897}\,{$\scriptscriptstyle\pm0.001$} & 0.891\,{$\scriptscriptstyle\pm0.001$} & 0.472\,{$\scriptscriptstyle\pm0.005$} & \textbf{0.928}\,{$\scriptscriptstyle\pm0.008$} & 0.853\,{$\scriptscriptstyle\pm0.001$} & 0.819\,{$\scriptscriptstyle\pm0.009$} & 0.255\,{$\scriptscriptstyle\pm0.087$} & 0.831\,{$\scriptscriptstyle\pm0.024$} & 0.555 & 0.384 & 0.840\,{$\scriptscriptstyle\pm0.007$} \\
\midrule
\multicolumn{12}{l}{\textbf{Source: gemma-2-2b}} \\
Llama-3.2-1B & 0.820\,{$\scriptscriptstyle\pm0.001$} & 0.923\,{$\scriptscriptstyle\pm0.001$} & 0.518\,{$\scriptscriptstyle\pm0.013$} & \underline{0.968}\,{$\scriptscriptstyle\pm0.001$} & 0.871\,{$\scriptscriptstyle\pm0.001$} & 0.829\,{$\scriptscriptstyle\pm0.004$} & \textbf{0.970}\,{$\scriptscriptstyle\pm0.001$} & 0.884\,{$\scriptscriptstyle\pm0.004$} & 0.614 & 0.437 & 0.856\,{$\scriptscriptstyle\pm0.011$} \\
Llama-3.2-3B & 0.866\,{$\scriptscriptstyle\pm0.001$} & 0.894\,{$\scriptscriptstyle\pm0.001$} & 0.565\,{$\scriptscriptstyle\pm0.004$} & \textbf{0.952}\,{$\scriptscriptstyle\pm0.002$} & 0.854\,{$\scriptscriptstyle\pm0.003$} & 0.795\,{$\scriptscriptstyle\pm0.008$} & \underline{0.951}\,{$\scriptscriptstyle\pm0.001$} & 0.835\,{$\scriptscriptstyle\pm0.033$} & 0.638 & 0.504 & 0.854\,{$\scriptscriptstyle\pm0.006$} \\
Llama-3.1-8B & 0.765\,{$\scriptscriptstyle\pm0.001$} & 0.916\,{$\scriptscriptstyle\pm0.001$} & 0.513\,{$\scriptscriptstyle\pm0.011$} & \textbf{0.969}\,{$\scriptscriptstyle\pm0.003$} & 0.822\,{$\scriptscriptstyle\pm0.005$} & 0.823\,{$\scriptscriptstyle\pm0.001$} & \underline{0.967}\,{$\scriptscriptstyle\pm0.002$} & 0.894\,{$\scriptscriptstyle\pm0.003$} & 0.605 & 0.552 & 0.892\,{$\scriptscriptstyle\pm0.001$} \\
Qwen2.5-1.5B & 0.733\,{$\scriptscriptstyle\pm0.002$} & 0.743\,{$\scriptscriptstyle\pm0.002$} & 0.330\,{$\scriptscriptstyle\pm0.003$} & \underline{0.912}\,{$\scriptscriptstyle\pm0.010$} & 0.787\,{$\scriptscriptstyle\pm0.003$} & 0.688\,{$\scriptscriptstyle\pm0.006$} & \textbf{0.919}\,{$\scriptscriptstyle\pm0.002$} & 0.628\,{$\scriptscriptstyle\pm0.017$} & 0.386 & 0.209 & 0.783\,{$\scriptscriptstyle\pm0.005$} \\
Qwen2.5-3B & 0.632\,{$\scriptscriptstyle\pm0.001$} & 0.731\,{$\scriptscriptstyle\pm0.002$} & 0.279\,{$\scriptscriptstyle\pm0.005$} & \underline{0.937}\,{$\scriptscriptstyle\pm0.002$} & 0.712\,{$\scriptscriptstyle\pm0.004$} & 0.625\,{$\scriptscriptstyle\pm0.013$} & \textbf{0.943}\,{$\scriptscriptstyle\pm0.001$} & 0.594\,{$\scriptscriptstyle\pm0.065$} & 0.351 & 0.203 & 0.758\,{$\scriptscriptstyle\pm0.012$} \\
Qwen2.5-7B & 0.837\,{$\scriptscriptstyle\pm0.000$} & \underline{0.855}\,{$\scriptscriptstyle\pm0.002$} & 0.372\,{$\scriptscriptstyle\pm0.011$} & \textbf{0.958}\,{$\scriptscriptstyle\pm0.011$} & 0.780\,{$\scriptscriptstyle\pm0.008$} & 0.704\,{$\scriptscriptstyle\pm0.006$} & 0.699\,{$\scriptscriptstyle\pm0.459$} & 0.759\,{$\scriptscriptstyle\pm0.043$} & 0.415 & 0.339 & 0.821\,{$\scriptscriptstyle\pm0.011$} \\
\midrule\midrule
\textbf{\#\,1st-place wins} & \textbf{0} & \textbf{2} & \textbf{0} & \textbf{19} & \textbf{0} & \textbf{0} & \textbf{21} & \textbf{0} & \textbf{0} & \textbf{0} & \textbf{0} \\
\#\,2nd-place wins & 2 & 9 & 0 & 20 & 0 & 0 & 7 & 0 & 0 & 0 & 4 \\
\end{longtable}\endgroup
\twocolumn

\onecolumn
\begingroup\fontsize{4.6pt}{5.6pt}\selectfont\setlength{\tabcolsep}{0.6pt}\renewcommand{\arraystretch}{1.1}
\setlength\LTleft{0pt}\setlength\LTright{0pt}\begin{longtable}{@{\extracolsep{\fill}}lccccccccccc}
\caption{Jaccard@5 ($\uparrow$): full per-pair results, seed mean\,$\pm$\,std. Best \textbf{bold}, second \underline{underlined}; grouped by source (donor) LLM.}\label{tab:pp-jaccard5}\\
\toprule
Target & \textbf{Term} & \textbf{Rel-dep} & \textbf{CKA} & \textbf{MATCHA} & \textbf{MATCHA-K} & \textbf{MATCHA-FL} & \textbf{MATCHA-avg} & \textbf{MATCHA-no FF} & \textbf{CmdV} & \textbf{HOT} & \textbf{MLP} \\
\midrule\endfirsthead
\multicolumn{12}{c}{\tablename\ \thetable{} -- continued}\\ \toprule
Target & \textbf{Term} & \textbf{Rel-dep} & \textbf{CKA} & \textbf{MATCHA} & \textbf{MATCHA-K} & \textbf{MATCHA-FL} & \textbf{MATCHA-avg} & \textbf{MATCHA-no FF} & \textbf{CmdV} & \textbf{HOT} & \textbf{MLP} \\
\midrule\endhead
\bottomrule\endfoot
\multicolumn{12}{l}{\textbf{Source: Llama-3.2-1B}} \\
Llama-3.2-3B & 0.271\,{$\scriptscriptstyle\pm0.000$} & 0.443\,{$\scriptscriptstyle\pm0.001$} & 0.320\,{$\scriptscriptstyle\pm0.002$} & 0.464\,{$\scriptscriptstyle\pm0.001$} & 0.415\,{$\scriptscriptstyle\pm0.005$} & 0.390\,{$\scriptscriptstyle\pm0.003$} & \underline{0.467}\,{$\scriptscriptstyle\pm0.000$} & \textbf{0.473}\,{$\scriptscriptstyle\pm0.006$} & 0.385 & 0.324 & 0.380\,{$\scriptscriptstyle\pm0.003$} \\
Llama-3.1-8B & 0.251\,{$\scriptscriptstyle\pm0.001$} & 0.435\,{$\scriptscriptstyle\pm0.001$} & 0.306\,{$\scriptscriptstyle\pm0.001$} & 0.455\,{$\scriptscriptstyle\pm0.002$} & 0.414\,{$\scriptscriptstyle\pm0.001$} & 0.371\,{$\scriptscriptstyle\pm0.002$} & \textbf{0.460}\,{$\scriptscriptstyle\pm0.002$} & \underline{0.460}\,{$\scriptscriptstyle\pm0.011$} & 0.369 & 0.326 & 0.385\,{$\scriptscriptstyle\pm0.004$} \\
Qwen2.5-1.5B & 0.299\,{$\scriptscriptstyle\pm0.001$} & 0.364\,{$\scriptscriptstyle\pm0.001$} & 0.177\,{$\scriptscriptstyle\pm0.002$} & \underline{0.380}\,{$\scriptscriptstyle\pm0.001$} & 0.345\,{$\scriptscriptstyle\pm0.000$} & 0.312\,{$\scriptscriptstyle\pm0.006$} & \textbf{0.384}\,{$\scriptscriptstyle\pm0.001$} & 0.341\,{$\scriptscriptstyle\pm0.013$} & 0.292 & 0.242 & 0.313\,{$\scriptscriptstyle\pm0.002$} \\
Qwen2.5-3B & 0.262\,{$\scriptscriptstyle\pm0.001$} & 0.352\,{$\scriptscriptstyle\pm0.000$} & 0.178\,{$\scriptscriptstyle\pm0.001$} & \underline{0.372}\,{$\scriptscriptstyle\pm0.002$} & 0.300\,{$\scriptscriptstyle\pm0.003$} & 0.292\,{$\scriptscriptstyle\pm0.005$} & \textbf{0.373}\,{$\scriptscriptstyle\pm0.002$} & 0.311\,{$\scriptscriptstyle\pm0.013$} & 0.268 & 0.236 & 0.298\,{$\scriptscriptstyle\pm0.008$} \\
Qwen2.5-7B & 0.292\,{$\scriptscriptstyle\pm0.002$} & 0.381\,{$\scriptscriptstyle\pm0.000$} & 0.217\,{$\scriptscriptstyle\pm0.001$} & \underline{0.398}\,{$\scriptscriptstyle\pm0.002$} & 0.317\,{$\scriptscriptstyle\pm0.007$} & 0.320\,{$\scriptscriptstyle\pm0.003$} & \textbf{0.399}\,{$\scriptscriptstyle\pm0.001$} & 0.371\,{$\scriptscriptstyle\pm0.008$} & 0.293 & 0.261 & 0.326\,{$\scriptscriptstyle\pm0.001$} \\
gemma-2-2b & 0.259\,{$\scriptscriptstyle\pm0.000$} & 0.391\,{$\scriptscriptstyle\pm0.001$} & 0.236\,{$\scriptscriptstyle\pm0.002$} & \textbf{0.406}\,{$\scriptscriptstyle\pm0.005$} & 0.369\,{$\scriptscriptstyle\pm0.004$} & 0.352\,{$\scriptscriptstyle\pm0.006$} & 0.319\,{$\scriptscriptstyle\pm0.106$} & \underline{0.399}\,{$\scriptscriptstyle\pm0.007$} & 0.313 & 0.250 & 0.360\,{$\scriptscriptstyle\pm0.002$} \\
\midrule
\multicolumn{12}{l}{\textbf{Source: Llama-3.2-3B}} \\
Llama-3.2-1B & 0.378\,{$\scriptscriptstyle\pm0.002$} & 0.492\,{$\scriptscriptstyle\pm0.001$} & 0.273\,{$\scriptscriptstyle\pm0.003$} & 0.497\,{$\scriptscriptstyle\pm0.002$} & 0.395\,{$\scriptscriptstyle\pm0.009$} & 0.389\,{$\scriptscriptstyle\pm0.007$} & \textbf{0.501}\,{$\scriptscriptstyle\pm0.002$} & \underline{0.498}\,{$\scriptscriptstyle\pm0.007$} & 0.478 & 0.313 & 0.351\,{$\scriptscriptstyle\pm0.001$} \\
Llama-3.1-8B & 0.461\,{$\scriptscriptstyle\pm0.000$} & 0.523\,{$\scriptscriptstyle\pm0.001$} & 0.272\,{$\scriptscriptstyle\pm0.007$} & 0.521\,{$\scriptscriptstyle\pm0.001$} & 0.455\,{$\scriptscriptstyle\pm0.003$} & 0.367\,{$\scriptscriptstyle\pm0.005$} & \underline{0.525}\,{$\scriptscriptstyle\pm0.001$} & \textbf{0.527}\,{$\scriptscriptstyle\pm0.001$} & 0.514 & 0.365 & 0.368\,{$\scriptscriptstyle\pm0.007$} \\
Qwen2.5-1.5B & 0.341\,{$\scriptscriptstyle\pm0.001$} & 0.341\,{$\scriptscriptstyle\pm0.001$} & 0.118\,{$\scriptscriptstyle\pm0.002$} & \textbf{0.391}\,{$\scriptscriptstyle\pm0.003$} & 0.334\,{$\scriptscriptstyle\pm0.003$} & 0.289\,{$\scriptscriptstyle\pm0.007$} & \underline{0.391}\,{$\scriptscriptstyle\pm0.002$} & 0.334\,{$\scriptscriptstyle\pm0.004$} & 0.317 & 0.239 & 0.294\,{$\scriptscriptstyle\pm0.000$} \\
Qwen2.5-3B & 0.330\,{$\scriptscriptstyle\pm0.001$} & 0.327\,{$\scriptscriptstyle\pm0.001$} & 0.150\,{$\scriptscriptstyle\pm0.001$} & \textbf{0.387}\,{$\scriptscriptstyle\pm0.001$} & 0.290\,{$\scriptscriptstyle\pm0.005$} & 0.256\,{$\scriptscriptstyle\pm0.003$} & \underline{0.387}\,{$\scriptscriptstyle\pm0.001$} & 0.350\,{$\scriptscriptstyle\pm0.007$} & 0.301 & 0.236 & 0.277\,{$\scriptscriptstyle\pm0.002$} \\
Qwen2.5-7B & 0.379\,{$\scriptscriptstyle\pm0.000$} & \underline{0.379}\,{$\scriptscriptstyle\pm0.000$} & 0.120\,{$\scriptscriptstyle\pm0.006$} & \textbf{0.424}\,{$\scriptscriptstyle\pm0.002$} & 0.281\,{$\scriptscriptstyle\pm0.004$} & 0.285\,{$\scriptscriptstyle\pm0.006$} & 0.357\,{$\scriptscriptstyle\pm0.116$} & 0.367\,{$\scriptscriptstyle\pm0.018$} & 0.337 & 0.264 & 0.307\,{$\scriptscriptstyle\pm0.006$} \\
gemma-2-2b & 0.416\,{$\scriptscriptstyle\pm0.001$} & \underline{0.420}\,{$\scriptscriptstyle\pm0.001$} & 0.186\,{$\scriptscriptstyle\pm0.003$} & \textbf{0.438}\,{$\scriptscriptstyle\pm0.007$} & 0.381\,{$\scriptscriptstyle\pm0.001$} & 0.343\,{$\scriptscriptstyle\pm0.005$} & 0.387\,{$\scriptscriptstyle\pm0.088$} & 0.417\,{$\scriptscriptstyle\pm0.002$} & 0.403 & 0.262 & 0.334\,{$\scriptscriptstyle\pm0.004$} \\
\midrule
\multicolumn{12}{l}{\textbf{Source: Llama-3.1-8B}} \\
Llama-3.2-1B & 0.354\,{$\scriptscriptstyle\pm0.001$} & 0.481\,{$\scriptscriptstyle\pm0.001$} & 0.286\,{$\scriptscriptstyle\pm0.003$} & 0.490\,{$\scriptscriptstyle\pm0.001$} & 0.397\,{$\scriptscriptstyle\pm0.007$} & 0.393\,{$\scriptscriptstyle\pm0.003$} & \textbf{0.495}\,{$\scriptscriptstyle\pm0.001$} & \underline{0.492}\,{$\scriptscriptstyle\pm0.002$} & 0.476 & 0.304 & 0.346\,{$\scriptscriptstyle\pm0.006$} \\
Llama-3.2-3B & 0.502\,{$\scriptscriptstyle\pm0.000$} & 0.543\,{$\scriptscriptstyle\pm0.001$} & 0.304\,{$\scriptscriptstyle\pm0.002$} & 0.530\,{$\scriptscriptstyle\pm0.001$} & 0.449\,{$\scriptscriptstyle\pm0.003$} & 0.378\,{$\scriptscriptstyle\pm0.009$} & 0.536\,{$\scriptscriptstyle\pm0.001$} & \underline{0.544}\,{$\scriptscriptstyle\pm0.001$} & \textbf{0.546} & 0.344 & 0.344\,{$\scriptscriptstyle\pm0.006$} \\
Qwen2.5-1.5B & 0.306\,{$\scriptscriptstyle\pm0.001$} & 0.341\,{$\scriptscriptstyle\pm0.001$} & 0.143\,{$\scriptscriptstyle\pm0.006$} & \textbf{0.395}\,{$\scriptscriptstyle\pm0.006$} & 0.329\,{$\scriptscriptstyle\pm0.008$} & 0.305\,{$\scriptscriptstyle\pm0.002$} & \underline{0.395}\,{$\scriptscriptstyle\pm0.005$} & 0.356\,{$\scriptscriptstyle\pm0.006$} & 0.333 & 0.236 & 0.285\,{$\scriptscriptstyle\pm0.002$} \\
Qwen2.5-3B & 0.334\,{$\scriptscriptstyle\pm0.001$} & 0.330\,{$\scriptscriptstyle\pm0.001$} & 0.150\,{$\scriptscriptstyle\pm0.010$} & \textbf{0.394}\,{$\scriptscriptstyle\pm0.003$} & 0.311\,{$\scriptscriptstyle\pm0.003$} & 0.275\,{$\scriptscriptstyle\pm0.009$} & \underline{0.392}\,{$\scriptscriptstyle\pm0.004$} & 0.337\,{$\scriptscriptstyle\pm0.015$} & 0.315 & 0.230 & 0.266\,{$\scriptscriptstyle\pm0.005$} \\
Qwen2.5-7B & 0.352\,{$\scriptscriptstyle\pm0.001$} & 0.389\,{$\scriptscriptstyle\pm0.001$} & 0.221\,{$\scriptscriptstyle\pm0.005$} & \textbf{0.440}\,{$\scriptscriptstyle\pm0.003$} & 0.321\,{$\scriptscriptstyle\pm0.002$} & 0.309\,{$\scriptscriptstyle\pm0.001$} & 0.352\,{$\scriptscriptstyle\pm0.151$} & \underline{0.404}\,{$\scriptscriptstyle\pm0.002$} & 0.358 & 0.259 & 0.295\,{$\scriptscriptstyle\pm0.002$} \\
gemma-2-2b & 0.389\,{$\scriptscriptstyle\pm0.001$} & 0.426\,{$\scriptscriptstyle\pm0.001$} & 0.219\,{$\scriptscriptstyle\pm0.001$} & \textbf{0.453}\,{$\scriptscriptstyle\pm0.001$} & 0.393\,{$\scriptscriptstyle\pm0.001$} & 0.356\,{$\scriptscriptstyle\pm0.004$} & 0.316\,{$\scriptscriptstyle\pm0.116$} & \underline{0.431}\,{$\scriptscriptstyle\pm0.001$} & 0.416 & 0.250 & 0.323\,{$\scriptscriptstyle\pm0.002$} \\
\midrule
\multicolumn{12}{l}{\textbf{Source: Qwen2.5-1.5B}} \\
Llama-3.2-1B & 0.387\,{$\scriptscriptstyle\pm0.001$} & 0.366\,{$\scriptscriptstyle\pm0.002$} & 0.294\,{$\scriptscriptstyle\pm0.003$} & \textbf{0.414}\,{$\scriptscriptstyle\pm0.005$} & 0.392\,{$\scriptscriptstyle\pm0.003$} & 0.365\,{$\scriptscriptstyle\pm0.006$} & \underline{0.409}\,{$\scriptscriptstyle\pm0.011$} & 0.359\,{$\scriptscriptstyle\pm0.008$} & 0.269 & 0.253 & 0.368\,{$\scriptscriptstyle\pm0.003$} \\
Llama-3.2-3B & 0.337\,{$\scriptscriptstyle\pm0.000$} & 0.336\,{$\scriptscriptstyle\pm0.001$} & 0.244\,{$\scriptscriptstyle\pm0.000$} & \textbf{0.394}\,{$\scriptscriptstyle\pm0.001$} & \underline{0.360}\,{$\scriptscriptstyle\pm0.002$} & 0.342\,{$\scriptscriptstyle\pm0.005$} & 0.346\,{$\scriptscriptstyle\pm0.067$} & 0.333\,{$\scriptscriptstyle\pm0.011$} & 0.245 & 0.246 & 0.350\,{$\scriptscriptstyle\pm0.001$} \\
Llama-3.1-8B & 0.294\,{$\scriptscriptstyle\pm0.001$} & 0.326\,{$\scriptscriptstyle\pm0.001$} & 0.278\,{$\scriptscriptstyle\pm0.001$} & \underline{0.388}\,{$\scriptscriptstyle\pm0.001$} & 0.366\,{$\scriptscriptstyle\pm0.002$} & 0.349\,{$\scriptscriptstyle\pm0.006$} & \textbf{0.389}\,{$\scriptscriptstyle\pm0.001$} & 0.355\,{$\scriptscriptstyle\pm0.001$} & 0.237 & 0.250 & 0.369\,{$\scriptscriptstyle\pm0.002$} \\
Qwen2.5-3B & 0.357\,{$\scriptscriptstyle\pm0.001$} & 0.398\,{$\scriptscriptstyle\pm0.002$} & 0.202\,{$\scriptscriptstyle\pm0.003$} & \textbf{0.404}\,{$\scriptscriptstyle\pm0.002$} & 0.353\,{$\scriptscriptstyle\pm0.002$} & 0.325\,{$\scriptscriptstyle\pm0.009$} & \underline{0.403}\,{$\scriptscriptstyle\pm0.002$} & 0.300\,{$\scriptscriptstyle\pm0.006$} & 0.293 & 0.250 & 0.329\,{$\scriptscriptstyle\pm0.001$} \\
Qwen2.5-7B & 0.403\,{$\scriptscriptstyle\pm0.001$} & 0.408\,{$\scriptscriptstyle\pm0.001$} & 0.220\,{$\scriptscriptstyle\pm0.001$} & \textbf{0.421}\,{$\scriptscriptstyle\pm0.002$} & 0.370\,{$\scriptscriptstyle\pm0.002$} & 0.348\,{$\scriptscriptstyle\pm0.007$} & \underline{0.420}\,{$\scriptscriptstyle\pm0.001$} & 0.377\,{$\scriptscriptstyle\pm0.029$} & 0.300 & 0.258 & 0.347\,{$\scriptscriptstyle\pm0.002$} \\
gemma-2-2b & 0.335\,{$\scriptscriptstyle\pm0.002$} & 0.315\,{$\scriptscriptstyle\pm0.001$} & 0.241\,{$\scriptscriptstyle\pm0.000$} & \textbf{0.376}\,{$\scriptscriptstyle\pm0.003$} & \underline{0.355}\,{$\scriptscriptstyle\pm0.003$} & 0.334\,{$\scriptscriptstyle\pm0.005$} & 0.342\,{$\scriptscriptstyle\pm0.059$} & 0.306\,{$\scriptscriptstyle\pm0.017$} & 0.216 & 0.224 & 0.354\,{$\scriptscriptstyle\pm0.003$} \\
\midrule
\multicolumn{12}{l}{\textbf{Source: Qwen2.5-3B}} \\
Llama-3.2-1B & 0.376\,{$\scriptscriptstyle\pm0.001$} & 0.374\,{$\scriptscriptstyle\pm0.001$} & 0.305\,{$\scriptscriptstyle\pm0.003$} & \textbf{0.423}\,{$\scriptscriptstyle\pm0.001$} & 0.399\,{$\scriptscriptstyle\pm0.002$} & 0.367\,{$\scriptscriptstyle\pm0.005$} & \underline{0.417}\,{$\scriptscriptstyle\pm0.015$} & 0.368\,{$\scriptscriptstyle\pm0.003$} & 0.288 & 0.244 & 0.356\,{$\scriptscriptstyle\pm0.003$} \\
Llama-3.2-3B & 0.368\,{$\scriptscriptstyle\pm0.001$} & 0.320\,{$\scriptscriptstyle\pm0.001$} & 0.313\,{$\scriptscriptstyle\pm0.000$} & \underline{0.387}\,{$\scriptscriptstyle\pm0.015$} & 0.378\,{$\scriptscriptstyle\pm0.003$} & 0.354\,{$\scriptscriptstyle\pm0.010$} & \textbf{0.388}\,{$\scriptscriptstyle\pm0.013$} & 0.366\,{$\scriptscriptstyle\pm0.006$} & 0.256 & 0.238 & 0.340\,{$\scriptscriptstyle\pm0.005$} \\
Llama-3.1-8B & 0.360\,{$\scriptscriptstyle\pm0.000$} & 0.329\,{$\scriptscriptstyle\pm0.001$} & 0.304\,{$\scriptscriptstyle\pm0.003$} & \underline{0.411}\,{$\scriptscriptstyle\pm0.001$} & 0.389\,{$\scriptscriptstyle\pm0.003$} & 0.360\,{$\scriptscriptstyle\pm0.007$} & \textbf{0.414}\,{$\scriptscriptstyle\pm0.001$} & 0.367\,{$\scriptscriptstyle\pm0.004$} & 0.258 & 0.255 & 0.367\,{$\scriptscriptstyle\pm0.004$} \\
Qwen2.5-1.5B & 0.406\,{$\scriptscriptstyle\pm0.001$} & 0.405\,{$\scriptscriptstyle\pm0.001$} & 0.240\,{$\scriptscriptstyle\pm0.002$} & \textbf{0.419}\,{$\scriptscriptstyle\pm0.013$} & 0.390\,{$\scriptscriptstyle\pm0.001$} & 0.356\,{$\scriptscriptstyle\pm0.004$} & \underline{0.414}\,{$\scriptscriptstyle\pm0.022$} & 0.340\,{$\scriptscriptstyle\pm0.012$} & 0.325 & 0.246 & 0.339\,{$\scriptscriptstyle\pm0.003$} \\
Qwen2.5-7B & \underline{0.421}\,{$\scriptscriptstyle\pm0.001$} & 0.417\,{$\scriptscriptstyle\pm0.001$} & 0.278\,{$\scriptscriptstyle\pm0.001$} & \textbf{0.446}\,{$\scriptscriptstyle\pm0.004$} & 0.392\,{$\scriptscriptstyle\pm0.002$} & 0.359\,{$\scriptscriptstyle\pm0.004$} & 0.367\,{$\scriptscriptstyle\pm0.144$} & 0.396\,{$\scriptscriptstyle\pm0.027$} & 0.342 & 0.267 & 0.350\,{$\scriptscriptstyle\pm0.014$} \\
gemma-2-2b & 0.361\,{$\scriptscriptstyle\pm0.001$} & 0.309\,{$\scriptscriptstyle\pm0.001$} & 0.248\,{$\scriptscriptstyle\pm0.002$} & \textbf{0.387}\,{$\scriptscriptstyle\pm0.003$} & \underline{0.362}\,{$\scriptscriptstyle\pm0.003$} & 0.335\,{$\scriptscriptstyle\pm0.009$} & 0.267\,{$\scriptscriptstyle\pm0.041$} & 0.319\,{$\scriptscriptstyle\pm0.029$} & 0.238 & 0.210 & 0.341\,{$\scriptscriptstyle\pm0.002$} \\
\midrule
\multicolumn{12}{l}{\textbf{Source: Qwen2.5-7B}} \\
Llama-3.2-1B & 0.406\,{$\scriptscriptstyle\pm0.000$} & 0.411\,{$\scriptscriptstyle\pm0.002$} & 0.292\,{$\scriptscriptstyle\pm0.007$} & \underline{0.434}\,{$\scriptscriptstyle\pm0.001$} & 0.403\,{$\scriptscriptstyle\pm0.005$} & 0.386\,{$\scriptscriptstyle\pm0.006$} & \textbf{0.439}\,{$\scriptscriptstyle\pm0.001$} & 0.400\,{$\scriptscriptstyle\pm0.010$} & 0.365 & 0.257 & 0.342\,{$\scriptscriptstyle\pm0.004$} \\
Llama-3.2-3B & 0.390\,{$\scriptscriptstyle\pm0.001$} & 0.388\,{$\scriptscriptstyle\pm0.001$} & 0.306\,{$\scriptscriptstyle\pm0.005$} & \underline{0.431}\,{$\scriptscriptstyle\pm0.001$} & 0.397\,{$\scriptscriptstyle\pm0.003$} & 0.378\,{$\scriptscriptstyle\pm0.009$} & \textbf{0.432}\,{$\scriptscriptstyle\pm0.001$} & 0.410\,{$\scriptscriptstyle\pm0.007$} & 0.352 & 0.256 & 0.339\,{$\scriptscriptstyle\pm0.002$} \\
Llama-3.1-8B & 0.351\,{$\scriptscriptstyle\pm0.000$} & 0.391\,{$\scriptscriptstyle\pm0.001$} & 0.270\,{$\scriptscriptstyle\pm0.003$} & \underline{0.447}\,{$\scriptscriptstyle\pm0.002$} & 0.415\,{$\scriptscriptstyle\pm0.003$} & 0.390\,{$\scriptscriptstyle\pm0.005$} & \textbf{0.447}\,{$\scriptscriptstyle\pm0.001$} & 0.415\,{$\scriptscriptstyle\pm0.009$} & 0.346 & 0.264 & 0.363\,{$\scriptscriptstyle\pm0.007$} \\
Qwen2.5-1.5B & 0.411\,{$\scriptscriptstyle\pm0.001$} & \underline{0.411}\,{$\scriptscriptstyle\pm0.001$} & 0.217\,{$\scriptscriptstyle\pm0.002$} & \textbf{0.429}\,{$\scriptscriptstyle\pm0.002$} & 0.390\,{$\scriptscriptstyle\pm0.002$} & 0.351\,{$\scriptscriptstyle\pm0.005$} & 0.358\,{$\scriptscriptstyle\pm0.121$} & 0.403\,{$\scriptscriptstyle\pm0.007$} & 0.373 & 0.256 & 0.314\,{$\scriptscriptstyle\pm0.013$} \\
Qwen2.5-3B & 0.385\,{$\scriptscriptstyle\pm0.001$} & 0.431\,{$\scriptscriptstyle\pm0.001$} & 0.203\,{$\scriptscriptstyle\pm0.000$} & \underline{0.444}\,{$\scriptscriptstyle\pm0.003$} & 0.362\,{$\scriptscriptstyle\pm0.002$} & 0.344\,{$\scriptscriptstyle\pm0.005$} & \textbf{0.445}\,{$\scriptscriptstyle\pm0.002$} & 0.389\,{$\scriptscriptstyle\pm0.022$} & 0.386 & 0.264 & 0.312\,{$\scriptscriptstyle\pm0.005$} \\
gemma-2-2b & \underline{0.386}\,{$\scriptscriptstyle\pm0.001$} & 0.370\,{$\scriptscriptstyle\pm0.001$} & 0.211\,{$\scriptscriptstyle\pm0.004$} & \textbf{0.416}\,{$\scriptscriptstyle\pm0.006$} & 0.385\,{$\scriptscriptstyle\pm0.001$} & 0.358\,{$\scriptscriptstyle\pm0.006$} & 0.247\,{$\scriptscriptstyle\pm0.028$} & 0.377\,{$\scriptscriptstyle\pm0.005$} & 0.323 & 0.232 & 0.336\,{$\scriptscriptstyle\pm0.004$} \\
\midrule
\multicolumn{12}{l}{\textbf{Source: gemma-2-2b}} \\
Llama-3.2-1B & 0.359\,{$\scriptscriptstyle\pm0.002$} & 0.410\,{$\scriptscriptstyle\pm0.000$} & 0.247\,{$\scriptscriptstyle\pm0.009$} & \textbf{0.444}\,{$\scriptscriptstyle\pm0.001$} & 0.401\,{$\scriptscriptstyle\pm0.001$} & 0.364\,{$\scriptscriptstyle\pm0.004$} & \underline{0.443}\,{$\scriptscriptstyle\pm0.001$} & 0.410\,{$\scriptscriptstyle\pm0.001$} & 0.363 & 0.260 & 0.345\,{$\scriptscriptstyle\pm0.002$} \\
Llama-3.2-3B & 0.375\,{$\scriptscriptstyle\pm0.000$} & 0.392\,{$\scriptscriptstyle\pm0.000$} & 0.258\,{$\scriptscriptstyle\pm0.001$} & \textbf{0.425}\,{$\scriptscriptstyle\pm0.001$} & 0.391\,{$\scriptscriptstyle\pm0.002$} & 0.346\,{$\scriptscriptstyle\pm0.002$} & \underline{0.425}\,{$\scriptscriptstyle\pm0.001$} & 0.384\,{$\scriptscriptstyle\pm0.015$} & 0.362 & 0.269 & 0.342\,{$\scriptscriptstyle\pm0.002$} \\
Llama-3.1-8B & 0.335\,{$\scriptscriptstyle\pm0.000$} & 0.405\,{$\scriptscriptstyle\pm0.001$} & 0.232\,{$\scriptscriptstyle\pm0.006$} & \underline{0.437}\,{$\scriptscriptstyle\pm0.000$} & 0.367\,{$\scriptscriptstyle\pm0.002$} & 0.374\,{$\scriptscriptstyle\pm0.003$} & \textbf{0.438}\,{$\scriptscriptstyle\pm0.002$} & 0.421\,{$\scriptscriptstyle\pm0.003$} & 0.362 & 0.280 & 0.365\,{$\scriptscriptstyle\pm0.002$} \\
Qwen2.5-1.5B & 0.314\,{$\scriptscriptstyle\pm0.001$} & 0.317\,{$\scriptscriptstyle\pm0.001$} & 0.155\,{$\scriptscriptstyle\pm0.002$} & \underline{0.367}\,{$\scriptscriptstyle\pm0.003$} & 0.323\,{$\scriptscriptstyle\pm0.001$} & 0.283\,{$\scriptscriptstyle\pm0.003$} & \textbf{0.368}\,{$\scriptscriptstyle\pm0.003$} & 0.275\,{$\scriptscriptstyle\pm0.002$} & 0.275 & 0.225 & 0.296\,{$\scriptscriptstyle\pm0.002$} \\
Qwen2.5-3B & 0.281\,{$\scriptscriptstyle\pm0.000$} & 0.303\,{$\scriptscriptstyle\pm0.001$} & 0.136\,{$\scriptscriptstyle\pm0.002$} & \textbf{0.351}\,{$\scriptscriptstyle\pm0.001$} & 0.286\,{$\scriptscriptstyle\pm0.002$} & 0.260\,{$\scriptscriptstyle\pm0.005$} & \underline{0.351}\,{$\scriptscriptstyle\pm0.001$} & 0.258\,{$\scriptscriptstyle\pm0.031$} & 0.253 & 0.217 & 0.274\,{$\scriptscriptstyle\pm0.002$} \\
Qwen2.5-7B & 0.342\,{$\scriptscriptstyle\pm0.001$} & \underline{0.346}\,{$\scriptscriptstyle\pm0.001$} & 0.171\,{$\scriptscriptstyle\pm0.007$} & \textbf{0.386}\,{$\scriptscriptstyle\pm0.002$} & 0.321\,{$\scriptscriptstyle\pm0.001$} & 0.283\,{$\scriptscriptstyle\pm0.003$} & 0.328\,{$\scriptscriptstyle\pm0.105$} & 0.332\,{$\scriptscriptstyle\pm0.012$} & 0.289 & 0.241 & 0.307\,{$\scriptscriptstyle\pm0.002$} \\
\midrule\midrule
\textbf{\#\,1st-place wins} & \textbf{0} & \textbf{0} & \textbf{0} & \textbf{24} & \textbf{0} & \textbf{0} & \textbf{15} & \textbf{2} & \textbf{1} & \textbf{0} & \textbf{0} \\
\#\,2nd-place wins & 2 & 4 & 0 & 12 & 3 & 0 & 14 & 7 & 0 & 0 & 0 \\
\end{longtable}\endgroup
\twocolumn

\subsubsection{MSE per Model Pair}
\label{sec:appendix-mse}
In \Cref{tab:mse-pp}, we report the pairwise breakdown of the results.


\onecolumn
\begingroup\fontsize{4.8pt}{5.8pt}\selectfont\setlength{\tabcolsep}{1pt}\renewcommand{\arraystretch}{1.1}
\setlength\LTleft{0pt}\setlength\LTright{0pt}\begin{longtable}{@{\extracolsep{\fill}}lcccccccccc}
\caption{MSE loss ($\downarrow$): full per-pair results, seed mean\,$\pm$\,std. Lowest \textbf{bold}, second \underline{underlined}; grouped by source (donor) LLM.}\label{tab:mse-pp}\\
\toprule
Target & \textbf{Terminal} & \textbf{Relative depth} & \textbf{CKA} & \textbf{MATCHA} & \textbf{MATCHA-K} & \textbf{MATCHA-FL} & \textbf{MATCHA-avg} & \textbf{MATCHA-no FF} & \textbf{CommandV} & \textbf{HOT} \\
\midrule\endfirsthead
\multicolumn{11}{c}{\tablename\ \thetable{} -- continued}\\ \toprule
Target & \textbf{Terminal} & \textbf{Relative depth} & \textbf{CKA} & \textbf{MATCHA} & \textbf{MATCHA-K} & \textbf{MATCHA-FL} & \textbf{MATCHA-avg} & \textbf{MATCHA-no FF} & \textbf{CommandV} & \textbf{HOT} \\
\midrule\endhead
\bottomrule\endfoot
\multicolumn{11}{l}{\textbf{Source: Llama-3.2-1B}} \\
Llama-3.2-3B & 0.077\,{$\scriptscriptstyle\pm0.000$} & 0.073\,{$\scriptscriptstyle\pm0.000$} & 0.120\,{$\scriptscriptstyle\pm0.000$} & \underline{0.071}\,{$\scriptscriptstyle\pm0.000$} & 0.079\,{$\scriptscriptstyle\pm0.001$} & 0.072\,{$\scriptscriptstyle\pm0.001$} & \textbf{0.070}\,{$\scriptscriptstyle\pm0.000$} & 0.077\,{$\scriptscriptstyle\pm0.000$} & 0.516 & 0.145 \\
Llama-3.1-8B & 0.090\,{$\scriptscriptstyle\pm0.000$} & 0.113\,{$\scriptscriptstyle\pm0.000$} & 0.195\,{$\scriptscriptstyle\pm0.002$} & \underline{0.085}\,{$\scriptscriptstyle\pm0.000$} & 0.090\,{$\scriptscriptstyle\pm0.000$} & 0.088\,{$\scriptscriptstyle\pm0.001$} & \textbf{0.083}\,{$\scriptscriptstyle\pm0.000$} & 0.098\,{$\scriptscriptstyle\pm0.000$} & 0.627 & 0.180 \\
Qwen2.5-1.5B & 0.576\,{$\scriptscriptstyle\pm0.000$} & 0.564\,{$\scriptscriptstyle\pm0.001$} & 1.576\,{$\scriptscriptstyle\pm0.020$} & \underline{0.451}\,{$\scriptscriptstyle\pm0.002$} & 0.501\,{$\scriptscriptstyle\pm0.002$} & 0.463\,{$\scriptscriptstyle\pm0.003$} & \textbf{0.446}\,{$\scriptscriptstyle\pm0.002$} & 0.611\,{$\scriptscriptstyle\pm0.021$} & 16.772 & 5.424 \\
Qwen2.5-3B & 0.932\,{$\scriptscriptstyle\pm0.001$} & 1.481\,{$\scriptscriptstyle\pm0.000$} & 1.655\,{$\scriptscriptstyle\pm0.016$} & \underline{0.724}\,{$\scriptscriptstyle\pm0.002$} & 0.797\,{$\scriptscriptstyle\pm0.005$} & 0.738\,{$\scriptscriptstyle\pm0.001$} & \textbf{0.718}\,{$\scriptscriptstyle\pm0.002$} & 0.953\,{$\scriptscriptstyle\pm0.027$} & 49.502 & 6.687 \\
Qwen2.5-7B & 1.214\,{$\scriptscriptstyle\pm0.000$} & 1.278\,{$\scriptscriptstyle\pm0.001$} & 3.237\,{$\scriptscriptstyle\pm0.055$} & \underline{1.113}\,{$\scriptscriptstyle\pm0.002$} & 1.269\,{$\scriptscriptstyle\pm0.011$} & 1.149\,{$\scriptscriptstyle\pm0.004$} & \textbf{1.104}\,{$\scriptscriptstyle\pm0.002$} & 1.370\,{$\scriptscriptstyle\pm0.008$} & 35.587 & 5.432 \\
gemma-2-2b & 9.008\,{$\scriptscriptstyle\pm0.012$} & 9.049\,{$\scriptscriptstyle\pm0.007$} & 19.246\,{$\scriptscriptstyle\pm0.515$} & \underline{7.825}\,{$\scriptscriptstyle\pm0.022$} & 8.293\,{$\scriptscriptstyle\pm0.038$} & \textbf{7.731}\,{$\scriptscriptstyle\pm0.045$} & 16.850\,{$\scriptscriptstyle\pm8.986$} & 9.206\,{$\scriptscriptstyle\pm0.058$} & 508.487 & 35.561 \\
\midrule
\multicolumn{11}{l}{\textbf{Source: Llama-3.2-3B}} \\
Llama-3.2-1B & 0.122\,{$\scriptscriptstyle\pm0.000$} & 0.122\,{$\scriptscriptstyle\pm0.000$} & 0.304\,{$\scriptscriptstyle\pm0.015$} & 0.121\,{$\scriptscriptstyle\pm0.001$} & 0.132\,{$\scriptscriptstyle\pm0.001$} & \underline{0.120}\,{$\scriptscriptstyle\pm0.001$} & \textbf{0.120}\,{$\scriptscriptstyle\pm0.001$} & 0.130\,{$\scriptscriptstyle\pm0.000$} & 0.245 & 0.272 \\
Llama-3.1-8B & 0.079\,{$\scriptscriptstyle\pm0.000$} & \underline{0.078}\,{$\scriptscriptstyle\pm0.000$} & 0.219\,{$\scriptscriptstyle\pm0.014$} & 0.079\,{$\scriptscriptstyle\pm0.000$} & 0.088\,{$\scriptscriptstyle\pm0.001$} & 0.091\,{$\scriptscriptstyle\pm0.002$} & \textbf{0.078}\,{$\scriptscriptstyle\pm0.000$} & 0.092\,{$\scriptscriptstyle\pm0.000$} & 0.093 & 0.170 \\
Qwen2.5-1.5B & 0.581\,{$\scriptscriptstyle\pm0.001$} & 0.583\,{$\scriptscriptstyle\pm0.000$} & 2.586\,{$\scriptscriptstyle\pm0.047$} & \textbf{0.488}\,{$\scriptscriptstyle\pm0.013$} & 0.539\,{$\scriptscriptstyle\pm0.004$} & 0.533\,{$\scriptscriptstyle\pm0.010$} & \underline{0.491}\,{$\scriptscriptstyle\pm0.012$} & 0.645\,{$\scriptscriptstyle\pm0.029$} & 2.699 & 5.445 \\
Qwen2.5-3B & 0.865\,{$\scriptscriptstyle\pm0.001$} & 0.879\,{$\scriptscriptstyle\pm0.000$} & 1.790\,{$\scriptscriptstyle\pm0.010$} & \textbf{0.755}\,{$\scriptscriptstyle\pm0.003$} & 0.806\,{$\scriptscriptstyle\pm0.009$} & 0.833\,{$\scriptscriptstyle\pm0.008$} & \underline{0.756}\,{$\scriptscriptstyle\pm0.001$} & 0.975\,{$\scriptscriptstyle\pm0.022$} & 3.065 & 6.679 \\
Qwen2.5-7B & \underline{1.214}\,{$\scriptscriptstyle\pm0.001$} & 1.217\,{$\scriptscriptstyle\pm0.001$} & 5.685\,{$\scriptscriptstyle\pm0.648$} & \textbf{1.142}\,{$\scriptscriptstyle\pm0.006$} & 1.369\,{$\scriptscriptstyle\pm0.005$} & 1.268\,{$\scriptscriptstyle\pm0.008$} & 2.206\,{$\scriptscriptstyle\pm1.844$} & 1.374\,{$\scriptscriptstyle\pm0.028$} & 5.013 & 5.453 \\
gemma-2-2b & 8.190\,{$\scriptscriptstyle\pm0.010$} & 8.136\,{$\scriptscriptstyle\pm0.008$} & 21.839\,{$\scriptscriptstyle\pm0.468$} & \textbf{7.529}\,{$\scriptscriptstyle\pm0.041$} & 8.260\,{$\scriptscriptstyle\pm0.030$} & \underline{7.975}\,{$\scriptscriptstyle\pm0.131$} & 14.003\,{$\scriptscriptstyle\pm11.374$} & 8.547\,{$\scriptscriptstyle\pm0.050$} & 25.640 & 35.658 \\
\midrule
\multicolumn{11}{l}{\textbf{Source: Llama-3.1-8B}} \\
Llama-3.2-1B & 0.123\,{$\scriptscriptstyle\pm0.000$} & 0.121\,{$\scriptscriptstyle\pm0.000$} & 0.284\,{$\scriptscriptstyle\pm0.023$} & \underline{0.117}\,{$\scriptscriptstyle\pm0.000$} & 0.129\,{$\scriptscriptstyle\pm0.000$} & 0.118\,{$\scriptscriptstyle\pm0.001$} & \textbf{0.117}\,{$\scriptscriptstyle\pm0.001$} & 0.125\,{$\scriptscriptstyle\pm0.000$} & 0.243 & 0.270 \\
Llama-3.2-3B & 0.066\,{$\scriptscriptstyle\pm0.000$} & \underline{0.064}\,{$\scriptscriptstyle\pm0.000$} & 0.145\,{$\scriptscriptstyle\pm0.004$} & 0.065\,{$\scriptscriptstyle\pm0.000$} & 0.073\,{$\scriptscriptstyle\pm0.001$} & 0.070\,{$\scriptscriptstyle\pm0.001$} & \textbf{0.064}\,{$\scriptscriptstyle\pm0.000$} & 0.069\,{$\scriptscriptstyle\pm0.000$} & 0.081 & 0.133 \\
Qwen2.5-1.5B & 0.556\,{$\scriptscriptstyle\pm0.001$} & 0.550\,{$\scriptscriptstyle\pm0.001$} & 1.774\,{$\scriptscriptstyle\pm0.074$} & \textbf{0.466}\,{$\scriptscriptstyle\pm0.002$} & 0.540\,{$\scriptscriptstyle\pm0.005$} & 0.486\,{$\scriptscriptstyle\pm0.007$} & \underline{0.470}\,{$\scriptscriptstyle\pm0.001$} & 0.565\,{$\scriptscriptstyle\pm0.028$} & 0.976 & 5.413 \\
Qwen2.5-3B & 0.812\,{$\scriptscriptstyle\pm0.002$} & 0.818\,{$\scriptscriptstyle\pm0.002$} & 1.906\,{$\scriptscriptstyle\pm0.133$} & \textbf{0.722}\,{$\scriptscriptstyle\pm0.002$} & 0.804\,{$\scriptscriptstyle\pm0.010$} & 0.773\,{$\scriptscriptstyle\pm0.015$} & \underline{0.725}\,{$\scriptscriptstyle\pm0.001$} & 0.855\,{$\scriptscriptstyle\pm0.022$} & 5.939 & 6.559 \\
Qwen2.5-7B & 1.150\,{$\scriptscriptstyle\pm0.001$} & \underline{1.139}\,{$\scriptscriptstyle\pm0.001$} & 2.770\,{$\scriptscriptstyle\pm0.186$} & \textbf{1.116}\,{$\scriptscriptstyle\pm0.008$} & 1.271\,{$\scriptscriptstyle\pm0.013$} & 1.162\,{$\scriptscriptstyle\pm0.005$} & 3.680\,{$\scriptscriptstyle\pm4.448$} & 1.251\,{$\scriptscriptstyle\pm0.007$} & 1.783 & 5.373 \\
gemma-2-2b & 7.940\,{$\scriptscriptstyle\pm0.013$} & 7.773\,{$\scriptscriptstyle\pm0.004$} & 19.424\,{$\scriptscriptstyle\pm0.160$} & \textbf{7.418}\,{$\scriptscriptstyle\pm0.020$} & 8.002\,{$\scriptscriptstyle\pm0.012$} & \underline{7.618}\,{$\scriptscriptstyle\pm0.096$} & 22.861\,{$\scriptscriptstyle\pm13.740$} & 8.073\,{$\scriptscriptstyle\pm0.013$} & 24.952 & 35.492 \\
\midrule
\multicolumn{11}{l}{\textbf{Source: Qwen2.5-1.5B}} \\
Llama-3.2-1B & 0.126\,{$\scriptscriptstyle\pm0.000$} & \textbf{0.119}\,{$\scriptscriptstyle\pm0.000$} & 0.245\,{$\scriptscriptstyle\pm0.006$} & 0.125\,{$\scriptscriptstyle\pm0.001$} & \underline{0.123}\,{$\scriptscriptstyle\pm0.000$} & 0.125\,{$\scriptscriptstyle\pm0.001$} & 0.134\,{$\scriptscriptstyle\pm0.018$} & 0.138\,{$\scriptscriptstyle\pm0.000$} & 0.275 & 0.303 \\
Llama-3.2-3B & \textbf{0.077}\,{$\scriptscriptstyle\pm0.000$} & \underline{0.077}\,{$\scriptscriptstyle\pm0.000$} & 0.107\,{$\scriptscriptstyle\pm0.001$} & 0.078\,{$\scriptscriptstyle\pm0.000$} & 0.078\,{$\scriptscriptstyle\pm0.000$} & 0.077\,{$\scriptscriptstyle\pm0.000$} & 0.113\,{$\scriptscriptstyle\pm0.035$} & 0.085\,{$\scriptscriptstyle\pm0.000$} & 0.240 & 0.161 \\
Llama-3.1-8B & 0.090\,{$\scriptscriptstyle\pm0.000$} & \underline{0.090}\,{$\scriptscriptstyle\pm0.000$} & 0.194\,{$\scriptscriptstyle\pm0.003$} & 0.090\,{$\scriptscriptstyle\pm0.000$} & 0.093\,{$\scriptscriptstyle\pm0.000$} & 0.092\,{$\scriptscriptstyle\pm0.001$} & \textbf{0.089}\,{$\scriptscriptstyle\pm0.000$} & 0.107\,{$\scriptscriptstyle\pm0.000$} & 0.223 & 0.201 \\
Qwen2.5-3B & 0.660\,{$\scriptscriptstyle\pm0.000$} & 0.669\,{$\scriptscriptstyle\pm0.001$} & 1.711\,{$\scriptscriptstyle\pm0.026$} & \underline{0.644}\,{$\scriptscriptstyle\pm0.003$} & 0.718\,{$\scriptscriptstyle\pm0.007$} & 0.689\,{$\scriptscriptstyle\pm0.007$} & \textbf{0.640}\,{$\scriptscriptstyle\pm0.001$} & 0.968\,{$\scriptscriptstyle\pm0.044$} & 2.678 & 5.955 \\
Qwen2.5-7B & 1.077\,{$\scriptscriptstyle\pm0.001$} & 1.051\,{$\scriptscriptstyle\pm0.002$} & 1.999\,{$\scriptscriptstyle\pm0.022$} & \underline{1.050}\,{$\scriptscriptstyle\pm0.002$} & 1.136\,{$\scriptscriptstyle\pm0.002$} & 1.099\,{$\scriptscriptstyle\pm0.014$} & \textbf{1.044}\,{$\scriptscriptstyle\pm0.001$} & 1.383\,{$\scriptscriptstyle\pm0.025$} & 4.463 & 5.380 \\
gemma-2-2b & \underline{7.902}\,{$\scriptscriptstyle\pm0.022$} & 7.964\,{$\scriptscriptstyle\pm0.013$} & 14.633\,{$\scriptscriptstyle\pm0.111$} & \textbf{7.850}\,{$\scriptscriptstyle\pm0.094$} & 8.296\,{$\scriptscriptstyle\pm0.020$} & 7.957\,{$\scriptscriptstyle\pm0.035$} & 11.975\,{$\scriptscriptstyle\pm7.376$} & 9.539\,{$\scriptscriptstyle\pm0.188$} & 59.072 & 36.695 \\
\midrule
\multicolumn{11}{l}{\textbf{Source: Qwen2.5-3B}} \\
Llama-3.2-1B & 0.127\,{$\scriptscriptstyle\pm0.000$} & 0.129\,{$\scriptscriptstyle\pm0.000$} & 0.250\,{$\scriptscriptstyle\pm0.005$} & 0.127\,{$\scriptscriptstyle\pm0.001$} & \textbf{0.122}\,{$\scriptscriptstyle\pm0.000$} & \underline{0.124}\,{$\scriptscriptstyle\pm0.001$} & 0.139\,{$\scriptscriptstyle\pm0.025$} & 0.145\,{$\scriptscriptstyle\pm0.000$} & 0.265 & 0.301 \\
Llama-3.2-3B & 0.076\,{$\scriptscriptstyle\pm0.000$} & 0.075\,{$\scriptscriptstyle\pm0.000$} & 0.112\,{$\scriptscriptstyle\pm0.000$} & 0.075\,{$\scriptscriptstyle\pm0.000$} & 0.076\,{$\scriptscriptstyle\pm0.000$} & \underline{0.075}\,{$\scriptscriptstyle\pm0.000$} & \textbf{0.074}\,{$\scriptscriptstyle\pm0.001$} & 0.087\,{$\scriptscriptstyle\pm0.000$} & 0.202 & 0.160 \\
Llama-3.1-8B & 0.089\,{$\scriptscriptstyle\pm0.000$} & \underline{0.087}\,{$\scriptscriptstyle\pm0.000$} & 0.190\,{$\scriptscriptstyle\pm0.009$} & 0.088\,{$\scriptscriptstyle\pm0.000$} & 0.091\,{$\scriptscriptstyle\pm0.000$} & 0.091\,{$\scriptscriptstyle\pm0.001$} & \textbf{0.087}\,{$\scriptscriptstyle\pm0.000$} & 0.108\,{$\scriptscriptstyle\pm0.001$} & 0.186 & 0.198 \\
Qwen2.5-1.5B & \underline{0.409}\,{$\scriptscriptstyle\pm0.000$} & 0.418\,{$\scriptscriptstyle\pm0.001$} & 0.812\,{$\scriptscriptstyle\pm0.010$} & \textbf{0.407}\,{$\scriptscriptstyle\pm0.001$} & 0.427\,{$\scriptscriptstyle\pm0.003$} & 0.424\,{$\scriptscriptstyle\pm0.006$} & 0.483\,{$\scriptscriptstyle\pm0.145$} & 0.540\,{$\scriptscriptstyle\pm0.003$} & 1.412 & 4.704 \\
Qwen2.5-7B & \underline{1.019}\,{$\scriptscriptstyle\pm0.000$} & 1.033\,{$\scriptscriptstyle\pm0.000$} & 2.145\,{$\scriptscriptstyle\pm0.007$} & \textbf{0.984}\,{$\scriptscriptstyle\pm0.004$} & 1.069\,{$\scriptscriptstyle\pm0.005$} & 1.079\,{$\scriptscriptstyle\pm0.003$} & 2.283\,{$\scriptscriptstyle\pm2.265$} & 1.355\,{$\scriptscriptstyle\pm0.084$} & 2.758 & 5.238 \\
gemma-2-2b & 7.792\,{$\scriptscriptstyle\pm0.009$} & 7.907\,{$\scriptscriptstyle\pm0.014$} & 14.352\,{$\scriptscriptstyle\pm0.115$} & \textbf{7.635}\,{$\scriptscriptstyle\pm0.048$} & 7.918\,{$\scriptscriptstyle\pm0.020$} & \underline{7.647}\,{$\scriptscriptstyle\pm0.072$} & 19.613\,{$\scriptscriptstyle\pm2.272$} & 9.021\,{$\scriptscriptstyle\pm0.507$} & 48.299 & 36.602 \\
\midrule
\multicolumn{11}{l}{\textbf{Source: Qwen2.5-7B}} \\
Llama-3.2-1B & 0.123\,{$\scriptscriptstyle\pm0.000$} & 0.126\,{$\scriptscriptstyle\pm0.000$} & 0.304\,{$\scriptscriptstyle\pm0.023$} & \textbf{0.118}\,{$\scriptscriptstyle\pm0.000$} & 0.122\,{$\scriptscriptstyle\pm0.001$} & 0.120\,{$\scriptscriptstyle\pm0.001$} & \underline{0.118}\,{$\scriptscriptstyle\pm0.000$} & 0.131\,{$\scriptscriptstyle\pm0.000$} & 0.250 & 0.288 \\
Llama-3.2-3B & 0.073\,{$\scriptscriptstyle\pm0.000$} & 0.073\,{$\scriptscriptstyle\pm0.000$} & 0.103\,{$\scriptscriptstyle\pm0.001$} & 0.072\,{$\scriptscriptstyle\pm0.000$} & 0.080\,{$\scriptscriptstyle\pm0.000$} & \underline{0.072}\,{$\scriptscriptstyle\pm0.001$} & \textbf{0.071}\,{$\scriptscriptstyle\pm0.000$} & 0.079\,{$\scriptscriptstyle\pm0.000$} & 0.125 & 0.154 \\
Llama-3.1-8B & 0.085\,{$\scriptscriptstyle\pm0.000$} & 0.085\,{$\scriptscriptstyle\pm0.000$} & 0.222\,{$\scriptscriptstyle\pm0.008$} & \underline{0.083}\,{$\scriptscriptstyle\pm0.000$} & 0.088\,{$\scriptscriptstyle\pm0.000$} & 0.087\,{$\scriptscriptstyle\pm0.001$} & \textbf{0.082}\,{$\scriptscriptstyle\pm0.000$} & 0.095\,{$\scriptscriptstyle\pm0.001$} & 0.154 & 0.190 \\
Qwen2.5-1.5B & 0.474\,{$\scriptscriptstyle\pm0.001$} & 0.475\,{$\scriptscriptstyle\pm0.001$} & 1.562\,{$\scriptscriptstyle\pm0.063$} & \textbf{0.420}\,{$\scriptscriptstyle\pm0.001$} & 0.444\,{$\scriptscriptstyle\pm0.001$} & \underline{0.422}\,{$\scriptscriptstyle\pm0.005$} & 1.260\,{$\scriptscriptstyle\pm1.464$} & 0.526\,{$\scriptscriptstyle\pm0.012$} & 0.946 & 5.128 \\
Qwen2.5-3B & 0.699\,{$\scriptscriptstyle\pm0.001$} & 0.709\,{$\scriptscriptstyle\pm0.001$} & 1.592\,{$\scriptscriptstyle\pm0.018$} & \underline{0.638}\,{$\scriptscriptstyle\pm0.004$} & 0.695\,{$\scriptscriptstyle\pm0.002$} & 0.654\,{$\scriptscriptstyle\pm0.007$} & \textbf{0.628}\,{$\scriptscriptstyle\pm0.004$} & 0.788\,{$\scriptscriptstyle\pm0.017$} & 1.120 & 6.181 \\
gemma-2-2b & 7.704\,{$\scriptscriptstyle\pm0.002$} & 7.789\,{$\scriptscriptstyle\pm0.003$} & 16.008\,{$\scriptscriptstyle\pm0.265$} & \underline{7.482}\,{$\scriptscriptstyle\pm0.112$} & 7.649\,{$\scriptscriptstyle\pm0.018$} & \textbf{7.294}\,{$\scriptscriptstyle\pm0.025$} & 26.885\,{$\scriptscriptstyle\pm4.770$} & 8.353\,{$\scriptscriptstyle\pm0.011$} & 31.048 & 35.384 \\
\midrule
\multicolumn{11}{l}{\textbf{Source: gemma-2-2b}} \\
Llama-3.2-1B & \textbf{0.121}\,{$\scriptscriptstyle\pm0.000$} & 0.125\,{$\scriptscriptstyle\pm0.000$} & 0.308\,{$\scriptscriptstyle\pm0.029$} & 0.126\,{$\scriptscriptstyle\pm0.000$} & \underline{0.123}\,{$\scriptscriptstyle\pm0.001$} & 0.125\,{$\scriptscriptstyle\pm0.001$} & 0.124\,{$\scriptscriptstyle\pm0.001$} & 0.142\,{$\scriptscriptstyle\pm0.003$} & 0.251 & 0.291 \\
Llama-3.2-3B & 0.072\,{$\scriptscriptstyle\pm0.000$} & \underline{0.072}\,{$\scriptscriptstyle\pm0.000$} & 0.135\,{$\scriptscriptstyle\pm0.001$} & 0.073\,{$\scriptscriptstyle\pm0.000$} & 0.075\,{$\scriptscriptstyle\pm0.000$} & 0.078\,{$\scriptscriptstyle\pm0.001$} & \textbf{0.072}\,{$\scriptscriptstyle\pm0.000$} & 0.080\,{$\scriptscriptstyle\pm0.001$} & 0.147 & 0.155 \\
Llama-3.1-8B & 0.084\,{$\scriptscriptstyle\pm0.000$} & \underline{0.084}\,{$\scriptscriptstyle\pm0.000$} & 0.247\,{$\scriptscriptstyle\pm0.027$} & 0.085\,{$\scriptscriptstyle\pm0.000$} & 0.086\,{$\scriptscriptstyle\pm0.000$} & 0.090\,{$\scriptscriptstyle\pm0.000$} & \textbf{0.084}\,{$\scriptscriptstyle\pm0.000$} & 0.109\,{$\scriptscriptstyle\pm0.000$} & 0.187 & 0.192 \\
Qwen2.5-1.5B & 0.520\,{$\scriptscriptstyle\pm0.001$} & 0.520\,{$\scriptscriptstyle\pm0.000$} & 1.566\,{$\scriptscriptstyle\pm0.056$} & \underline{0.448}\,{$\scriptscriptstyle\pm0.005$} & 0.500\,{$\scriptscriptstyle\pm0.002$} & 0.495\,{$\scriptscriptstyle\pm0.006$} & \textbf{0.447}\,{$\scriptscriptstyle\pm0.004$} & 0.622\,{$\scriptscriptstyle\pm0.006$} & 1.362 & 5.360 \\
Qwen2.5-3B & 0.800\,{$\scriptscriptstyle\pm0.001$} & 0.812\,{$\scriptscriptstyle\pm0.000$} & 1.871\,{$\scriptscriptstyle\pm0.020$} & \underline{0.706}\,{$\scriptscriptstyle\pm0.005$} & 0.786\,{$\scriptscriptstyle\pm0.020$} & 0.772\,{$\scriptscriptstyle\pm0.014$} & \textbf{0.704}\,{$\scriptscriptstyle\pm0.006$} & 1.014\,{$\scriptscriptstyle\pm0.022$} & 1.725 & 6.652 \\
Qwen2.5-7B & \underline{1.157}\,{$\scriptscriptstyle\pm0.000$} & 1.158\,{$\scriptscriptstyle\pm0.001$} & 2.798\,{$\scriptscriptstyle\pm0.048$} & \textbf{1.088}\,{$\scriptscriptstyle\pm0.010$} & 1.214\,{$\scriptscriptstyle\pm0.011$} & 1.249\,{$\scriptscriptstyle\pm0.015$} & 2.001\,{$\scriptscriptstyle\pm1.586$} & 1.411\,{$\scriptscriptstyle\pm0.035$} & 2.513 & 5.458 \\
\midrule\midrule
\textbf{\#\,1st (lowest)} & \textbf{2} & \textbf{1} & \textbf{0} & \textbf{15} & \textbf{1} & \textbf{2} & \textbf{21} & \textbf{0} & \textbf{0} & \textbf{0} \\
\#\,2nd & 5 & 8 & 0 & 14 & 2 & 8 & 5 & 0 & 0 & 0 \\
\end{longtable}\endgroup
\twocolumn

\subsubsection{Probing}
\label{sec:appendix-probing}

Beyond simply applying the receiver's own (globally-best) probe to the transferred
activations, we additionally consider a setting where the receiver's probe is
\emph{re-selected by validation after alignment}: among the receiver's per-layer
probes, we pick whichever has the highest validation AUC on the aligned
(transferred) features, rather than reusing the receiver's globally-best probe
as-is. Both settings score on the same transferred representation; they differ
only in which probe is used.

\paragraph{Results per model pair.} The complete per-pair probing numbers behind
the main probing tables: the transferred-representation test AUC with the best
receiver probe by validation (Table~\ref{tab:probing-pp-aligned}) and the
receiver's own probe (Table~\ref{tab:probing-pp-target}). Rows are the 42 pairs
grouped by source; best per row \textbf{bold}, second \underline{underlined}, with
per-method win counts.

\onecolumn
\begingroup\fontsize{4.6pt}{5.6pt}\selectfont\setlength{\tabcolsep}{0.6pt}\renewcommand{\arraystretch}{1.1}
\setlength\LTleft{0pt}\setlength\LTright{0pt}\begin{longtable}{@{\extracolsep{\fill}}lccccccccccc}
\caption{Probing --- \textbf{best probe by validation} (test AUC $\uparrow$): the receiver probe re-selected by validation on the transferred features. Best \textbf{bold}, second \underline{underlined}; grouped by source (donor) LLM.}\label{tab:probing-pp-aligned}\\
\toprule
Target & \textbf{Term} & \textbf{Rel-dep} & \textbf{CKA} & \textbf{MATCHA} & \textbf{MATCHA-K} & \textbf{MATCHA-FL} & \textbf{MATCHA-avg} & \textbf{MATCHA-no FF} & \textbf{CmdV} & \textbf{HOT} & \textbf{MLP} \\
\midrule\endfirsthead
\multicolumn{12}{c}{\tablename\ \thetable{} -- continued}\\ \toprule
Target & \textbf{Term} & \textbf{Rel-dep} & \textbf{CKA} & \textbf{MATCHA} & \textbf{MATCHA-K} & \textbf{MATCHA-FL} & \textbf{MATCHA-avg} & \textbf{MATCHA-no FF} & \textbf{CmdV} & \textbf{HOT} & \textbf{MLP} \\
\midrule\endhead
\bottomrule\endfoot
\multicolumn{12}{l}{\textbf{Source: Llama-3.2-1B}} \\
Llama-3.2-3B & -- & 0.731\,{$\scriptscriptstyle\pm0.004$} & 0.696\,{$\scriptscriptstyle\pm0.005$} & 0.723\,{$\scriptscriptstyle\pm0.003$} & 0.701\,{$\scriptscriptstyle\pm0.009$} & 0.668\,{$\scriptscriptstyle\pm0.014$} & 0.726\,{$\scriptscriptstyle\pm0.003$} & \underline{0.733}\,{$\scriptscriptstyle\pm0.003$} & 0.653 & \textbf{0.734} & 0.696\,{$\scriptscriptstyle\pm0.010$} \\
Llama-3.1-8B & -- & 0.730\,{$\scriptscriptstyle\pm0.002$} & 0.705\,{$\scriptscriptstyle\pm0.004$} & 0.722\,{$\scriptscriptstyle\pm0.005$} & 0.704\,{$\scriptscriptstyle\pm0.006$} & 0.703\,{$\scriptscriptstyle\pm0.012$} & 0.724\,{$\scriptscriptstyle\pm0.002$} & \underline{0.731}\,{$\scriptscriptstyle\pm0.003$} & 0.663 & \textbf{0.743} & 0.693\,{$\scriptscriptstyle\pm0.007$} \\
Qwen2.5-1.5B & -- & 0.662\,{$\scriptscriptstyle\pm0.010$} & 0.648\,{$\scriptscriptstyle\pm0.021$} & \textbf{0.671}\,{$\scriptscriptstyle\pm0.009$} & \underline{0.669}\,{$\scriptscriptstyle\pm0.019$} & 0.656\,{$\scriptscriptstyle\pm0.010$} & 0.668\,{$\scriptscriptstyle\pm0.005$} & 0.644\,{$\scriptscriptstyle\pm0.011$} & 0.641 & 0.630 & 0.659\,{$\scriptscriptstyle\pm0.009$} \\
Qwen2.5-3B & -- & 0.662\,{$\scriptscriptstyle\pm0.012$} & 0.607\,{$\scriptscriptstyle\pm0.018$} & \textbf{0.691}\,{$\scriptscriptstyle\pm0.011$} & 0.635\,{$\scriptscriptstyle\pm0.006$} & 0.620\,{$\scriptscriptstyle\pm0.019$} & \underline{0.683}\,{$\scriptscriptstyle\pm0.006$} & 0.659\,{$\scriptscriptstyle\pm0.019$} & 0.645 & 0.638 & 0.662\,{$\scriptscriptstyle\pm0.017$} \\
Qwen2.5-7B & -- & \underline{0.679}\,{$\scriptscriptstyle\pm0.007$} & 0.638\,{$\scriptscriptstyle\pm0.006$} & \textbf{0.680}\,{$\scriptscriptstyle\pm0.008$} & 0.619\,{$\scriptscriptstyle\pm0.013$} & 0.653\,{$\scriptscriptstyle\pm0.010$} & 0.663\,{$\scriptscriptstyle\pm0.011$} & 0.642\,{$\scriptscriptstyle\pm0.009$} & 0.602 & 0.631 & 0.660\,{$\scriptscriptstyle\pm0.014$} \\
gemma-2-2b & -- & 0.610\,{$\scriptscriptstyle\pm0.004$} & 0.611\,{$\scriptscriptstyle\pm0.019$} & 0.627\,{$\scriptscriptstyle\pm0.012$} & 0.613\,{$\scriptscriptstyle\pm0.022$} & 0.595\,{$\scriptscriptstyle\pm0.028$} & 0.633\,{$\scriptscriptstyle\pm0.004$} & 0.633\,{$\scriptscriptstyle\pm0.005$} & 0.606 & \underline{0.639} & \textbf{0.645}\,{$\scriptscriptstyle\pm0.017$} \\
\midrule
\multicolumn{12}{l}{\textbf{Source: Llama-3.2-3B}} \\
Llama-3.2-1B & 0.673\,{$\scriptscriptstyle\pm0.005$} & 0.726\,{$\scriptscriptstyle\pm0.003$} & 0.671\,{$\scriptscriptstyle\pm0.020$} & 0.721\,{$\scriptscriptstyle\pm0.004$} & 0.681\,{$\scriptscriptstyle\pm0.002$} & 0.644\,{$\scriptscriptstyle\pm0.029$} & \underline{0.733}\,{$\scriptscriptstyle\pm0.003$} & 0.728\,{$\scriptscriptstyle\pm0.007$} & 0.716 & \textbf{0.765} & 0.647\,{$\scriptscriptstyle\pm0.010$} \\
Llama-3.1-8B & -- & \underline{0.781}\,{$\scriptscriptstyle\pm0.005$} & 0.709\,{$\scriptscriptstyle\pm0.007$} & 0.771\,{$\scriptscriptstyle\pm0.005$} & 0.725\,{$\scriptscriptstyle\pm0.009$} & 0.681\,{$\scriptscriptstyle\pm0.006$} & 0.774\,{$\scriptscriptstyle\pm0.003$} & 0.766\,{$\scriptscriptstyle\pm0.005$} & 0.748 & \textbf{0.796} & 0.698\,{$\scriptscriptstyle\pm0.009$} \\
Qwen2.5-1.5B & 0.632\,{$\scriptscriptstyle\pm0.021$} & 0.632\,{$\scriptscriptstyle\pm0.021$} & 0.572\,{$\scriptscriptstyle\pm0.009$} & \textbf{0.682}\,{$\scriptscriptstyle\pm0.018$} & 0.649\,{$\scriptscriptstyle\pm0.013$} & 0.606\,{$\scriptscriptstyle\pm0.001$} & \underline{0.678}\,{$\scriptscriptstyle\pm0.010$} & 0.662\,{$\scriptscriptstyle\pm0.015$} & 0.645 & 0.667 & 0.615\,{$\scriptscriptstyle\pm0.017$} \\
Qwen2.5-3B & -- & 0.656\,{$\scriptscriptstyle\pm0.018$} & 0.589\,{$\scriptscriptstyle\pm0.027$} & \textbf{0.681}\,{$\scriptscriptstyle\pm0.014$} & 0.647\,{$\scriptscriptstyle\pm0.012$} & 0.589\,{$\scriptscriptstyle\pm0.035$} & \underline{0.673}\,{$\scriptscriptstyle\pm0.014$} & 0.670\,{$\scriptscriptstyle\pm0.007$} & 0.619 & 0.662 & 0.617\,{$\scriptscriptstyle\pm0.023$} \\
Qwen2.5-7B & 0.643\,{$\scriptscriptstyle\pm0.005$} & 0.643\,{$\scriptscriptstyle\pm0.008$} & 0.610\,{$\scriptscriptstyle\pm0.019$} & \underline{0.693}\,{$\scriptscriptstyle\pm0.008$} & 0.606\,{$\scriptscriptstyle\pm0.025$} & 0.604\,{$\scriptscriptstyle\pm0.025$} & 0.672\,{$\scriptscriptstyle\pm0.013$} & \textbf{0.699}\,{$\scriptscriptstyle\pm0.010$} & 0.589 & 0.626 & 0.644\,{$\scriptscriptstyle\pm0.015$} \\
gemma-2-2b & 0.599\,{$\scriptscriptstyle\pm0.003$} & 0.617\,{$\scriptscriptstyle\pm0.012$} & 0.580\,{$\scriptscriptstyle\pm0.009$} & 0.614\,{$\scriptscriptstyle\pm0.009$} & 0.573\,{$\scriptscriptstyle\pm0.007$} & 0.560\,{$\scriptscriptstyle\pm0.028$} & 0.604\,{$\scriptscriptstyle\pm0.006$} & 0.596\,{$\scriptscriptstyle\pm0.015$} & \underline{0.639} & \textbf{0.659} & 0.607\,{$\scriptscriptstyle\pm0.006$} \\
\midrule
\multicolumn{12}{l}{\textbf{Source: Llama-3.1-8B}} \\
Llama-3.2-1B & 0.670\,{$\scriptscriptstyle\pm0.001$} & 0.717\,{$\scriptscriptstyle\pm0.002$} & 0.689\,{$\scriptscriptstyle\pm0.004$} & 0.730\,{$\scriptscriptstyle\pm0.003$} & 0.693\,{$\scriptscriptstyle\pm0.003$} & 0.659\,{$\scriptscriptstyle\pm0.013$} & 0.729\,{$\scriptscriptstyle\pm0.008$} & \underline{0.736}\,{$\scriptscriptstyle\pm0.007$} & 0.688 & \textbf{0.744} & 0.607\,{$\scriptscriptstyle\pm0.021$} \\
Llama-3.2-3B & 0.748\,{$\scriptscriptstyle\pm0.003$} & 0.762\,{$\scriptscriptstyle\pm0.003$} & 0.689\,{$\scriptscriptstyle\pm0.003$} & 0.762\,{$\scriptscriptstyle\pm0.006$} & 0.732\,{$\scriptscriptstyle\pm0.011$} & 0.675\,{$\scriptscriptstyle\pm0.028$} & 0.757\,{$\scriptscriptstyle\pm0.005$} & \underline{0.765}\,{$\scriptscriptstyle\pm0.007$} & 0.745 & \textbf{0.783} & 0.639\,{$\scriptscriptstyle\pm0.005$} \\
Qwen2.5-1.5B & 0.642\,{$\scriptscriptstyle\pm0.010$} & 0.647\,{$\scriptscriptstyle\pm0.015$} & 0.624\,{$\scriptscriptstyle\pm0.017$} & \textbf{0.683}\,{$\scriptscriptstyle\pm0.018$} & 0.639\,{$\scriptscriptstyle\pm0.014$} & 0.634\,{$\scriptscriptstyle\pm0.020$} & \underline{0.677}\,{$\scriptscriptstyle\pm0.016$} & 0.675\,{$\scriptscriptstyle\pm0.029$} & 0.674 & 0.677 & 0.598\,{$\scriptscriptstyle\pm0.011$} \\
Qwen2.5-3B & -- & 0.637\,{$\scriptscriptstyle\pm0.012$} & 0.571\,{$\scriptscriptstyle\pm0.022$} & \textbf{0.689}\,{$\scriptscriptstyle\pm0.009$} & 0.654\,{$\scriptscriptstyle\pm0.027$} & 0.566\,{$\scriptscriptstyle\pm0.037$} & 0.683\,{$\scriptscriptstyle\pm0.007$} & 0.674\,{$\scriptscriptstyle\pm0.020$} & 0.628 & \underline{0.683} & 0.605\,{$\scriptscriptstyle\pm0.008$} \\
Qwen2.5-7B & 0.627\,{$\scriptscriptstyle\pm0.008$} & 0.654\,{$\scriptscriptstyle\pm0.003$} & 0.636\,{$\scriptscriptstyle\pm0.001$} & \underline{0.707}\,{$\scriptscriptstyle\pm0.008$} & 0.645\,{$\scriptscriptstyle\pm0.013$} & 0.624\,{$\scriptscriptstyle\pm0.040$} & 0.692\,{$\scriptscriptstyle\pm0.012$} & \textbf{0.720}\,{$\scriptscriptstyle\pm0.010$} & 0.633 & 0.636 & 0.623\,{$\scriptscriptstyle\pm0.011$} \\
gemma-2-2b & 0.624\,{$\scriptscriptstyle\pm0.023$} & 0.600\,{$\scriptscriptstyle\pm0.016$} & 0.589\,{$\scriptscriptstyle\pm0.031$} & 0.610\,{$\scriptscriptstyle\pm0.014$} & \textbf{0.644}\,{$\scriptscriptstyle\pm0.020$} & 0.584\,{$\scriptscriptstyle\pm0.035$} & 0.633\,{$\scriptscriptstyle\pm0.028$} & 0.633\,{$\scriptscriptstyle\pm0.007$} & 0.636 & \underline{0.640} & 0.632\,{$\scriptscriptstyle\pm0.007$} \\
\midrule
\multicolumn{12}{l}{\textbf{Source: Qwen2.5-1.5B}} \\
Llama-3.2-1B & 0.714\,{$\scriptscriptstyle\pm0.005$} & \underline{0.734}\,{$\scriptscriptstyle\pm0.005$} & 0.701\,{$\scriptscriptstyle\pm0.010$} & 0.718\,{$\scriptscriptstyle\pm0.003$} & 0.699\,{$\scriptscriptstyle\pm0.001$} & 0.677\,{$\scriptscriptstyle\pm0.044$} & 0.711\,{$\scriptscriptstyle\pm0.014$} & \textbf{0.749}\,{$\scriptscriptstyle\pm0.005$} & 0.670 & 0.732 & 0.706\,{$\scriptscriptstyle\pm0.011$} \\
Llama-3.2-3B & 0.720\,{$\scriptscriptstyle\pm0.005$} & \underline{0.721}\,{$\scriptscriptstyle\pm0.005$} & 0.657\,{$\scriptscriptstyle\pm0.022$} & 0.713\,{$\scriptscriptstyle\pm0.015$} & 0.669\,{$\scriptscriptstyle\pm0.016$} & 0.675\,{$\scriptscriptstyle\pm0.019$} & 0.714\,{$\scriptscriptstyle\pm0.002$} & \textbf{0.747}\,{$\scriptscriptstyle\pm0.016$} & 0.627 & 0.720 & 0.713\,{$\scriptscriptstyle\pm0.006$} \\
Llama-3.1-8B & -- & 0.726\,{$\scriptscriptstyle\pm0.002$} & 0.705\,{$\scriptscriptstyle\pm0.006$} & 0.745\,{$\scriptscriptstyle\pm0.005$} & 0.701\,{$\scriptscriptstyle\pm0.026$} & 0.678\,{$\scriptscriptstyle\pm0.035$} & 0.749\,{$\scriptscriptstyle\pm0.006$} & \textbf{0.756}\,{$\scriptscriptstyle\pm0.001$} & 0.657 & \underline{0.750} & 0.729\,{$\scriptscriptstyle\pm0.008$} \\
Qwen2.5-3B & -- & \textbf{0.756}\,{$\scriptscriptstyle\pm0.003$} & 0.626\,{$\scriptscriptstyle\pm0.018$} & 0.744\,{$\scriptscriptstyle\pm0.007$} & 0.713\,{$\scriptscriptstyle\pm0.022$} & 0.691\,{$\scriptscriptstyle\pm0.007$} & \underline{0.748}\,{$\scriptscriptstyle\pm0.007$} & 0.659\,{$\scriptscriptstyle\pm0.026$} & 0.628 & 0.733 & 0.701\,{$\scriptscriptstyle\pm0.005$} \\
Qwen2.5-7B & \textbf{0.763}\,{$\scriptscriptstyle\pm0.003$} & 0.733\,{$\scriptscriptstyle\pm0.009$} & 0.667\,{$\scriptscriptstyle\pm0.005$} & \underline{0.747}\,{$\scriptscriptstyle\pm0.006$} & 0.722\,{$\scriptscriptstyle\pm0.019$} & 0.702\,{$\scriptscriptstyle\pm0.007$} & 0.742\,{$\scriptscriptstyle\pm0.012$} & 0.679\,{$\scriptscriptstyle\pm0.029$} & 0.617 & 0.734 & 0.715\,{$\scriptscriptstyle\pm0.003$} \\
gemma-2-2b & 0.664\,{$\scriptscriptstyle\pm0.025$} & \underline{0.670}\,{$\scriptscriptstyle\pm0.016$} & 0.655\,{$\scriptscriptstyle\pm0.008$} & 0.661\,{$\scriptscriptstyle\pm0.025$} & 0.652\,{$\scriptscriptstyle\pm0.003$} & 0.589\,{$\scriptscriptstyle\pm0.009$} & 0.631\,{$\scriptscriptstyle\pm0.026$} & 0.646\,{$\scriptscriptstyle\pm0.016$} & 0.593 & 0.652 & \textbf{0.696}\,{$\scriptscriptstyle\pm0.019$} \\
\midrule
\multicolumn{12}{l}{\textbf{Source: Qwen2.5-3B}} \\
Llama-3.2-1B & 0.683\,{$\scriptscriptstyle\pm0.005$} & 0.688\,{$\scriptscriptstyle\pm0.008$} & 0.677\,{$\scriptscriptstyle\pm0.017$} & 0.703\,{$\scriptscriptstyle\pm0.004$} & 0.700\,{$\scriptscriptstyle\pm0.012$} & 0.691\,{$\scriptscriptstyle\pm0.017$} & 0.708\,{$\scriptscriptstyle\pm0.002$} & \textbf{0.738}\,{$\scriptscriptstyle\pm0.002$} & 0.620 & \underline{0.722} & 0.690\,{$\scriptscriptstyle\pm0.011$} \\
Llama-3.2-3B & 0.721\,{$\scriptscriptstyle\pm0.010$} & 0.712\,{$\scriptscriptstyle\pm0.004$} & 0.685\,{$\scriptscriptstyle\pm0.001$} & 0.713\,{$\scriptscriptstyle\pm0.028$} & 0.696\,{$\scriptscriptstyle\pm0.006$} & 0.686\,{$\scriptscriptstyle\pm0.006$} & 0.717\,{$\scriptscriptstyle\pm0.014$} & \underline{0.723}\,{$\scriptscriptstyle\pm0.023$} & 0.654 & \textbf{0.725} & 0.686\,{$\scriptscriptstyle\pm0.017$} \\
Llama-3.1-8B & 0.710\,{$\scriptscriptstyle\pm0.003$} & 0.705\,{$\scriptscriptstyle\pm0.006$} & 0.624\,{$\scriptscriptstyle\pm0.012$} & 0.728\,{$\scriptscriptstyle\pm0.008$} & 0.698\,{$\scriptscriptstyle\pm0.017$} & 0.652\,{$\scriptscriptstyle\pm0.025$} & 0.728\,{$\scriptscriptstyle\pm0.003$} & \underline{0.730}\,{$\scriptscriptstyle\pm0.004$} & 0.645 & \textbf{0.747} & 0.713\,{$\scriptscriptstyle\pm0.003$} \\
Qwen2.5-1.5B & \underline{0.737}\,{$\scriptscriptstyle\pm0.004$} & \textbf{0.745}\,{$\scriptscriptstyle\pm0.002$} & 0.697\,{$\scriptscriptstyle\pm0.031$} & 0.733\,{$\scriptscriptstyle\pm0.017$} & 0.704\,{$\scriptscriptstyle\pm0.007$} & 0.678\,{$\scriptscriptstyle\pm0.013$} & 0.732\,{$\scriptscriptstyle\pm0.013$} & 0.678\,{$\scriptscriptstyle\pm0.017$} & 0.691 & 0.736 & 0.697\,{$\scriptscriptstyle\pm0.010$} \\
Qwen2.5-7B & \underline{0.743}\,{$\scriptscriptstyle\pm0.006$} & \textbf{0.747}\,{$\scriptscriptstyle\pm0.010$} & 0.675\,{$\scriptscriptstyle\pm0.007$} & 0.734\,{$\scriptscriptstyle\pm0.011$} & 0.705\,{$\scriptscriptstyle\pm0.009$} & 0.661\,{$\scriptscriptstyle\pm0.023$} & 0.738\,{$\scriptscriptstyle\pm0.002$} & 0.660\,{$\scriptscriptstyle\pm0.006$} & 0.639 & 0.703 & 0.689\,{$\scriptscriptstyle\pm0.021$} \\
gemma-2-2b & \textbf{0.679}\,{$\scriptscriptstyle\pm0.008$} & 0.659\,{$\scriptscriptstyle\pm0.011$} & 0.609\,{$\scriptscriptstyle\pm0.009$} & 0.635\,{$\scriptscriptstyle\pm0.018$} & 0.631\,{$\scriptscriptstyle\pm0.010$} & 0.593\,{$\scriptscriptstyle\pm0.009$} & 0.656\,{$\scriptscriptstyle\pm0.035$} & 0.636\,{$\scriptscriptstyle\pm0.053$} & 0.605 & 0.661 & \underline{0.675}\,{$\scriptscriptstyle\pm0.018$} \\
\midrule
\multicolumn{12}{l}{\textbf{Source: Qwen2.5-7B}} \\
Llama-3.2-1B & 0.725\,{$\scriptscriptstyle\pm0.013$} & 0.718\,{$\scriptscriptstyle\pm0.014$} & 0.668\,{$\scriptscriptstyle\pm0.017$} & 0.736\,{$\scriptscriptstyle\pm0.009$} & 0.675\,{$\scriptscriptstyle\pm0.019$} & 0.667\,{$\scriptscriptstyle\pm0.025$} & \underline{0.738}\,{$\scriptscriptstyle\pm0.001$} & \textbf{0.739}\,{$\scriptscriptstyle\pm0.006$} & 0.695 & 0.710 & 0.682\,{$\scriptscriptstyle\pm0.015$} \\
Llama-3.2-3B & 0.744\,{$\scriptscriptstyle\pm0.011$} & 0.743\,{$\scriptscriptstyle\pm0.011$} & 0.611\,{$\scriptscriptstyle\pm0.010$} & 0.748\,{$\scriptscriptstyle\pm0.004$} & 0.728\,{$\scriptscriptstyle\pm0.010$} & 0.696\,{$\scriptscriptstyle\pm0.027$} & \underline{0.752}\,{$\scriptscriptstyle\pm0.007$} & \textbf{0.769}\,{$\scriptscriptstyle\pm0.005$} & 0.730 & 0.733 & 0.668\,{$\scriptscriptstyle\pm0.022$} \\
Llama-3.1-8B & -- & 0.748\,{$\scriptscriptstyle\pm0.009$} & 0.662\,{$\scriptscriptstyle\pm0.017$} & 0.763\,{$\scriptscriptstyle\pm0.004$} & 0.739\,{$\scriptscriptstyle\pm0.006$} & 0.682\,{$\scriptscriptstyle\pm0.032$} & \underline{0.764}\,{$\scriptscriptstyle\pm0.004$} & \textbf{0.774}\,{$\scriptscriptstyle\pm0.003$} & 0.729 & 0.751 & 0.713\,{$\scriptscriptstyle\pm0.009$} \\
Qwen2.5-1.5B & \textbf{0.763}\,{$\scriptscriptstyle\pm0.004$} & \underline{0.763}\,{$\scriptscriptstyle\pm0.003$} & 0.675\,{$\scriptscriptstyle\pm0.014$} & 0.754\,{$\scriptscriptstyle\pm0.007$} & 0.707\,{$\scriptscriptstyle\pm0.010$} & 0.679\,{$\scriptscriptstyle\pm0.041$} & 0.758\,{$\scriptscriptstyle\pm0.007$} & 0.753\,{$\scriptscriptstyle\pm0.006$} & 0.711 & 0.732 & 0.702\,{$\scriptscriptstyle\pm0.023$} \\
Qwen2.5-3B & -- & \underline{0.761}\,{$\scriptscriptstyle\pm0.005$} & 0.608\,{$\scriptscriptstyle\pm0.018$} & 0.754\,{$\scriptscriptstyle\pm0.004$} & 0.672\,{$\scriptscriptstyle\pm0.016$} & 0.623\,{$\scriptscriptstyle\pm0.009$} & \textbf{0.767}\,{$\scriptscriptstyle\pm0.009$} & 0.748\,{$\scriptscriptstyle\pm0.014$} & 0.682 & 0.724 & 0.705\,{$\scriptscriptstyle\pm0.025$} \\
gemma-2-2b & \underline{0.674}\,{$\scriptscriptstyle\pm0.017$} & \textbf{0.683}\,{$\scriptscriptstyle\pm0.010$} & 0.556\,{$\scriptscriptstyle\pm0.007$} & 0.632\,{$\scriptscriptstyle\pm0.008$} & 0.628\,{$\scriptscriptstyle\pm0.012$} & 0.579\,{$\scriptscriptstyle\pm0.023$} & 0.641\,{$\scriptscriptstyle\pm0.037$} & 0.633\,{$\scriptscriptstyle\pm0.004$} & 0.621 & 0.637 & 0.671\,{$\scriptscriptstyle\pm0.009$} \\
\midrule
\multicolumn{12}{l}{\textbf{Source: gemma-2-2b}} \\
Llama-3.2-1B & 0.697\,{$\scriptscriptstyle\pm0.002$} & 0.705\,{$\scriptscriptstyle\pm0.007$} & 0.655\,{$\scriptscriptstyle\pm0.019$} & \textbf{0.714}\,{$\scriptscriptstyle\pm0.005$} & 0.704\,{$\scriptscriptstyle\pm0.004$} & 0.677\,{$\scriptscriptstyle\pm0.013$} & 0.707\,{$\scriptscriptstyle\pm0.005$} & \underline{0.711}\,{$\scriptscriptstyle\pm0.001$} & 0.649 & 0.695 & 0.679\,{$\scriptscriptstyle\pm0.010$} \\
Llama-3.2-3B & -- & \underline{0.715}\,{$\scriptscriptstyle\pm0.006$} & 0.631\,{$\scriptscriptstyle\pm0.015$} & \textbf{0.720}\,{$\scriptscriptstyle\pm0.005$} & 0.687\,{$\scriptscriptstyle\pm0.010$} & 0.662\,{$\scriptscriptstyle\pm0.010$} & 0.714\,{$\scriptscriptstyle\pm0.004$} & 0.713\,{$\scriptscriptstyle\pm0.006$} & 0.658 & 0.699 & 0.676\,{$\scriptscriptstyle\pm0.009$} \\
Llama-3.1-8B & -- & \textbf{0.720}\,{$\scriptscriptstyle\pm0.005$} & 0.667\,{$\scriptscriptstyle\pm0.004$} & \underline{0.719}\,{$\scriptscriptstyle\pm0.000$} & 0.707\,{$\scriptscriptstyle\pm0.006$} & 0.663\,{$\scriptscriptstyle\pm0.030$} & 0.715\,{$\scriptscriptstyle\pm0.002$} & 0.716\,{$\scriptscriptstyle\pm0.003$} & 0.685 & 0.700 & 0.668\,{$\scriptscriptstyle\pm0.019$} \\
Qwen2.5-1.5B & -- & \textbf{0.684}\,{$\scriptscriptstyle\pm0.006$} & 0.599\,{$\scriptscriptstyle\pm0.012$} & \underline{0.678}\,{$\scriptscriptstyle\pm0.005$} & 0.677\,{$\scriptscriptstyle\pm0.010$} & 0.630\,{$\scriptscriptstyle\pm0.035$} & 0.673\,{$\scriptscriptstyle\pm0.018$} & 0.653\,{$\scriptscriptstyle\pm0.015$} & 0.631 & 0.650 & 0.653\,{$\scriptscriptstyle\pm0.004$} \\
Qwen2.5-3B & -- & \textbf{0.700}\,{$\scriptscriptstyle\pm0.011$} & 0.566\,{$\scriptscriptstyle\pm0.011$} & \underline{0.679}\,{$\scriptscriptstyle\pm0.009$} & 0.662\,{$\scriptscriptstyle\pm0.009$} & 0.605\,{$\scriptscriptstyle\pm0.013$} & 0.671\,{$\scriptscriptstyle\pm0.004$} & 0.609\,{$\scriptscriptstyle\pm0.009$} & 0.631 & 0.655 & 0.640\,{$\scriptscriptstyle\pm0.009$} \\
Qwen2.5-7B & -- & 0.671\,{$\scriptscriptstyle\pm0.006$} & 0.620\,{$\scriptscriptstyle\pm0.006$} & \textbf{0.696}\,{$\scriptscriptstyle\pm0.016$} & \underline{0.676}\,{$\scriptscriptstyle\pm0.012$} & 0.643\,{$\scriptscriptstyle\pm0.017$} & 0.670\,{$\scriptscriptstyle\pm0.023$} & 0.649\,{$\scriptscriptstyle\pm0.011$} & 0.657 & 0.627 & 0.633\,{$\scriptscriptstyle\pm0.012$} \\
\midrule\midrule
\textbf{\#\,1st-place wins} & \textbf{3} & \textbf{7} & \textbf{0} & \textbf{10} & \textbf{1} & \textbf{0} & \textbf{1} & \textbf{9} & \textbf{0} & \textbf{9} & \textbf{2} \\
\#\,2nd-place wins & 3 & 8 & 0 & 6 & 2 & 0 & 9 & 7 & 1 & 5 & 1 \\
\end{longtable}\endgroup
\twocolumn

\onecolumn
\begingroup\fontsize{4.6pt}{5.6pt}\selectfont\setlength{\tabcolsep}{0.6pt}\renewcommand{\arraystretch}{1.1}
\setlength\LTleft{0pt}\setlength\LTright{0pt}\begin{longtable}{@{\extracolsep{\fill}}lccccccccccc}
\caption{Probing --- \textbf{target's own probe} (test AUC $\uparrow$): the receiver's globally-best probe applied to the transferred features (transfer-free ceiling). Best \textbf{bold}, second \underline{underlined}; grouped by source (donor) LLM.}\label{tab:probing-pp-target}\\
\toprule
Target & \textbf{Term} & \textbf{Rel-dep} & \textbf{CKA} & \textbf{MATCHA} & \textbf{MATCHA-K} & \textbf{MATCHA-FL} & \textbf{MATCHA-avg} & \textbf{MATCHA-no FF} & \textbf{CmdV} & \textbf{HOT} & \textbf{MLP} \\
\midrule\endfirsthead
\multicolumn{12}{c}{\tablename\ \thetable{} -- continued}\\ \toprule
Target & \textbf{Term} & \textbf{Rel-dep} & \textbf{CKA} & \textbf{MATCHA} & \textbf{MATCHA-K} & \textbf{MATCHA-FL} & \textbf{MATCHA-avg} & \textbf{MATCHA-no FF} & \textbf{CmdV} & \textbf{HOT} & \textbf{MLP} \\
\midrule\endhead
\bottomrule\endfoot
\multicolumn{12}{l}{\textbf{Source: Llama-3.2-1B}} \\
Llama-3.2-3B & -- & 0.700\,{$\scriptscriptstyle\pm0.004$} & 0.696\,{$\scriptscriptstyle\pm0.005$} & \underline{0.723}\,{$\scriptscriptstyle\pm0.004$} & 0.683\,{$\scriptscriptstyle\pm0.010$} & 0.646\,{$\scriptscriptstyle\pm0.030$} & 0.716\,{$\scriptscriptstyle\pm0.008$} & \textbf{0.728}\,{$\scriptscriptstyle\pm0.002$} & 0.592 & 0.695 & 0.678\,{$\scriptscriptstyle\pm0.011$} \\
Llama-3.1-8B & -- & 0.705\,{$\scriptscriptstyle\pm0.003$} & 0.663\,{$\scriptscriptstyle\pm0.002$} & 0.715\,{$\scriptscriptstyle\pm0.001$} & 0.692\,{$\scriptscriptstyle\pm0.005$} & 0.677\,{$\scriptscriptstyle\pm0.008$} & 0.714\,{$\scriptscriptstyle\pm0.005$} & \underline{0.716}\,{$\scriptscriptstyle\pm0.001$} & 0.607 & \textbf{0.722} & 0.672\,{$\scriptscriptstyle\pm0.005$} \\
Qwen2.5-1.5B & -- & 0.555\,{$\scriptscriptstyle\pm0.005$} & 0.592\,{$\scriptscriptstyle\pm0.007$} & 0.619\,{$\scriptscriptstyle\pm0.013$} & 0.589\,{$\scriptscriptstyle\pm0.029$} & 0.579\,{$\scriptscriptstyle\pm0.052$} & \underline{0.623}\,{$\scriptscriptstyle\pm0.003$} & 0.596\,{$\scriptscriptstyle\pm0.028$} & 0.606 & 0.545 & \textbf{0.654}\,{$\scriptscriptstyle\pm0.017$} \\
Qwen2.5-3B & -- & 0.633\,{$\scriptscriptstyle\pm0.004$} & 0.599\,{$\scriptscriptstyle\pm0.016$} & \textbf{0.681}\,{$\scriptscriptstyle\pm0.012$} & 0.633\,{$\scriptscriptstyle\pm0.020$} & 0.605\,{$\scriptscriptstyle\pm0.031$} & \underline{0.665}\,{$\scriptscriptstyle\pm0.012$} & 0.657\,{$\scriptscriptstyle\pm0.012$} & 0.571 & 0.598 & 0.649\,{$\scriptscriptstyle\pm0.019$} \\
Qwen2.5-7B & -- & 0.640\,{$\scriptscriptstyle\pm0.018$} & 0.622\,{$\scriptscriptstyle\pm0.003$} & \textbf{0.662}\,{$\scriptscriptstyle\pm0.009$} & 0.580\,{$\scriptscriptstyle\pm0.018$} & 0.616\,{$\scriptscriptstyle\pm0.013$} & 0.638\,{$\scriptscriptstyle\pm0.012$} & \underline{0.647}\,{$\scriptscriptstyle\pm0.008$} & 0.562 & 0.547 & 0.628\,{$\scriptscriptstyle\pm0.003$} \\
gemma-2-2b & -- & 0.568\,{$\scriptscriptstyle\pm0.025$} & 0.590\,{$\scriptscriptstyle\pm0.006$} & \underline{0.606}\,{$\scriptscriptstyle\pm0.047$} & 0.600\,{$\scriptscriptstyle\pm0.006$} & 0.566\,{$\scriptscriptstyle\pm0.044$} & 0.580\,{$\scriptscriptstyle\pm0.028$} & 0.605\,{$\scriptscriptstyle\pm0.012$} & 0.543 & 0.601 & \textbf{0.642}\,{$\scriptscriptstyle\pm0.015$} \\
\midrule
\multicolumn{12}{l}{\textbf{Source: Llama-3.2-3B}} \\
Llama-3.2-1B & 0.613\,{$\scriptscriptstyle\pm0.011$} & 0.677\,{$\scriptscriptstyle\pm0.010$} & 0.624\,{$\scriptscriptstyle\pm0.040$} & 0.685\,{$\scriptscriptstyle\pm0.014$} & 0.632\,{$\scriptscriptstyle\pm0.013$} & 0.609\,{$\scriptscriptstyle\pm0.013$} & \underline{0.696}\,{$\scriptscriptstyle\pm0.005$} & 0.669\,{$\scriptscriptstyle\pm0.003$} & 0.662 & \textbf{0.721} & 0.613\,{$\scriptscriptstyle\pm0.017$} \\
Llama-3.1-8B & -- & \underline{0.781}\,{$\scriptscriptstyle\pm0.005$} & 0.661\,{$\scriptscriptstyle\pm0.003$} & 0.769\,{$\scriptscriptstyle\pm0.009$} & 0.697\,{$\scriptscriptstyle\pm0.008$} & 0.620\,{$\scriptscriptstyle\pm0.009$} & 0.771\,{$\scriptscriptstyle\pm0.004$} & 0.765\,{$\scriptscriptstyle\pm0.004$} & 0.706 & \textbf{0.791} & 0.679\,{$\scriptscriptstyle\pm0.014$} \\
Qwen2.5-1.5B & 0.520\,{$\scriptscriptstyle\pm0.022$} & 0.521\,{$\scriptscriptstyle\pm0.022$} & 0.515\,{$\scriptscriptstyle\pm0.044$} & \underline{0.587}\,{$\scriptscriptstyle\pm0.025$} & 0.574\,{$\scriptscriptstyle\pm0.032$} & 0.528\,{$\scriptscriptstyle\pm0.044$} & 0.575\,{$\scriptscriptstyle\pm0.041$} & 0.551\,{$\scriptscriptstyle\pm0.021$} & 0.535 & 0.505 & \textbf{0.597}\,{$\scriptscriptstyle\pm0.014$} \\
Qwen2.5-3B & -- & 0.616\,{$\scriptscriptstyle\pm0.017$} & 0.553\,{$\scriptscriptstyle\pm0.014$} & \textbf{0.668}\,{$\scriptscriptstyle\pm0.019$} & 0.577\,{$\scriptscriptstyle\pm0.024$} & 0.569\,{$\scriptscriptstyle\pm0.040$} & \underline{0.664}\,{$\scriptscriptstyle\pm0.013$} & 0.647\,{$\scriptscriptstyle\pm0.015$} & 0.588 & 0.662 & 0.578\,{$\scriptscriptstyle\pm0.030$} \\
Qwen2.5-7B & 0.604\,{$\scriptscriptstyle\pm0.013$} & 0.604\,{$\scriptscriptstyle\pm0.014$} & 0.575\,{$\scriptscriptstyle\pm0.013$} & \underline{0.671}\,{$\scriptscriptstyle\pm0.023$} & 0.568\,{$\scriptscriptstyle\pm0.009$} & 0.560\,{$\scriptscriptstyle\pm0.051$} & 0.670\,{$\scriptscriptstyle\pm0.024$} & \textbf{0.696}\,{$\scriptscriptstyle\pm0.022$} & 0.566 & 0.582 & 0.607\,{$\scriptscriptstyle\pm0.019$} \\
gemma-2-2b & 0.548\,{$\scriptscriptstyle\pm0.024$} & 0.584\,{$\scriptscriptstyle\pm0.020$} & 0.560\,{$\scriptscriptstyle\pm0.023$} & 0.567\,{$\scriptscriptstyle\pm0.007$} & 0.554\,{$\scriptscriptstyle\pm0.036$} & 0.529\,{$\scriptscriptstyle\pm0.023$} & 0.575\,{$\scriptscriptstyle\pm0.034$} & 0.585\,{$\scriptscriptstyle\pm0.019$} & \underline{0.610} & \textbf{0.620} & 0.573\,{$\scriptscriptstyle\pm0.020$} \\
\midrule
\multicolumn{12}{l}{\textbf{Source: Llama-3.1-8B}} \\
Llama-3.2-1B & 0.613\,{$\scriptscriptstyle\pm0.008$} & 0.661\,{$\scriptscriptstyle\pm0.016$} & 0.645\,{$\scriptscriptstyle\pm0.007$} & \underline{0.708}\,{$\scriptscriptstyle\pm0.009$} & 0.654\,{$\scriptscriptstyle\pm0.007$} & 0.621\,{$\scriptscriptstyle\pm0.010$} & 0.701\,{$\scriptscriptstyle\pm0.014$} & 0.700\,{$\scriptscriptstyle\pm0.003$} & 0.632 & \textbf{0.712} & 0.590\,{$\scriptscriptstyle\pm0.009$} \\
Llama-3.2-3B & 0.723\,{$\scriptscriptstyle\pm0.005$} & 0.739\,{$\scriptscriptstyle\pm0.009$} & 0.689\,{$\scriptscriptstyle\pm0.003$} & 0.744\,{$\scriptscriptstyle\pm0.005$} & 0.726\,{$\scriptscriptstyle\pm0.013$} & 0.667\,{$\scriptscriptstyle\pm0.024$} & 0.740\,{$\scriptscriptstyle\pm0.003$} & \underline{0.754}\,{$\scriptscriptstyle\pm0.006$} & 0.721 & \textbf{0.765} & 0.623\,{$\scriptscriptstyle\pm0.008$} \\
Qwen2.5-1.5B & 0.518\,{$\scriptscriptstyle\pm0.035$} & 0.535\,{$\scriptscriptstyle\pm0.034$} & 0.531\,{$\scriptscriptstyle\pm0.019$} & 0.576\,{$\scriptscriptstyle\pm0.045$} & 0.544\,{$\scriptscriptstyle\pm0.032$} & 0.547\,{$\scriptscriptstyle\pm0.014$} & \textbf{0.593}\,{$\scriptscriptstyle\pm0.032$} & 0.575\,{$\scriptscriptstyle\pm0.015$} & \underline{0.588} & 0.486 & 0.562\,{$\scriptscriptstyle\pm0.024$} \\
Qwen2.5-3B & -- & 0.576\,{$\scriptscriptstyle\pm0.003$} & 0.550\,{$\scriptscriptstyle\pm0.011$} & \textbf{0.680}\,{$\scriptscriptstyle\pm0.009$} & 0.621\,{$\scriptscriptstyle\pm0.029$} & 0.544\,{$\scriptscriptstyle\pm0.029$} & \underline{0.678}\,{$\scriptscriptstyle\pm0.010$} & 0.661\,{$\scriptscriptstyle\pm0.036$} & 0.623 & 0.629 & 0.577\,{$\scriptscriptstyle\pm0.004$} \\
Qwen2.5-7B & 0.576\,{$\scriptscriptstyle\pm0.009$} & 0.607\,{$\scriptscriptstyle\pm0.012$} & 0.590\,{$\scriptscriptstyle\pm0.011$} & \underline{0.695}\,{$\scriptscriptstyle\pm0.019$} & 0.593\,{$\scriptscriptstyle\pm0.017$} & 0.602\,{$\scriptscriptstyle\pm0.034$} & 0.685\,{$\scriptscriptstyle\pm0.009$} & \textbf{0.708}\,{$\scriptscriptstyle\pm0.014$} & 0.615 & 0.622 & 0.605\,{$\scriptscriptstyle\pm0.021$} \\
gemma-2-2b & 0.544\,{$\scriptscriptstyle\pm0.027$} & 0.556\,{$\scriptscriptstyle\pm0.029$} & 0.548\,{$\scriptscriptstyle\pm0.028$} & 0.566\,{$\scriptscriptstyle\pm0.019$} & 0.575\,{$\scriptscriptstyle\pm0.032$} & 0.564\,{$\scriptscriptstyle\pm0.045$} & 0.600\,{$\scriptscriptstyle\pm0.038$} & 0.584\,{$\scriptscriptstyle\pm0.022$} & \textbf{0.618} & \underline{0.610} & 0.596\,{$\scriptscriptstyle\pm0.027$} \\
\midrule
\multicolumn{12}{l}{\textbf{Source: Qwen2.5-1.5B}} \\
Llama-3.2-1B & 0.642\,{$\scriptscriptstyle\pm0.026$} & \underline{0.694}\,{$\scriptscriptstyle\pm0.004$} & 0.681\,{$\scriptscriptstyle\pm0.010$} & 0.682\,{$\scriptscriptstyle\pm0.012$} & 0.645\,{$\scriptscriptstyle\pm0.012$} & 0.632\,{$\scriptscriptstyle\pm0.038$} & 0.676\,{$\scriptscriptstyle\pm0.011$} & \textbf{0.714}\,{$\scriptscriptstyle\pm0.019$} & 0.632 & 0.685 & 0.664\,{$\scriptscriptstyle\pm0.032$} \\
Llama-3.2-3B & 0.704\,{$\scriptscriptstyle\pm0.006$} & 0.703\,{$\scriptscriptstyle\pm0.005$} & 0.611\,{$\scriptscriptstyle\pm0.023$} & \underline{0.705}\,{$\scriptscriptstyle\pm0.009$} & 0.645\,{$\scriptscriptstyle\pm0.024$} & 0.670\,{$\scriptscriptstyle\pm0.017$} & 0.696\,{$\scriptscriptstyle\pm0.012$} & \textbf{0.721}\,{$\scriptscriptstyle\pm0.005$} & 0.556 & 0.687 & 0.701\,{$\scriptscriptstyle\pm0.004$} \\
Llama-3.1-8B & -- & 0.701\,{$\scriptscriptstyle\pm0.002$} & 0.653\,{$\scriptscriptstyle\pm0.012$} & \underline{0.737}\,{$\scriptscriptstyle\pm0.006$} & 0.681\,{$\scriptscriptstyle\pm0.031$} & 0.628\,{$\scriptscriptstyle\pm0.047$} & 0.736\,{$\scriptscriptstyle\pm0.013$} & \textbf{0.741}\,{$\scriptscriptstyle\pm0.002$} & 0.495 & 0.701 & 0.716\,{$\scriptscriptstyle\pm0.014$} \\
Qwen2.5-3B & -- & 0.706\,{$\scriptscriptstyle\pm0.009$} & 0.614\,{$\scriptscriptstyle\pm0.005$} & \underline{0.740}\,{$\scriptscriptstyle\pm0.007$} & 0.669\,{$\scriptscriptstyle\pm0.042$} & 0.676\,{$\scriptscriptstyle\pm0.038$} & \textbf{0.740}\,{$\scriptscriptstyle\pm0.014$} & 0.620\,{$\scriptscriptstyle\pm0.018$} & 0.553 & 0.680 & 0.688\,{$\scriptscriptstyle\pm0.004$} \\
Qwen2.5-7B & 0.684\,{$\scriptscriptstyle\pm0.018$} & 0.679\,{$\scriptscriptstyle\pm0.019$} & 0.489\,{$\scriptscriptstyle\pm0.014$} & \underline{0.712}\,{$\scriptscriptstyle\pm0.012$} & 0.663\,{$\scriptscriptstyle\pm0.019$} & 0.587\,{$\scriptscriptstyle\pm0.039$} & \textbf{0.729}\,{$\scriptscriptstyle\pm0.020$} & 0.650\,{$\scriptscriptstyle\pm0.030$} & 0.571 & 0.684 & 0.671\,{$\scriptscriptstyle\pm0.033$} \\
gemma-2-2b & \underline{0.664}\,{$\scriptscriptstyle\pm0.025$} & \textbf{0.670}\,{$\scriptscriptstyle\pm0.016$} & 0.623\,{$\scriptscriptstyle\pm0.010$} & 0.606\,{$\scriptscriptstyle\pm0.089$} & 0.606\,{$\scriptscriptstyle\pm0.004$} & 0.563\,{$\scriptscriptstyle\pm0.112$} & 0.586\,{$\scriptscriptstyle\pm0.006$} & 0.601\,{$\scriptscriptstyle\pm0.005$} & 0.568 & 0.557 & 0.604\,{$\scriptscriptstyle\pm0.024$} \\
\midrule
\multicolumn{12}{l}{\textbf{Source: Qwen2.5-3B}} \\
Llama-3.2-1B & 0.593\,{$\scriptscriptstyle\pm0.018$} & \underline{0.675}\,{$\scriptscriptstyle\pm0.008$} & 0.565\,{$\scriptscriptstyle\pm0.035$} & 0.664\,{$\scriptscriptstyle\pm0.027$} & 0.633\,{$\scriptscriptstyle\pm0.029$} & 0.630\,{$\scriptscriptstyle\pm0.009$} & 0.650\,{$\scriptscriptstyle\pm0.020$} & \textbf{0.696}\,{$\scriptscriptstyle\pm0.009$} & 0.554 & 0.670 & 0.656\,{$\scriptscriptstyle\pm0.013$} \\
Llama-3.2-3B & \textbf{0.713}\,{$\scriptscriptstyle\pm0.004$} & 0.681\,{$\scriptscriptstyle\pm0.004$} & 0.617\,{$\scriptscriptstyle\pm0.008$} & 0.704\,{$\scriptscriptstyle\pm0.052$} & 0.674\,{$\scriptscriptstyle\pm0.018$} & 0.670\,{$\scriptscriptstyle\pm0.018$} & 0.707\,{$\scriptscriptstyle\pm0.026$} & \underline{0.708}\,{$\scriptscriptstyle\pm0.022$} & 0.554 & 0.698 & 0.668\,{$\scriptscriptstyle\pm0.019$} \\
Llama-3.1-8B & 0.694\,{$\scriptscriptstyle\pm0.001$} & 0.676\,{$\scriptscriptstyle\pm0.003$} & 0.587\,{$\scriptscriptstyle\pm0.015$} & \textbf{0.711}\,{$\scriptscriptstyle\pm0.013$} & 0.669\,{$\scriptscriptstyle\pm0.017$} & 0.618\,{$\scriptscriptstyle\pm0.021$} & 0.703\,{$\scriptscriptstyle\pm0.001$} & 0.675\,{$\scriptscriptstyle\pm0.026$} & 0.587 & \underline{0.709} & 0.684\,{$\scriptscriptstyle\pm0.014$} \\
Qwen2.5-1.5B & \underline{0.659}\,{$\scriptscriptstyle\pm0.021$} & 0.651\,{$\scriptscriptstyle\pm0.008$} & 0.646\,{$\scriptscriptstyle\pm0.003$} & \textbf{0.667}\,{$\scriptscriptstyle\pm0.016$} & 0.577\,{$\scriptscriptstyle\pm0.027$} & 0.606\,{$\scriptscriptstyle\pm0.013$} & 0.647\,{$\scriptscriptstyle\pm0.039$} & 0.578\,{$\scriptscriptstyle\pm0.020$} & 0.632 & 0.591 & 0.637\,{$\scriptscriptstyle\pm0.018$} \\
Qwen2.5-7B & 0.688\,{$\scriptscriptstyle\pm0.008$} & \underline{0.712}\,{$\scriptscriptstyle\pm0.007$} & 0.578\,{$\scriptscriptstyle\pm0.042$} & 0.683\,{$\scriptscriptstyle\pm0.032$} & 0.674\,{$\scriptscriptstyle\pm0.007$} & 0.601\,{$\scriptscriptstyle\pm0.024$} & \textbf{0.715}\,{$\scriptscriptstyle\pm0.005$} & 0.636\,{$\scriptscriptstyle\pm0.036$} & 0.572 & 0.669 & 0.658\,{$\scriptscriptstyle\pm0.019$} \\
gemma-2-2b & \textbf{0.679}\,{$\scriptscriptstyle\pm0.008$} & 0.623\,{$\scriptscriptstyle\pm0.026$} & 0.564\,{$\scriptscriptstyle\pm0.011$} & \underline{0.641}\,{$\scriptscriptstyle\pm0.010$} & 0.598\,{$\scriptscriptstyle\pm0.010$} & 0.565\,{$\scriptscriptstyle\pm0.078$} & 0.606\,{$\scriptscriptstyle\pm0.043$} & 0.599\,{$\scriptscriptstyle\pm0.061$} & 0.559 & 0.579 & 0.639\,{$\scriptscriptstyle\pm0.022$} \\
\midrule
\multicolumn{12}{l}{\textbf{Source: Qwen2.5-7B}} \\
Llama-3.2-1B & 0.633\,{$\scriptscriptstyle\pm0.014$} & 0.679\,{$\scriptscriptstyle\pm0.012$} & 0.541\,{$\scriptscriptstyle\pm0.022$} & \underline{0.679}\,{$\scriptscriptstyle\pm0.022$} & 0.624\,{$\scriptscriptstyle\pm0.040$} & 0.592\,{$\scriptscriptstyle\pm0.018$} & 0.670\,{$\scriptscriptstyle\pm0.026$} & \textbf{0.696}\,{$\scriptscriptstyle\pm0.014$} & 0.644 & 0.595 & 0.618\,{$\scriptscriptstyle\pm0.021$} \\
Llama-3.2-3B & 0.690\,{$\scriptscriptstyle\pm0.012$} & 0.691\,{$\scriptscriptstyle\pm0.012$} & 0.556\,{$\scriptscriptstyle\pm0.014$} & 0.732\,{$\scriptscriptstyle\pm0.009$} & 0.713\,{$\scriptscriptstyle\pm0.013$} & 0.695\,{$\scriptscriptstyle\pm0.026$} & \underline{0.743}\,{$\scriptscriptstyle\pm0.012$} & \textbf{0.754}\,{$\scriptscriptstyle\pm0.007$} & 0.667 & 0.709 & 0.660\,{$\scriptscriptstyle\pm0.020$} \\
Llama-3.1-8B & -- & 0.720\,{$\scriptscriptstyle\pm0.001$} & 0.561\,{$\scriptscriptstyle\pm0.028$} & 0.750\,{$\scriptscriptstyle\pm0.006$} & 0.705\,{$\scriptscriptstyle\pm0.015$} & 0.639\,{$\scriptscriptstyle\pm0.043$} & \underline{0.754}\,{$\scriptscriptstyle\pm0.003$} & \textbf{0.763}\,{$\scriptscriptstyle\pm0.004$} & 0.671 & 0.714 & 0.704\,{$\scriptscriptstyle\pm0.010$} \\
Qwen2.5-1.5B & 0.681\,{$\scriptscriptstyle\pm0.018$} & 0.680\,{$\scriptscriptstyle\pm0.019$} & 0.548\,{$\scriptscriptstyle\pm0.008$} & \underline{0.690}\,{$\scriptscriptstyle\pm0.006$} & 0.535\,{$\scriptscriptstyle\pm0.019$} & 0.584\,{$\scriptscriptstyle\pm0.043$} & \textbf{0.691}\,{$\scriptscriptstyle\pm0.030$} & 0.690\,{$\scriptscriptstyle\pm0.021$} & 0.637 & 0.684 & 0.651\,{$\scriptscriptstyle\pm0.003$} \\
Qwen2.5-3B & -- & \underline{0.761}\,{$\scriptscriptstyle\pm0.005$} & 0.521\,{$\scriptscriptstyle\pm0.005$} & 0.752\,{$\scriptscriptstyle\pm0.007$} & 0.632\,{$\scriptscriptstyle\pm0.019$} & 0.585\,{$\scriptscriptstyle\pm0.042$} & \textbf{0.766}\,{$\scriptscriptstyle\pm0.004$} & 0.748\,{$\scriptscriptstyle\pm0.014$} & 0.685 & 0.724 & 0.690\,{$\scriptscriptstyle\pm0.018$} \\
gemma-2-2b & \underline{0.674}\,{$\scriptscriptstyle\pm0.017$} & \textbf{0.683}\,{$\scriptscriptstyle\pm0.010$} & 0.510\,{$\scriptscriptstyle\pm0.020$} & 0.622\,{$\scriptscriptstyle\pm0.012$} & 0.564\,{$\scriptscriptstyle\pm0.013$} & 0.550\,{$\scriptscriptstyle\pm0.075$} & 0.634\,{$\scriptscriptstyle\pm0.011$} & 0.615\,{$\scriptscriptstyle\pm0.003$} & 0.545 & 0.576 & 0.638\,{$\scriptscriptstyle\pm0.019$} \\
\midrule
\multicolumn{12}{l}{\textbf{Source: gemma-2-2b}} \\
Llama-3.2-1B & 0.631\,{$\scriptscriptstyle\pm0.011$} & \textbf{0.671}\,{$\scriptscriptstyle\pm0.011$} & 0.566\,{$\scriptscriptstyle\pm0.021$} & 0.651\,{$\scriptscriptstyle\pm0.013$} & 0.651\,{$\scriptscriptstyle\pm0.023$} & 0.595\,{$\scriptscriptstyle\pm0.013$} & 0.645\,{$\scriptscriptstyle\pm0.028$} & 0.637\,{$\scriptscriptstyle\pm0.007$} & 0.604 & 0.640 & \underline{0.660}\,{$\scriptscriptstyle\pm0.015$} \\
Llama-3.2-3B & -- & 0.700\,{$\scriptscriptstyle\pm0.005$} & 0.631\,{$\scriptscriptstyle\pm0.015$} & \textbf{0.717}\,{$\scriptscriptstyle\pm0.001$} & 0.678\,{$\scriptscriptstyle\pm0.015$} & 0.619\,{$\scriptscriptstyle\pm0.017$} & \underline{0.715}\,{$\scriptscriptstyle\pm0.006$} & 0.689\,{$\scriptscriptstyle\pm0.008$} & 0.594 & 0.648 & 0.657\,{$\scriptscriptstyle\pm0.006$} \\
Llama-3.1-8B & -- & \underline{0.709}\,{$\scriptscriptstyle\pm0.003$} & 0.637\,{$\scriptscriptstyle\pm0.012$} & 0.707\,{$\scriptscriptstyle\pm0.013$} & 0.699\,{$\scriptscriptstyle\pm0.007$} & 0.603\,{$\scriptscriptstyle\pm0.053$} & 0.706\,{$\scriptscriptstyle\pm0.007$} & \textbf{0.716}\,{$\scriptscriptstyle\pm0.003$} & 0.629 & 0.648 & 0.644\,{$\scriptscriptstyle\pm0.015$} \\
Qwen2.5-1.5B & -- & 0.598\,{$\scriptscriptstyle\pm0.010$} & 0.461\,{$\scriptscriptstyle\pm0.002$} & 0.585\,{$\scriptscriptstyle\pm0.030$} & 0.606\,{$\scriptscriptstyle\pm0.013$} & 0.560\,{$\scriptscriptstyle\pm0.059$} & 0.590\,{$\scriptscriptstyle\pm0.022$} & 0.561\,{$\scriptscriptstyle\pm0.005$} & 0.600 & \underline{0.622} & \textbf{0.639}\,{$\scriptscriptstyle\pm0.006$} \\
Qwen2.5-3B & -- & 0.625\,{$\scriptscriptstyle\pm0.010$} & 0.513\,{$\scriptscriptstyle\pm0.031$} & \textbf{0.677}\,{$\scriptscriptstyle\pm0.009$} & 0.662\,{$\scriptscriptstyle\pm0.009$} & 0.579\,{$\scriptscriptstyle\pm0.031$} & \underline{0.664}\,{$\scriptscriptstyle\pm0.006$} & 0.593\,{$\scriptscriptstyle\pm0.017$} & 0.572 & 0.643 & 0.618\,{$\scriptscriptstyle\pm0.010$} \\
Qwen2.5-7B & -- & 0.667\,{$\scriptscriptstyle\pm0.004$} & 0.620\,{$\scriptscriptstyle\pm0.006$} & \textbf{0.688}\,{$\scriptscriptstyle\pm0.010$} & 0.661\,{$\scriptscriptstyle\pm0.003$} & 0.627\,{$\scriptscriptstyle\pm0.004$} & \underline{0.669}\,{$\scriptscriptstyle\pm0.034$} & 0.643\,{$\scriptscriptstyle\pm0.014$} & 0.624 & 0.562 & 0.629\,{$\scriptscriptstyle\pm0.019$} \\
\midrule\midrule
\textbf{\#\,1st-place wins} & \textbf{2} & \textbf{3} & \textbf{0} & \textbf{9} & \textbf{0} & \textbf{0} & \textbf{6} & \textbf{11} & \textbf{1} & \textbf{6} & \textbf{4} \\
\#\,2nd-place wins & 3 & 6 & 0 & 13 & 0 & 0 & 10 & 4 & 2 & 3 & 1 \\
\end{longtable}\endgroup
\twocolumn

\paragraph{Probing on the best layer selected by validation.}
Table~\ref{tab:probing-pp-layer} lists, per pair and method, the receiver layer
whose \emph{validation} AUC is highest (the layer whose \emph{test} AUC is reported
in Table~\ref{tab:probing-pp-aligned}). Receiver depth varies (16--36 layers), so
the selected layer index is method- and pair-dependent.

\onecolumn
\begingroup\fontsize{5.5pt}{6.5pt}\selectfont\setlength{\tabcolsep}{1.5pt}\renewcommand{\arraystretch}{1.15}
\setlength\LTleft{0pt}\setlength\LTright{0pt}\begin{longtable}{@{\extracolsep{\fill}}lccccccccccc}
\caption{Winning probe layer (receiver index) per method and pair, i.e.\ the layer whose test AUC is in Table~\ref{tab:probing-pp-aligned}. Grouped by source (donor) LLM.}\label{tab:probing-pp-layer}\\
\toprule
Target & \textbf{Term} & \textbf{Rel-dep} & \textbf{CKA} & \textbf{MATCHA} & \textbf{MATCHA-K} & \textbf{MATCHA-FL} & \textbf{MATCHA-avg} & \textbf{MATCHA-no FF} & \textbf{CmdV} & \textbf{HOT} & \textbf{MLP} \\
\midrule\endfirsthead
\multicolumn{12}{c}{\tablename\ \thetable{} -- continued}\\ \toprule
Target & \textbf{Term} & \textbf{Rel-dep} & \textbf{CKA} & \textbf{MATCHA} & \textbf{MATCHA-K} & \textbf{MATCHA-FL} & \textbf{MATCHA-avg} & \textbf{MATCHA-no FF} & \textbf{CmdV} & \textbf{HOT} & \textbf{MLP} \\
\midrule\endhead
\bottomrule\endfoot
\multicolumn{12}{l}{\textbf{Source: Llama-3.2-1B}} \\
Llama-3.2-3B & -- & 20 & 13 & 13 & 21 & -- & -- & -- & 21 & 19 & 21 \\
Llama-3.1-8B & -- & 21 & 16 & 18 & 17 & -- & -- & -- & 29 & 18 & 21 \\
Qwen2.5-1.5B & -- & 16 & 18 & 20 & 18 & -- & -- & -- & 21 & 8 & 21 \\
Qwen2.5-3B & -- & 18 & 27 & 25 & 18 & -- & -- & -- & 27 & 10 & 24 \\
Qwen2.5-7B & -- & 18 & 5 & 18 & 5 & -- & -- & -- & 21 & 2 & 8 \\
gemma-2-2b & -- & 8 & 8 & 8 & 25 & -- & -- & -- & 12 & 23 & 15 \\
\midrule
\multicolumn{12}{l}{\textbf{Source: Llama-3.2-3B}} \\
Llama-3.2-1B & 14 & 8 & 11 & 10 & 11 & -- & -- & -- & 9 & 8 & 10 \\
Llama-3.1-8B & -- & 17 & 16 & 17 & 17 & -- & -- & -- & 18 & 19 & 19 \\
Qwen2.5-1.5B & 13 & 13 & 10 & 18 & 18 & -- & -- & -- & 17 & 16 & 18 \\
Qwen2.5-3B & -- & 21 & 34 & 24 & 20 & -- & -- & -- & 24 & 26 & 20 \\
Qwen2.5-7B & 17 & 17 & 5 & 18 & 10 & -- & -- & -- & 17 & 18 & 18 \\
gemma-2-2b & 2 & 9 & 16 & 7 & 19 & -- & -- & -- & 11 & 11 & 2 \\
\midrule
\multicolumn{12}{l}{\textbf{Source: Llama-3.1-8B}} \\
Llama-3.2-1B & 11 & 9 & 8 & 10 & 10 & -- & -- & -- & 9 & 10 & 8 \\
Llama-3.2-3B & 15 & 15 & 13 & 15 & 13 & -- & -- & -- & 17 & 19 & 20 \\
Qwen2.5-1.5B & 11 & 13 & 16 & 18 & 17 & -- & -- & -- & 16 & 18 & 18 \\
Qwen2.5-3B & -- & 16 & 24 & 24 & 20 & -- & -- & -- & 24 & 24 & 22 \\
Qwen2.5-7B & 17 & 14 & 16 & 19 & 6 & -- & -- & -- & 19 & 18 & 18 \\
gemma-2-2b & 7 & 8 & 8 & 20 & 21 & -- & -- & -- & 13 & 12 & 9 \\
\midrule
\multicolumn{12}{l}{\textbf{Source: Qwen2.5-1.5B}} \\
Llama-3.2-1B & 7 & 9 & 6 & 9 & 7 & -- & -- & -- & 10 & 9 & 7 \\
Llama-3.2-3B & 16 & 16 & 19 & 12 & 12 & -- & -- & -- & 9 & 14 & 8 \\
Llama-3.1-8B & -- & 19 & 16 & 14 & 16 & -- & -- & -- & 25 & 16 & 11 \\
Qwen2.5-3B & -- & 24 & 16 & 24 & 18 & -- & -- & -- & 23 & 24 & 25 \\
Qwen2.5-7B & 16 & 16 & 1 & 18 & 14 & -- & -- & -- & 23 & 18 & 18 \\
gemma-2-2b & 14 & 14 & 3 & 8 & 23 & -- & -- & -- & 24 & 11 & 2 \\
\midrule
\multicolumn{12}{l}{\textbf{Source: Qwen2.5-3B}} \\
Llama-3.2-1B & 7 & 10 & 10 & 8 & 9 & -- & -- & -- & 13 & 8 & 8 \\
Llama-3.2-3B & 14 & 19 & 23 & 12 & 9 & -- & -- & -- & 22 & 15 & 12 \\
Llama-3.1-8B & 19 & 21 & 16 & 15 & 16 & -- & -- & -- & 22 & 13 & 12 \\
Qwen2.5-1.5B & 15 & 18 & 23 & 19 & 13 & -- & -- & -- & 23 & 17 & 20 \\
Qwen2.5-7B & 15 & 18 & 4 & 16 & 16 & -- & -- & -- & 23 & 18 & 19 \\
gemma-2-2b & 14 & 16 & 5 & 8 & 18 & -- & -- & -- & 18 & 8 & 2 \\
\midrule
\multicolumn{12}{l}{\textbf{Source: Qwen2.5-7B}} \\
Llama-3.2-1B & 7 & 8 & 14 & 8 & 8 & -- & -- & -- & 11 & 8 & 8 \\
Llama-3.2-3B & 19 & 19 & 26 & 16 & 12 & -- & -- & -- & 19 & 12 & 19 \\
Llama-3.1-8B & -- & 22 & 23 & 16 & 15 & -- & -- & -- & 28 & 15 & 16 \\
Qwen2.5-1.5B & 19 & 19 & 9 & 17 & 16 & -- & -- & -- & 24 & 18 & 18 \\
Qwen2.5-3B & -- & 26 & 9 & 25 & 15 & -- & -- & -- & 22 & 26 & 24 \\
gemma-2-2b & 14 & 14 & 11 & 3 & 22 & -- & -- & -- & 12 & 5 & 11 \\
\midrule
\multicolumn{12}{l}{\textbf{Source: gemma-2-2b}} \\
Llama-3.2-1B & 8 & 9 & 8 & 9 & 8 & -- & -- & -- & 11 & 8 & 8 \\
Llama-3.2-3B & -- & 15 & 13 & 12 & 19 & -- & -- & -- & 20 & 14 & 21 \\
Llama-3.1-8B & -- & 21 & 16 & 15 & 16 & -- & -- & -- & 23 & 16 & 24 \\
Qwen2.5-1.5B & -- & 17 & 7 & 19 & 18 & -- & -- & -- & 20 & 18 & 22 \\
Qwen2.5-3B & -- & 22 & 7 & 27 & 26 & -- & -- & -- & 24 & 24 & 20 \\
Qwen2.5-7B & -- & 17 & 20 & 19 & 14 & -- & -- & -- & 19 & 18 & 18 \\
\end{longtable}\endgroup
\twocolumn


\subsubsection{Refusal-Suppression Steering}
\label{sec:appendix-steering}

We construct refusal-suppression directions separately for each model. To
transfer an intervention defined in a target model $\mathcal{M}_t$ to a source
model $\mathcal{M}_s$, we use the frozen source-to-target and target-to-source
alignment maps $\mathcal{A}^{s\to t}$ and $\mathcal{A}^{t\to s}$ described in
Section~\ref{sec:appendix-training}. Neither map is trained on refusal data.

\paragraph{Native steering-vector construction.}
For a model $\mathcal{M}_m$, let $h^m_\ell(p)$ denote the last-token hidden
state at layer $\ell$ for prompt $p$. Given harmful prompts $\mathcal{H}$ and
harmless prompts $\mathcal{N}$, we form a mean-difference direction at every
layer,
\begin{equation}
    u^m_\ell
    =
    \frac{1}{|\mathcal{H}|}\sum_{p\in\mathcal{H}} h^m_\ell(p)
    -
    \frac{1}{|\mathcal{N}|}\sum_{p\in\mathcal{N}} h^m_\ell(p).
    \label{eq:steering-direction}
\end{equation}
Following Arditi et al.\ (2024), we rescale each direction to the typical
activation magnitude at that layer,
\begin{equation}
    v^m_\ell
    =
    u^m_\ell\,
    \frac{s^m_\ell}
         {\operatorname{RMS}(u^m_\ell)+\varepsilon},
    \label{eq:steering-rms}
\end{equation}
where $s^m_\ell$ is the RMS magnitude of $h^m_\ell(p)$ over the prompts used to
construct the direction. We use a fixed sample of $128$ harmful and $128$
harmless prompts from the public \texttt{refusal\_direction} training splits,
sampled once with seed $42$. Direction construction is therefore identical
across alignment methods and alignment-training seeds.

We next select the model's native steering layer and coefficient. For every
layer $\ell$, we apply
\begin{equation}
    h^m_\ell \leftarrow h^m_\ell + \lambda v^m_\ell
    \label{eq:native-steering}
\end{equation}
using both signs of the coefficient magnitudes
\begin{equation}
\begin{aligned}
    \mathcal{C}_{\mathrm{native}}
    ={}& \{0.5,1,1.2,1.5,2,3,4,5,6,8,10,12,16\},\\
    &\lambda\in\{\pm c:c\in\mathcal{C}_{\mathrm{native}}\}.
\end{aligned}
\end{equation}
Each $(\ell,\lambda)$ pair is evaluated on the held-out
\texttt{refusal\_val} split, which contains $39$ harmful prompts and a matched
harmless sample. We select the configuration with the highest attack-success
rate (ASR), excluding configurations whose outputs are predominantly empty,
repetitive, or unusually short. This gives a native layer $\ell_m^\star$ and
coefficient $\lambda_{m,\mathrm{native}}^\star$ for each model. Test prompts
are not used for this selection.

\paragraph{Transferred steering.}
Consider an intervention defined in the target model $\mathcal{M}_t$ and
transferred to the source model $\mathcal{M}_s$. The target direction
$v^t_{\ell_t^\star}$ and target layer $\ell_t^\star$ remain fixed. At each
decoding step, the current last token is processed in two source-model forward
passes.

The first pass is unmodified and is used to collect the full stack of
last-token source activations,
\begin{equation}
    X^s_\tau
    =
    \begin{bmatrix}
        h^s_{1,\tau} \\
        \vdots \\
        h^s_{L_s,\tau}
    \end{bmatrix}
    \in \mathbb{R}^{L_s\times D_s},
\end{equation}
where $\tau$ indexes the current decoding step. We map this stack into target
space,
\begin{equation}
    Z^t_\tau = \mathcal{A}^{s\to t}(X^s_\tau),
\end{equation}
and edit the selected target layer,
\begin{equation}
    \widetilde Z^t_\tau
    =
    Z^t_\tau
    +
    \mathcal{E}_{\ell_t^\star}
        \!\left(\lambda v^t_{\ell_t^\star}\right),
    \label{eq:target-stack-edit}
\end{equation}
where $\mathcal{E}_{\ell_t^\star}(x)$ denotes a target-layer stack whose
$\ell_t^\star$-th row is $x$ and whose remaining rows are zero. Both the edited
and unedited target stacks are then mapped back to source space. Their
difference defines the transferred update,
\begin{equation}
\begin{aligned}
    \Delta X^s_\tau
    &={}
    \mathcal{A}^{t\to s}(\widetilde Z^t_\tau)
    -
    \mathcal{A}^{t\to s}(Z^t_\tau) \\
    &={}
    \begin{bmatrix}
        \Delta h^s_{1,\tau} \\
        \vdots \\
        \Delta h^s_{L_s,\tau}
    \end{bmatrix}.
    \label{eq:steering-transfer-delta}
\end{aligned}
\end{equation}
The reference term is the alignment's own round-trip reconstruction,
$\mathcal{A}^{t\to s}(Z^t_\tau)$, rather than the original source stack
$X^s_\tau$. Thus, $\Delta X^s_\tau$ isolates the change induced by the
intervention in target space and does not include the alignment's
reconstruction error. In particular, $\lambda=0$ gives
$\Delta X^s_\tau=0$ exactly.

In the second pass, the same token is run through the source model again. When
layer $\ell$ produces its last-token state, we add the corresponding row of
Eq.~\ref{eq:steering-transfer-delta},
\begin{equation}
    h^s_{\ell,\tau}
    \leftarrow
    h^s_{\ell,\tau} + \Delta h^s_{\ell,\tau},
    \qquad \ell=1,\ldots,L_s.
    \label{eq:steering-transfer-injection}
\end{equation}
Generation then continues from this steered second pass. The alignment is
therefore evaluated once on the complete source activation stack to compute
all source-layer updates; it is not applied independently at each source layer
during the steered pass.

Because the alignment can change the effective scale of the intervention, we
re-select only the coefficient on \texttt{refusal\_val}, while keeping
$\ell_t^\star$ fixed. We evaluate both signs of
\begin{equation}
    \mathcal{C}_{\mathrm{transfer}}
    = \{0.5,1,1.5,2,4,6,8\},
\end{equation}
as well as the target model's native coefficient
$\lambda_{t,\mathrm{native}}^\star$. We report two settings:
\textbf{best-target}, which uses $\lambda_{t,\mathrm{native}}^\star$ without
transfer-side tuning, and \textbf{best-by-val}, which selects $\lambda$ using
transferred ASR on \texttt{refusal\_val}. Both settings use the same target
layer and direction.

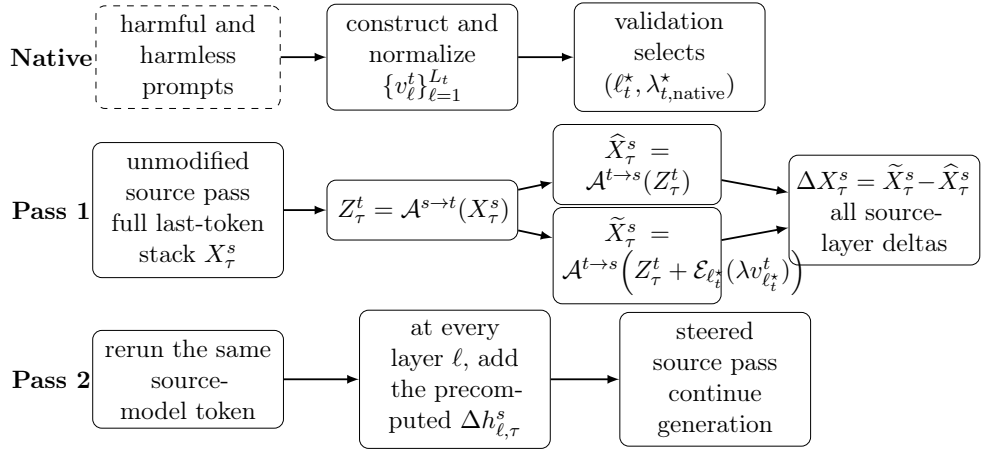
\begin{figure*}[t]
\centering
\resizebox{0.98\textwidth}{!}{%
\begin{tikzpicture}[
    box/.style={draw,rounded corners,minimum height=0.92cm,
                text width=2.65cm,align=center,inner sep=4pt},
    smallbox/.style={draw,rounded corners,minimum height=0.88cm,
                     text width=2.30cm,align=center,inner sep=4pt},
    data/.style={draw,dashed,rounded corners,minimum height=0.92cm,
                 text width=2.55cm,align=center,inner sep=4pt},
    arrow/.style={->,thick,>=latex},
    stage/.style={font=\bfseries}
]
\node[stage] at (-1.0,2.7) {Native};
\node[data] (contrast) at (1.1,2.7)
    {harmful and harmless\\prompts};
\node[box] (directions) at (4.7,2.7)
    {construct and normalize\\$\{v^t_\ell\}_{\ell=1}^{L_t}$};
\node[box] (select) at (8.5,2.7)
    {validation selects\\$(\ell_t^\star,\lambda_{t,\mathrm{native}}^\star)$};
\draw[arrow] (contrast) -- (directions);
\draw[arrow] (directions) -- (select);

\node[stage] at (-1.0,0.35) {Pass 1};
\node[box] (passone) at (1.1,0.35)
    {unmodified source pass\\full last-token stack $X^s_\tau$};
\node[box] (forward) at (4.7,0.35)
    {$Z^t_\tau=\mathcal{A}^{s\to t}(X^s_\tau)$};
\node[smallbox] (base) at (8.0,1.05)
    {$\widehat X^s_\tau=\mathcal{A}^{t\to s}(Z^t_\tau)$};
\node[smallbox] (edited) at (8.0,-0.35)
    {$\widetilde X^s_\tau=\mathcal{A}^{t\to s}\!\left(Z^t_\tau+\mathcal{E}_{\ell_t^\star}(\lambda v^t_{\ell_t^\star})\right)$};
\node[box] (delta) at (11.8,0.35)
    {$\Delta X^s_\tau=\widetilde X^s_\tau-\widehat X^s_\tau$\\all source-layer deltas};
\draw[arrow] (passone) -- (forward);
\draw[arrow] (forward) -- (base);
\draw[arrow] (forward) -- (edited);
\draw[arrow] (base) -- (delta);
\draw[arrow] (edited) -- (delta);

\node[stage] at (-1.0,-2.25) {Pass 2};
\node[box] (passtwo) at (1.1,-2.25)
    {rerun the same\\source-model token};
\node[box] (inject) at (5.2,-2.25)
    {at every layer $\ell$, add\\the precomputed $\Delta h^s_{\ell,\tau}$};
\node[box] (continue) at (9.2,-2.25)
    {steered source pass\\continue generation};
\draw[arrow] (passtwo) -- (inject);
\draw[arrow] (inject) -- (continue);
\end{tikzpicture}%
}
\caption{Refusal-suppression steering. A direction and native coefficient are
selected once for each target model (top). At each decoding step, an unmodified
source-model pass provides the complete last-token activation stack. The stack
is mapped into target space, edited at the selected target layer, and mapped
back both with and without the edit; subtracting the two outputs yields one
delta for each source layer (middle). The same token is then processed a second
time by the source model, with these deltas added at the corresponding source
layers (bottom). The alignment maps are trained beforehand and remain frozen
throughout steering.}
\label{fig:steering-transfer}
\end{figure*}

\paragraph{Evaluation.}
After validation fixes the target layer and coefficient setting, we evaluate
once on JailbreakBench ($100$ prompts), AdvBench ($387$ prompts), and HarmBench
($200$ prompts). We report LlamaGuard attack-success rate and substring-based
refusal rate, together with the source model's no-intervention performance on
the same prompts. All model pairs, alignment variants, and alignment-training
seeds use the same direction-construction, validation, and test splits.

Figure~\ref{fig:steering-bar-val-appx} aggregates the best-by-val results below:
mean test ASR per method and dataset (42 pairs), against the no-steering floor
and native-adapter ceiling.

\begin{figure}[htbp]
\centering
\includegraphics[width=\linewidth]{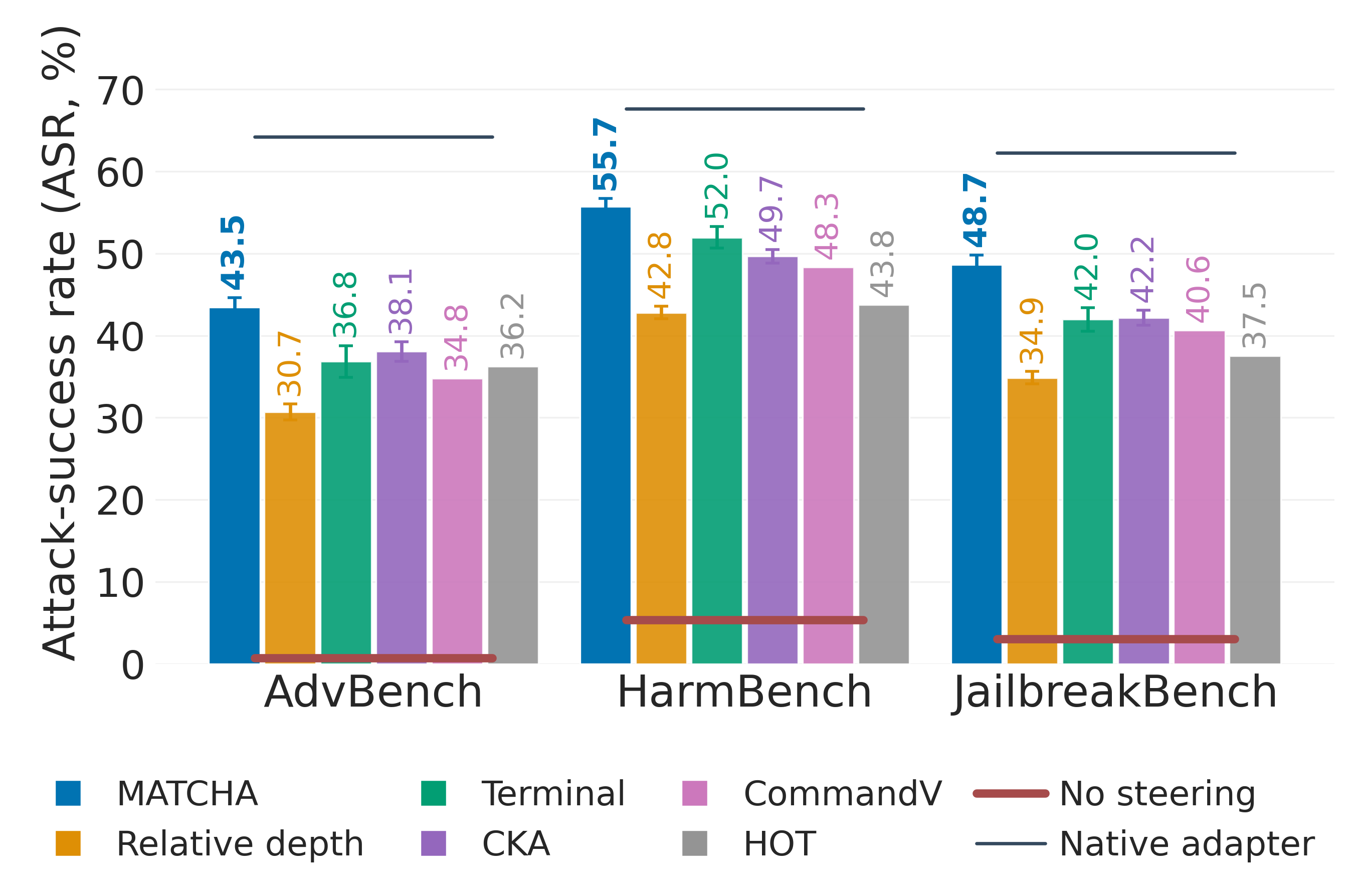}
\caption{Best-by-val transfer, aggregated: mean test ASR per method and dataset
(42 pairs), against the no-steering floor and native-adapter ceiling --- the same
quantity reported in the main results.}
\label{fig:steering-bar-val-appx}
\end{figure}

\paragraph{Results per model pair.} Tables~\ref{tab:steering-main-val}--\ref{tab:steering-main-val-jailbreakbench}
(best-by-val) and Tables~\ref{tab:steering-main-tgt}--\ref{tab:steering-main-tgt-jailbreakbench}
(best-target) give the full per-pair ASR, one table per dataset (rows grouped by
donor LLM; blank $=$ not yet available; trailing rows count per-column wins).
Methods: MATCHA, Relative depth, Terminal, CKA, CommandV, HOT, and (per-pair
tables only) MLP.

\onecolumn
\begingroup\fontsize{4.8pt}{5.8pt}\selectfont\setlength{\tabcolsep}{1pt}\renewcommand{\arraystretch}{1.05}
\setlength\LTleft{0pt}\setlength\LTright{0pt}
\begin{longtable}{@{\extracolsep{\fill}}lccccccc}
\caption{Refusal-steering transfer on \textbf{AdvBench} --- \textbf{best-by-val}: per-pair test ASR (\%, $\uparrow$), steering coefficient chosen by the highest average held-out validation ASR (\texttt{refusal\_val}) of the transfer. Each cell is the mean $\pm$ std over the pair's seeds/runs (std blank for the deterministic CommandV/HOT). Methods: Ours $=$ our method, Rel-dep $=$ Relative depth, Term $=$ Terminal, CmdV $=$ CommandV. Best method per row in \textbf{bold}, second \underline{underlined}; the last two rows count, per column, how many pairs each method won (1st) or came 2nd on. Rows grouped by source (donor) LLM; \,--\, $=$ run not yet available (and, for Terminal, every smaller$\to$bigger pair, where the terminal map is degenerate).}\label{tab:steering-main-val}\\
\toprule
Target & \textbf{\textbf{Ours}} & \textbf{Rel-dep} & \textbf{Term} & \textbf{CKA} & \textbf{CmdV} & \textbf{HOT} & \textbf{MLP} \\
\midrule\endfirsthead
\multicolumn{8}{c}{\tablename\ \thetable{} -- continued}\\ \toprule
Target & \textbf{\textbf{Ours}} & \textbf{Rel-dep} & \textbf{Term} & \textbf{CKA} & \textbf{CmdV} & \textbf{HOT} & \textbf{MLP} \\
\midrule\endhead
\bottomrule\endfoot
\multicolumn{8}{l}{\textbf{Source: Llama-3.2-1B}} \\
Llama-3.2-3B & \textbf{89.5}\,{\tiny$\pm$1.1} & 0.3\,{\tiny$\pm$0.0} & -- & 31.5\,{\tiny$\pm$4.0} & \underline{84.0} & 62.9 & 80.4\,{\tiny$\pm$7.9} \\
Llama-3.1-8B & \textbf{88.3}\,{\tiny$\pm$0.4} & 0.3\,{\tiny$\pm$0.0} & -- & 55.0\,{\tiny$\pm$21.2} & \underline{85.1} & 42.5 & 82.7\,{\tiny$\pm$6.9} \\
Qwen2.5-1.5B & 33.8\,{\tiny$\pm$5.1} & 57.6\,{\tiny$\pm$3.0} & -- & \underline{67.4}\,{\tiny$\pm$13.3} & \textbf{72.7} & 0.5 & 38.6\,{\tiny$\pm$22.8} \\
Qwen2.5-3B & \textbf{49.1}\,{\tiny$\pm$7.2} & 0.3\,{\tiny$\pm$0.0} & -- & \underline{34.0}\,{\tiny$\pm$5.5} & 33.5 & 0.5 & 30.6\,{\tiny$\pm$13.3} \\
Qwen2.5-7B & \textbf{90.5}\,{\tiny$\pm$0.8} & 44.2\,{\tiny$\pm$12.5} & -- & 35.7\,{\tiny$\pm$12.8} & 77.3 & 51.3 & \underline{80.9}\,{\tiny$\pm$10.1} \\
gemma-2-2b & \underline{86.7}\,{\tiny$\pm$2.3} & 78.0\,{\tiny$\pm$16.7} & -- & 53.2\,{\tiny$\pm$4.2} & 35.8 & \textbf{94.8} & 64.5\,{\tiny$\pm$2.3} \\
\midrule
\multicolumn{8}{l}{\textbf{Source: Llama-3.2-3B}} \\
\end{longtable}\endgroup
\begingroup\fontsize{4.8pt}{5.8pt}\selectfont\setlength{\tabcolsep}{1pt}\renewcommand{\arraystretch}{1.05}
\setlength\LTleft{0pt}\setlength\LTright{0pt}\begin{longtable}{@{\extracolsep{\fill}}lccccccc}
\caption{Refusal-steering transfer on \textbf{HarmBench} --- \textbf{best-by-val}: per-pair test ASR (\%, $\uparrow$), steering coefficient chosen by the highest average held-out validation ASR (\texttt{refusal\_val}) of the transfer. Each cell is the mean $\pm$ std over the pair's seeds/runs (std blank for the deterministic CommandV/HOT). Methods: Ours $=$ our method, Rel-dep $=$ Relative depth, Term $=$ Terminal, CmdV $=$ CommandV. Best method per row in \textbf{bold}, second \underline{underlined}; the last two rows count, per column, how many pairs each method won (1st) or came 2nd on. Rows grouped by source (donor) LLM; \,--\, $=$ run not yet available (and, for Terminal, every smaller$\to$bigger pair, where the terminal map is degenerate).}\label{tab:steering-main-val-harmbench}\\
\toprule
Target & \textbf{\textbf{Ours}} & \textbf{Rel-dep} & \textbf{Term} & \textbf{CKA} & \textbf{CmdV} & \textbf{HOT} & \textbf{MLP} \\
\midrule\endfirsthead
\multicolumn{8}{c}{\tablename\ \thetable{} -- continued}\\ \toprule
Target & \textbf{\textbf{Ours}} & \textbf{Rel-dep} & \textbf{Term} & \textbf{CKA} & \textbf{CmdV} & \textbf{HOT} & \textbf{MLP} \\
\midrule\endhead
\bottomrule\endfoot
\multicolumn{8}{l}{\textbf{Source: Llama-3.2-1B}} \\
Llama-3.2-3B & \textbf{89.1}\,{\tiny$\pm$1.1} & 2.3\,{\tiny$\pm$0.4} & -- & 49.6\,{\tiny$\pm$4.8} & \underline{87.5} & 63.0 & 79.4\,{\tiny$\pm$9.0} \\
Llama-3.1-8B & \textbf{89.2}\,{\tiny$\pm$1.7} & 2.4\,{\tiny$\pm$0.2} & -- & 65.4\,{\tiny$\pm$15.3} & 82.0 & 51.0 & \underline{83.8}\,{\tiny$\pm$6.4} \\
Qwen2.5-1.5B & 37.3\,{\tiny$\pm$4.8} & \textbf{68.8}\,{\tiny$\pm$6.2} & -- & 57.8\,{\tiny$\pm$7.4} & \underline{61.0} & 3.0 & 54.0\,{\tiny$\pm$12.3} \\
Qwen2.5-3B & \textbf{51.1}\,{\tiny$\pm$7.5} & 2.5\,{\tiny$\pm$0.2} & -- & 44.0\,{\tiny$\pm$7.2} & 40.0 & 3.5 & \underline{49.1}\,{\tiny$\pm$11.7} \\
Qwen2.5-7B & \textbf{89.9}\,{\tiny$\pm$1.1} & 61.3\,{\tiny$\pm$4.4} & -- & 48.1\,{\tiny$\pm$7.4} & 73.5 & 65.0 & \underline{80.4}\,{\tiny$\pm$11.5} \\
gemma-2-2b & \underline{83.8}\,{\tiny$\pm$3.5} & 80.4\,{\tiny$\pm$13.1} & -- & 56.9\,{\tiny$\pm$8.0} & 46.0 & \textbf{85.0} & 64.6\,{\tiny$\pm$4.6} \\
\midrule
\multicolumn{8}{l}{\textbf{Source: Llama-3.2-3B}} \\
\end{longtable}\endgroup
\begingroup\fontsize{4.8pt}{5.8pt}\selectfont\setlength{\tabcolsep}{1pt}\renewcommand{\arraystretch}{1.05}
\setlength\LTleft{0pt}\setlength\LTright{0pt}\begin{longtable}{@{\extracolsep{\fill}}lccccccc}
\caption{Refusal-steering transfer on \textbf{JailbreakBench} --- \textbf{best-by-val}: per-pair test ASR (\%, $\uparrow$), steering coefficient chosen by the highest average held-out validation ASR (\texttt{refusal\_val}) of the transfer. Each cell is the mean $\pm$ std over the pair's seeds/runs (std blank for the deterministic CommandV/HOT). Methods: Ours $=$ our method, Rel-dep $=$ Relative depth, Term $=$ Terminal, CmdV $=$ CommandV. Best method per row in \textbf{bold}, second \underline{underlined}; the last two rows count, per column, how many pairs each method won (1st) or came 2nd on. Rows grouped by source (donor) LLM; \,--\, $=$ run not yet available (and, for Terminal, every smaller$\to$bigger pair, where the terminal map is degenerate).}\label{tab:steering-main-val-jailbreakbench}\\
\toprule
Target & \textbf{\textbf{Ours}} & \textbf{Rel-dep} & \textbf{Term} & \textbf{CKA} & \textbf{CmdV} & \textbf{HOT} & \textbf{MLP} \\
\midrule\endfirsthead
\multicolumn{8}{c}{\tablename\ \thetable{} -- continued}\\ \toprule
Target & \textbf{\textbf{Ours}} & \textbf{Rel-dep} & \textbf{Term} & \textbf{CKA} & \textbf{CmdV} & \textbf{HOT} & \textbf{MLP} \\
\midrule\endhead
\bottomrule\endfoot
\multicolumn{8}{l}{\textbf{Source: Llama-3.2-1B}} \\
Llama-3.2-3B & \underline{83.6}\,{\tiny$\pm$1.6} & 1.7\,{\tiny$\pm$0.5} & -- & 39.0\,{\tiny$\pm$3.8} & \textbf{86.0} & 55.0 & 75.2\,{\tiny$\pm$9.2} \\
Llama-3.1-8B & \textbf{84.7}\,{\tiny$\pm$1.2} & 1.3\,{\tiny$\pm$0.5} & -- & 57.7\,{\tiny$\pm$15.6} & \underline{82.0} & 45.0 & 77.9\,{\tiny$\pm$5.9} \\
Qwen2.5-1.5B & 36.3\,{\tiny$\pm$5.5} & \textbf{63.0}\,{\tiny$\pm$4.3} & -- & 57.3\,{\tiny$\pm$11.0} & \underline{61.0} & 1.0 & 44.4\,{\tiny$\pm$14.4} \\
Qwen2.5-3B & \textbf{47.1}\,{\tiny$\pm$8.8} & 1.4\,{\tiny$\pm$0.5} & -- & 40.1\,{\tiny$\pm$6.0} & 33.0 & 1.0 & \underline{40.2}\,{\tiny$\pm$12.4} \\
Qwen2.5-7B & \textbf{84.0}\,{\tiny$\pm$1.3} & 49.6\,{\tiny$\pm$6.6} & -- & 43.3\,{\tiny$\pm$7.6} & 64.0 & 61.0 & \underline{74.3}\,{\tiny$\pm$10.4} \\
gemma-2-2b & \underline{81.7}\,{\tiny$\pm$4.7} & 74.5\,{\tiny$\pm$12.9} & -- & 52.4\,{\tiny$\pm$5.6} & 43.0 & \textbf{92.0} & 58.0\,{\tiny$\pm$4.4} \\
\midrule
\multicolumn{8}{l}{\textbf{Source: Llama-3.2-3B}} \\
\end{longtable}\endgroup
\twocolumn

\onecolumn
\begingroup\fontsize{4.8pt}{5.8pt}\selectfont\setlength{\tabcolsep}{1pt}\renewcommand{\arraystretch}{1.05}
\setlength\LTleft{0pt}\setlength\LTright{0pt}\begin{longtable}{@{\extracolsep{\fill}}lccccccc}
\caption{Refusal-steering transfer on \textbf{AdvBench} --- \textbf{best-target}: per-pair test ASR (\%, $\uparrow$), steering coefficient the target model's own native steering coefficient (chosen on the target's native validation, no transfer-side tuning). Each cell is the mean $\pm$ std over the pair's seeds/runs (std blank for the deterministic CommandV/HOT). Methods: Ours $=$ our method, Rel-dep $=$ Relative depth, Term $=$ Terminal, CmdV $=$ CommandV. Best method per row in \textbf{bold}, second \underline{underlined}; the last two rows count, per column, how many pairs each method won (1st) or came 2nd on. Rows grouped by source (donor) LLM; \,--\, $=$ run not yet available (and, for Terminal, every smaller$\to$bigger pair, where the terminal map is degenerate).}\label{tab:steering-main-tgt}\\
\toprule
Target & \textbf{\textbf{Ours}} & \textbf{Rel-dep} & \textbf{Term} & \textbf{CKA} & \textbf{CmdV} & \textbf{HOT} & \textbf{MLP} \\
\midrule\endfirsthead
\multicolumn{8}{c}{\tablename\ \thetable{} -- continued}\\ \toprule
Target & \textbf{\textbf{Ours}} & \textbf{Rel-dep} & \textbf{Term} & \textbf{CKA} & \textbf{CmdV} & \textbf{HOT} & \textbf{MLP} \\
\midrule\endhead
\bottomrule\endfoot
\multicolumn{8}{l}{\textbf{Source: Llama-3.2-1B}} \\
Llama-3.2-3B & \textbf{89.5}\,{\tiny$\pm$1.1} & 0.3\,{\tiny$\pm$0.0} & -- & 42.3\,{\tiny$\pm$16.3} & \underline{82.2} & 62.9 & 80.4\,{\tiny$\pm$7.9} \\
Llama-3.1-8B & \textbf{87.0}\,{\tiny$\pm$1.3} & 0.3\,{\tiny$\pm$0.0} & -- & 40.7\,{\tiny$\pm$23.3} & 76.0 & 4.1 & \underline{79.6}\,{\tiny$\pm$7.9} \\
Qwen2.5-1.5B & 15.3\,{\tiny$\pm$13.6} & \underline{38.5}\,{\tiny$\pm$5.6} & -- & \textbf{41.4}\,{\tiny$\pm$7.7} & 18.8 & 0.0 & 31.1\,{\tiny$\pm$18.7} \\
Qwen2.5-3B & 17.4\,{\tiny$\pm$10.8} & 0.3\,{\tiny$\pm$0.0} & -- & \textbf{32.9}\,{\tiny$\pm$8.5} & 12.1 & 0.3 & \underline{21.9}\,{\tiny$\pm$13.3} \\
Qwen2.5-7B & 19.0\,{\tiny$\pm$5.9} & 44.2\,{\tiny$\pm$12.5} & -- & 3.7\,{\tiny$\pm$1.5} & \textbf{77.3} & 0.8 & \underline{66.8}\,{\tiny$\pm$13.3} \\
gemma-2-2b & \underline{86.7}\,{\tiny$\pm$2.3} & 78.1\,{\tiny$\pm$15.8} & -- & 43.9\,{\tiny$\pm$14.1} & 33.5 & \textbf{94.3} & 58.1\,{\tiny$\pm$6.6} \\
\midrule
\multicolumn{8}{l}{\textbf{Source: Llama-3.2-3B}} \\
\end{longtable}\endgroup
\begingroup\fontsize{4.8pt}{5.8pt}\selectfont\setlength{\tabcolsep}{1pt}\renewcommand{\arraystretch}{1.05}
\setlength\LTleft{0pt}\setlength\LTright{0pt}\begin{longtable}{@{\extracolsep{\fill}}lccccccc}
\caption{Refusal-steering transfer on \textbf{HarmBench} --- \textbf{best-target}: per-pair test ASR (\%, $\uparrow$), steering coefficient the target model's own native steering coefficient (chosen on the target's native validation, no transfer-side tuning). Each cell is the mean $\pm$ std over the pair's seeds/runs (std blank for the deterministic CommandV/HOT). Methods: Ours $=$ our method, Rel-dep $=$ Relative depth, Term $=$ Terminal, CmdV $=$ CommandV. Best method per row in \textbf{bold}, second \underline{underlined}; the last two rows count, per column, how many pairs each method won (1st) or came 2nd on. Rows grouped by source (donor) LLM; \,--\, $=$ run not yet available (and, for Terminal, every smaller$\to$bigger pair, where the terminal map is degenerate).}\label{tab:steering-main-tgt-harmbench}\\
\toprule
Target & \textbf{\textbf{Ours}} & \textbf{Rel-dep} & \textbf{Term} & \textbf{CKA} & \textbf{CmdV} & \textbf{HOT} & \textbf{MLP} \\
\midrule\endfirsthead
\multicolumn{8}{c}{\tablename\ \thetable{} -- continued}\\ \toprule
Target & \textbf{\textbf{Ours}} & \textbf{Rel-dep} & \textbf{Term} & \textbf{CKA} & \textbf{CmdV} & \textbf{HOT} & \textbf{MLP} \\
\midrule\endhead
\bottomrule\endfoot
\multicolumn{8}{l}{\textbf{Source: Llama-3.2-1B}} \\
Llama-3.2-3B & \textbf{89.1}\,{\tiny$\pm$1.1} & 2.3\,{\tiny$\pm$0.4} & -- & 50.5\,{\tiny$\pm$6.6} & \underline{83.0} & 63.0 & 79.4\,{\tiny$\pm$9.0} \\
Llama-3.1-8B & \textbf{84.0}\,{\tiny$\pm$1.3} & 2.4\,{\tiny$\pm$0.2} & -- & 57.9\,{\tiny$\pm$14.9} & 71.0 & 12.0 & \underline{79.9}\,{\tiny$\pm$6.8} \\
Qwen2.5-1.5B & 20.4\,{\tiny$\pm$15.1} & \textbf{51.8}\,{\tiny$\pm$4.6} & -- & 42.4\,{\tiny$\pm$6.2} & 28.5 & 2.5 & \underline{43.8}\,{\tiny$\pm$12.2} \\
Qwen2.5-3B & 28.6\,{\tiny$\pm$11.4} & 2.5\,{\tiny$\pm$0.2} & -- & \textbf{38.7}\,{\tiny$\pm$8.4} & 23.5 & 0.0 & \underline{35.4}\,{\tiny$\pm$12.8} \\
Qwen2.5-7B & 31.7\,{\tiny$\pm$6.1} & 61.3\,{\tiny$\pm$4.4} & -- & 14.4\,{\tiny$\pm$3.4} & \textbf{73.5} & 4.5 & \underline{65.2}\,{\tiny$\pm$10.2} \\
gemma-2-2b & \underline{83.8}\,{\tiny$\pm$3.5} & 80.0\,{\tiny$\pm$13.0} & -- & 53.8\,{\tiny$\pm$13.1} & 51.0 & \textbf{84.5} & 59.8\,{\tiny$\pm$7.2} \\
\midrule
\multicolumn{8}{l}{\textbf{Source: Llama-3.2-3B}} \\
\end{longtable}\endgroup
\begingroup\fontsize{4.8pt}{5.8pt}\selectfont\setlength{\tabcolsep}{1pt}\renewcommand{\arraystretch}{1.05}
\setlength\LTleft{0pt}\setlength\LTright{0pt}\begin{longtable}{@{\extracolsep{\fill}}lccccccc}
\caption{Refusal-steering transfer on \textbf{JailbreakBench} --- \textbf{best-target}: per-pair test ASR (\%, $\uparrow$), steering coefficient the target model's own native steering coefficient (chosen on the target's native validation, no transfer-side tuning). Each cell is the mean $\pm$ std over the pair's seeds/runs (std blank for the deterministic CommandV/HOT). Methods: Ours $=$ our method, Rel-dep $=$ Relative depth, Term $=$ Terminal, CmdV $=$ CommandV. Best method per row in \textbf{bold}, second \underline{underlined}; the last two rows count, per column, how many pairs each method won (1st) or came 2nd on. Rows grouped by source (donor) LLM; \,--\, $=$ run not yet available (and, for Terminal, every smaller$\to$bigger pair, where the terminal map is degenerate).}\label{tab:steering-main-tgt-jailbreakbench}\\
\toprule
Target & \textbf{\textbf{Ours}} & \textbf{Rel-dep} & \textbf{Term} & \textbf{CKA} & \textbf{CmdV} & \textbf{HOT} & \textbf{MLP} \\
\midrule\endfirsthead
\multicolumn{8}{c}{\tablename\ \thetable{} -- continued}\\ \toprule
Target & \textbf{\textbf{Ours}} & \textbf{Rel-dep} & \textbf{Term} & \textbf{CKA} & \textbf{CmdV} & \textbf{HOT} & \textbf{MLP} \\
\midrule\endhead
\bottomrule\endfoot
\multicolumn{8}{l}{\textbf{Source: Llama-3.2-1B}} \\
Llama-3.2-3B & \underline{83.6}\,{\tiny$\pm$1.6} & 1.7\,{\tiny$\pm$0.5} & -- & 45.3\,{\tiny$\pm$11.2} & \textbf{86.0} & 55.0 & 75.2\,{\tiny$\pm$9.2} \\
Llama-3.1-8B & \textbf{80.2}\,{\tiny$\pm$3.2} & 1.3\,{\tiny$\pm$0.5} & -- & 47.7\,{\tiny$\pm$15.3} & 65.0 & 9.0 & \underline{74.7}\,{\tiny$\pm$8.1} \\
Qwen2.5-1.5B & 20.0\,{\tiny$\pm$14.1} & \textbf{48.9}\,{\tiny$\pm$6.8} & -- & \underline{40.4}\,{\tiny$\pm$6.8} & 25.0 & 1.0 & 35.2\,{\tiny$\pm$15.9} \\
Qwen2.5-3B & 24.7\,{\tiny$\pm$15.2} & 1.4\,{\tiny$\pm$0.5} & -- & \textbf{40.3}\,{\tiny$\pm$5.8} & 12.0 & 1.0 & \underline{26.4}\,{\tiny$\pm$14.7} \\
Qwen2.5-7B & 30.0\,{\tiny$\pm$7.0} & 49.6\,{\tiny$\pm$6.6} & -- & 9.2\,{\tiny$\pm$2.4} & \textbf{64.0} & 1.0 & \underline{59.0}\,{\tiny$\pm$9.2} \\
gemma-2-2b & \underline{81.7}\,{\tiny$\pm$4.7} & 73.7\,{\tiny$\pm$14.2} & -- & 48.0\,{\tiny$\pm$13.9} & 47.0 & \textbf{92.0} & 57.4\,{\tiny$\pm$5.2} \\
\midrule
\multicolumn{8}{l}{\textbf{Source: Llama-3.2-3B}} \\
\end{longtable}\endgroup
\twocolumn

Table~\ref{tab:steering-coef} lists the best-by-val chosen coefficient per
(pair, method); negative values suppress the refusal direction.

\onecolumn
\begingroup\footnotesize\setlength{\tabcolsep}{5pt}\renewcommand{\arraystretch}{1.15}
\setlength\LTleft{0pt}\setlength\LTright{0pt}\begin{longtable}{@{\extracolsep{\fill}}lccccccc}
\caption{Best-by-val chosen steering coefficient per (pair, method): the coefficient with the highest average validation ASR (\texttt{refusal\_val}), used for the \textbf{val} table (Table~\ref{tab:steering-main-val}). Negative coefficients suppress the refusal direction. Rows grouped by source (donor) LLM; \,--\, = not yet available (and, for Terminal, every smaller$\to$bigger pair, where the terminal map is degenerate).}\label{tab:steering-coef}\\
\toprule
Target & \textbf{\textbf{Ours}} & \textbf{Rel-dep} & \textbf{Term} & \textbf{CKA} & \textbf{CmdV} & \textbf{HOT} & \textbf{MLP} \\
\midrule\endfirsthead
\multicolumn{8}{c}{\tablename\ \thetable{} -- continued}\\ \toprule
Target & \textbf{\textbf{Ours}} & \textbf{Rel-dep} & \textbf{Term} & \textbf{CKA} & \textbf{CmdV} & \textbf{HOT} & \textbf{MLP} \\
\midrule\endhead
\bottomrule\endfoot
\multicolumn{8}{l}{\textbf{Source: Llama-3.2-1B}} \\
Llama-3.2-3B & -10 & -0.5 & -- & -4 & -6 & -10 & -10 \\
Llama-3.1-8B & -8 & -0.5 & -- & -6 & -2 & -8 & -8 \\
Qwen2.5-1.5B & 6 & 8 & -- & -8 & 1.5 & -0.5 & 8 \\
Qwen2.5-3B & -8 & -0.5 & -- & 6 & 2 & -0.5 & -8 \\
Qwen2.5-7B & -4 & -1 & -- & -6 & -1 & -8 & -8 \\
gemma-2-2b & -6 & -8 & -- & -8 & -1.5 & -4 & -2 \\
\midrule
\multicolumn{8}{l}{\textbf{Source: Llama-3.2-3B}} \\
\end{longtable}\endgroup
\twocolumn

\subsection{Extended Quantitative Analysis}
\label{sec:appendix-analysis}

\subsubsection{Additional Results -- Per-Pair Performance Difference (MATCHA $-$ Relative Depth)}
\label{sec:appendix-heatmap}
Results are in \Cref{fig:wk-vs-linear-recall1,fig:wk-vs-linear-jaccard5,fig:wk-vs-linear-tj5}.

\begin{figure*}[htbp]
\centering
\begin{minipage}[t]{0.49\linewidth}
\centering
\includegraphics[width=\linewidth]{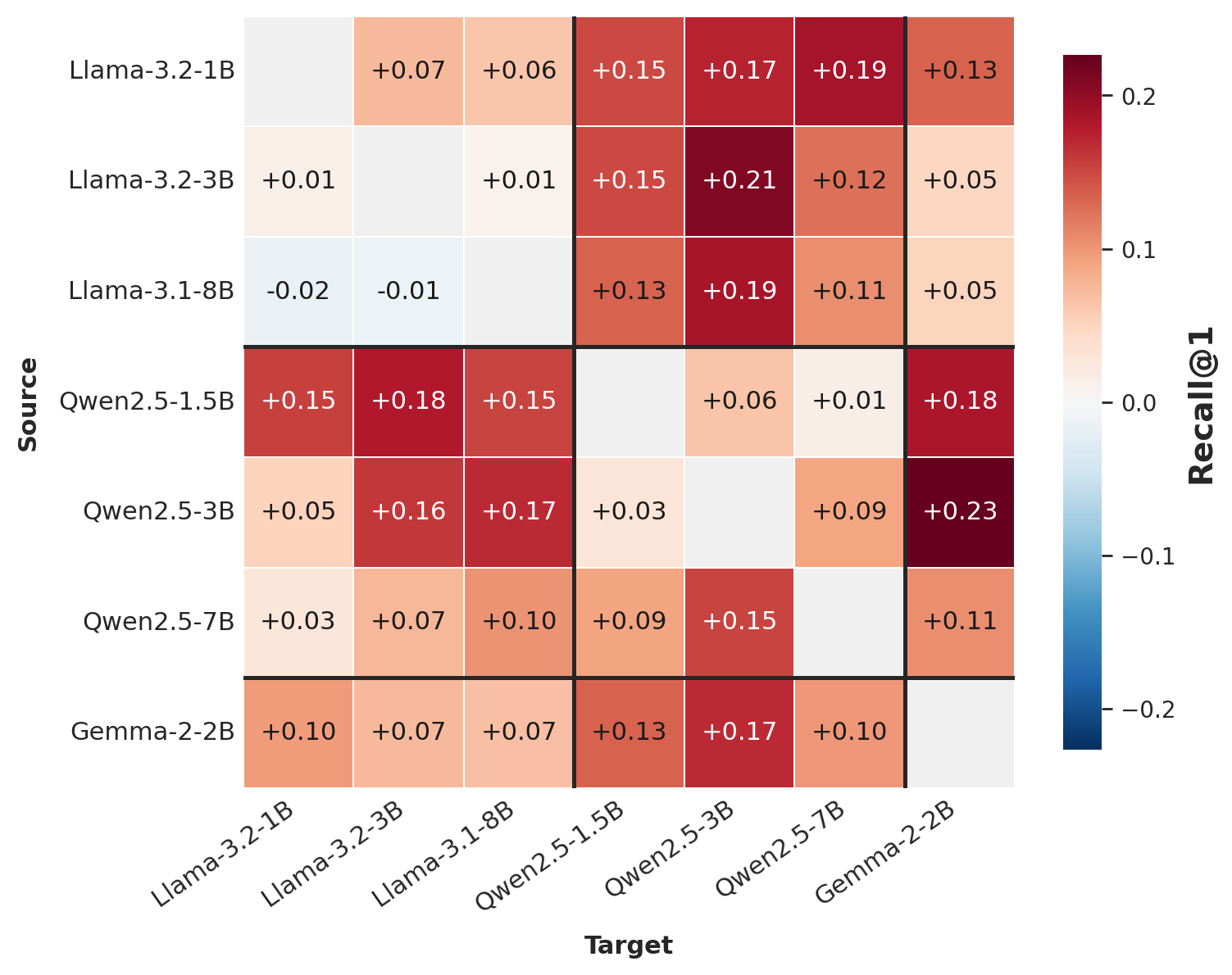}
\caption{Per-pair \textbf{Recall@1} difference, MATCHA $-$ Relative depth.}
\label{fig:wk-vs-linear-recall1}
\end{minipage}\hfill
\begin{minipage}[t]{0.49\linewidth}
\centering
\includegraphics[width=\linewidth]{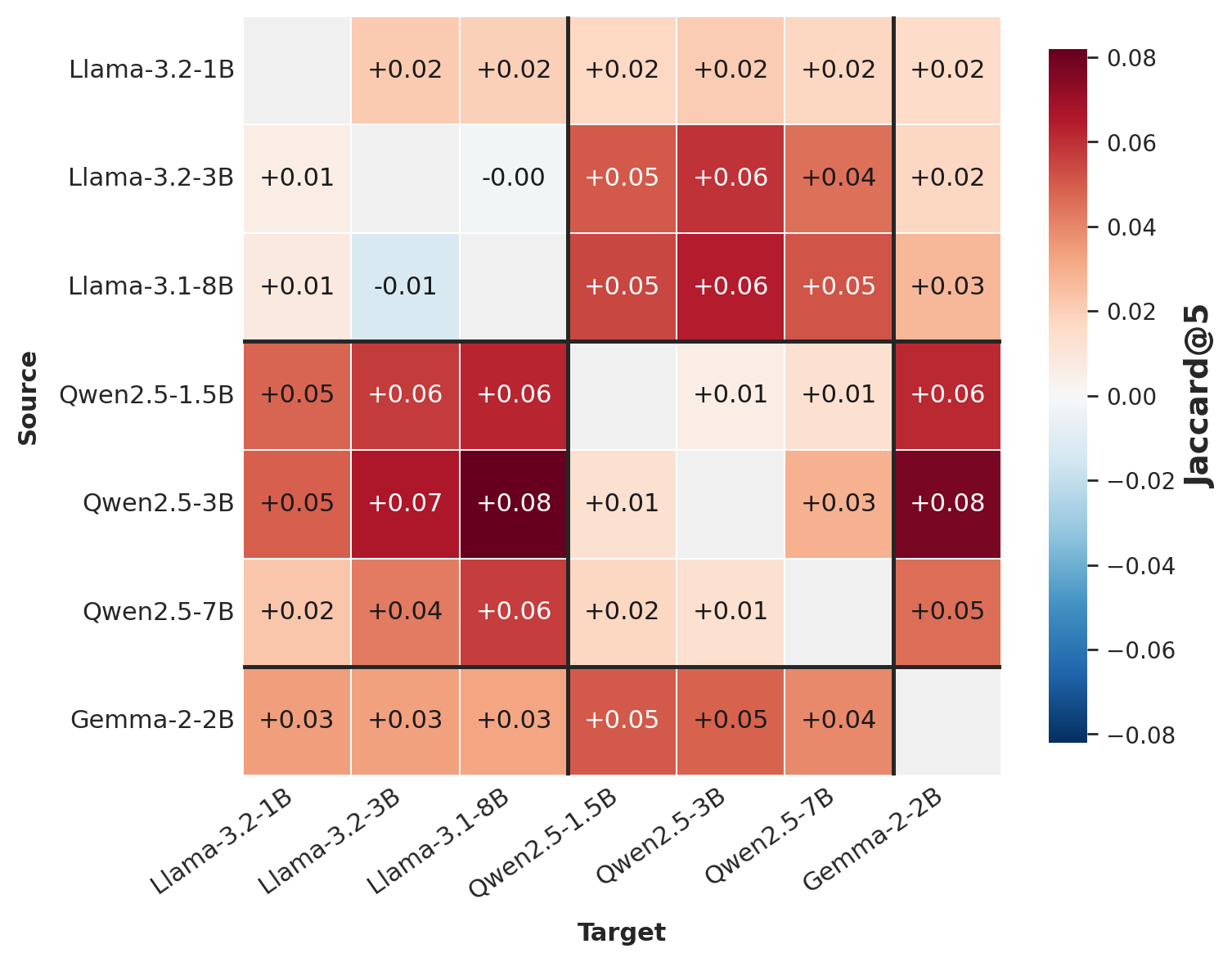}
\caption{Per-pair \textbf{Jaccard@5} difference, MATCHA $-$ Relative depth. Axes
as in Figure~\ref{fig:wk-vs-linear-recall1}.}
\label{fig:wk-vs-linear-jaccard5}
\end{minipage}
\end{figure*}

\begin{figure*}[htbp]
\centering
\includegraphics[width=0.49\linewidth]{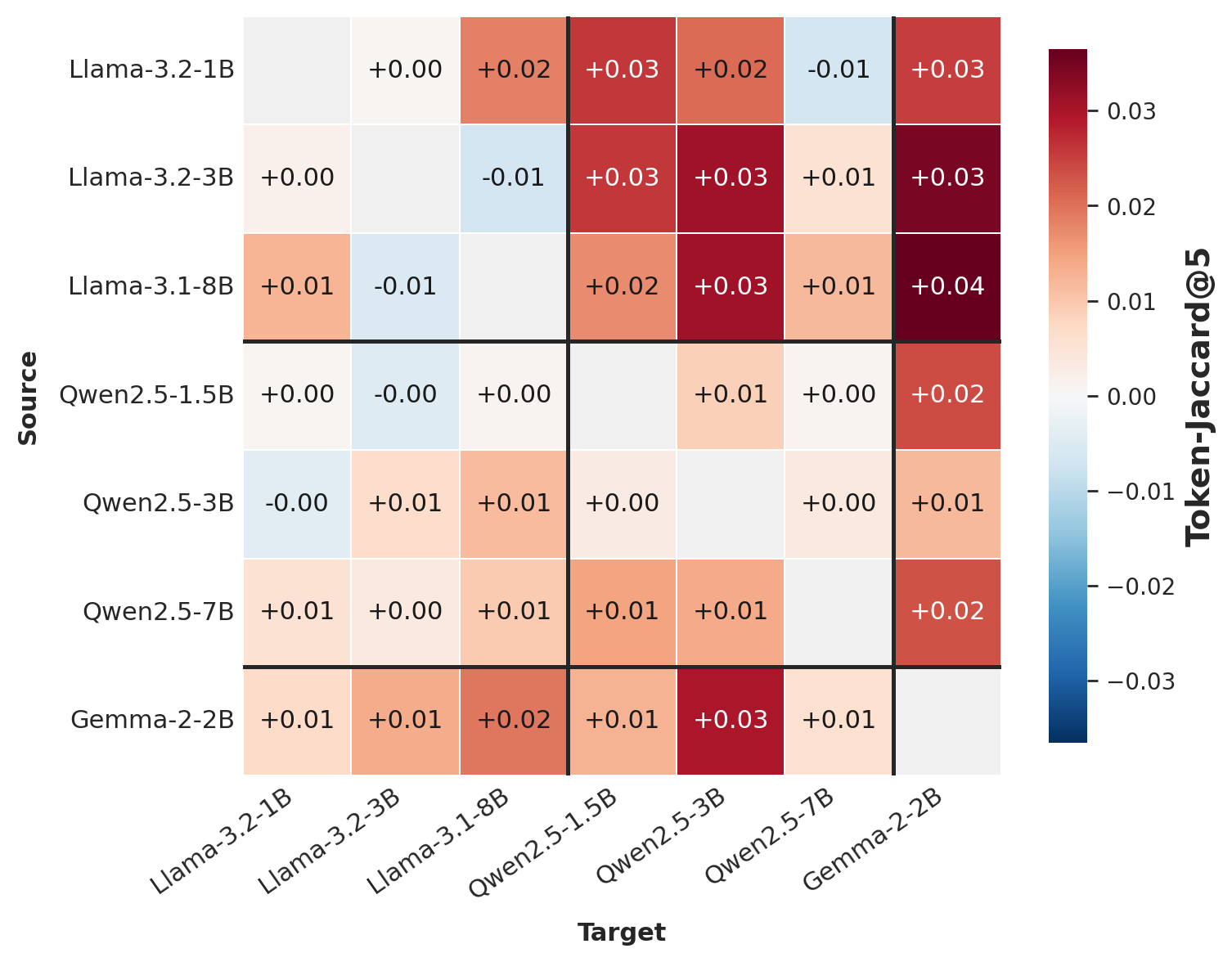}
\caption{Per-pair \textbf{Token-Jaccard@5} difference, MATCHA $-$ Relative depth.
Axes as in Figure~\ref{fig:wk-vs-linear-recall1}.}
\label{fig:wk-vs-linear-tj5}
\end{figure*}

\FloatBarrier

\RenewDocumentEnvironment{figure}{O{!htbp}+b}
  {%
    \MATCHAorigfigure[#1]
    \centering
    \begin{minipage}{0.80\textwidth}
      \MATCHAsetlocalwidth
      \centering
      #2
    \end{minipage}
    \MATCHAendorigfigure
  }
  {}

\RenewDocumentEnvironment{figure*}{O{!htbp}+b}
  {%
    \MATCHAorigfigure[#1]
    \centering
    \begin{minipage}{1\textwidth}
      \MATCHAsetlocalwidth
      \centering
      #2
    \end{minipage}
    \MATCHAendorigfigure
  }
  {}

%
%
%
%
%
%

\providecommand{\softmax}{\operatorname{softmax}}
\providecommand{\rank}{\operatorname{rank}}

\makeatletter
\@ifundefined{proposition}{%
  \theoremstyle{plain}%
  \newtheorem{proposition}{Proposition}%
  \newtheorem{lemma}[proposition]{Lemma}%
  \newtheorem{corollary}[proposition]{Corollary}%
  \theoremstyle{definition}%
  \newtheorem{assumption}[proposition]{Assumption}%
  \theoremstyle{remark}%
  \newtheorem{remark}[proposition]{Remark}%
}{}
\makeatother

\section{Expressivity of the layer map}
\label{app:expressivity}
\paragraph{Overview.} Two questions matter: does the architecture subsume the fixed one-to-one maps of prior
work, and how much of the space of \emph{soft} correspondences is reachable? In this subsection we show: (i) every one-to-one map
is approximable to arbitrary accuracy at attention width $d=2$
(\Cref{prop:dyn-onehot}), and (ii) every soft
correspondence is realized \emph{exactly}, assuming $d\ge\min(L_t,L_s)$
(\Cref{prop:dyn-interior}).

\paragraph{Notation.}
Let $\{ e_1,\dots,e_n \}$ be the standard basis of $\mathbb R^{n}$, and let 
\[
  \|\mathbf M\|_{1,\infty}\;:=\;\max_{1\le\ell\le a}\|\mathbf M_\ell\|_1
  \;=\;\max_{\ell}\sum_{j=1}^{b}\lvert \mathbf M_{\ell,j}\rvert
\]
be the largest $\ell_1$ norm
of any single row, so that bounding it constrains every target layer rather than the
average one. Throughout, $S:=QK(X)^\top/\sqrt d$ is the \emph{pre-softmax} logit matrix,
$\alpha=\softmax(S)$ the row-wise layer map of \Cref{eq:map}, and $x_j(X)$ the $j$-th row
of the normalized source stack $X$. Each row of $\alpha$ is a distribution over source layers,
so every layer map lies in the set $\Delta$ of row-stochastic matrices; the two families
we care about occupy different parts of it.

\paragraph{One-to-one maps: the boundary of $\Delta$.}
Given a layer assignment $\pi:[L_t]\to[L_s]$, let $P_\pi\in\{0,1\}^{L_t\times L_s}$ have
rows $(P_\pi)_\ell=e_{\pi(\ell)}^{\top}$: target layer $\ell$ draws on source layer
$\pi(\ell)$ alone. Every fixed correspondence used in prior work --- relative-depth,
terminal, identity --- is of this form. These matrices lie on the
boundary of $\Delta$ and are never attained exactly, since the softmax assigns strictly
positive weight to every source layer. This costs nothing in practice: the strictly
positive row-stochastic matrices are dense in $\Delta$,\footnote{Every point of $\Delta$,
vertices included, is a limit of strictly positive row-stochastic matrices --- exactness
fails but $\varepsilon$-approximation does not.} and by \Cref{lem:margin} the approach to
the boundary is exponentially fast in the logit scale. We therefore state these results
with a tolerance $\varepsilon>0$.

\paragraph{Soft maps: the interior of $\Delta$.}
The interior $\Delta^{\circ}\subset\Delta$ consists of the row-stochastic matrices with
strictly positive entries, in which each target layer draws on several source layers at
once --- the many-to-many correspondences our parameterization is meant to permit and
which the learned maps of \Cref{eq:map} in fact exhibit. Here the situation is the
reverse: $\Delta^{\circ}$ is exactly the image of the softmax, so these maps are
representable \emph{exactly}, and the binding constraint is not the tolerance but the
attention width $d$.

Before proving the main results, we state a lemma regarding the conversion of logit gaps
into layer maps via the softmax: if in every row of $S$ the designated source layer leads
the runner-up by a margin $m$, then the resulting $\alpha$ is within $2(L_s-1)e^{-m}$ of
the corresponding one-hot matrix $P_\pi$ in $\|\cdot\|_{1,\infty}$. This gives us a way to
control, through $m$ alone, the error incurred in approximating a one-to-one
correspondence.

\begin{lemma}[Logit margin $\Rightarrow$ near-one-hot map]
\label{lem:margin}
Let $\pi:[L_t]\to[L_s]$ be a layer assignment, let $S\in\mathbb R^{L_t\times L_s}$, and
let $m>0$ be a \emph{margin} for $\pi$ in $S$:
\begin{equation}\label{eq:margin}
  S_{\ell,\pi(\ell)}\;\ge\;S_{\ell,j}+m
  \qquad\text{for all }\ell\in[L_t],\;j\neq\pi(\ell).
\end{equation}
That is, in every row the designated source layer $\pi(\ell)$ is the strict maximizer,
leading the runner-up by at least $m$. Then
\begin{equation}\label{eq:margin-bound}
  \big\|\softmax(S)-P_\pi\big\|_{1,\infty}\;\le\;2(L_s-1)\,e^{-m}.
\end{equation}
\end{lemma}

\begin{proof}
Fix $\ell$, abbreviate $p=\pi(\ell)$ and $\alpha_\ell=\softmax(S_\ell)$. For $j\neq p$
the shared normalizer cancels, so by \eqref{eq:margin}
\[
  \frac{\alpha_{\ell,j}}{\alpha_{\ell,p}}
  \;=\;e^{\,S_{\ell,j}-S_{\ell,p}}\;\le\;e^{-m}.
\]

Since $\alpha_\ell$ is a probability vector, $\lvert\alpha_{\ell,p}-1\rvert=1-\alpha_{\ell,p}$
and $\lvert\alpha_{\ell,j}\rvert=\alpha_{\ell,j}$ for $j\neq p$, so
\begin{align}
  \|\alpha_\ell-e_p\|_1
  &\;\overset{(a)}{=}\;\sum_{j=1}^{L_s}\big\lvert\alpha_{\ell,j}-(e_p)_j\big\rvert \notag\\
  &\;\overset{(b)}{=}\;\lvert\alpha_{\ell,p}-1\rvert
      \;+\!\!\sum_{j\in[L_s]\setminus\{p\}}\!\!\lvert\alpha_{\ell,j}\rvert \notag\\
  &\;\overset{(c)}{=}\;\big(1-\alpha_{\ell,p}\big)
      \;+\!\!\sum_{j\in[L_s]\setminus\{p\}}\!\!\alpha_{\ell,j} \notag\\
  &\;\overset{(d)}{=}\;2\!\!\sum_{j\in[L_s]\setminus\{p\}}\!\!\alpha_{\ell,j} \notag\\
  &\;\overset{(e)}{\le}\;2\!\!\sum_{j\in[L_s]\setminus\{p\}}\!\! e^{-m}\alpha_{\ell,p} \notag\\
  &\;\overset{(f)}{=}\;2(L_s-1)\,e^{-m}\,\alpha_{\ell,p} \notag\\
  &\;\overset{(g)}{\le}\;2(L_s-1)\,e^{-m},
\end{align}
where
\begin{enumerate}[label=(\alph*),itemsep=1pt,topsep=3pt,leftmargin=*]
  \item is the definition of the $\ell_1$ norm;
  \item splits off the coordinate $j=p$, at which $(e_p)_p=1$, from the remaining
        $j\neq p$, at which $(e_p)_j=0$;
  \item uses that $\alpha_\ell$ is a probability vector, so $\alpha_{\ell,p}\le1$ and
        $\alpha_{\ell,j}\ge0$;
  \item uses $1-\alpha_{\ell,p}=\sum_{j\neq p}\alpha_{\ell,j}$, again because the entries
        of $\alpha_\ell$ sum to one;
  \item multiplies the ratio bound $\alpha_{\ell,j}\le e^{-m}\alpha_{\ell,p}$ through by
        $\alpha_{\ell,p}>0$ and sums;
  \item counts the $L_s-1$ identical summands; and
  \item uses $\alpha_{\ell,p}\le1$ once more.
\end{enumerate}
This completes the proof.
\end{proof}

\paragraph{Reading the bound.}
Setting the right-hand side of \eqref{eq:margin-bound} to $\varepsilon$ and solving for
$m$,
\begin{align}
  &2(L_s-1)e^{-m}\!\le\!\varepsilon
  \iff
  e^{-m}\;\le\;\frac{\varepsilon}{2(L_s-1)}
   \\
  &\iff m\;\ge\;\log\!\frac{2(L_s-1)}{\varepsilon},
\end{align}
so a margin of this size suffices for $\varepsilon$-accuracy.

We now turn to prove that one-to-one maps are contained in our family.

\subsubsection{One-to-one correspondences are contained in our family}

The keys are computed from the input, $K(X)=XW_k$, so the only thing we choose is the projection $W_k$,
and the keys it produces vary with the prompt. Two difficulties follow. First, we can not guarantee a particular geometry --- whether the keys end up in distinct directions
is a fact about the activations, not something we impose. Second, the conclusion must
hold \emph{uniformly}: since the alignment is evaluated on many prompts, a single $W_k$
must induce the correct layer map for every source stack $X$ it may encounter, not merely for one. 
What makes this possible is that we do not need the keys to be well separated in any absolute sense; we need only that, within each stack, the projection scores each layer
above the others of that same stack. If some direction $w_m$ responds more strongly to
layer $m$'s activation than to any other layer's, then a query aligned with $w_m$ selects
layer $m$ --- exactly as $q_\ell$ selected $k_{\pi(\ell)}$ before. In other words, the
construction goes through provided the source stack carries layer identity in a linearly
readable form, with a margin that does not degrade across inputs. Since this is a property
of the models rather than of our architecture, we state it as an assumption.

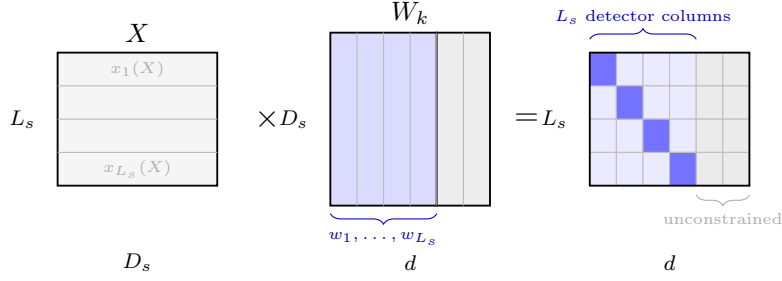
\begin{figure*}[t]
\centering
\begin{tikzpicture}[scale=0.88, every node/.style={font=\scriptsize}]

\draw[fill=gray!8,draw=black,line width=0.7pt] (0,0) rectangle (2.4,-2.0);
\foreach \y in {-0.5,-1.0,-1.5}{\draw[gray!55] (0,\y) -- (2.4,\y);}
\node[font=\small] at (1.2,0.30) {$X$};
\node[anchor=east] at (-0.18,-1.0) {$L_s$};
\node at (1.2,-3.15) {$D_s$};
\node[font=\tiny,gray!65] at (1.2,-0.25) {$x_1(X)$};
\node[font=\tiny,gray!65] at (1.2,-1.75) {$x_{L_s}(X)$};

\node[font=\large] at (3.15,-1.0) {$\times$};

\begin{scope}[xshift=4.1cm]
  \draw[fill=blue!14,draw=black,line width=0.7pt] (0,0.3) rectangle (1.6,-2.3);
  \draw[fill=gray!14,draw=black,line width=0.7pt] (1.6,0.3) rectangle (2.4,-2.3);
  \foreach \x in {0.4,0.8,1.2,1.6,2.0}{\draw[gray!55] (\x,0.3) -- (\x,-2.3);}
  \node[font=\small] at (1.2,0.60) {$W_k$};
  \node[anchor=east] at (-0.18,-1.0) {$D_s$};
  \draw[decorate,decoration={brace,mirror,amplitude=4pt},blue!70!black]
        (0,-2.42) -- (1.6,-2.42)
        node[midway,below,yshift=-3pt,font=\tiny,blue!70!black] {$w_1,\dots,w_{L_s}$};
  \node at (1.2,-3.15) {$d$};
\end{scope}

\node[font=\large] at (7.05,-1.0) {$=$};

\begin{scope}[xshift=8.0cm]
  \foreach \r in {1,2,3,4}{
    \foreach \j in {1,2,3,4}{
      \draw[fill=blue!8,draw=gray!50] ({0.4*\r-0.4},{-0.5*\j+0.5})
            rectangle ({0.4*\r},{-0.5*\j});
    }
  }
  \foreach \r in {5,6}{
    \foreach \j in {1,2,3,4}{
      \draw[fill=gray!16,draw=gray!50] ({0.4*\r-0.4},{-0.5*\j+0.5})
            rectangle ({0.4*\r},{-0.5*\j});
    }
  }
  \foreach \d in {1,2,3,4}{
    \draw[fill=blue!55,draw=gray!50] ({0.4*\d-0.4},{-0.5*\d+0.5})
          rectangle ({0.4*\d},{-0.5*\d});
  }
  \draw[draw=black,line width=0.7pt] (0,0) rectangle (2.4,-2.0);

  \draw[decorate,decoration={brace,amplitude=4pt},blue!70!black]
        (0,0.20) -- (1.6,0.20)
        node[midway,above,yshift=2pt,font=\tiny,blue!70!black] {$L_s$ detector columns};
  \draw[decorate,decoration={brace,mirror,amplitude=4pt},gray!65]
        (1.6,-2.14) -- (2.4,-2.14)
        node[midway,below,yshift=-3pt,font=\tiny,gray!65] {unconstrained};

  \node[anchor=east] at (-0.18,-1.0) {$L_s$};
  \node at (1.2,-3.15) {$d$};
\end{scope}

\end{tikzpicture}
\caption{\Cref{ass:margin} as a condition on the keys. Entry $(j,r)$ of $XW_k$ is the score
$\langle w_r,x_j(X)\rangle$ that detector $r$ assigns to source layer $j$. Reading down any
of the leading $L_s$ columns, the entry at $j=r$ (dark) must lead every other by $\gamma$,
uniformly over $X\in\mathcal X$; off-diagonal magnitudes and the remaining $d-L_s$ columns
are unconstrained.}
\label{fig:assumption}
\end{figure*}
\begin{assumption}[Linearly decodable layer identity]\label{ass:margin}
Let $\mathcal X\subseteq\mathbb R^{L_s\times D_s}$ be the set of normalized source stacks.
There exist $d\ge L_s$, a projection $W_k\in\mathbb R^{D_s\times d}$ with columns
$w_1,\dots,w_d$, and $\gamma>0$ such that the first $L_s$ columns act as per-layer
detectors:
\begin{align}
  &\langle w_m,x_m(X)\rangle\;\ge\;\langle w_m,x_j(X)\rangle+\gamma
  \\
  &\text{for all }X\in\mathcal X,\;m\in[L_s],\;j\in[L_s]\setminus\{m\} \nonumber.
\end{align}
The remaining $d-L_s$ columns are unconstrained.
\end{assumption}

A visualization of this assumption is given in \Cref{fig:assumption}.

\paragraph{Remark.} Should the assumption fail, it can be enforced by a minor modification:
adding a learned per-layer embedding to the keys, $K(X)=XW_k+P$, satisfies it trivially
with $W_k=0$ and $P$ the matrix whose $j$-th row is $e_j$, so where the activations do not
identify their own layer this can be supplied by parameters instead.
%
\begin{figure*}[t]
\centering
\providecommand{\mcell}[4]{%
  \node[draw=gray!55,line width=0.3pt,minimum width=0.50cm,minimum height=0.45cm,
        inner sep=0pt,anchor=north west,fill=#4,font=\tiny] at (#1,#2) {#3};}
\resizebox{\textwidth}{!}{%
\begin{tikzpicture}[every node/.style={font=\scriptsize}]

\node[font=\small] at (1.5,0.62) {$Q/(c\sqrt d)$};
\mcell{0.0}{-0.675}{$0$}{blue!8} \mcell{0.5}{-0.675}{$0$}{blue!8}
\mcell{1.0}{-0.675}{$1$}{blue!55}\mcell{1.5}{-0.675}{$0$}{blue!8}
\mcell{0.0}{-1.125}{$0$}{blue!8} \mcell{0.5}{-1.125}{$1$}{blue!55}
\mcell{1.0}{-1.125}{$0$}{blue!8} \mcell{1.5}{-1.125}{$0$}{blue!8}
\mcell{0.0}{-1.575}{$0$}{blue!8} \mcell{0.5}{-1.575}{$0$}{blue!8}
\mcell{1.0}{-1.575}{$0$}{blue!8} \mcell{1.5}{-1.575}{$1$}{blue!55}
\draw[fill=gray!16,draw=gray!55] (2.0,-0.675) rectangle (3.0,-2.025);
\node[font=\tiny,gray!60] at (2.5,-1.35) {$\mathbf 0$};
\draw[draw=black,line width=0.7pt] (0,-0.675) rectangle (3.0,-2.025);
\draw[decorate,decoration={brace,amplitude=3pt},blue!70!black]
      (0,-0.58) -- (2.0,-0.58)
      node[midway,above,yshift=1pt,font=\tiny,blue!70!black] {$L_s$};
\node[font=\tiny,gray!70] at (1.5,-2.34) {$L_t\times d$};

\node[font=\large] at (3.40,-1.35) {$\times$};

\begin{scope}[xshift=3.8cm]
  \node[font=\small] at (1.0,0.62) {$(XW_k)^{\!\top}$};
  \mcell{0.0}{0.00}{2.0}{blue!55}\mcell{0.5}{0.00}{0.6}{blue!8}
  \mcell{1.0}{0.00}{0.2}{blue!8} \mcell{1.5}{0.00}{0.1}{blue!8}
  \mcell{0.0}{-0.45}{0.7}{blue!8} \mcell{0.5}{-0.45}{2.1}{blue!55}
  \mcell{1.0}{-0.45}{0.8}{blue!8} \mcell{1.5}{-0.45}{0.3}{blue!8}
  \mcell{0.0}{-0.90}{0.3}{blue!8} \mcell{0.5}{-0.90}{0.9}{blue!8}
  \mcell{1.0}{-0.90}{1.9}{blue!55}\mcell{1.5}{-0.90}{0.6}{blue!8}
  \mcell{0.0}{-1.35}{0.1}{blue!8} \mcell{0.5}{-1.35}{0.2}{blue!8}
  \mcell{1.0}{-1.35}{0.7}{blue!8} \mcell{1.5}{-1.35}{2.2}{blue!55}
  \draw[fill=gray!16,draw=gray!55] (0,-1.8) rectangle (2.0,-2.7);
  \node[font=\tiny,gray!60] at (1.0,-2.25) {unconstrained};
  \draw[draw=black,line width=0.7pt] (0,0) rectangle (2.0,-2.7);
  \foreach \r in {2,3,4}{
    \draw[orange!85!black,line width=1.0pt]
          (0,{-0.45*(\r-1)}) rectangle (2.0,{-0.45*\r});
  }
  \node[font=\tiny,gray!70] at (1.0,-3.02) {$d\times L_s$};
\end{scope}

\node[font=\large] at (6.20,-1.35) {$=$};

\begin{scope}[xshift=6.6cm]
  \node[font=\small] at (1.0,0.62) {$S/c$};
  \mcell{0.0}{-0.675}{0.3}{blue!8} \mcell{0.5}{-0.675}{0.9}{blue!8}
  \mcell{1.0}{-0.675}{1.9}{blue!55}\mcell{1.5}{-0.675}{0.6}{blue!8}
  \mcell{0.0}{-1.125}{0.7}{blue!8} \mcell{0.5}{-1.125}{2.1}{blue!55}
  \mcell{1.0}{-1.125}{0.8}{blue!8} \mcell{1.5}{-1.125}{0.3}{blue!8}
  \mcell{0.0}{-1.575}{0.1}{blue!8} \mcell{0.5}{-1.575}{0.2}{blue!8}
  \mcell{1.0}{-1.575}{0.7}{blue!8} \mcell{1.5}{-1.575}{2.2}{blue!55}
  \draw[orange!85!black,line width=1.0pt] (0,-0.675) rectangle (2.0,-2.025);
  \node[font=\tiny,gray!70] at (1.0,-2.34) {$L_t\times L_s$};
\end{scope}

\draw[->,very thick] (8.85,-1.35) -- (9.80,-1.35)
      node[midway,above,font=\tiny] {$\softmax$};

\begin{scope}[xshift=10.1cm]
  \node[font=\small] at (1.0,0.62) {$\alpha$};
  \node[font=\tiny,gray!70] at (1.0,0.24) {$c=3$};
  \mcell{0.0}{-0.675}{0.01}{gray!6} \mcell{0.5}{-0.675}{0.05}{gray!6}
  \mcell{1.0}{-0.675}{0.93}{blue!35}\mcell{1.5}{-0.675}{0.02}{gray!6}
  \mcell{0.0}{-1.125}{0.01}{gray!6} \mcell{0.5}{-1.125}{0.96}{blue!35}
  \mcell{1.0}{-1.125}{0.02}{gray!6} \mcell{1.5}{-1.125}{0.00}{gray!6}
  \mcell{0.0}{-1.575}{0.00}{gray!6} \mcell{0.5}{-1.575}{0.00}{gray!6}
  \mcell{1.0}{-1.575}{0.01}{gray!6} \mcell{1.5}{-1.575}{0.98}{blue!35}
  \draw[draw=black,line width=0.7pt] (0,-0.675) rectangle (2.0,-2.025);
  \node[font=\tiny,gray!70] at (1.0,-2.34) {$L_t\times L_s$};
\end{scope}

\node[font=\large] at (12.50,-1.35) {$\approx$};

\begin{scope}[xshift=12.9cm]
  \node[font=\small] at (1.0,0.62) {$P_\pi$};
  \mcell{0.0}{-0.675}{$0$}{gray!6} \mcell{0.5}{-0.675}{$0$}{gray!6}
  \mcell{1.0}{-0.675}{$1$}{blue!55}\mcell{1.5}{-0.675}{$0$}{gray!6}
  \mcell{0.0}{-1.125}{$0$}{gray!6} \mcell{0.5}{-1.125}{$1$}{blue!55}
  \mcell{1.0}{-1.125}{$0$}{gray!6} \mcell{1.5}{-1.125}{$0$}{gray!6}
  \mcell{0.0}{-1.575}{$0$}{gray!6} \mcell{0.5}{-1.575}{$0$}{gray!6}
  \mcell{1.0}{-1.575}{$0$}{gray!6} \mcell{1.5}{-1.575}{$1$}{blue!55}
  \draw[draw=black,line width=0.7pt] (0,-0.675) rectangle (2.0,-2.025);
  \node[font=\tiny,gray!70] at (1.0,-2.34) {$L_t\times L_s$};
\end{scope}

\end{tikzpicture}%
}
\caption{\Cref{prop:dyn-onehot} on a concrete instance: $L_t=3$, $L_s=4$, and $\pi(1)=3$,
$\pi(2)=2$, $\pi(3)=4$. Each query is one-hot at $\pi(\ell)$ and vanishes on the last
$d-L_s$ coordinates, so the product annihilates the unconstrained rows of
$(XW_k)^{\!\top}$ and copies its row $\pi(\ell)$ into row $\ell$ of $S$ (orange). By
\Cref{ass:margin} each such row is largest at its own layer (dark), so after the softmax
every row of $\alpha$ concentrates on $\pi(\ell)$. With $\gamma=1.0$ and $c=3$,
\Cref{lem:margin} gives $\|\alpha-P_\pi\|_{1,\infty}\le2(L_s-1)e^{-3}\approx0.30$; larger
$c$ drives $\alpha\to P_\pi$ exponentially with $W_k$ unchanged. Note that detector $1$ is
selected by no target layer, and that $\pi$ need not be injective.}
\label{fig:dyn-construction}
\end{figure*}
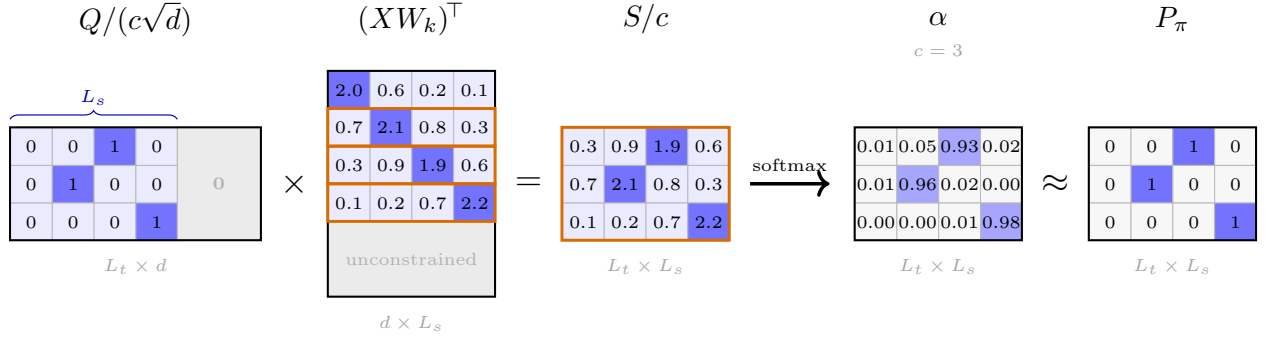
\begin{proposition}[One-to-one]\label{prop:dyn-onehot}
Under Assumption~\ref{ass:margin}, for every $\pi$ and every $\varepsilon>0$ there exist
$Q$ and $W_k$ such that
$\sup_{X\in\mathcal X}\|\alpha(X)-P_\pi\|_{1,\infty}\le\varepsilon$, where
$\alpha(X)=\softmax\!\big(Q(XW_k)^\top/\sqrt d\big)$.
\end{proposition}

\begin{proof}
Let $W_k$ be the projection supplied by \Cref{ass:margin}, with columns
$w_1,\dots,w_d$ and margin $\gamma>0$, and take to be as follows,
$q_\ell=c\sqrt d\,e_{\pi(\ell)}$ with $c>0$ free. The keys are determined by the input: the $j$-th row of $XW_k$ is,
\begin{align}
  (XW_k)_j&\;=\;\big(\langle w_1,x_j(X)\rangle,\;\dots,\;\langle w_d,x_j(X)\rangle\big)  \nonumber\\
  &\;\in\;\mathbb R^{d}, \nonumber
\end{align}
that is, the $r$-th coordinate of the $j$-th key is the score that detector $w_r$ assigns
to source layer $j$. Fix any $X\in\mathcal X$. Then
\begin{align}
  S_{\ell,j}(X)
  &\;\overset{(a)}{=}\;\frac{\big\langle q_\ell,\,(XW_k)_j\big\rangle}{\sqrt d} \notag\\
  &\;\overset{(b)}{=}\;\frac{1}{\sqrt d}\sum_{r=1}^{d}(q_\ell)_r\,\big(XW_k\big)_{j,r}
    \notag\\
  &\;\overset{(c)}{=}\;\frac{1}{\sqrt d}\,(q_\ell)_{\pi(\ell)}\,
      \big(XW_k\big)_{j,\pi(\ell)} \notag\\
  &\;\overset{(d)}{=}\;\frac{1}{\sqrt d}\cdot c\sqrt d\cdot
      \big(XW_k\big)_{j,\pi(\ell)} \notag\\
  &\;\overset{(e)}{=}\;c\,\big\langle w_{\pi(\ell)},\,x_j(X)\big\rangle,
\end{align}
where
\begin{enumerate}[label=(\alph*),itemsep=1pt,topsep=3pt,leftmargin=*]
  \item is the definition of the logit matrix, $S(X)=Q\,K(X)^\top/\sqrt d$ with
        $K(X)=XW_k$, read entrywise;
  \item expands the inner product in coordinates of $\mathbb R^{d}$;
  \item drops all but one term: $q_\ell=c\sqrt d\,e_{\pi(\ell)}$ is supported on the single
        coordinate $r=\pi(\ell)$, so $(q_\ell)_r=0$ for $r\neq\pi(\ell)$. Note
        $\pi(\ell)\le L_s\le d$, so this coordinate exists and is one of the $L_s$
        detector columns constrained by \Cref{ass:margin};
  \item substitutes $(q_\ell)_{\pi(\ell)}=c\sqrt d$ and cancels against $1/\sqrt d$; and
  \item identifies the surviving coordinate, $\big(XW_k\big)_{j,\pi(\ell)}
        =\langle w_{\pi(\ell)},x_j(X)\rangle$, by the description of $(XW_k)_j$ above.
\end{enumerate}
 
Row $\ell$ of the logit matrix is therefore, up to the positive factor $c$, the vector of
scores that the single detector $w_{\pi(\ell)}$ assigns to each of the source layers of
$X$. This is exactly the quantity Assumption \ref{ass:margin} controls: applying it with
$m=\pi(\ell)$ gives, for every $j\neq\pi(\ell)$,
\begin{align}
  &S_{\ell,\pi(\ell)}(X)-S_{\ell,j}(X) \\
  &\;\overset{(f)}{=}\;c\Big[\big\langle w_{\pi(\ell)},x_{\pi(\ell)}(X)\big\rangle
      -\big\langle w_{\pi(\ell)},x_j(X)\big\rangle\Big]\\
  &\;\overset{(g)}{\ge}\;c\,\gamma,
\end{align}
where $(f)$ substitutes the expression just derived at $j$ and at $\pi(\ell)$, and $(g)$
applies the assumed margin and uses $c>0$. Crucially, $\gamma$ is uniform over
$\mathcal X$ by hypothesis, so the bound holds for every $X$ with the same constant --- a
single choice of $(Q,W_k)$ works for all inputs, which is what the supremum in the
statement requires.
 
Thus every row of $S(X)$ has margin $m\ge c\gamma$ for $\pi$, and \Cref{lem:margin} gives
\begin{align}
  &\big\|\alpha(X)-P_\pi\big\|_{1,\infty}\;\le\;2(L_s-1)\,e^{-c\gamma} \\
 &\text{for every }X\in\mathcal X \nonumber.
\end{align}
The right-hand side does not depend on $X$, so it also bounds the supremum. Choosing
\[
  c\;\ge\;\frac{1}{\gamma}\,\log\!\frac{2(L_s-1)}{\varepsilon}
\]
makes it at most $\varepsilon$, which proves the claim.
\end{proof}
A visualization of a specific example for \Cref{prop:dyn-onehot} is given in \Cref{fig:dyn-construction}.

Finally, we state the concluding corollary, stating that the one-to-one correspondences assumed by prior work all lie in the closure of the layer maps our attention can produce, and are therefore special cases of what we can learn rather than
alternatives to it.

\begin{corollary}[Containment of fixed one-to-one layer maps]
\label{cor:contained}
Suppose that \Cref{ass:margin} holds and let $d \ge 2$. Then, for every one-to-one layer assignment $\pi$ and every $\varepsilon>0$,
\[
    \sup_{X \in \mathcal{X}}
    \left\lVert \alpha(X)-P_{\pi} \right\rVert_{1,\infty}
    \le \varepsilon.
\]
\end{corollary}

\begin{proof}

Fix $\pi$ and $\varepsilon>0$. \Cref{prop:dyn-onehot} supplies
$Q$ and $W_k$ with $\|\alpha(X)-P_\pi\|_{1,\infty}\le\varepsilon$ for every
$X\in\mathcal X$, hence for the supremum. Equality again fails for the same reason.
\end{proof}

\subsubsection{Arbitrary soft correspondences}

One-to-one maps are the boundary of $\Delta$. We now characterize the interior, where a
rank constraint --- absent from the results above --- becomes the binding one.
To establish this result, we require a stronger assumption than Assumption~\ref{ass:margin}: reproducing an interior target requires matching the logit \emph{values}, not merely their ordering.


\begin{assumption}[Exactly decodable layer identity]\label{ass:exact}
There exist $u_1,\dots,u_{L_s}\in\mathbb R^{D_s}$ with
\begin{align}
  & \langle u_m,x_j(X)\rangle\;=\;\mathbf 1\{j=m\} \nonumber \\
  & \text{for all }X\in\mathcal X,\;m,j\in[L_s] \nonumber.
\end{align}
\end{assumption}

\paragraph{Remark.} Should the assumption fail, it can be enforced by the same modification
as before: adding a learned per-layer embedding to the keys, $K(X)=XW_k+P$, satisfies it
with $W_k=0$ and $P$ the matrix whose $j$-th row is $e_j$. 

Where \Cref{ass:margin} constrained only the \emph{ordering} of the detector scores, this
assumption fixes their \emph{values}: reading layer $j$'s activation with the direction
$u_m$ must return exactly $1$ when $j=m$ and exactly $0$ otherwise. Collecting the
directions into $U:=[\,u_1\;\cdots\;u_{L_s}\,]\in\mathbb R^{D_s\times L_s}$, the condition
says simply
\[
  X U\;=\;I_{L_s}\qquad\text{for every }X\in\mathcal X,
\]
that is, the same identity matrix for every input (\Cref{fig:assumption-exact}). The
strengthening is what the interior case requires: a one-to-one target is determined by which
logit is largest, so an ordering suffices, whereas a soft target
$\alpha^\star\in\Delta^{\circ}$ is determined by the logit values themselves, and these can
only be dialled in if the scores are known exactly rather than merely ranked.

\begin{figure}[t]
\centering
\providecommand{\mcell}[4]{%
  \node[draw=gray!55,line width=0.3pt,minimum width=0.50cm,minimum height=0.45cm,
        inner sep=0pt,anchor=north west,fill=#4,font=\tiny] at (#1,#2) {#3};}
\resizebox{\columnwidth}{!}{%
\begin{tikzpicture}[every node/.style={font=\scriptsize}]

\node[font=\small] at (1.2,0.28) {$X$};
\draw[fill=gray!8,draw=black,line width=0.7pt] (0,-0.30) rectangle (2.4,-2.10);
\foreach \y in {-0.75,-1.20,-1.65}{\draw[gray!55] (0,\y) -- (2.4,\y);}
\node[font=\tiny,gray!65] at (1.2,-0.525) {$x_1(X)$};
\node[font=\tiny,gray!65] at (1.2,-1.875) {$x_{L_s}(X)$};
\node[font=\tiny,gray!70] at (1.2,-2.42) {$L_s\times D_s$};

\node[font=\large] at (2.85,-1.20) {$\times$};

\begin{scope}[xshift=3.3cm]
  \node[font=\small] at (1.0,0.28) {$U$};
  \draw[fill=blue!12,draw=black,line width=0.7pt] (0,0) rectangle (2.0,-2.40);
  \foreach \x in {0.5,1.0,1.5}{\draw[gray!55] (\x,0) -- (\x,-2.40);}
  \node[font=\tiny,blue!70!black] at (0.25,-2.62) {$u_1$};
  \node[font=\tiny,blue!70!black] at (1.75,-2.62) {$u_{L_s}$};
  \node[font=\tiny,gray!70] at (1.0,-2.98) {$D_s\times L_s$};
\end{scope}

\node[font=\large] at (5.75,-1.20) {$=$};

\begin{scope}[xshift=6.2cm]
  \node[font=\small] at (1.0,0.28) {$I_{L_s}$};
  \mcell{0.0}{-0.30}{$1$}{blue!55}\mcell{0.5}{-0.30}{$0$}{gray!6}
  \mcell{1.0}{-0.30}{$0$}{gray!6} \mcell{1.5}{-0.30}{$0$}{gray!6}
  \mcell{0.0}{-0.75}{$0$}{gray!6} \mcell{0.5}{-0.75}{$1$}{blue!55}
  \mcell{1.0}{-0.75}{$0$}{gray!6} \mcell{1.5}{-0.75}{$0$}{gray!6}
  \mcell{0.0}{-1.20}{$0$}{gray!6} \mcell{0.5}{-1.20}{$0$}{gray!6}
  \mcell{1.0}{-1.20}{$1$}{blue!55}\mcell{1.5}{-1.20}{$0$}{gray!6}
  \mcell{0.0}{-1.65}{$0$}{gray!6} \mcell{0.5}{-1.65}{$0$}{gray!6}
  \mcell{1.0}{-1.65}{$0$}{gray!6} \mcell{1.5}{-1.65}{$1$}{blue!55}
  \draw[draw=black,line width=0.7pt] (0,-0.30) rectangle (2.0,-2.10);
  \node[font=\tiny,gray!70] at (1.0,-2.42) {$L_s\times L_s$};
\end{scope}

\node[font=\small,red!75!black] at (4.10,-3.45) {for \emph{every} $X\in\mathcal X$};

\end{tikzpicture}%
}
\caption{\Cref{ass:exact} as a condition on the source stacks. Reading each layer's
activation with each of the $L_s$ decoding directions must return exactly one on the
matching layer and exactly zero elsewhere --- that is, $XU=I_{L_s}$, with the same identity
matrix for every input. This is strictly stronger than \Cref{ass:margin}
(\Cref{fig:assumption}), which constrains only the ordering within each column and leaves
all magnitudes free.}
\label{fig:assumption-exact}
\end{figure}

\begin{proposition}[Soft map]\label{prop:dyn-interior}
Under \Cref{ass:exact}, if $d\ge\min(L_t,L_s)$ then every $\alpha^\star\in\Delta^{\circ}$ is
realized exactly and input-independently: there exist $Q$ and $W_k$ with
$\alpha(X)=\alpha^\star$ for all $X\in\mathcal X$.
\end{proposition}

\begin{proof}
The target is reached by realizing the logit matrix
$\Lambda:=\log\alpha^\star$, which satisfies $\softmax(\Lambda)=\alpha^\star$. The challenge here is to show that $\Lambda$ can be produced by \emph{input-dependent}
keys, i.e.\ that a single pair $(Q,W_k)$ yields $S(X)=\Lambda$ for every $X\in\mathcal X$.

\emph{Step 1: a direction realizing each row of $\Lambda$.} Let
$M\in\mathbb R^{L_t\times D_s}$ have rows
\[
  M_\ell\;:=\;\sum_{m=1}^{L_s}\Lambda_{\ell,m}\,u_m ,
\]
a linear combination of the decoding directions of \Cref{ass:exact} weighted by the desired
logits. Reading any source activation with $M_\ell$ returns the corresponding logit:
$\langle M_\ell,x_j(X)\rangle=\Lambda_{\ell,j}$, as verified in Step~3 below.

\emph{Step 2: factoring $M$ through width $d$.} Every row of $M$ lies in
$\operatorname{span}\{u_1,\dots,u_{L_s}\}$, a subspace of dimension at most $L_s$, and $M$
has $L_t$ rows, so
\[
  \rank(M)\;\le\;\min(L_t,L_s)\;\le\;d .
\]
Writing $r:=\rank(M)\le d$ and taking a thin singular value decomposition
$\sqrt d\,M=U_M\Sigma V_M^\top$ with $U_M\in\mathbb R^{L_t\times r}$,
$V_M\in\mathbb R^{D_s\times r}$, set $Q:=[\,U_M\Sigma\;\;\mathbf 0\,]\in\mathbb R^{L_t\times d}$
and $W_k:=[\,V_M\;\;\mathbf 0\,]\in\mathbb R^{D_s\times d}$, padding with $d-r$ zero columns.
The padded coordinates contribute nothing to the product, so
\begin{align}
  & QW_k^\top\;=\;U_M\Sigma V_M^\top\;=\;\sqrt d\,M, \nonumber \\
  & \text{i.e.}\qquad
  W_k\,q_\ell\;=\;\sqrt d\,M_\ell \ \ \text{for each }\ell \nonumber .
\end{align}

\emph{Step 3: the logits are $\Lambda$, for every input.} Fix any $X\in\mathcal X$. Then
\begin{align}
  S_{\ell,j}(X)
  &\;\overset{(a)}{=}\;\frac{\big\langle q_\ell,\,(XW_k)_j\big\rangle}{\sqrt d} \notag\\
  &\;\overset{(b)}{=}\;\frac{\big\langle q_\ell,\,W_k^\top x_j(X)\big\rangle}{\sqrt d}
    \notag\\
  &\;\overset{(c)}{=}\;\frac{\big\langle W_k\,q_\ell,\,x_j(X)\big\rangle}{\sqrt d} \notag\\
  &\;\overset{(d)}{=}\;\frac{\big\langle \sqrt d\,M_\ell,\,x_j(X)\big\rangle}{\sqrt d}
    \notag\\
  &\;\overset{(e)}{=}\;\big\langle M_\ell,\,x_j(X)\big\rangle \notag\\
  &\;\overset{(f)}{=}\;\Big\langle \sum_{m=1}^{L_s}\Lambda_{\ell,m}u_m,\;x_j(X)\Big\rangle
    \notag\\
  &\;\overset{(g)}{=}\;\sum_{m=1}^{L_s}\Lambda_{\ell,m}\big\langle u_m,x_j(X)\big\rangle
    \notag\\
  &\;\overset{(h)}{=}\;\sum_{m=1}^{L_s}\Lambda_{\ell,m}\,\mathbf 1\{j=m\} \notag\\
  &\;\overset{(i)}{=}\;\Lambda_{\ell,j},
\end{align}
where
\begin{enumerate}[label=(\alph*),itemsep=1pt,topsep=3pt,leftmargin=*]
  \item is the definition of the logit matrix, $S(X)=Q\,K(X)^\top/\sqrt d$ with
        $K(X)=XW_k$, read entrywise;
  \item writes the $j$-th row of $XW_k$ as $W_k^\top x_j(X)$;
  \item moves $W_k$ across the inner product, $\langle a,W_k^\top b\rangle=\langle W_k a,b\rangle$;
  \item substitutes $W_k q_\ell=\sqrt d\,M_\ell$ from Step~2;
  \item cancels $\sqrt d$ against $1/\sqrt d$;
  \item substitutes the definition of $M_\ell$ from Step~1;
  \item uses bilinearity of the inner product;
  \item applies \Cref{ass:exact}, which fixes each score at exactly $\mathbf 1\{j=m\}$; and
  \item evaluates the sum, in which only the term $m=j$ survives.
\end{enumerate}
The right-hand side does not depend on $X$, so $S(X)=\Lambda$ for every $X\in\mathcal X$ and
therefore $\alpha(X)=\softmax(\Lambda)=\alpha^\star$ throughout.
\end{proof}

\begin{figure*}[t]
\centering
\providecommand{\scell}[4]{%
  \node[draw=gray!55,line width=0.3pt,minimum width=0.58cm,minimum height=0.45cm,
        inner sep=0pt,anchor=north west,fill=#4,font=\tiny] at (#1,#2) {#3};}
\resizebox{\textwidth}{!}{%
\begin{tikzpicture}[every node/.style={font=\scriptsize}]

\node[font=\small] at (1.16,0.30) {$Q/\sqrt d=\Lambda$};
\scell{0.00}{-0.525}{-0.69}{blue!8}\scell{0.58}{-0.525}{-1.20}{blue!8}
\scell{1.16}{-0.525}{-1.90}{blue!8}\scell{1.74}{-0.525}{-3.00}{blue!8}
\scell{0.00}{-0.975}{-2.30}{blue!8}\scell{0.58}{-0.975}{-0.51}{blue!8}
\scell{1.16}{-0.975}{-1.61}{blue!8}\scell{1.74}{-0.975}{-2.30}{blue!8}
\scell{0.00}{-1.425}{-3.00}{blue!8}\scell{0.58}{-1.425}{-1.90}{blue!8}
\scell{1.16}{-1.425}{-1.20}{blue!8}\scell{1.74}{-1.425}{-0.69}{blue!8}
\draw[draw=black,line width=0.7pt] (0,-0.525) rectangle (2.32,-1.875);
\node[font=\tiny,gray!70] at (1.16,-2.18) {$L_t\times L_s$};

\node[font=\large] at (2.70,-1.20) {$\times$};

\begin{scope}[xshift=3.08cm]
  \node[font=\small] at (1.16,0.30) {$(XW_k)^{\!\top}=I_{L_s}$};
  \scell{0.00}{-0.30}{$1$}{blue!55}\scell{0.58}{-0.30}{$0$}{gray!6}
  \scell{1.16}{-0.30}{$0$}{gray!6} \scell{1.74}{-0.30}{$0$}{gray!6}
  \scell{0.00}{-0.75}{$0$}{gray!6} \scell{0.58}{-0.75}{$1$}{blue!55}
  \scell{1.16}{-0.75}{$0$}{gray!6} \scell{1.74}{-0.75}{$0$}{gray!6}
  \scell{0.00}{-1.20}{$0$}{gray!6} \scell{0.58}{-1.20}{$0$}{gray!6}
  \scell{1.16}{-1.20}{$1$}{blue!55}\scell{1.74}{-1.20}{$0$}{gray!6}
  \scell{0.00}{-1.65}{$0$}{gray!6} \scell{0.58}{-1.65}{$0$}{gray!6}
  \scell{1.16}{-1.65}{$0$}{gray!6} \scell{1.74}{-1.65}{$1$}{blue!55}
  \draw[draw=black,line width=0.7pt] (0,-0.30) rectangle (2.32,-2.10);
  \node[font=\tiny,gray!70] at (1.16,-2.40) {$L_s\times L_s$};
  \node[font=\tiny,red!75!black] at (1.16,-2.76) {same for \emph{every} $X$};
\end{scope}

\node[font=\large] at (5.78,-1.20) {$=$};

\begin{scope}[xshift=6.16cm]
  \node[font=\small] at (1.16,0.30) {$S(X)=\Lambda$};
  \scell{0.00}{-0.525}{-0.69}{blue!8}\scell{0.58}{-0.525}{-1.20}{blue!8}
  \scell{1.16}{-0.525}{-1.90}{blue!8}\scell{1.74}{-0.525}{-3.00}{blue!8}
  \scell{0.00}{-0.975}{-2.30}{blue!8}\scell{0.58}{-0.975}{-0.51}{blue!8}
  \scell{1.16}{-0.975}{-1.61}{blue!8}\scell{1.74}{-0.975}{-2.30}{blue!8}
  \scell{0.00}{-1.425}{-3.00}{blue!8}\scell{0.58}{-1.425}{-1.90}{blue!8}
  \scell{1.16}{-1.425}{-1.20}{blue!8}\scell{1.74}{-1.425}{-0.69}{blue!8}
  \draw[draw=black,line width=0.7pt] (0,-0.525) rectangle (2.32,-1.875);
  \node[font=\tiny,gray!70] at (1.16,-2.18) {$L_t\times L_s$};
\end{scope}

\draw[->,very thick] (8.86,-1.20) -- (9.80,-1.20)
      node[midway,above,font=\tiny] {$\softmax$};

\begin{scope}[xshift=10.10cm]
  \node[font=\small] at (1.16,0.30) {$\alpha=\alpha^\star$};
  \scell{0.00}{-0.525}{0.50}{blue!50}\scell{0.58}{-0.525}{0.30}{blue!30}
  \scell{1.16}{-0.525}{0.15}{blue!15}\scell{1.74}{-0.525}{0.05}{blue!5}
  \scell{0.00}{-0.975}{0.10}{blue!10}\scell{0.58}{-0.975}{0.60}{blue!60}
  \scell{1.16}{-0.975}{0.20}{blue!20}\scell{1.74}{-0.975}{0.10}{blue!10}
  \scell{0.00}{-1.425}{0.05}{blue!5} \scell{0.58}{-1.425}{0.15}{blue!15}
  \scell{1.16}{-1.425}{0.30}{blue!30}\scell{1.74}{-1.425}{0.50}{blue!50}
  \draw[draw=black,line width=0.7pt] (0,-0.525) rectangle (2.32,-1.875);
  \node[font=\tiny,gray!70] at (1.16,-2.18) {$L_t\times L_s$};
\end{scope}

\end{tikzpicture}%
}
\caption{\Cref{prop:dyn-interior} on a concrete instance, with $L_t=3$, $L_s=4$ and
$d=L_s$, for which the construction takes the transparent form $W_k=U$ and
$Q=\sqrt d\,\Lambda$. \Cref{ass:exact} makes the key matrix $XW_k=XU$ equal to the identity
for \emph{every} source stack, so the input drops out of the product entirely and the logits
are $\Lambda=\log\alpha^\star$ regardless of $X$; the softmax then returns $\alpha^\star$
itself. Unlike the one-to-one case, the target here is soft --- no row concentrates on a
single source layer --- and is reached exactly rather than in the limit of a large logit
scale. For general $d\ge\min(L_t,L_s)$ the same argument runs with $Q$ and $W_k$ obtained
from a rank factorization of $\sqrt d\,M$, as in Step~2 of the proof.}
\label{fig:dyn-interior}
\end{figure*}

\FloatBarrier
\end{document}